%% file: ScaledGD.tex
\documentclass[english]{article}

\usepackage{algorithmic}
\usepackage{algorithm}
\usepackage{amsfonts}
\usepackage{amsmath}
\usepackage{amssymb}
\usepackage{amsthm}
\usepackage{array}
\usepackage{arydshln}
\usepackage{bm}
\usepackage{cite}
\usepackage{color}
\usepackage{comment}
\usepackage{dsfont}
\usepackage{enumitem}
\usepackage{float}
\usepackage[T1]{fontenc}
\usepackage{geometry}
\usepackage{graphicx}
\usepackage{grffile}
\usepackage{hyperref}\hypersetup{colorlinks,linkcolor=red,anchorcolor=blue,citecolor=blue}
\usepackage[latin9]{inputenc}
\usepackage{letltxmacro}
\usepackage{mathrsfs}
\usepackage{mathtools}
\usepackage{multirow}
\usepackage{nicefrac}   
\usepackage{verbatim}

\newcommand{\ba}{\bm{a}}

\newcommand{\bu}{\bm{u}}
\newcommand{\bv}{\bm{v}}
\newcommand{\bw}{\bm{w}}

\newcommand{\by}{\bm{y}}

\newcommand{\bA}{\bm{A}}
\newcommand{\bB}{\bm{B}}

\newcommand{\bF}{\bm{F}}
\newcommand{\bG}{\bm{G}}
\newcommand{\bH}{\bm{H}}
\newcommand{\bI}{\bm{I}}

\newcommand{\bL}{\bm{L}}
\newcommand{\bM}{\bm{M}}

\newcommand{\bO}{\bm{O}}

\newcommand{\bQ}{\bm{Q}}
\newcommand{\bR}{\bm{R}}
\newcommand{\bS}{\bm{S}}

\newcommand{\bU}{\bm{U}}
\newcommand{\bV}{\bm{V}}
\newcommand{\bW}{\bm{W}}
\newcommand{\bX}{\bm{X}}
\newcommand{\bY}{\bm{Y}}

\newcommand{\bDelta}{\bm{\Delta}}

\newcommand{\bSigma}{\bm{\Sigma}}

\newcommand{\cA}{\mathcal{A}}

\newcommand{\cE}{\mathcal{E}}

\newcommand{\cH}{\mathcal{H}}
\newcommand{\cI}{\mathcal{I}}

\newcommand{\cL}{\mathcal{L}}

\newcommand{\cN}{\mathcal{N}}

\newcommand{\cP}{\mathcal{P}}

\newcommand{\cS}{\mathcal{S}}
\newcommand{\cT}{\mathcal{T}}

\newcommand{\RR}{\mathbb{R}}

\newcommand{\VV}{\mathbb{V}}

\newcommand{\mfk}{\mathfrak} 

\newcommand{\zero}{\bm{0}}

\newcommand{\argmin}{\mathop{\mathrm{argmin}}}
\newcommand{\minimize}{\mathop{\mathrm{minimize}}}

\DeclareMathOperator{\dist}{\mathrm{dist}}
\DeclareMathOperator{\fro}{\mathsf{F}}
\DeclareMathOperator{\GL}{\mathrm{GL}}

\DeclareMathOperator{\op}{\mathsf{op}}
\DeclareMathOperator{\rank}{\mathrm{rank}}
\DeclareMathOperator{\sgn}{\mathrm{sgn}}
\DeclareMathOperator{\tr}{\mathrm{tr}}
\DeclareMathOperator{\vc}{\mathrm{vec}}
 
\allowdisplaybreaks
\makeatletter
\providecommand{\tabularnewline}{\\}

\let\tilde\widetilde

\setlist[itemize]{leftmargin=1em}
\setlist[enumerate]{leftmargin=1em}

\theoremstyle{plain}\newtheorem{lemma}{\textbf{Lemma}} 
\newtheorem{proposition}{\textbf{Proposition}}
\newtheorem{theorem}{\textbf{Theorem}}

\newtheorem{claim}{\textbf{Claim}} 
\theoremstyle{definition}\newtheorem{definition}{\textbf{Definition}}
 
\theoremstyle{remark}\newtheorem{remark}{\textbf{Remark}}

\definecolor{tian}{RGB}{0,150,0}
\definecolor{cm}{RGB}{250,0,200}
\definecolor{yc}{RGB}{255,0,0}

\begin{document}
\title{Accelerating Ill-Conditioned Low-Rank Matrix Estimation \\ via Scaled Gradient Descent}

\author
{
	Tian Tong\thanks{Department of Electrical and Computer Engineering, Carnegie Mellon University, Pittsburgh, PA 15213, USA; Emails:
		\texttt{\{ttong1,yuejiec\}@andrew.cmu.edu}.} \\
		Carnegie Mellon University \\
		\and
	Cong Ma\thanks{Department of Electrical Engineering and Computer Science, UC Berkeley, Berkeley, CA 94720, USA; Email:
		\texttt{congm@berkeley.edu}.} \\
		UC Berkeley \\
		\and
	Yuejie Chi\footnotemark[1] \\
	Carnegie Mellon University
}

\date{May 2020; Revised June 2021}
 
\setcounter{tocdepth}{2}
\maketitle

\input{abstract.tex}


\tableofcontents{}

\input{introduction.tex}

\input{problem-formulation.tex}

\input{proof-outline.tex}

\input{numerical.tex}

\input{discussion.tex}

\section*{Acknowledgements}

The work of T.~Tong and Y.~Chi is supported in part by ONR under the grants N00014-18-1-2142 and N00014-19-1-2404, by ARO under the grant W911NF-18-1-0303, and by NSF under the grants CAREER ECCS-1818571, CCF-1806154 and CCF-1901199.

\bibliographystyle{alphaabbr}
\bibliography{bibfileNonconvex_TSP}

\appendix

\input{appendix.tex}

\input{appendix-sensing.tex}

\input{appendix-rpca.tex}

\input{appendix-mc.tex}

\input{appendix-general-loss.tex}

\end{document}

%% file: abstract.tex
\begin{abstract}

Low-rank matrix estimation is a canonical problem that finds numerous applications in signal processing, machine learning and imaging science. A popular approach in practice is to factorize the matrix into two compact low-rank factors, and then optimize these factors directly via simple iterative methods such as gradient descent and alternating minimization. Despite nonconvexity, recent literatures have shown that these simple heuristics in fact achieve linear convergence when initialized properly for a growing number of problems of interest. However, upon closer examination, existing approaches can still be computationally expensive especially for ill-conditioned matrices: the convergence rate of gradient descent depends linearly on the condition number of the low-rank matrix, while the per-iteration cost of alternating minimization is often prohibitive for large matrices.

The goal of this paper is to set forth a competitive algorithmic approach dubbed {\em Scaled Gradient Descent} (\texttt{ScaledGD}) which can be viewed as preconditioned or diagonally-scaled gradient descent, where the preconditioners are adaptive and iteration-varying with a minimal computational overhead. With tailored variants for low-rank matrix sensing, robust principal component analysis and matrix completion, we theoretically show that \texttt{ScaledGD} achieves the best of both worlds: it converges linearly at a rate independent of the condition number of the low-rank matrix similar as alternating minimization, while maintaining the low per-iteration cost of gradient descent. Our analysis is also applicable to general loss functions that are restricted strongly convex and smooth over low-rank matrices. To the best of our knowledge, \texttt{ScaledGD} is the first algorithm that provably has such properties over a wide range of low-rank matrix estimation tasks. At the core of our analysis is the introduction of a new distance function that takes account of the preconditioners when measuring the distance between the iterates and the ground truth. Finally, numerical examples are provided to demonstrate the effectiveness of \texttt{ScaledGD} in accelerating  the convergence rate of ill-conditioned low-rank matrix estimation in a wide number of applications.

\end{abstract}

%% file: introduction.tex
\section{Introduction}

Low-rank matrix estimation plays a critical role in fields such as machine learning, signal processing, imaging science, and many others. 
Broadly speaking, one aims to recover a rank-$r$ matrix $\bX_{\star}\in\RR^{n_{1}\times n_{2}}$ from a set of observations $\by=\cA(\bX_{\star})$, where the operator $\cA(\cdot)$ models the measurement process. It is natural to minimize the least-squares loss function subject to a rank constraint: 
\begin{align}
\minimize_{\bX\in\RR^{n_1\times n_2}}\;f(\bX)\coloneqq\tfrac{1}{2}\|\cA(\bX)-\by\|_{2}^{2}\qquad\mbox{s.t.}\quad\rank(\bX)\le r,
\end{align}
which is, however, computationally intractable in general due to the rank constraint. 
Moreover, as the size of the matrix increases, the costs involved in optimizing over the full matrix space (i.e.~$\RR^{n_{1}\times n_{2}}$) are prohibitive in terms of both memory and computation. To cope with these challenges, one popular approach is to parametrize $\bX=\bL\bR^{\top}$ by two low-rank factors $\bL\in\RR^{n_{1}\times r}$ and $\bR\in\RR^{n_{2}\times r}$ that are more memory-efficient, and then to optimize over the factors instead: 
\begin{align}
\minimize_{\bL\in\RR^{n_1\times r},\bR\in\RR^{n_2\times r}}\;\cL(\bL,\bR)\coloneqq f(\bL\bR^{\top}).\label{eq:problem}
\end{align}
Although this leads to a nonconvex optimization problem over the factors, recent breakthroughs have shown that simple algorithms (e.g.~gradient descent, alternating minimization), when properly initialized (e.g.~via the spectral method), can provably converge to the true low-rank factors under mild statistical assumptions. 
These benign convergence guarantees hold for a growing number of problems such as low-rank matrix sensing, matrix completion, robust principal component analysis (robust PCA), phase synchronization, and so on. 

However, upon closer examination, existing approaches such as gradient descent and alternating minimization are still computationally expensive, especially for ill-conditioned matrices. Take low-rank matrix sensing as an example: although the per-iteration cost is small, the iteration complexity of gradient descent scales linearly with respect to the condition number of the low-rank matrix $\bX_{\star}$ \cite{tu2015low}; on the other end, while the iteration complexity of alternating minimization \cite{jain2013low} is independent of the condition number, each iteration requires inverting a linear system whose size is proportional to the dimension of the matrix and thus the per-iteration cost is prohibitive for large-scale problems. 
These together raise an important open question: {\em can one design an algorithm with a comparable per-iteration cost as gradient descent, but converges much faster at a rate that is independent of the condition number as alternating minimization in a provable manner for a wide variety of low-rank matrix estimation tasks?}

\subsection{Preconditioning helps: scaled gradient descent}

In this paper, we answer this question affirmatively by studying the following scaled gradient descent (\texttt{ScaledGD}) algorithm to optimize \eqref{eq:problem}. Given an initialization $(\bL_{0}, \bR_{0})$, \texttt{ScaledGD} proceeds as follows
\begin{align}
\begin{split} \bL_{t+1} & =\bL_{t}-\eta\nabla_{\bL}\cL(\bL_{t},\bR_{t})(\bR_{t}^{\top}\bR_{t})^{-1},\\
 \bR_{t+1} & =\bR_{t}-\eta\nabla_{\bR}\cL(\bL_{t},\bR_{t})(\bL_{t}^{\top}\bL_{t})^{-1},
\end{split}\label{eq:scaledGD}
\end{align}
where $\eta > 0$ is the step size and $\nabla_{\bL}\cL(\bL_{t},\bR_{t})$ (resp.~$\nabla_{\bR}\cL(\bL_{t},\bR_{t})$) is the gradient of the loss function $\cL$ with respect to the factor $\bL_{t}$ (resp.~$\bR_{t}$) at the $t$-th iteration. Comparing to vanilla gradient descent, the search directions of the low-rank factors $\bL_{t},\bR_{t}$ in \eqref{eq:scaledGD} are {\em scaled} by $(\bR_{t}^{\top}\bR_{t})^{-1}$ and $(\bL_{t}^{\top}\bL_{t})^{-1}$ respectively. 
Intuitively, the scaling serves as a preconditioner as in quasi-Newton type algorithms, with the hope of improving the quality of the search direction to allow larger step sizes. Since the computation of the Hessian is extremely expensive, it is necessary to design preconditioners that are both theoretically sound and practically cheap to compute. Such requirements are met by \texttt{ScaledGD}, where the preconditioners are computed by inverting two $r\times r$ matrices, whose size is much smaller than the dimension of matrix factors.
Therefore, each iteration of \texttt{ScaledGD} adds minimal overhead to the gradient computation and has the order-wise same per-iteration cost as gradient descent. Moreover, the preconditioners are adaptive and iteration-varying. Another key property of \texttt{ScaledGD} is that it ensures the iterates are covariant with respect to the parameterization of low-rank factors up to invertible transforms. 

While \texttt{ScaledGD} and its alternating variants have been proposed in  \cite{mishra2012riemannian,mishra2016riemannian,tanner2016low} for a subset of the problems we studied, none of these prior art provides any theoretical validations to the empirical success.
In this work, we confirm {\em theoretically} that \texttt{ScaledGD} achieves linear convergence at a rate {\em independent of} the condition number of the matrix when initialized properly, e.g.~using the standard spectral method, for several canonical problems: low-rank matrix sensing, robust PCA, and matrix completion. Table~\ref{tab:performance-guarantees-ScaledGD} summarizes the performance guarantees of \texttt{ScaledGD} in terms of both statistical and computational complexities with comparisons to prior algorithms using the vanilla gradient method. 
\begin{itemize}
\item {\em Low-rank matrix sensing.} As long as the measurement operator satisfies the standard restricted isometry property (RIP) with an RIP constant $\delta_{2r}\lesssim 1/(\sqrt{r}\kappa)$, where $\kappa$ is the condition number of $\bX_{\star}$, \texttt{ScaledGD} reaches $\epsilon$-accuracy in $O(\log(1/\epsilon))$ iterations when initialized by the spectral method. This strictly improves the iteration complexity $O(\kappa\log(1/\epsilon))$ of gradient descent in \cite{tu2015low} under the same sample complexity requirement.

\item {\em Robust PCA.} Under the deterministic corruption model \cite{chandrasekaran2011siam},  as long as the fraction $\alpha$ of corruptions per row\;/\;column satisfies $\alpha \lesssim 1/(\mu r^{3/2}\kappa)$, where $\mu$ is the incoherence parameter of $\bX_{\star}$, \texttt{ScaledGD} in conjunction with hard thresholding reaches $\epsilon$-accuracy in $O(\log(1/\epsilon))$ iterations when initialized by the spectral method. This strictly improves the iteration complexity of projected gradient descent \cite{yi2016fast}.

\item {\em Matrix completion.} Under the random Bernoulli observation model, as long as the sample complexity satisfies $n_{1}n_{2}p\gtrsim (\mu\kappa^2\vee \log n)\mu nr^2\kappa^2$ with $n=n_1 \vee n_2$, \texttt{ScaledGD} in conjunction with a properly designed projection operator reaches $\epsilon$-accuracy in $O(\log(1/\epsilon))$ iterations when initialized by the spectral method. This improves the iteration complexity of projected gradient descent \cite{zheng2016convergence} at the expense of requiring a larger sample size. 
\end{itemize}
In addition, \texttt{ScaledGD} does not require any explicit regularizations that balance the norms of two low-rank factors as required in \cite{tu2015low,yi2016fast,zheng2016convergence}, and removed the additional projection that maintains the incoherence properties in robust PCA \cite{yi2016fast}, thus unveiling the implicit regularization property of \texttt{ScaledGD}.
To the best of our knowledge, this is the first factored gradient descent algorithm that achieves a fast convergence rate that is independent of the condition number of the low-rank matrix at near-optimal sample complexities without increasing the per-iteration computational cost. Our analysis is also applicable to general loss functions that are restricted strongly convex and smooth over low-rank matrices.

\begin{table}[t]
\centering %
\resizebox{\textwidth}{!}{\begin{tabular}{c||c|c||c|c||c|c}
\hline 
 & \multicolumn{2}{c||}{Matrix sensing}   & \multicolumn{2}{c||}{Robust PCA}  & \multicolumn{2}{c}{Matrix completion}  \tabularnewline
\hline 
\hline 
\multirow{2}{*}{Algorithms} & sample  & iteration & corruption  & iteration  & sample  & iteration \tabularnewline
 & complexity  & complexity     & fraction  & complexity &   complexity  & complexity \tabularnewline
\hline 
\multirow{2}{*}{\texttt{GD}}  & \multirow{2}{*}{$nr^{2}\kappa^{2}$} & \multirow{2}{*}{$\kappa\log\frac{1}{\epsilon}$} &   \multirow{2}{*}{$\frac{1}{\mu r^{3/2}\kappa^{3/2} \vee \mu r\kappa^{2}}$} & \multirow{2}{*}{$\kappa\log\frac{1}{\epsilon}$}   & \multirow{2}{*}{$(\mu\vee\log n)\mu n r^{2}\kappa^{2}$ } & \multirow{2}{*}{$\kappa\log\frac{1}{\epsilon}$} \tabularnewline
 &  &    &  &   &  &  \tabularnewline 
\hline 
\texttt{ScaledGD}  & \multirow{2}{*}{$nr^{2}\kappa^{2}$} & \multirow{2}{*}{$\log\frac{1}{\epsilon}$}   & \multirow{2}{*}{$\frac{1}{\mu r^{3/2}\kappa}$} & \multirow{2}{*}{$\log\frac{1}{\epsilon}$}    & \multirow{2}{*}{$(\mu\kappa^{2}\vee\log n)\mu n r^{2}\kappa^{2} $} & \multirow{2}{*}{$\log\frac{1}{\epsilon}$}\tabularnewline
(this paper)  &  &    &  &     & & \tabularnewline
\hline
\end{tabular}}\vspace{0.04in}
 \caption{Comparisons of \texttt{ScaledGD} with gradient descent (\texttt{GD}) when tailored to various problems (with spectral initialization) \cite{tu2015low,yi2016fast,zheng2016convergence}, where they have comparable per-iteration costs. Here, we say that the output $\bX$ of an algorithm reaches $\epsilon$-accuracy, if it satisfies $\|\bX-\bX_{\star}\|_{\fro}\le\epsilon\sigma_{r}(\bX_{\star})$. Here, $n\coloneqq n_1 \vee n_2 = \max\{n_{1},n_{2}\}$, $\kappa$ and $\mu$ are the condition number and incoherence parameter of $\bX_{\star}$.
 \label{tab:performance-guarantees-ScaledGD}  }
\end{table}

At the core of our analysis, we introduce a new distance metric (i.e.~Lyapunov function) that accounts for the preconditioners, and carefully show the contraction of the \texttt{ScaledGD} iterates under the new distance metric.
We expect that the \texttt{ScaledGD} algorithm can accelerate the convergence for other low-rank matrix estimation problems, as well as facilitate the design and analysis of other quasi-Newton first-order algorithms. 
As a teaser, Figure~\ref{fig:scaledGD_teaser} illustrates the relative error of completing a $1000\times 1000$ incoherent matrix of rank $10$ with varying condition numbers from $20\%$ of its entries, using either \texttt{ScaledGD} or vanilla GD with spectral initialization. Even for moderately ill-conditioned matrices, the convergence rate of vanilla GD slows down dramatically, while it is evident that \texttt{ScaledGD} converges at a rate independent of the condition number and therefore is much more efficient. 
 
\begin{figure}[t]
\centering
\includegraphics[width=0.45\textwidth]{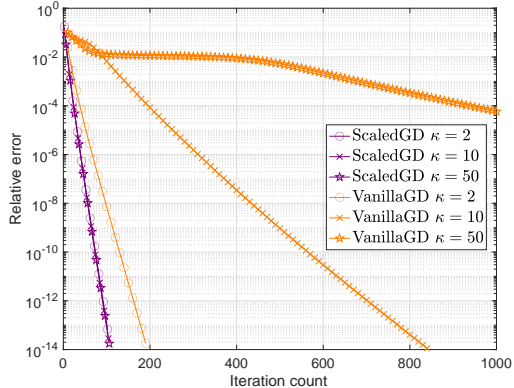}
\caption{Performance of \texttt{ScaledGD} and vanilla GD for completing a $1000\times 1000$ incoherent matrix of rank $10$ with different condition numbers $\kappa=2, 10, 50$, where each entry is observed independently with probability $0.2$. Here, both methods are initialized via the spectral method. It can be seen that \texttt{ScaledGD} converges much faster than vanilla GD even for moderately large condition numbers.}\label{fig:scaledGD_teaser}
\end{figure}

\begin{remark}[\texttt{ScaledGD} for PSD matrices] When the low-rank matrix of interest is positive semi-definite (PSD), we factorize the matrix $\bX\in\RR^{n\times n}$ as $\bX =\bL\bL^{\top}$, with $\bL\in\RR^{n\times r}$. The update rule of \texttt{ScaledGD} simplifies to
\begin{align}
\bL_{t+1}=\bL_{t}-\eta\nabla_{\bL}\cL(\bL_{t})(\bL_{t}^{\top}\bL_{t})^{-1}.\label{eq:scaledGD_PSD} 
\end{align}
We focus on the asymmetric case since the analysis is more involved with two factors. Our theory applies to the PSD case without loss of generality.
\end{remark}

\input{related-work.tex}

\subsection{Paper organization and notation}

The rest of this paper is organized as follows. Section~\ref{sec:problem_formulation} describes the proposed \texttt{ScaledGD} method and details its application to low-rank matrix sensing, robust PCA and matrix completion with theoretical guarantees in terms of both statistical and computational complexities, highlighting the role of a new distance metric. The convergence guarantee of \texttt{ScaledGD} under the general loss function is also presented.
In Section~\ref{sec:proof-outline}, we outline the proof for our main results. Section~\ref{sec:numerical} illustrates the excellent empirical performance of \texttt{ScaledGD} in a variety of low-rank matrix estimation problems. Finally, we conclude in Section~\ref{sec:discussion}. 

Before continuing, we introduce several notation used throughout the paper. First of all, we use boldfaced symbols for vectors and matrices. For a vector $\bv$, we use $\|\bv\|_{0}$ to denote its $\ell_0$ counting norm, and $\|\bv\|_{2}$ to denote the $\ell_2$ norm. For any matrix $\bA$, we use $\sigma_{i}(\bA)$ to denote its $i$-th largest singular value, and let $\bA_{i,\cdot}$ and $\bA_{\cdot,j}$ denote its $i$-th row and $j$-th column, respectively. In addition, $\|\bA\|_{\op}$, $\|\bA\|_{\fro}$, $\|\bA\|_{1,\infty}$, $\|\bA\|_{2,\infty}$, and $\|\bA\|_{\infty}$ stand for the spectral norm (i.e.~the largest singular value), the Frobenius norm, the $\ell_{1,\infty}$ norm (i.e.~the largest $\ell_1$ norm of the rows), the $\ell_{2,\infty}$ norm (i.e.~the largest $\ell_2$ norm of the rows), and the entrywise $\ell_{\infty}$ norm (the largest magnitude of all entries) of a matrix $\bA$. We denote 
\begin{align}
\cP_{r}(\bA)=\min_{\tilde{\bA}:\rank(\tilde{\bA}) \le r}\;\|\bA-\tilde{\bA}\|_{\fro}^2\label{eq:rank_r_proj}
\end{align}
as the rank-$r$ approximation of $\bA$, which is given by the top-$r$ SVD of $\bA$ by the Eckart-Young-Mirsky theorem. We also use $\vc(\bA)$ to denote the vectorization of a matrix $\bA$. For matrices $\bA,\bB$ of the same size, we use $\langle\bA,\bB\rangle=\sum_{i,j}\bA_{i,j}\bB_{i,j}=\tr(\bA^{\top}\bB)$ to denote their inner product.  The set of invertible matrices in $\RR^{r\times r}$ is denoted by $\GL(r)$.  
Let $a\vee b=\max\{a,b\}$ and $a\wedge b=\min\{a,b\}$. Throughout, $f(n)\lesssim g(n)$ or $f(n)=O(g(n))$ means $|f(n)|/|g(n)|\le C$
for some constant $C>0$ when $n$ is sufficiently large; $f(n)\gtrsim g(n)$ means $|f(n)|/|g(n)|\ge C$
for some constant $C>0$ when $n$ is sufficiently large.
Last but not least, we use the terminology ``with overwhelming probability'' to denote the event happens with probability at least $1-c_{1}n^{-c_{2}}$, where $c_{1},c_{2}>0$ are some universal constants, whose values may vary from line to line.

%% file: related-work.tex
\subsection{Related work} \label{sec:related}

Our work contributes to the growing literature of design and analysis of provable nonconvex optimization procedures for high-dimensional signal estimation; see e.g.~\cite{jain2017non,chen2018harnessing,chi2019nonconvex} for recent overviews.
A growing number of problems have been demonstrated to possess benign geometry that is amenable for optimization \cite{mei2016landscape} either globally or locally under appropriate statistical models. On one end, it is shown that there are no spurious local minima in the optimization landscape of matrix sensing and completion \cite{ge2016matrix,bhojanapalli2016global,park2017non,ge2017no}, phase retrieval \cite{sun2018geometric,davis2017nonsmooth}, dictionary learning \cite{sun2015complete}, kernel PCA \cite{chen2019model} and linear neural networks \cite{baldi1989neural,kawaguchi2016deep}.
Such landscape analysis facilitates the adoption of generic saddle-point escaping algorithms \cite{nesterov2006cubic,ge2015escaping,jin2017escape} to ensure global convergence. However, the resulting iteration complexity is typically high. On the other end, local refinements with carefully-designed initializations often admit fast convergence, for example in phase retrieval \cite{candes2015phase,ma2017implicit}, matrix sensing \cite{jain2013low,zheng2015convergent,wei2016guarantees}, matrix completion \cite{sun2016guaranteed,chen2015fast,ma2017implicit,chen2019nonconvex,zheng2016convergence,chen2019noisy}, blind deconvolution \cite{li2019rapid,ma2017implicit}, and robust PCA \cite{netrapalli2014non,yi2016fast,chen2020bridging}, to name a few. 

Existing approaches for asymmetric low-rank matrix estimation often require additional regularization terms to balance the two factors, either in the form of $\frac{1}{2}\|\bL^{\top}\bL-\bR^{\top}\bR\|_{\fro}^{2}$ \cite{tu2015low,park2017non} or $\frac{1}{2}\|\bL\|_{\fro}^{2}+\frac{1}{2}\|\bR\|_{\fro}^{2}$ \cite{zhu2017global,chen2019noisy, chen2020bridging}, which ease the theoretical analysis but are often unnecessary for the practical success, as long as the initialization is balanced. Some recent work studies the unregularized gradient descent for low-rank matrix factorization and sensing including \cite{charisopoulos2019low,du2018algorithmic,ma2021beyond}.
However, the iteration complexity of all these approaches scales at least linearly with respect to the condition number $\kappa$ of the low-rank matrix, e.g.~$O(\kappa\log(1/\epsilon))$, to reach $\epsilon$-accuracy, therefore they converge slowly when the underlying matrix becomes ill-conditioned.
In contrast, \texttt{ScaledGD} enjoys a local convergence rate of $O(\log(1/\epsilon))$, therefore incurring a much smaller computational footprint when $\kappa$ is large. Last but not least, alternating minimization \cite{jain2013low,hardt2014fast} (which alternatively updates $\bL_{t}$ and $\bR_{t}$) or singular value projection \cite{netrapalli2014non,jain2010guaranteed} (which operates in the matrix space) also converge at the rate $O(\log(1/\epsilon))$, but the per-iteration cost is much higher than \texttt{ScaledGD}. Another notable algorithm is the Riemannian gradient descent algorithm in \cite{wei2016guarantees}, which also converges at the rate $O(\log(1/\epsilon))$ under the same sample complexity for low-rank matrix sensing, but requires a higher memory complexity since it operates in the matrix space rather than the factor space.

From an algorithmic perspective, our approach is closely related to the alternating steepest descent (ASD) method in \cite{tanner2016low} for low-rank matrix completion, which performs the proposed updates \eqref{eq:scaledGD} for the low-rank factors in an alternating manner. Furthermore, the scaled gradient updates were also introduced in \cite{mishra2012riemannian,mishra2016riemannian} for low-rank matrix completion from the perspective of Riemannian optimization. However, none of \cite{tanner2016low,mishra2012riemannian,mishra2016riemannian} offered any statistical nor computational guarantees for global convergence. Our analysis of \texttt{ScaledGD} can be viewed as providing justifications to these precursors. Moreover, we have systematically extended the framework of \texttt{ScaledGD} to work in a large number of low-rank matrix estimation tasks such as robust PCA.

%% file: problem-formulation.tex
\section{Scaled Gradient Descent for Low-Rank Matrix Estimation}\label{sec:problem_formulation}

This section is devoted to introducing \texttt{ScaledGD} and establishing its statistical and computational guarantees for various low-rank matrix estimation problems. Before we instantiate tailored versions of \texttt{ScaledGD} on concrete low-rank matrix estimation problems, we first pause to provide more insights of the update rule of \texttt{ScaledGD}, by connecting it to the quasi-Newton method. Note that the update rule \eqref{eq:scaledGD} for \texttt{ScaledGD} can be equivalently written in a vectorization form as
\begin{align}
\vc(\bF_{t+1}) & =\vc(\bF_{t})-\eta\begin{bmatrix}(\bR_{t}^{\top}\bR_{t})^{-1}\otimes\bI_{n_{1}} & \zero\\
\zero & (\bL_{t}^{\top}\bL_{t})^{-1}\otimes\bI_{n_{2}}
\end{bmatrix}\vc(\nabla_{\bF}\cL(\bF_{t})) \nonumber\\
 & =\vc(\bF_{t})-\eta\bH_{t}^{-1}\vc(\nabla_{\bF}\cL(\bF_{t})),\label{eq:scaledGD_vec}
\end{align}
where we denote $\bF_{t}=[\bL_{t}^{\top},\bR_{t}^{\top}]^{\top}\in\RR^{(n_{1}+n_{2})\times r}$, and by $\otimes$ the Kronecker product. Here, the block diagonal matrix $\bH_{t}$ is set to be 
\begin{align*}
\bH_t \coloneqq \begin{bmatrix}(\bR_{t}^{\top}\bR_{t})\otimes\bI_{n_{1}} & \zero\\
\zero & (\bL_{t}^{\top}\bL_{t})\otimes\bI_{n_{2}}
\end{bmatrix}.
\end{align*}
The form \eqref{eq:scaledGD_vec} makes it apparent that \texttt{ScaledGD} can be interpreted as a quasi-Newton algorithm, where the inverse of $\bH_{t}$ can be cheaply computed through inverting two rank-$r$ matrices. 

\subsection{Assumptions and error metric} \label{subsec:scaledGD_assumptions}

Denote by $\bU_{\star}\bSigma_{\star}\bV_{\star}^{\top}$ the compact singular value decomposition (SVD) of the rank-$r$ matrix $\bX_{\star}\in\RR^{n_{1}\times n_{2}}$.
Here $\bU_{\star}\in\RR^{n_{1}\times r}$ and $\bV_{\star}\in\RR^{n_{2}\times r}$ are composed of $r$ left and right singular vectors, respectively, and $\bSigma_{\star}\in\RR^{r\times r}$ is a diagonal matrix consisting of $r$ singular values of $\bX_{\star}$ organized in a non-increasing order, i.e.~$\sigma_{1}(\bX_{\star})\ge\dots\ge\sigma_{r}(\bX_{\star})>0$. Define 
\begin{align}
\kappa\coloneqq\sigma_{1}(\bX_{\star})/\sigma_{r}(\bX_{\star})\label{eq:kappa}
\end{align} 
as the condition number of $\bX_{\star}$. Define the ground truth low-rank factors as 
\begin{align}
\bL_{\star}\coloneqq\bU_{\star}\bSigma_{\star}^{1/2}, \qquad\mbox{and}\qquad \bR_{\star}\coloneqq\bV_{\star}\bSigma_{\star}^{1/2},\label{eq:true_SVD}
\end{align} 
so that $\bX_{\star}=\bL_{\star}\bR_{\star}^{\top}$. Correspondingly, denote the stacked factor matrix as
\begin{align}
\bF_{\star}\coloneqq\begin{bmatrix}\bL_{\star} \\ \bR_{\star}\end{bmatrix}\in\RR^{(n_{1}+n_{2})\times r}.\label{eq:true_factor}
\end{align} 

Next, we are in need of a right metric to measure the performance of the \texttt{ScaledGD} iterates $\bF_{t}\coloneqq[\bL_{t}^{\top},\bR_{t}^{\top}]^{\top}$. 
Obviously, the factored representation is not unique in that for any invertible matrix $\bQ\in\GL(r)$, one has $\bL\bR^\top=(\bL\bQ)(\bR\bQ^{-\top})^{\top}$. Therefore, the reconstruction error metric needs to take into account this identifiability issue. More importantly, we need a diagonal scaling in the distance error metric to properly account for the effect of preconditioning. To provide intuition, note that the update rule \eqref{eq:scaledGD} can be viewed as finding the best local quadratic approximation of $\cL(\cdot)$ in the following sense:
\begin{align*}
(\bL_{t+1},\bR_{t+1})=\argmin_{\bL,\bR}\; & \cL(\bL_{t},\bR_{t})+\langle\nabla_{\bL}\cL(\bL_{t},\bR_{t}),\bL-\bL_{t}\rangle+\langle\nabla_{\bR}\cL(\bL_{t},\bR_{t}),\bR-\bR_{t}\rangle \\ 
 &\quad +\frac{1}{2\eta}\left(\left\Vert (\bL-\bL_{t})(\bR_{t}^{\top}\bR_{t})^{1/2}\right\Vert _{\fro}^{2}+\left\Vert (\bR-\bR_{t})(\bL_{t}^{\top}\bL_{t})^{1/2}\right\Vert _{\fro}^{2}\right),
\end{align*}
where it is different from the common interpretation of gradient descent in the way the quadratic approximation is taken by a scaled norm. When $\bL_{t}\approx\bL_{\star}$ and $\bR_{t}\approx\bR_{\star}$ are approaching the ground truth, the additional scaling factors can be approximated by $\bL_{t}^{\top}\bL_{t}\approx\bSigma_{\star}$ and $\bR_{t}^{\top}\bR_{t}\approx\bSigma_{\star}$, leading to the following error metric
\begin{align}
\dist^{2}(\bF,\bF_{\star})\coloneqq\inf_{\bQ\in\GL(r)}\;\left\Vert (\bL\bQ-\bL_{\star})\bSigma_{\star}^{1/2}\right\Vert _{\fro}^{2}+\left\Vert (\bR\bQ^{-\top}-\bR_{\star})\bSigma_{\star}^{1/2}\right\Vert _{\fro}^{2}.\label{eq:dist}
\end{align}
Correspondingly, we define the optimal alignment matrix $\bQ$ between $\bF$ and $\bF_{\star}$ as 
\begin{align}
\bQ\coloneqq\argmin_{\bQ\in\GL(r)}\;\left\Vert (\bL\bQ-\bL_{\star})\bSigma_{\star}^{1/2}\right\Vert _{\fro}^{2}+\left\Vert (\bR\bQ^{-\top}-\bR_{\star})\bSigma_{\star}^{1/2}\right\Vert _{\fro}^{2},\label{eq:Q_def}
\end{align}
whenever the minimum is achieved.\footnote{If there are multiple minimizers, we can arbitrarily take one to be $\bQ$.} It turns out that for the \texttt{ScaledGD} iterates $\{\bF_{t}\}$, the optimal alignment matrices $\{\bQ_{t}\}$ always exist (at least when properly initialized) and hence are well-defined. The design and analysis of this new distance metric are of crucial importance in obtaining the improved rate of \texttt{ScaledGD}; see Appendix~\ref{subsec:distance_metric} for a collection of its properties.
In comparison, the previously studied distance metrics (proposed mainly for GD) either do not include the diagonal scaling \cite{ma2021beyond,tu2015low}, or only consider the ambiguity class up to orthonormal transforms \cite{tu2015low}, which fail to unveil the benefit of \texttt{ScaledGD}.

\subsection{Matrix sensing}\label{subsec:scaledGD_MS}

Assume that we have collected a set of linear measurements about a rank-$r$ matrix $\bX_{\star}\in\RR^{n_{1}\times n_{2}}$, given as 
\begin{align}
\by=\cA(\bX_{\star})\in\RR^{m},\label{eq:sensing_measurements}
\end{align}
where $\cA(\bX)=\{\langle\bA_{k},\bX\rangle\}_{k=1}^{m}:\RR^{n_{1}\times n_{2}}\mapsto\RR^{m}$ is the linear map modeling the measurement process. The goal of low-rank matrix sensing is to recover $\bX_{\star}$ from $\by$, especially when the number of measurements $m\ll n_{1}n_{2}$, by exploiting the low-rank property. 
This problem has wide applications in medical imaging, signal processing, and data compression \cite{candes2011tight}. 

\paragraph{Algorithm.} Writing $\bX\in\RR^{n_{1}\times n_{2}}$ into a factored form $\bX=\bL\bR^{\top}$, we consider the following optimization problem: 
\begin{align}
\minimize_{\bF\in\RR^{(n_{1}+n_{2})\times r}}\;\cL(\bF)=\frac{1}{2}\left\Vert \cA(\bL\bR^{\top})-\by\right\Vert _{2}^{2}.\label{eq:loss_MS}
\end{align}
Here as before, $\bF$ denotes the stacked factor matrix $[\bL^\top, \bR^\top]^\top$. We suggest running \texttt{ScaledGD} \eqref{eq:scaledGD} with the spectral initialization to solve \eqref{eq:loss_MS}, which performs the top-$r$ SVD on $\cA^{*}(\by)$, where $\cA^{*}(\cdot)$ is the adjoint operator of $\cA(\cdot)$. The full algorithm is stated in Algorithm~\ref{alg:MS}. The low-rank matrix can be estimated as $\bX_{T}=\bL_{T}\bR_{T}^{\top}$ after running $T$ iterations of \texttt{ScaledGD}.

\begin{algorithm}[ht]
\caption{\texttt{ScaledGD} for low-rank matrix sensing with spectral initialization}\label{alg:MS} 
\begin{algorithmic} \STATE \textbf{{Spectral initialization}}: Let $\bU_{0}\bSigma_{0}\bV_{0}^{\top}$ be the top-$r$ SVD of $\cA^{*}(\by)$, and set 
\begin{align}
\bL_{0}=\bU_{0}\bSigma_{0}^{1/2},\quad\mbox{and}\quad\bR_{0}=\bV_{0}\bSigma_{0}^{1/2}.\label{eq:init_MS}
\end{align}
\vspace{-0.1in}
\STATE \textbf{{Scaled gradient updates}}: \textbf{for} $t=0,1,2,\dots,T-1$ \textbf{do} 
\begin{align}
\begin{split} & \bL_{t+1}=\bL_{t}-\eta\cA^{*}(\cA(\bL_{t}\bR_{t}^{\top})-\by)\bR_{t}(\bR_{t}^{\top}\bR_{t})^{-1},\\
 & \bR_{t+1}=\bR_{t}-\eta\cA^{*}(\cA(\bL_{t}\bR_{t}^{\top})-\by)^{\top}\bL_{t}(\bL_{t}^{\top}\bL_{t})^{-1}.
\end{split}\label{eq:iterates_MS}
\end{align}
\end{algorithmic} 
\end{algorithm}

\paragraph{Theoretical guarantees.} To understand the performance of \texttt{ScaledGD} for low-rank matrix sensing, we adopt a standard assumption on the sensing operator $\cA(\cdot)$, namely the Restricted Isometry Property (RIP). 
\begin{definition}[RIP \cite{recht2010guaranteed}] The linear map $\cA(\cdot)$ is said to obey the rank-$r$ RIP with a constant $\delta_{r}\in[0,1)$, if for all matrices $\bM\in\RR^{n_{1}\times n_{2}}$ of rank at most $r$, one has 
\begin{align*}
(1-\delta_{r})\|\bM\|_{\fro}^{2}\le\|\cA(\bM)\|_{2}^{2}\le(1+\delta_{r})\|\bM\|_{\fro}^{2}.
\end{align*}
\end{definition}

It is well-known that many measurement ensembles satisfy the RIP property \cite{recht2010guaranteed,candes2011tight}. For example, if the entries of $\bA_{i}$'s are composed of i.i.d.~Gaussian entries $\cN(0,1/m)$, then the RIP is satisfied for a constant $\delta_{r}$ as long as $m$ is on the order of $(n_{1}+n_{2})r/\delta_{r}^{2}$. With the RIP condition in place, the following theorem demonstrates that \texttt{ScaledGD} converges linearly --- in terms of the new distance metric (cf.~\eqref{eq:dist}) --- at a constant rate as long as the sensing operator $\cA(\cdot)$ has a sufficiently small RIP constant.

\begin{theorem}\label{thm:MS} Suppose that $\cA(\cdot)$ obeys the $2r$-RIP with $\delta_{2r}\le0.02/(\sqrt{r}\kappa)$. If the step size obeys $0<\eta\le2/3$, then for all $t\ge0$, the iterates of the \texttt{ScaledGD} method in Algorithm~\ref{alg:MS} satisfy 
\begin{align*}
\dist(\bF_{t},\bF_{\star})\le(1-0.6\eta)^{t}0.1\sigma_{r}(\bX_{\star}),\quad\mbox{and}\quad\left\Vert \bL_{t}\bR_{t}^{\top}-\bX_{\star}\right\Vert _{\fro}\le(1-0.6\eta)^{t}0.15\sigma_{r}(\bX_{\star}).
\end{align*}
\end{theorem}

Theorem~\ref{thm:MS} establishes that the distance $\dist(\bF_{t},\bF_{\star})$ contracts linearly at a constant rate, as long as the sample size satisfies $m=O(nr^{2}\kappa^{2})$ with Gaussian random measurements \cite{recht2010guaranteed}, where we recall that $n=n_{1}\vee n_{2}$.
To reach $\epsilon$-accuracy, i.e.~$\|\bL_{t}\bR_{t}^{\top}-\bX_{\star}\|_{\fro}\le\epsilon\sigma_{r}(\bX_{\star})$, \texttt{ScaledGD} takes at most $T=O(\log(1/\epsilon))$ iterations, which is {\em independent} of the condition number $\kappa$ of $\bX_{\star}$. 
In comparison, alternating minimization with spectral initialization (\texttt{AltMinSense}) converges in $O(\log(1/\epsilon))$ iterations as long as $m=O(nr^{3}\kappa^{4})$ \cite{jain2013low}, where the per-iteration cost is much higher.\footnote{The exact per-iteration complexity of \texttt{AltMinSense} depends on how the least-squares subproblems are solved with $m$ equations and $nr$ unknowns; see \cite[Table 1]{luo2020recursive} for detailed comparisons.} On the other end, gradient descent with spectral initialization in \cite{tu2015low} converges in $O(\kappa\log(1/\epsilon))$ iterations as long as $m=O(nr^{2}\kappa^{2})$.
Therefore, \texttt{ScaledGD} converges at a much faster rate than GD at the same sample complexity while requiring a significantly lower per-iteration cost than \texttt{AltMinSense}.

\begin{remark} \cite{tu2015low} suggested that one can employ a more expensive initialization scheme, e.g.~performing multiple projected gradient descent steps over the low-rank matrix, to reduce the sample complexity. By seeding \texttt{ScaledGD} with the output of updates of the form $\bX_{\tau+1}=\cP_{r}\left(\bX_{\tau}-\cA^{*}(\cA(\bX_{\tau})-\by)\right)$ after $T_{0}\gtrsim\log (\sqrt{r}\kappa)$ iterations, where $\cP_r(\cdot)$ is defined in \eqref{eq:rank_r_proj}, \texttt{ScaledGD} succeeds with the sample size $O(nr)$ which is information theoretically optimal.
\end{remark}

\subsection{Robust PCA}\label{subsec:scaledGD_RPCA}

Assume that we have observed the data matrix 
\begin{align*}
\bY=\bX_{\star}+\bS_{\star},
\end{align*}
which is a superposition of a rank-$r$ matrix $\bX_{\star}$, modeling the clean data, and a sparse matrix $\bS_{\star}$, modeling the corruption or outliers.
The goal of robust PCA~\cite{candes2009robustPCA,chandrasekaran2011siam} is to separate the two matrices $\bX_{\star}$ and $\bS_{\star}$ from their mixture $\bY$. This problem finds numerous applications in video surveillance, image processing, and so on. 

Following~\cite{chandrasekaran2011siam,netrapalli2014non,yi2016fast}, we consider a deterministic sparsity model for $\bS_{\star}$, in which $\bS_{\star}$ contains at most $\alpha$-fraction of nonzero entries per row and column for some $\alpha\in[0,1)$, i.e.~$\bS_{\star}\in\cS_{\alpha}$, where we denote
\begin{align}
\cS_{\alpha}\coloneqq\{\bS\in\RR^{n_{1}\times n_{2}}:\|\bS_{i,\cdot}\|_{0}\le\alpha n_{2}\mbox{ for all }i, \mbox{ and }\|\bS_{\cdot,j}\|_{0}\le\alpha n_{1}\mbox{ for all }j\}.\label{eq:S_alpha}
\end{align}

\paragraph{Algorithm.} Writing $\bX\in\RR^{n_{1}\times n_{2}}$ into the factored form $\bX=\bL\bR^{\top}$, we consider the following optimization problem: 
\begin{align}
\minimize_{\bF\in\RR^{(n_{1}+n_{2})\times r},\bS\in\cS_{\alpha}}\;\cL(\bF,\bS)=\frac{1}{2}\left\Vert \bL\bR^{\top}+\bS-\bY\right\Vert _{\fro}^{2}.\label{eq:loss_RPCA}
\end{align}
It is thus natural to alternatively update $\bF=[\bL^{\top},\bR^{\top}]^{\top}$ and $\bS$, where $\bF$ is updated via the proposed \texttt{ScaledGD} algorithm, and $\bS$ is updated by hard thresholding, which trims the small entries of the residual matrix $\bY-\bL\bR^{\top}$. More specifically, for some truncation level $0\le\bar{\alpha}\le1$, we define the sparsification operator that only keeps $\bar{\alpha}$ fraction of largest entries in each row and column:
\begin{align}
(\cT_{\bar{\alpha}}[\bA])_{i,j}=\begin{cases}
\bA_{i,j}, & \mbox{if }|\bA|_{i,j}\ge|\bA|_{i,(\bar{\alpha}n_{2})},\mbox{ and }|\bA|_{i,j}\ge|\bA|_{(\bar{\alpha}n_{1}),j}\\
0, & \mbox{otherwise}
\end{cases},\label{eq:T_alpha}
\end{align}
where $|\bA|_{i,(k)}$ (resp.~$|\bA|_{(k),j}$) denote the $k$-th largest element in magnitude in the $i$-th row (resp.~$j$-th column).

The \texttt{ScaledGD} algorithm with the spectral initialization for solving robust PCA is formally stated in Algorithm~\ref{alg:RPCA}. Note that, comparing with \cite{yi2016fast}, we do not require a balancing term $\|\bL^{\top}\bL-\bR^{\top}\bR\|_{\fro}^{2}$ in the loss function \eqref{eq:loss_RPCA}, nor the projection of the low-rank factors onto the $\ell_{2,\infty}$ ball in each iteration.
\begin{algorithm}[ht]
\caption{\texttt{ScaledGD} for robust PCA with spectral initialization}\label{alg:RPCA} 
\begin{algorithmic} \STATE \textbf{{Spectral initialization}}: Let $\bU_{0}\bSigma_{0}\bV_{0}^{\top}$ be the top-$r$ SVD of $\bY-\cT_{\alpha}[\bY]$, and set 
\begin{align}
\bL_{0}=\bU_{0}\bSigma_{0}^{1/2},\quad\mbox{and}\quad\bR_{0}=\bV_{0}\bSigma_{0}^{1/2}.\label{eq:init_RPCA}
\end{align}
\vspace{-0.1in}
\STATE \textbf{{Scaled gradient updates}}: \textbf{for} $t=0,1,2,\dots,T-1$ \textbf{do} 
\begin{align}
\begin{split} \bS_{t} & =\cT_{2\alpha}[\bY-\bL_{t}\bR_{t}^{\top}],\\
 \bL_{t+1} & =\bL_{t}-\eta(\bL_{t}\bR_{t}^{\top}+\bS_{t}-\bY)\bR_{t}(\bR_{t}^{\top}\bR_{t})^{-1},\\
 \bR_{t+1} & =\bR_{t}-\eta(\bL_{t}\bR_{t}^{\top}+\bS_{t}-\bY)^{\top}\bL_{t}(\bL_{t}^{\top}\bL_{t})^{-1}.
\end{split}\label{eq:iterates_RPCA}
\end{align}
\end{algorithmic} 
\end{algorithm}

\paragraph{Theoretical guarantee.} Before stating our main result for robust PCA, we introduce the incoherence condition which is known to be crucial for reliable estimation of the low-rank matrix $\bX_{\star}$ in robust PCA~\cite{chen2015incoherence}. 
\begin{definition}[Incoherence]\label{def:incoherence} A rank-$r$ matrix $\bX_{\star}\in\RR^{n_{1}\times n_{2}}$ with compact SVD as $\bX_{\star}=\bU_{\star}\bSigma_{\star}\bV_{\star}^{\top}$ is said to be $\mu$-incoherent if
\begin{align*}
\|\bU_{\star}\|_{2,\infty}\le\sqrt{\frac{\mu}{n_{1}} }\|\bU_{\star}\|_{\fro}=\sqrt{\frac{\mu r}{n_{1}}},\quad\mbox{and}\quad\|\bV_{\star}\|_{2,\infty}\le\sqrt{\frac{\mu}{n_{2}}}\|\bV_{\star}\|_{\fro}=\sqrt{ \frac{\mu r}{n_{2}}}.
\end{align*}
\end{definition}

The following theorem establishes that \texttt{ScaledGD} converges linearly at a constant rate as long as the fraction $\alpha$ of corruptions is sufficiently small.

\begin{theorem}\label{thm:RPCA} Suppose that $\bX_{\star}$ is $\mu$-incoherent and that the corruption fraction $\alpha$ obeys $\alpha\le c/(\mu r^{3/2}\kappa)$ for some sufficiently small constant $c>0$. If the step size obeys $0.1\le\eta\le2/3$, then for all $t\ge0$, the iterates of \texttt{ScaledGD} in Algorithm~\ref{alg:RPCA} satisfy 
\begin{align*}
\dist(\bF_{t},\bF_{\star}) & \le(1-0.6\eta)^{t}0.02\sigma_{r}(\bX_{\star}),\quad\mbox{and}\quad\left\Vert \bL_{t}\bR_{t}^{\top}-\bX_{\star}\right\Vert _{\fro}\le(1-0.6\eta)^{t}0.03\sigma_{r}(\bX_{\star}).
\end{align*}
\end{theorem}

Theorem~\ref{thm:RPCA} establishes that the distance $\dist(\bF_{t},\bF_{\star})$ contracts linearly at a constant rate, as long as the fraction of corruptions satisfies $\alpha\lesssim1/(\mu r^{3/2}\kappa)$.
To reach $\epsilon$-accuracy, i.e.~$\|\bL_{t}\bR_{t}^{\top}-\bX_{\star}\|_{\fro}\le\epsilon\sigma_{r}(\bX_{\star})$, \texttt{ScaledGD} takes at most $T=O(\log(1/\epsilon))$ iterations, which is {\em independent} of $\kappa$. 
In comparison, the \texttt{AltProj} algorithm\footnote{\texttt{AltProj} employs a multi-stage strategy to remove the dependence on $\kappa$ in $\alpha$, which we do not consider here. The same strategy might also improve the dependence on $\kappa$ for \texttt{ScaledGD}, which we leave for future work.} with spectral initialization converges in $O(\log(1/\epsilon))$ iterations as long as $\alpha\lesssim1/(\mu r)$ \cite{netrapalli2014non},
 where the per-iteration cost is much higher both in terms of computation and memory as it requires the computation of the low-rank SVD of the full matrix. 
On the other hand, projected gradient descent with spectral initialization in \cite{yi2016fast} converges in $O(\kappa\log(1/\epsilon))$ iterations as long as $\alpha\lesssim 1/(\mu r^{3/2}\kappa^{3/2}\vee\mu r\kappa^{2})$.
Therefore, \texttt{ScaledGD} converges at a much faster rate than GD while requesting a significantly lower per-iteration cost than \texttt{AltProj}. In addition, our theory suggests that \texttt{ScaledGD} maintains the incoherence and balancedness of the low-rank factors without imposing explicit regularizations, which is not captured in previous analysis \cite{yi2016fast}.

\subsection{Matrix completion}\label{subsec:scaledGD_MC}

Assume that we have observed a subset $\Omega$ of entries of $\bX_{\star}$ given as 
$\cP_{\Omega}(\bX_{\star})$, where $\cP_{\Omega}:\RR^{n_{1}\times n_{2}}\mapsto\RR^{n_{1}\times n_{2}}$ is a projection such that
\begin{align}
(\cP_{\Omega}(\bX))_{i,j}=\begin{cases}\bX_{i,j}, & \mbox{if }(i,j)\in\Omega \\
0, & \mbox{otherwise}\end{cases}.
\end{align}
Here $\Omega$ is generated according to the Bernoulli model in the sense that each $(i,j) \in \Omega$ independent with probability $p$. 
The goal of matrix completion is to recover the matrix $\bX_{\star}$ from its partial observation $\cP_{\Omega}(\bX_{\star})$. This problem has many applications in recommendation systems, signal processing, sensor network localization, and so on \cite{candes_mc}.

\paragraph{Algorithm.} Again, writing $\bX\in\RR^{n_{1}\times n_{2}}$ into the factored form $\bX=\bL\bR^{\top}$, we consider the following optimization problem: 
\begin{align}
\minimize_{\bF\in\RR^{(n_{1}+n_{2})\times r}}\;\cL(\bF)\coloneqq\frac{1}{2p}\left\Vert \cP_{\Omega}(\bL\bR^{\top}-\bX_{\star})\right\Vert _{\fro}^2.\label{eq:loss_MC}
\end{align}
Similarly to robust PCA, the underlying low-rank matrix $\bX_{\star}$ needs to be incoherent (cf.~Definition~\ref{def:incoherence}) to avoid ill-posedness. One typical strategy to ensure the incoherence condition is to perform projection after the gradient update, by projecting the iterates to maintain small $\ell_{2,\infty}$ norms of the factor matrices. However, the standard projection operator \cite{chen2015fast} is not covariant with respect to invertible transforms, and consequently, needs to be modified when using scaled gradient updates. To that end, we introduce the following new projection operator: for every $\tilde{\bF}\in\RR^{(n_{1}+n_{2})\times r} = [\tilde{\bL}^{\top}, \tilde{\bR}^{\top}]^{\top}$,
\begin{align}
\begin{split}
 \cP_{B}(\tilde{\bF}) =& \argmin_{\bF\in\RR^{(n_{1}+n_{2})\times r}}\;\left\Vert(\bL-\tilde{\bL})(\tilde{\bR}^{\top}\tilde{\bR})^{1/2}\right\Vert_{\fro}^{2} + \left\Vert(\bR-\tilde{\bR})(\tilde{\bL}^{\top}\tilde{\bL})^{1/2}\right\Vert_{\fro}^{2} \\
 & \quad\mbox{s.t.}\quad\sqrt{n_{1}}\left\Vert\bL(\tilde{\bR}^{\top}\tilde{\bR})^{1/2}\right\Vert_{2,\infty} \vee \sqrt{n_{2}}\left\Vert\bR(\tilde{\bL}^{\top}\tilde{\bL})^{1/2}\right\Vert_{2,\infty} \le B
\end{split},\label{eq:scaled_proj_opt}
\end{align} 
which finds a factored matrix that is closest to $\tilde{\bF}$ and stays incoherent in a weighted sense. Luckily, the solution to the above scaled projection admits a simple closed-form solution, as stated below.
\begin{proposition}\label{prop:scaled_proj_sol} The solution to \eqref{eq:scaled_proj_opt} is given by
\begin{align}
\begin{split}
 \cP_{B}(\tilde{\bF}) \coloneqq\begin{bmatrix}\bL \\ \bR \end{bmatrix},\quad\mbox{where}\quad
\bL_{i,\cdot} & \coloneqq \left(1\wedge \frac{B}{\sqrt{n_{1}}\|\tilde{\bL}_{i,\cdot}\tilde{\bR}^{\top}\|_{2}}\right)\tilde{\bL}_{i,\cdot}, \quad 1\le i \le n_{1}, \\ 
\bR_{j,\cdot} & \coloneqq \left(1\wedge \frac{B}{\sqrt{n_{2}}\|\tilde{\bR}_{j,\cdot}\tilde{\bL}^{\top}\|_{2}}\right)\tilde{\bR}_{j,\cdot}, \quad 1\le j \le n_{2}.
\end{split}\label{eq:scaled_proj}
\end{align}
\end{proposition}
\begin{proof}
See Appendix~\ref{proof:scaled_proj_sol}.
\end{proof} 
With the new projection operator in place, we propose the scaled projected gradient descent (\texttt{ScaledPGD}) method with the spectral initialization for solving matrix completion, formally stated in Algorithm~\ref{alg:MC}.

\begin{algorithm}[ht]
\caption{\texttt{ScaledPGD} for matrix completion with spectral initialization}\label{alg:MC} 
\begin{algorithmic} \STATE \textbf{{Spectral initialization}}: Let $\bU_{0}\bSigma_{0}\bV_{0}^{\top}$ be the top-$r$ SVD of $\frac{1}{p}\cP_{\Omega}(\bX_{\star})$, and set
\begin{align}
\begin{bmatrix}\bL_{0} \\ \bR_{0}\end{bmatrix}=\cP_{B}\left(\begin{bmatrix}\bU_{0}\bSigma_{0}^{1/2} \\ \bV_{0}\bSigma_{0}^{1/2}\end{bmatrix}\right).\label{eq:init_MC}
\end{align}
\vspace{-0.1in}
\STATE \textbf{{Scaled projected gradient updates}}: \textbf{for} $t=0,1,2,\dots,T-1$ \textbf{do} 
\begin{align}
\begin{bmatrix}\bL_{t+1} \\ \bR_{t+1}\end{bmatrix}=\cP_{B}\left(\begin{bmatrix}
\bL_{t}-\frac{\eta}{p}\cP_{\Omega}(\bL_{t}\bR_{t}^{\top}-\bX_{\star})\bR_{t}(\bR_{t}^{\top}\bR_{t})^{-1} \\
\bR_{t}-\frac{\eta}{p}\cP_{\Omega}(\bL_{t}\bR_{t}^{\top}-\bX_{\star})^{\top}\bL_{t}(\bL_{t}^{\top}\bL_{t})^{-1}
\end{bmatrix}\right).\label{eq:iterates_MC}
\end{align}
\end{algorithmic} 
\end{algorithm}

\paragraph{Theoretical guarantee.} Consider a random observation model, where each index $(i,j)$ belongs to the index set $\Omega$ independently with probability $0<p\le 1$. The following theorem establishes that \texttt{ScaledPGD} converges linearly at a constant rate as long as the number of observations is sufficiently large.
 
\begin{theorem}\label{thm:MC} Suppose that $\bX_{\star}$ is $\mu$-incoherent, and that $p$ satisfies $p\ge C(\mu\kappa^{2}\vee \log(n_{1}\vee n_{2}))\mu r^{2}\kappa^{2}/(n_{1}\wedge n_{2})$ for some sufficiently large constant $C$. Set the projection radius as $B=C_{B}\sqrt{\mu r}\sigma_{1}(\bX_{\star})$ for some constant $C_{B}\ge1.02$. If the step size obeys $0<\eta\le2/3$, then with probability at least $1-c_{1}(n_{1}\vee n_{2})^{-c_{2}}$, for all $t\ge0$, the iterates of \texttt{ScaledPGD} in \eqref{eq:iterates_MC} satisfy
\begin{align*}
\dist(\bF_{t},\bF_{\star})\le(1-0.6\eta)^{t}0.02\sigma_{r}(\bX_{\star}),\quad\mbox{and}\quad\left\|\bL_{t}\bR_{t}^{\top}-\bX_{\star}\right\|_{\fro}\le(1-0.6\eta)^{t}0.03\sigma_{r}(\bX_{\star}).
\end{align*} 
Here $c_{1},c_{2}>0$ are two universal constants. 
\end{theorem}

Theorem~\ref{thm:MC} establishes that the distance $\dist(\bF_{t},\bF_{\star})$ contracts linearly at a constant rate, as long as the probability of observation satisfies $p \gtrsim (\mu\kappa^{2}\vee \log(n_{1}\vee n_{2}))\mu r^{2}\kappa^{2}/(n_{1}\wedge n_{2})$.
To reach $\epsilon$-accuracy, i.e.~$\|\bL_{t}\bR_{t}^{\top}-\bX_{\star}\|_{\fro}\le\epsilon\sigma_{r}(\bX_{\star})$, \texttt{ScaledPGD} takes at most $T=O(\log(1/\epsilon))$ iterations, which is {\em independent} of $\kappa$. In comparison, projected gradient descent \cite{zheng2016convergence} with spectral initialization converges in $O(\kappa\log(1/\epsilon))$ iterations as long as $p \gtrsim (\mu\vee\log(n_{1}\vee n_{2}))\mu r^{2}\kappa^{2}/(n_{1}\wedge n_{2})$. Therefore, \texttt{ScaledPGD} achieves much faster convergence than its unscaled counterpart, at an expense of higher sample complexity. We believe this higher sample complexity is an artifact of our proof techniques, as numerically we do not observe a degradation in terms of sample complexity.

\subsection{Optimizing general loss functions}

Last but not least, we generalize our analysis of \texttt{ScaledGD} to minimize a general loss function in the form of~\eqref{eq:problem}, where the update rule of \texttt{ScaledGD} is given by
\begin{align}
\begin{split} \bL_{t+1} &=\bL_{t}-\eta\nabla f(\bL_{t}\bR_{t}^{\top})\bR_{t}(\bR_{t}^{\top}\bR_{t})^{-1},\\
 \bR_{t+1} &=\bR_{t}-\eta\nabla f(\bL_{t}\bR_{t}^{\top})^{\top}\bL_{t}(\bL_{t}^{\top}\bL_{t})^{-1}.
\end{split}\label{eq:scaledGD_GL}
\end{align}

Two important properties of the loss function $f:\RR^{n_{1}\times n_{2}}\mapsto\RR$ play a key role in the analysis. 
\begin{definition}[Restricted smoothness]\label{def:restricted_smooth} A differentiable function $f:\RR^{n_{1}\times n_{2}}\mapsto\RR$ is said to be rank-$r$ restricted $L$-smooth for some $L>0$ if 
\begin{align*}
f(\bX_{2})\le f(\bX_{1})+\langle\nabla f(\bX_{1}),\bX_{2}-\bX_{1}\rangle+\frac{L}{2}\|\bX_{2}-\bX_{1}\|_{\fro}^2,
\end{align*}
for any $\bX_{1},\bX_{2}\in\RR^{n_{1}\times n_{2}}$ with rank at most $r$.
\end{definition}
\begin{definition}[Restricted strong convexity]\label{def:restricted_strongly_convex} A differentiable function $f:\RR^{n_{1}\times n_{2}}\mapsto\RR$ is said to be rank-$r$ restricted $\mu$-strongly convex for some $\mu\ge0$ if 
\begin{align*}
f(\bX_{2})\ge f(\bX_{1})+\langle\nabla f(\bX_{1}),\bX_{2}-\bX_{1}\rangle+\frac{\mu}{2}\|\bX_{2}-\bX_{1}\|_{\fro}^2,
\end{align*}
for any $\bX_{1},\bX_{2}\in\RR^{n_{1}\times n_{2}}$ with rank at most $r$. When $\mu=0$, we simply say $f(\cdot)$ is rank-$r$ restricted convex.
\end{definition}
Further, when $\mu>0$, define the condition number of the loss function $f(\cdot)$ over rank-$r$ matrices as
\begin{align} 
\kappa_{f}\coloneqq L/\mu.\label{eq:kappa_f}
\end{align}
Encouragingly, many problems can be viewed as a special case of optimizing this general loss \eqref{eq:scaledGD_GL}, including but not limited to:
\begin{itemize}
\item {\em low-rank matrix factorization}, where the loss function $f(\bX)=\frac{1}{2}\|\bX-\bX_{\star}\|_{\fro}^2$ in \eqref{eq:loss_MF} satisfies $\kappa_{f}=1$;
\item {\em low-rank matrix sensing}, where the loss function $f(\bX)=\frac{1}{2}\|\cA(\bX-\bX_{\star})\|_{2}^2$ in \eqref{eq:loss_MS} satisfies $\kappa_{f}\approx 1$ when $\cA(\cdot)$ obeys the rank-$r$ RIP with a sufficiently small RIP constant;
\item {\em quadratic sampling}, where the loss function $f(\bX)=\frac{1}{2}\sum_{i=1}^{m}|\langle \ba_i\ba_i^{\top}, \bX-\bX_{\star} \rangle|^2$ satisfies restricted strong convexity and smoothness when $\ba_i$'s are i.i.d.~Gaussian vectors for sufficiently large $m$ \cite{sanghavi2017local,li2018nonconvex};
\item {\em exponential-family PCA}, where the loss function $f(\bX)=-\sum_{i,j} \log p(\bY_{i,j}|\bX_{i,j})$, where $p(\bY_{i,j}|\bX_{i,j})$ is the probability density function of $\bY_{i,j}$ conditional on $\bX_{i,j}$, following an exponential-family distribution such as Bernoulli and Poisson distributions. The resulting loss function satisfies restricted strong convexity and smoothness with a condition number $\kappa_{f}>1$ depending on the property of the specific distribution \cite{gunasekar2014exponential,lafond2015low}.
\end{itemize}

Indeed, the treatment of a general loss function brings the condition number of $f(\cdot)$ under the spotlight, since in our earlier case studies $\kappa_{f}\approx1$. Our purpose is thus to understand the interplay of two types of conditioning numbers in the convergence of first-order methods. For simplicity, we assume that $f(\cdot)$ is minimized at the ground truth rank-$r$ matrix $\bX_{\star}$.\footnote{In practice, due to the presence of statistical noise, the minimizer of $f(\cdot)$ might be only approximately low-rank, to which our analysis can be extended in a straightforward fashion.} The following theorem establishes that as long as properly initialized, then \texttt{ScaledGD} converges linearly at a constant rate. 

\begin{theorem}\label{thm:GL} Suppose that $f(\cdot)$ is rank-$2r$ restricted $L$-smooth and $\mu$-strongly convex, of which $\bX_{\star}$ is a minimizer, and that the initialization $\bF_{0}$ satisfies $\dist(\bF_{0},\bF_{\star})\le0.1\sigma_{r}(\bX_{\star})/\sqrt{\kappa_{f}}$. If the step size obeys $0<\eta\le0.4/L$, then for all $t\ge0$, the iterates of \texttt{ScaledGD} in \eqref{eq:scaledGD_GL} satisfy
\begin{align*}
\dist(\bF_{t},\bF_{\star})\le(1-0.7\eta\mu)^{t}0.1\sigma_{r}(\bX_{\star})/\sqrt{\kappa_{f}},\quad\mbox{and}\quad\left\Vert \bL_{t}\bR_{t}^{\top}-\bX_{\star}\right\Vert _{\fro}\le(1-0.7\eta\mu)^{t}0.15\sigma_{r}(\bX_{\star})/\sqrt{\kappa_{f}}.
\end{align*} 
\end{theorem}

Theorem~\ref{thm:GL} establishes that the distance $\dist(\bF_{t},\bF_{\star})$ contracts linearly at a constant rate, as long as the initialization $\bF_{0}$ is sufficiently close to $\bF_{\star}$.
To reach $\epsilon$-accuracy, i.e.~$\|\bL_{t}\bR_{t}^{\top}-\bX_{\star}\|_{\fro}\le\epsilon\sigma_{r}(\bX_{\star})$, \texttt{ScaledGD} takes at most $T=O(\kappa_{f}\log(1/\epsilon))$ iterations, which depends only on the condition number $\kappa_{f}$ of $f(\cdot)$, but is independent of the condition number $\kappa$ of the matrix $\bX_{\star}$. In contrast, prior theory of vanilla gradient descent \cite{park2018finding,bhojanapalli2016dropping} requires $O(\kappa_{f}\kappa\log(1/\epsilon))$ iterations, which is worse than our rate by a factor of $\kappa$.

%% file: proof-outline.tex
\section{Proof Sketch}\label{sec:proof-outline}

In this section, we sketch the proof of the main theorems, highlighting the role of the scaled distance metric (cf.~\eqref{eq:dist}) in these analyses.

\subsection{A warm-up analysis: matrix factorization}

Let us consider the problem of factorizing a matrix $\bX_{\star}$ into two low-rank factors: 
\begin{align}
\minimize_{\bF\in\RR^{(n_{1}+n_{2})\times r}}\;\cL(\bF)=\frac{1}{2}\left\Vert \bL\bR^{\top}-\bX_{\star}\right\Vert _{\fro}^{2}.\label{eq:loss_MF}
\end{align}
For this toy problem, the update rule of \texttt{ScaledGD} is given as 
\begin{align}
\begin{split} \bL_{t+1} & =\bL_{t}-\eta(\bL_{t}\bR_{t}^{\top}-\bX_{\star})\bR_{t}(\bR_{t}^{\top}\bR_{t})^{-1},\\
 \bR_{t+1} & =\bR_{t}-\eta(\bL_{t}\bR_{t}^{\top}-\bX_{\star})^{\top}\bL_{t}(\bL_{t}^{\top}\bL_{t})^{-1}.
\end{split}\label{eq:scaledGD_MF}
\end{align}

To shed light on why \texttt{ScaledGD} is robust to ill-conditioning, it is worthwhile to think of \texttt{ScaledGD} as a quasi-Newton algorithm: the following proposition (proven in Appendix~\ref{subsec:quasi_hessian}) reveals that \texttt{ScaledGD} is equivalent to approximating the Hessian of the loss function in \eqref{eq:loss_MF} by only keeping its diagonal blocks.  

\begin{proposition}\label{prop:vectorize} For the matrix factorization problem \eqref{eq:loss_MF}, \texttt{ScaledGD} is equivalent to the following update rule 
\begin{align*}
\vc(\bF_{t+1})=\vc(\bF_{t})-\eta\begin{bmatrix}\nabla_{\bL,\bL}^{2}\cL(\bF_{t})& \zero\\
\zero & \nabla_{\bR,\bR}^{2}\cL(\bF_{t})
\end{bmatrix}^{-1}\vc(\nabla_{\bF}\cL(\bF_{t})).
\end{align*}
Here, $\nabla_{\bL,\bL}^{2}\cL(\bF_{t})$ (resp.~$\nabla_{\bR,\bR}^{2}\cL(\bF_{t})$) denotes the second order derivative w.r.t.~$\bL$ (resp.~$\bR$) at $\bF_{t}$. 
\end{proposition}
 
%\item {\em \texttt{ScaledGD} as a convex combination of GD and least-squares:} Perhaps more interestingly, it is also worth noting that \eqref{eq:scaledGD_MF} can be rewritten as 
%\begin{align}
%\begin{bmatrix} 
%\bL_{t+1} \\
%\bR_{t+1}
%\end{bmatrix} & =(1-\eta)\begin{bmatrix}
%\bL_{t} \\
%\bR_{t}
%\end{bmatrix} + \eta\begin{bmatrix}
%\bX_{\star}\bR_{t}(\bR_{t}^{\top}\bR_{t})^{-1}\\
%\bX_{\star}^{\top}\bL_{t}(\bL_{t}^{\top}\bL_{t})^{-1}
%\end{bmatrix},\label{eq:scaledGD_MF_reformulation}
%\end{align}
%where the second term is the least-squares update of the factors when fixing the other:
%\begin{align}  
%\bX_{\star}\bR_{t}(\bR_{t}^{\top}\bR_{t})^{-1}=\argmin_{\bL}\;\cL(\bL,\bR_t),\quad\mbox{and}\quad \bX_{\star}^{\top}\bL_{t}(\bL_{t}^{\top}\bL_{t})^{-1}=\argmin_{\bR}\;\cL(\bL_t,\bR).\label{eq:alt_min}
%\end{align}
%Therefore, \eqref{eq:scaledGD_MF_reformulation} shows that with $\eta\in[0,1]$, the next iterate of \texttt{ScaledGD} can be interpreted as a convex combination of the current iterate and the least-squares update \eqref{eq:alt_min}, where the latter is robust to ill-conditioning.  

The following theorem, whose proof can be found in Appendix~\ref{subsec:proof_MF}, formally establishes that as long as \texttt{ScaledGD} is initialized close
to the ground truth, $\dist(\bF_{t},\bF_{\star})$ will contract at a constant linear rate for the matrix factorization problem.

\begin{theorem} \label{thm:MF} Suppose that the initialization $\bF_{0}$ satisfies $\dist(\bF_{0},\bF_{\star})\le0.1\sigma_{r}(\bX_{\star})$. If the step size obeys $0<\eta\le2/3$, then for all $t\ge0$, the iterates of the \texttt{ScaledGD} method in \eqref{eq:scaledGD_MF} satisfy 
\begin{align*}
\dist(\bF_{t},\bF_{\star})\le(1-0.7\eta)^{t}0.1\sigma_{r}(\bX_{\star}),\quad\mbox{and}\quad\left\Vert \bL_{t}\bR_{t}^{\top}-\bX_{\star}\right\Vert _{\fro}\le(1-0.7\eta)^{t}0.15\sigma_{r}(\bX_{\star}).
\end{align*}
\end{theorem}

Comparing to the rate of contraction $(1-1/\kappa)$ of gradient descent for matrix factorization \cite{ma2021beyond,chi2019nonconvex}, Theorem~\ref{thm:MF} demonstrates that the preconditioners indeed allow better search directions in the local neighborhood of the ground truth, and hence a faster convergence rate.

\subsection{Proof outline for matrix sensing}\label{subsec:proof_outline_MS}

It can be seen that the update rule \eqref{eq:iterates_MS} of \texttt{ScaledGD} in Algorithm~\ref{alg:MS} closely mimics \eqref{eq:scaledGD_MF} when $\cA(\cdot)$ satisfies the RIP. Therefore, leveraging the RIP of $\cA(\cdot)$ and Theorem~\ref{thm:MF}, we can establish the following local convergence guarantee of Algorithm~\ref{alg:MS}, which has a weaker requirement on $\delta_{2r}$ than the main theorem (cf.~Theorem~\ref{thm:MS}).

\begin{lemma}\label{lemma:contraction_MS} Suppose that $\cA(\cdot)$ obeys the $2r$-RIP with $\delta_{2r}\le0.02$. If the $t$-th iterate satisfies $\dist(\bF_{t},\bF_{\star})\le0.1\sigma_{r}(\bX_{\star})$, then $\|\bL_{t}\bR_{t}^{\top}-\bX_{\star}\|_{\fro}\le1.5\dist(\bF_{t},\bF_{\star})$. In addition, if the step size obeys $0<\eta\le2/3$, then the $(t+1)$-th iterate $\bF_{t+1}$ of the \texttt{ScaledGD} method in \eqref{eq:iterates_MS} of Algorithm~\ref{alg:MS} satisfies
\begin{align*}
\dist(\bF_{t+1},\bF_{\star})\le(1-0.6\eta)\dist(\bF_{t},\bF_{\star}).
\end{align*}
\end{lemma}

It then boils to down to finding a good initialization, for which we have the following lemma on the quality of the spectral initialization.

\begin{lemma}\label{lemma:init_MS} Suppose that $\cA(\cdot)$ obeys the $2r$-RIP with a constant $\delta_{2r}$. Then the spectral initialization in \eqref{eq:init_MS} for low-rank matrix sensing satisfies
\begin{align*}
\dist(\bF_{0},\bF_{\star})\le 5\delta_{2r}\sqrt{r}\kappa\sigma_{r}(\bX_{\star}).
\end{align*}
\end{lemma}
Therefore, as long as $\delta_{2r}$ is small enough, say $\delta_{2r}\le0.02/(\sqrt{r}\kappa)$ as specified in Theorem~\ref{thm:MS}, the initial distance satisfies $\dist(\bF_{0},\bF_{\star})\le0.1\sigma_{r}(\bX_{\star})$, allowing us to invoke Lemma~\ref{lemma:contraction_MS} recursively. The proof of Theorem~\ref{thm:MS} is then complete.
The proofs of Lemmas~\ref{lemma:contraction_MS}-\ref{lemma:init_MS} can be found in Appendix~\ref{sec:proof_MS}.

\subsection{Proof outline for robust PCA}\label{subsec:proof_outline_RPCA}

As before, we begin with the following local convergence guarantee of Algorithm~\ref{alg:RPCA}, which has a weaker requirement on $\alpha$ than the main theorem (cf.~Theorem~\ref{thm:RPCA}). The difference with low-rank matrix sensing is that local convergence for robust PCA requires a further incoherence condition on the iterates (cf.~\eqref{eq:incoherence_cond_RPCA}), where we recall from \eqref{eq:Q_def} that $\bQ_{t}$ is the optimal alignment matrix between $\bF_{t}$ and $\bF_{\star}$. 

\begin{lemma}\label{lemma:contraction_RPCA} Suppose that $\bX_{\star}$ is $\mu$-incoherent and $\alpha\le10^{-4}/(\mu r)$. If the $t$-th iterate satisfies $\dist(\bF_{t},\bF_{\star})\le 0.02\sigma_{r}(\bX_{\star})$ and the incoherence condition
\begin{align}
\sqrt{n_{1}}\left\Vert (\bL_{t}\bQ_{t}-\bL_{\star})\bSigma_{\star}^{1/2}\right\Vert _{2,\infty}\vee\sqrt{n_{2}}\left\Vert (\bR_{t}\bQ_{t}^{-\top}-\bR_{\star})\bSigma_{\star}^{1/2}\right\Vert _{2,\infty}\le\sqrt{\mu r}\sigma_{r}(\bX_{\star}),\label{eq:incoherence_cond_RPCA}
\end{align}
then $\|\bL_{t}\bR_{t}^{\top}-\bX_{\star}\|_{\fro}\le1.5\dist(\bF_{t},\bF_{\star})$. In addition, if the step size obeys $0.1\le\eta\le2/3$, then the $(t+1)$-th iterate $\bF_{t+1}$ of the \texttt{ScaledGD} method in \eqref{eq:iterates_RPCA} of Algorithm~\ref{alg:RPCA} satisfies
\begin{align*}
\dist(\bF_{t+1},\bF_{\star})\le(1-0.6\eta)\dist(\bF_{t},\bF_{\star}),
\end{align*}
and the incoherence condition 
\begin{align*}
\sqrt{n_{1}}\left\Vert (\bL_{t+1}\bQ_{t+1}-\bL_{\star})\bSigma_{\star}^{1/2}\right\Vert _{2,\infty}\vee\sqrt{n_{2}}\left\Vert (\bR_{t+1}\bQ_{t+1}^{-\top}-\bR_{\star})\bSigma_{\star}^{1/2}\right\Vert _{2,\infty}\le\sqrt{\mu r}\sigma_{r}(\bX_{\star}).
\end{align*}
\end{lemma} 

As long as the initialization is close to the ground truth and satisfies the incoherence condition, Lemma~\ref{lemma:contraction_RPCA} ensures that the iterates of \texttt{ScaledGD} remain incoherent and converge linearly. This allows us to remove the unnecessary projection step in~\cite{yi2016fast}, whose main objective is to ensure the incoherence of the iterates. 

We are left with checking the initial conditions. The following lemma ensures that the spectral initialization in \eqref{eq:init_RPCA} is close to the ground truth as long as $\alpha$ is sufficiently small.
\begin{lemma}\label{lemma:init_RPCA} Suppose that $\bX_{\star}$ is $\mu$-incoherent. Then the spectral initialization~\eqref{eq:init_RPCA} for robust PCA satisfies 
\begin{align*}
\dist(\bF_{0},\bF_{\star}) \le 20\alpha\mu r^{3/2}\kappa\sigma_{r}(\bX_{\star}).
\end{align*}
\end{lemma}

As a result, setting $\alpha\le 10^{-3}/(\mu r^{3/2}\kappa)$, the spectral initialization satisfies $\dist(\bF_{0},\bF_{\star})\le 0.02\sigma_{r}(\bX_{\star})$. 
In addition, we need to make sure that the spectral initialization satisfies the incoherence condition, which is provided in the following lemma.
\begin{lemma}\label{lemma:init_2inf_RPCA} Suppose that $\bX_{\star}$ is $\mu$-incoherent and $\alpha\le0.1/(\mu r\kappa)$, and that $\dist(\bF_{0},\bF_{\star})\le0.02\sigma_{r}(\bX_{\star})$. Then the spectral initialization~\eqref{eq:init_RPCA} satisfies the incoherence condition
\begin{align*}
\sqrt{n_{1}}\left\Vert (\bL_{0}\bQ_{0}-\bL_{\star})\bSigma_{\star}^{1/2}\right\Vert _{2,\infty}\vee\sqrt{n_{2}}\left\Vert (\bR_{0}\bQ_{0}^{-\top}-\bR_{\star})\bSigma_{\star}^{1/2}\right\Vert _{2,\infty}\le\sqrt{\mu r}\sigma_{r}(\bX_{\star}).
\end{align*} 
\end{lemma}

Combining Lemmas~\ref{lemma:contraction_RPCA}-\ref{lemma:init_2inf_RPCA} finishes the proof of Theorem~\ref{thm:RPCA}. The proofs of the the three supporting lemmas can be found in Section~\ref{sec:proof_RPCA}.

\subsection{Proof outline for matrix completion} \label{subsec:proof_outline_MC}
 
\paragraph{A key property of the new projection operator.}

We start with the following lemma that entails a key property of the scaled projection \eqref{eq:scaled_proj}, which ensures the scaled projection satisfies both non-expansiveness and incoherence under the scaled metric.

\begin{lemma}\label{lemma:scaled_proj} Suppose that $\bX_{\star}$ is $\mu$-incoherent, and $\dist(\tilde{\bF},\bF_{\star})\le \epsilon\sigma_{r}(\bX_{\star})$ for some $\epsilon<1$. Set $B\ge(1+\epsilon)\sqrt{\mu r}\sigma_{1}(\bX_{\star})$, then $\cP_{B}(\tilde{\bF})$ satisfies the non-expansiveness
\begin{align*}
\dist(\cP_{B}(\tilde{\bF}),\bF_{\star})\le\dist(\tilde{\bF},\bF_{\star}),
\end{align*}
and the incoherence condition
\begin{align*}
\sqrt{n_{1}}\|\bL\bR^{\top}\|_{2,\infty}\vee \sqrt{n_{2}}\|\bR\bL^{\top}\|_{2,\infty}\le B.
\end{align*}
\end{lemma}
 
It is worth noting that the incoherence condition adopts a slightly different form than that of robust PCA, which is more convenient for matrix completion. The next lemma guarantees the fast local convergence of Algorithm~\ref{alg:MC} as long as the sample complexity is large enough and the parameter $B$ is set properly. 
 
\begin{lemma}\label{lemma:contraction_MC} Suppose that $\bX_{\star}$ is $\mu$-incoherent, and $p\ge C(\mu r\kappa^{4}\vee\log(n_{1}\vee n_{2}))\mu r/(n_{1}\wedge n_{2})$ for some sufficiently large constant $C$. Set the projection radius as $B=C_{B}\sqrt{\mu r}\sigma_{1}(\bX_{\star})$ for some constant $C_{B}\ge1.02$. Under an event $\cE$ which happens with overwhelming probability (i.e.~at least $1-c_{1}(n_{1}\vee n_{2})^{-c_2}$), if the $t$-th iterate satisfies $\dist(\bF_{t},\bF_{\star})\le0.02\sigma_{r}(\bX_{\star})$, and the incoherence condition
\begin{align*}
\sqrt{n_{1}}\|\bL_{t}\bR_{t}^{\top}\|_{2,\infty}\vee\sqrt{n_{1}}\|\bR_{t}\bL_{t}^{\top}\|_{2,\infty} \le B,
\end{align*}
then $\|\bL_{t}\bR_{t}^{\top}-\bX_{\star}\|_{\fro}\le1.5\dist(\bF_{t},\bF_{\star})$. In addition, if the step size obeys $0<\eta\le2/3$, then the $(t+1)$-th iterate $\bF_{t+1}$ of the \texttt{ScaledPGD} method in \eqref{eq:iterates_MC} of Algorithm~\ref{alg:MC} satisfies
\begin{align*}
\dist(\bF_{t+1},\bF_{\star})\le(1-0.6\eta)\dist(\bF_{t},\bF_{\star}),
\end{align*}
and the incoherence condition
\begin{align*}
\sqrt{n_{1}}\|\bL_{t+1}\bR_{t+1}^{\top}\|_{2,\infty}\vee\sqrt{n_{2}}\|\bR_{t+1}\bL_{t+1}^{\top}\|_{2,\infty} \le B.
\end{align*}
\end{lemma}
As long as we can find an initialization that is close to the ground truth and satisfies the incoherence condition, Lemma~\ref{lemma:contraction_MC} ensures that the iterates of \texttt{ScaledPGD} remain incoherent and converge linearly. The follow lemma ensures that such an initialization can be ensured via the spectral method. 

\begin{lemma}\label{lemma:init_MC} Suppose that $\bX_{\star}$ is $\mu$-incoherent, then with overwhelming probability, the spectral initialization before projection $\tilde{\bF}_{0}\coloneqq\begin{bmatrix}\bU_{0}\bSigma_{0}^{1/2} \\ \bV_{0}\bSigma_{0}^{1/2}\end{bmatrix}$ in \eqref{eq:init_MC} satisfies 
\begin{align*}
\dist(\tilde{\bF}_{0},\bF_{\star})\le C_{0}\left(\frac{\mu r\log(n_{1}\vee n_{2})}{p\sqrt{n_{1}n_{2}}} + \sqrt{\frac{\mu r\log(n_{1}\vee n_{2})}{p(n_{1}\wedge n_{2})}}\right)5\sqrt{r}\kappa\sigma_{r}(\bX_{\star}).
\end{align*}
\end{lemma}
Therefore, as long as $p\ge C\mu r^{2}\kappa^{2}\log(n_{1}\vee n_{2})/(n_{1}\wedge n_{2})$ for some sufficiently large constant $C$, the initial distance satisfies $\dist(\tilde{\bF}_{0},\bF_{\star})\le0.02\sigma_{r}(\bX_{\star})$. One can then invoke Lemma~\ref{lemma:scaled_proj} to see that $\bF_{0}=\cP_{B}(\tilde{\bF}_{0})$ meets the requirements of Lemma~\ref{lemma:contraction_MC} due to the non-expansiveness and incoherence properties of the projection operator. The proofs of the the the supporting lemmas can be found in Section~\ref{sec:proof_MC}.

%% file: numerical.tex
\section{Numerical Experiments}\label{sec:numerical}

In this section, we provide numerical experiments to corroborate our theoretical findings, with the codes available at 
\begin{center}
\url{https://github.com/Titan-Tong/ScaledGD}.
\end{center} 
The simulations are performed in Matlab with a 3.6 GHz Intel Xeon Gold 6244 CPU.

\subsection{Comparison with vanilla GD}
To begin, we compare the iteration complexity of \texttt{ScaledGD} with vanilla gradient descent (GD). The update rule of vanilla GD for solving \eqref{eq:problem} is given as
\begin{align}
\begin{split} \bL_{t+1} & =\bL_{t} - \eta_{\texttt{GD}} \nabla_{\bL}\cL(\bL_{t},\bR_{t}),\\
\bR_{t+1} & =\bR_{t} - \eta_{\texttt{GD}} \nabla_{\bR}\cL(\bL_{t},\bR_{t}),
\end{split}\label{eq:vanillaGD}
\end{align}
where $\eta_{\texttt{GD}}=\eta/\sigma_{1}(\bX_{\star})$ stands for the step size for gradient descent. This choice is often recommended by the theory of vanilla GD \cite{tu2015low,yi2016fast,ma2017implicit} and the scaling by $\sigma_{1}(\bX_{\star})$ is needed for its convergence. For ease of comparison, we fix $\eta=0.5$ for both \texttt{ScaledGD} and vanilla GD (see Figure~\ref{fig:scaledGD_stepsizes} for justifications). Both algorithms start from the same spectral initialization. To avoid notational clutter, we work on square {\it asymmetric} matrices with $n_{1}=n_{2}=n$. We consider four low-rank matrix estimation tasks:
\begin{figure}[!ht]
\centering
\begin{tabular}{cc}
 \includegraphics[width=0.45\textwidth]{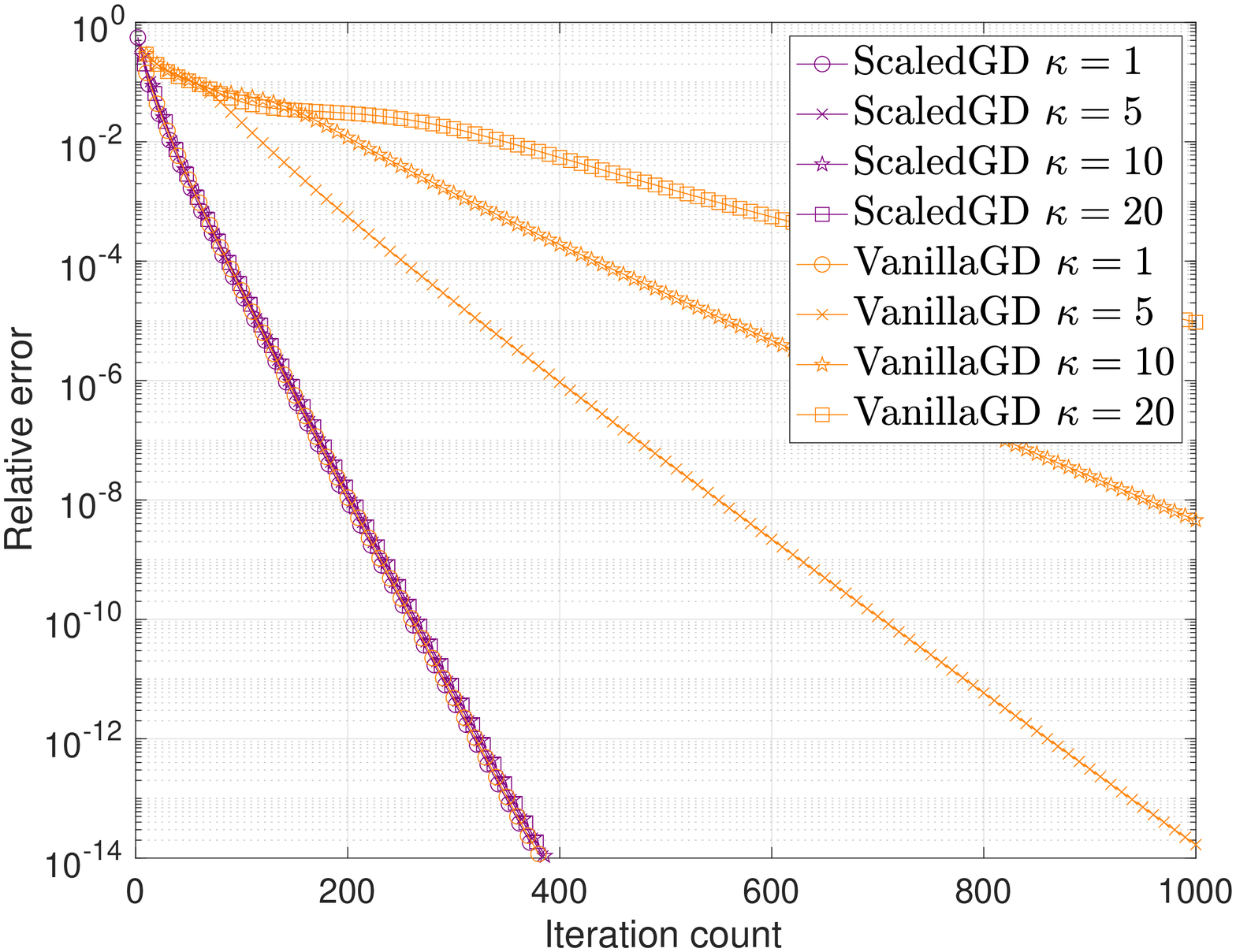} & 
 \includegraphics[width=0.45\textwidth]{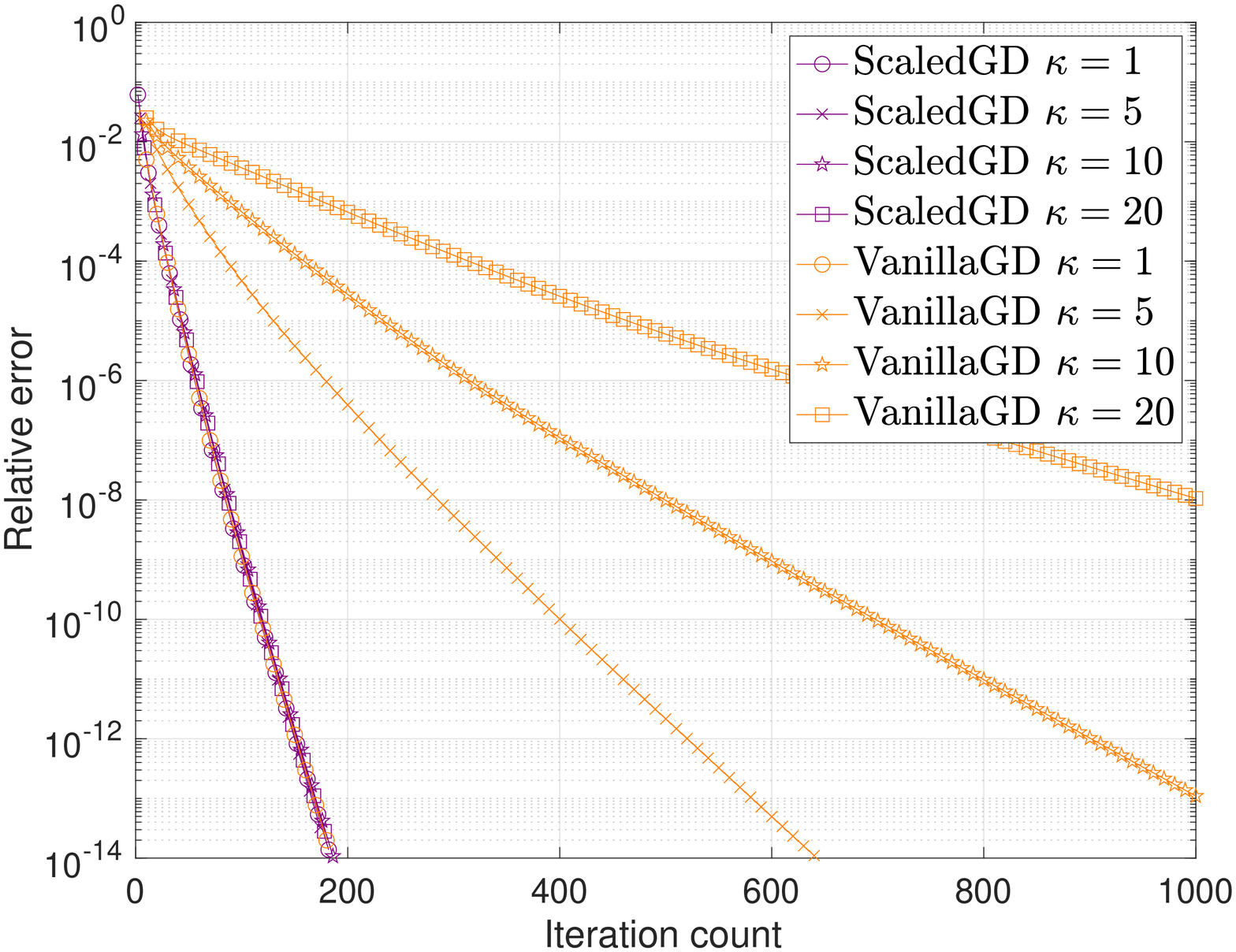} \\
	(a) Matrix sensing & (b) Robust PCA \\
	 $n=200, r=10, m=5nr$ &  $n=1000, r=10, \alpha=0.1$ \\
 \includegraphics[width=0.45\textwidth]{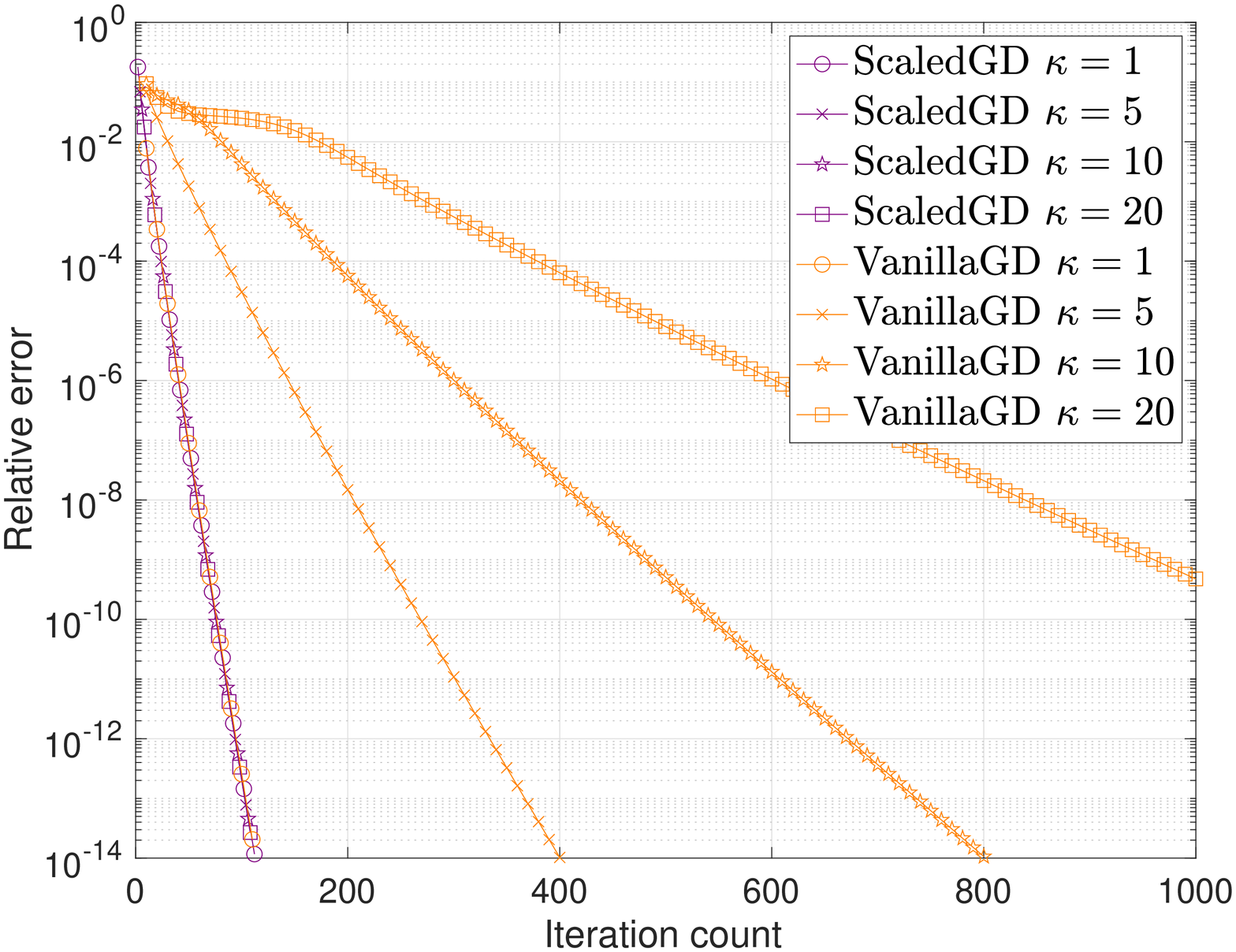}  &
 \includegraphics[width=0.45\textwidth]{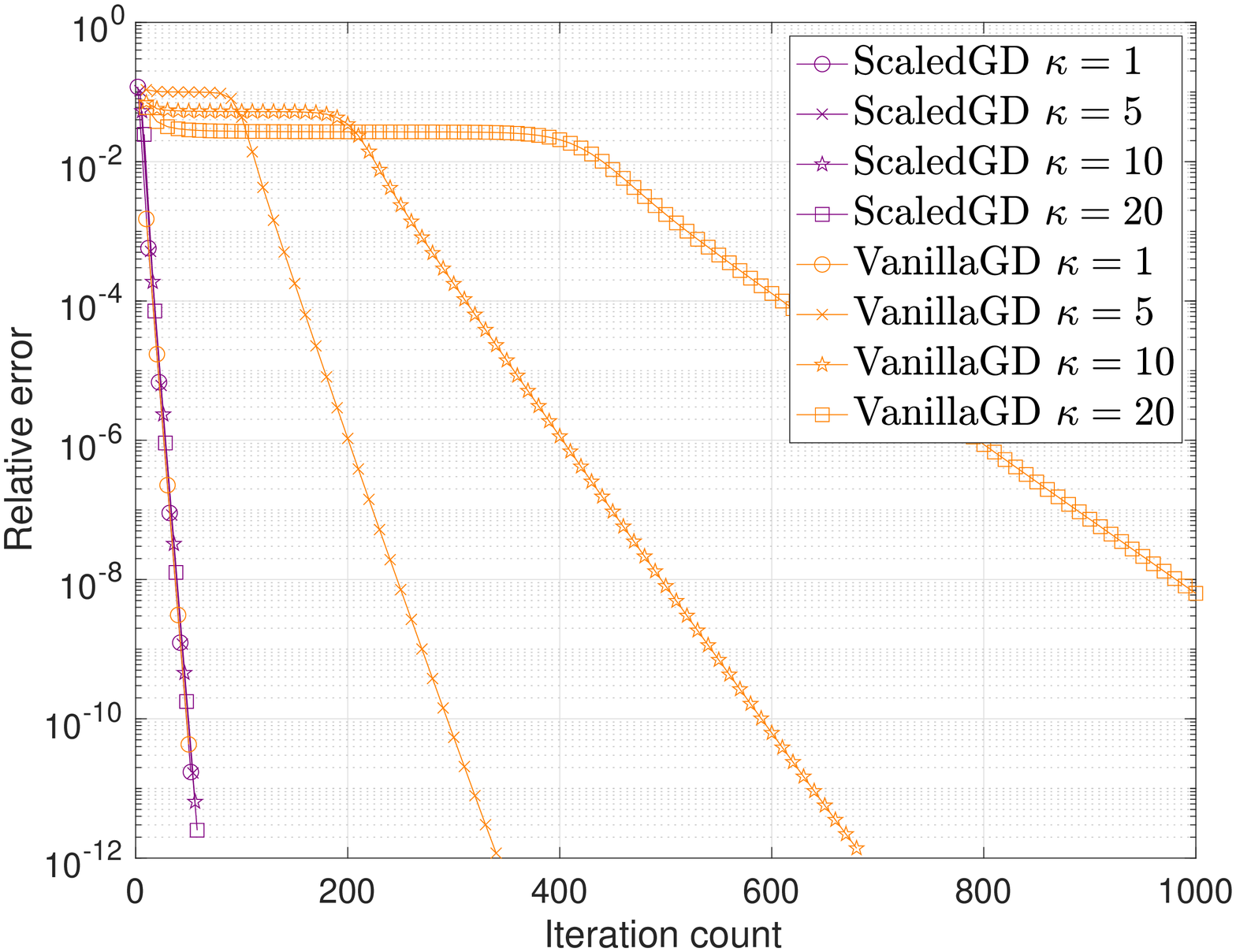} \\
	 (c) Matrix completion  & (d) Hankel matrix completion\\
	  $n=1000, r=10, p=0.2$ & $n=1000, r=10, p=0.2$
\end{tabular}
\caption{The relative errors of \texttt{ScaledGD} and vanilla GD with respect to the iteration count under different condition numbers $\kappa=1,5,10,20$ for (a) matrix sensing, (b) robust PCA, (c) matrix completion, and (d) Hankel matrix completion.}\label{fig:scaledGD_all}
\end{figure} 
\begin{figure}[!ht]
\centering
\begin{tabular}{cc}
 \includegraphics[width=0.45\textwidth]{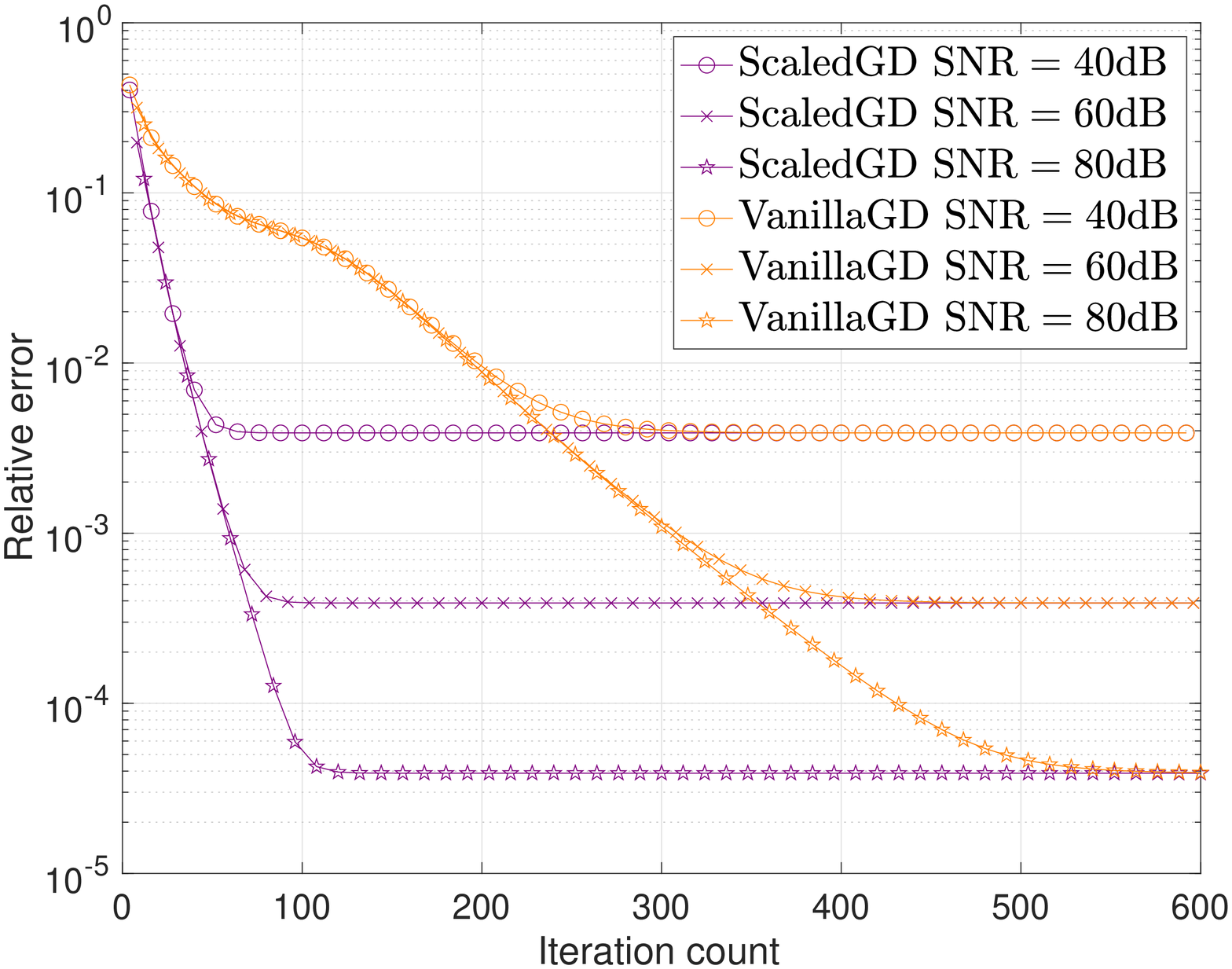} & 
 \includegraphics[width=0.45\textwidth]{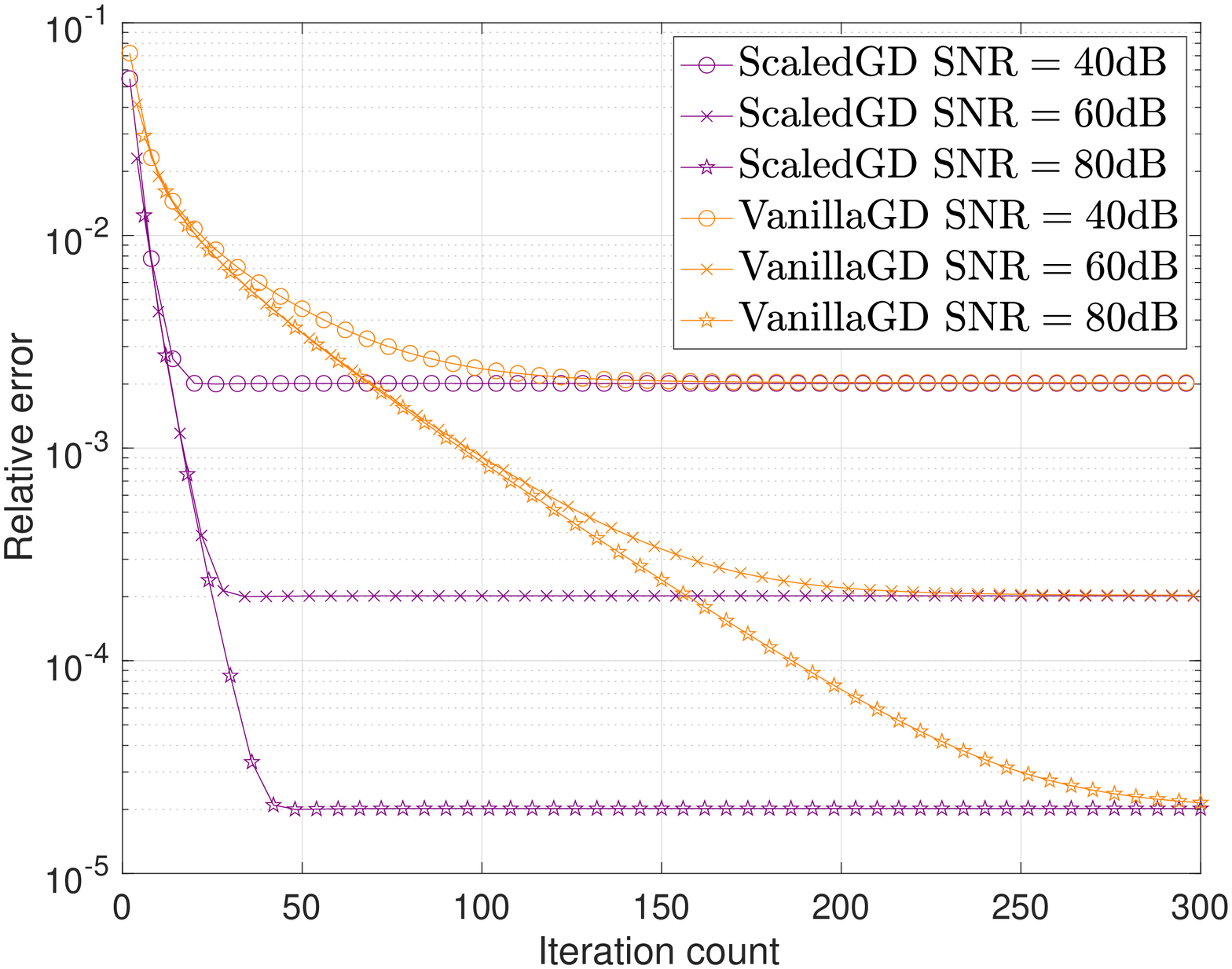} \\
	(a) Matrix Sensing & (b) Robust PCA \\
	 $n=200, r=10, m=5nr$ &  $n=1000, r=10, \alpha=0.1$ \\
 \includegraphics[width=0.45\textwidth]{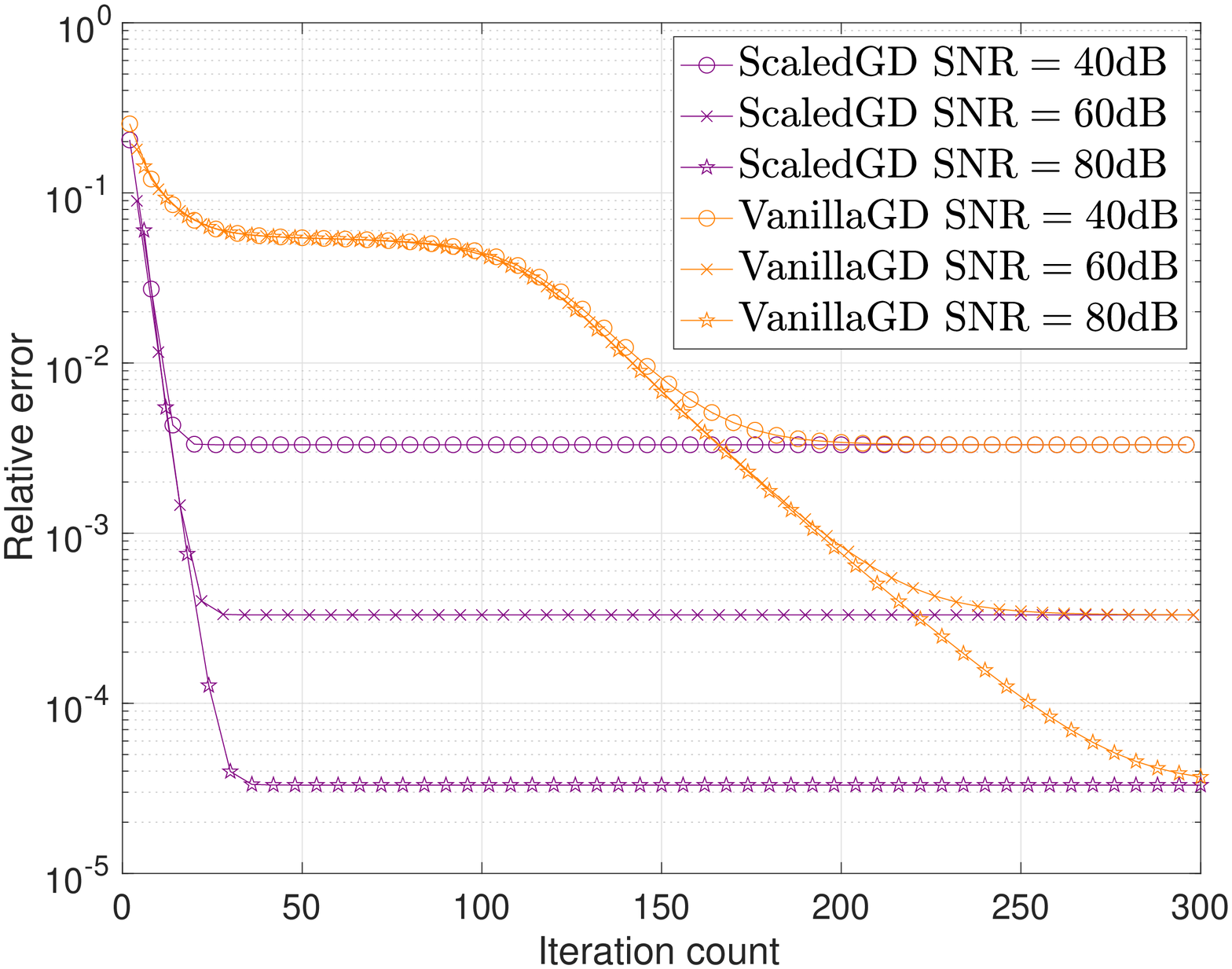}  &
 \includegraphics[width=0.45\textwidth]{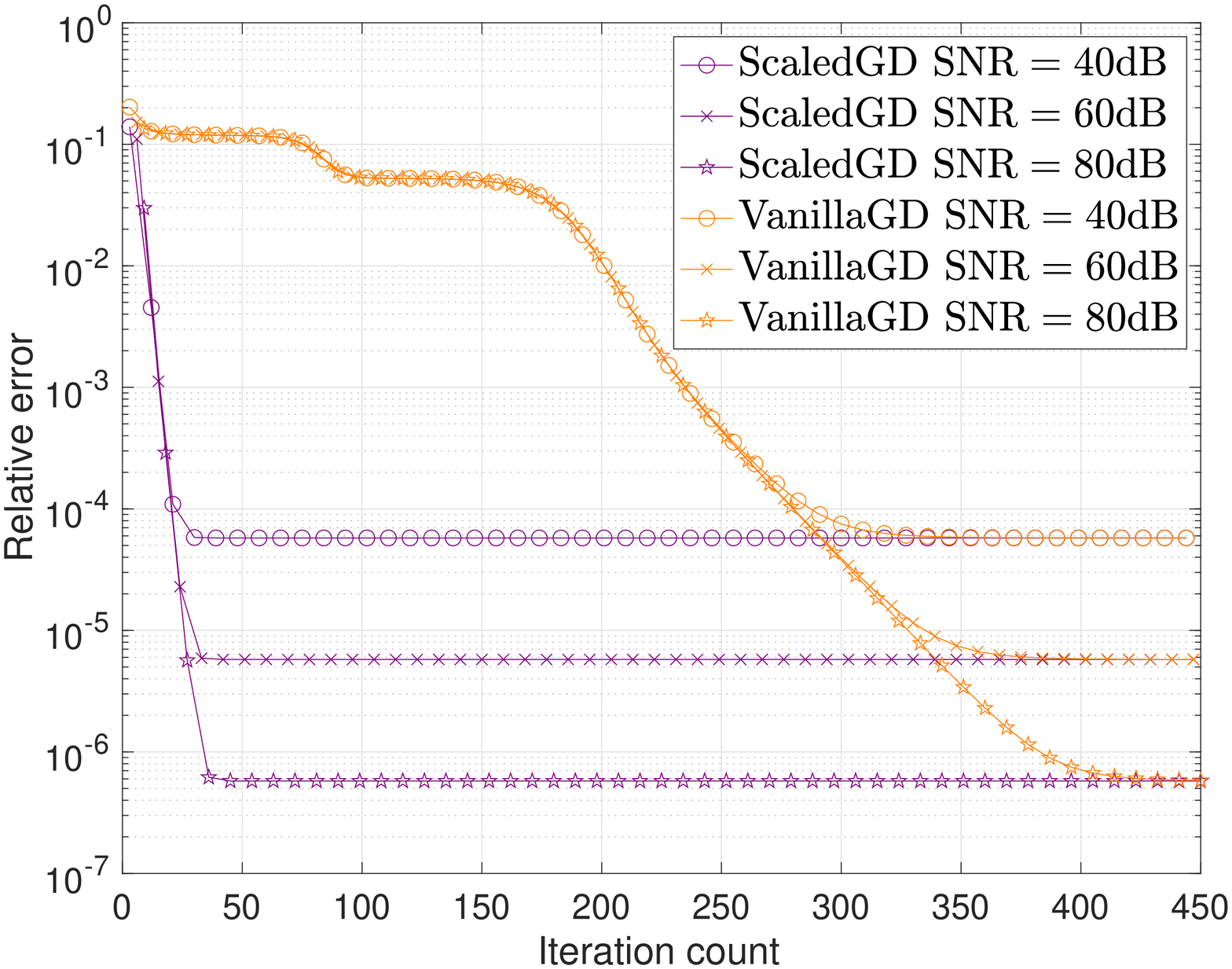} \\
	 (c) Matrix completion  & (d) Hankel matrix completion\\
	  $n=1000, r=10, p=0.2$ & $n=1000, r=10, p=0.2$
\end{tabular}
\caption{The relative errors of \texttt{ScaledGD} and vanilla GD with respect to the iteration count under the condition number $\kappa=10$ and signal-to-noise ratios $\mathrm{SNR}=40,60,80\mathrm{dB}$ for (a) matrix sensing, (b) robust PCA, (c) matrix completion, and (d) Hankel matrix completion.}\label{fig:scaledGD_noise}
\end{figure}
\begin{itemize}
\item {\em Low-rank matrix sensing.} The problem formulation is detailed in Section~\ref{subsec:scaledGD_MS}. Here, we collect $m=5nr$ measurements in the form of $\by_{k} = \langle\bA_{k}, \bX_{\star}\rangle+\bw_{k}$, in which the measurement matrices $\bA_{k}$ are generated with i.i.d.~Gaussian entries with zero mean and variance $1/m$, and $\bw_{k}\sim\cN(0,\sigma_{w}^{2})$ are i.i.d.~Gaussian noises.
\item {\em Robust PCA.} The problem formulation is stated in Section~\ref{subsec:scaledGD_RPCA}. We generate the corruption with a sparse matrix $\bS_{\star} \in \cS_{\alpha}$ with $\alpha = 0.1$. More specifically, we generate a matrix with standard Gaussian entries and pass it through $\cT_{\alpha}[\cdot]$ to obtain $\bS_{\star}$. The observation is $\bY=\bX_{\star}+\bS_{\star}+\bW$, where $\bW_{i,j}\sim\cN(0,\sigma_{w}^{2})$ are i.i.d.~Gaussian noises.
\item {\em Matrix completion.} The problem formulation is stated in Section~\ref{subsec:scaledGD_MC}. We assume random Bernoulli observations, where each entry of $\bX_{\star}$ is observed with probability $p=0.2$ independently. The observation is $\bY=\cP_{\Omega}(\bX_{\star}+\bW)$, where $\bW_{i,j}\sim\cN(0,\sigma_{w}^{2})$ are i.i.d.~Gaussian noises. Moreover, we perform the scaled gradient updates without projections. 
\item {\em Hankel matrix completion.} Briefly speaking, a Hankel matrix shares the same value along each skew-diagonal, and we aim at recovering a low-rank Hankel matrix from observing a few skew-diagonals \cite{chen2014robust,cai2018spectral}. We assume random Bernoulli observations, where each skew-diagonal of $\bX_{\star}$ is observed with probability $p=0.2$ independently. 
The loss function is
\begin{align}
\cL(\bL,\bR) = \frac{1}{2p}\left\Vert\cH_{\Omega}(\bL\bR^{\top}-\bY)\right\Vert_{\fro}^2 + \frac{1}{2}\left\Vert(\cI-\cH)(\bL\bR^{\top})\right\Vert_{\fro}^2,\label{eq:loss_HankelMC}
\end{align}
where $\cI(\cdot)$ denotes the identity operator, and the Hankel projection is defined as $\cH(\bX)\coloneqq\sum_{k=1}^{2n-1} \langle\bH_{k}, \bX\rangle \bH_{k}$, which maps $\bX$ to its closest Hankel matrix. Here, the Hankel basis matrix $\bH_{k}$ is the $n\times n$ matrix with the entries in the $k$-th skew diagonal as $\frac{1}{\sqrt{\omega_{k}}}$, and all other entries as $0$, where $\omega_{k}$ is the length of the $k$-th skew diagonal. Note that $\bX$ is a Hankel matrix if and only if $(\cI-\cH)(\bX)=\zero$. The Hankel projection on the observation index set $\Omega$ is defined as $\cH_{\Omega}(\bX)\coloneqq\sum_{k\in\Omega}\langle\bH_{k}, \bX\rangle\bH_{k}$. The observation is $\bY=\cH_{\Omega}(\bX_{\star}+\bW)$, where $\bW$ is a Hankel matrix whose entries along each skew-diagonal are i.i.d.~Gaussian noises $\cN(0,\sigma_{w}^{2})$.
\end{itemize}

For the first three problems, we generate the ground truth matrix $\bX_{\star}\in\RR^{n\times n}$ in the following way. We first generate an $n\times r$ matrix with i.i.d.~random signs, and take its $r$ left singular vectors as $\bU_{\star}$, and similarly for $\bV_{\star}$. The singular values are set to be linearly distributed from $1$ to $1/\kappa$. The ground truth is then defined as $\bX_{\star}=\bU_{\star}\bSigma_{\star}\bV_{\star}^{\top}$ which has the specified condition number $\kappa$ and rank $r$. For Hankel matrix completion, we generate $\bX_{\star}$ as an $n\times n$ Hankel matrix with entries given as
\begin{align*}
(\bX_{\star})_{i,j} = \sum_{\ell=1}^{r} \frac{\sigma_{\ell}}{n} e^{2\pi\imath (i+j-2)f_{\ell}}, \quad i,j=1,\dots,n,
\end{align*}
where $f_{\ell}$, $\ell=1,\dots,r$ are randomly chosen from $1/n, 2/n, \dots, 1$, and $\sigma_{\ell}$ are linearly distributed from $1$ to $1/\kappa$. The Vandermonde decomposition lemma tells that $\bX_{\star}$ has rank $r$ and singular values $\sigma_{\ell}$, $\ell=1,\dots,r$.

We first illustrate the convergence performance under noise-free observations, i.e.~$\sigma_{w}=0$. We plot the relative reconstruction error $\|\bX_{t}-\bX_{\star}\|_{\fro}/\|\bX_{\star}\|_{\fro}$ with respect to the iteration count $t$ in Figure~\ref{fig:scaledGD_all} for the four problems under different condition numbers $\kappa = 1,5,10,20$.
For all these models, we can see that \texttt{ScaledGD} has a convergence rate independent of $\kappa$, with all curves almost overlay on each other. Under good conditioning $\kappa=1$, \texttt{ScaledGD} converges at the same rate as vanilla GD; under ill conditioning, i.e.~when $\kappa$ is large, \texttt{ScaledGD} converges much faster than vanilla GD and leads to significant computational savings.

We next move to demonstrate that \texttt{ScaledGD} is robust to small additive noises. Denote the signal-to-noise ratio as 
$\mathrm{SNR}\coloneqq 10\log_{10}\frac{\|\bX_{\star}\|_{\fro}^2}{n^2\sigma_{w}^2}$ in dB.
We plot the reconstruction error $\|\bX_{t}-\bX_{\star}\|_{\fro}/\|\bX_{\star}\|_{\fro}$ with respect to the iteration count $t$ in Figure~\ref{fig:scaledGD_noise} under the condition number $\kappa=10$ and various $\mathrm{SNR}=40,60,80\mathrm{dB}$. We can see that \texttt{ScaledGD} and vanilla GD achieve the same statistical error eventually, but \texttt{ScaledGD} converges much faster. In addition, the convergence speeds are not influenced by the noise levels.

Careful readers might wonder how sensitivity our comparisons are with respect to the choice of step sizes. To address this, we illustrate the convergence speeds of both \texttt{ScaledGD} and vanilla GD under different step sizes $\eta$ for matrix completion (under the same setting as Figure~\ref{fig:scaledGD_all}~(c)), where similar plots can be obtained for other problems as well. We run both algorithms for at most $80$ iterations, and terminate if the relative error exceeds $10^{2}$ (which happens if the step size is too large and the algorithm diverges). Figure~\ref{fig:scaledGD_stepsizes} plots the relative error with respect to the step size $\eta$ for both algorithms, where we can see that \texttt{ScaledGD} outperforms vanilla GD over a large range of step sizes, even under optimized values for performance. Hence, our choice of $\eta = 0.5$ in previous experiments renders a typical comparison between  \texttt{ScaledGD} and vanilla GD. 
\begin{figure}[ht]
\centering
 \includegraphics[width=0.45\textwidth]{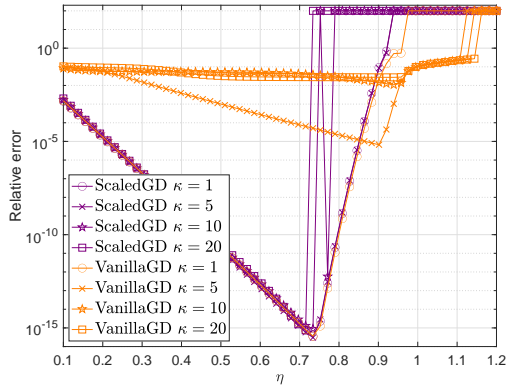} 
\caption{The relative errors of \texttt{ScaledGD} and vanilla GD after $80$ iterations with respect to different step sizes $\eta$ from $0.1$ to $1.2$, for matrix completion with $n=1000, r=10, p=0.2$.  }\label{fig:scaledGD_stepsizes}
\end{figure}

\subsection{Run time comparisons}

\begin{figure}[!ht]
\centering
\begin{tabular}{cc}
 \includegraphics[width=0.45\textwidth]{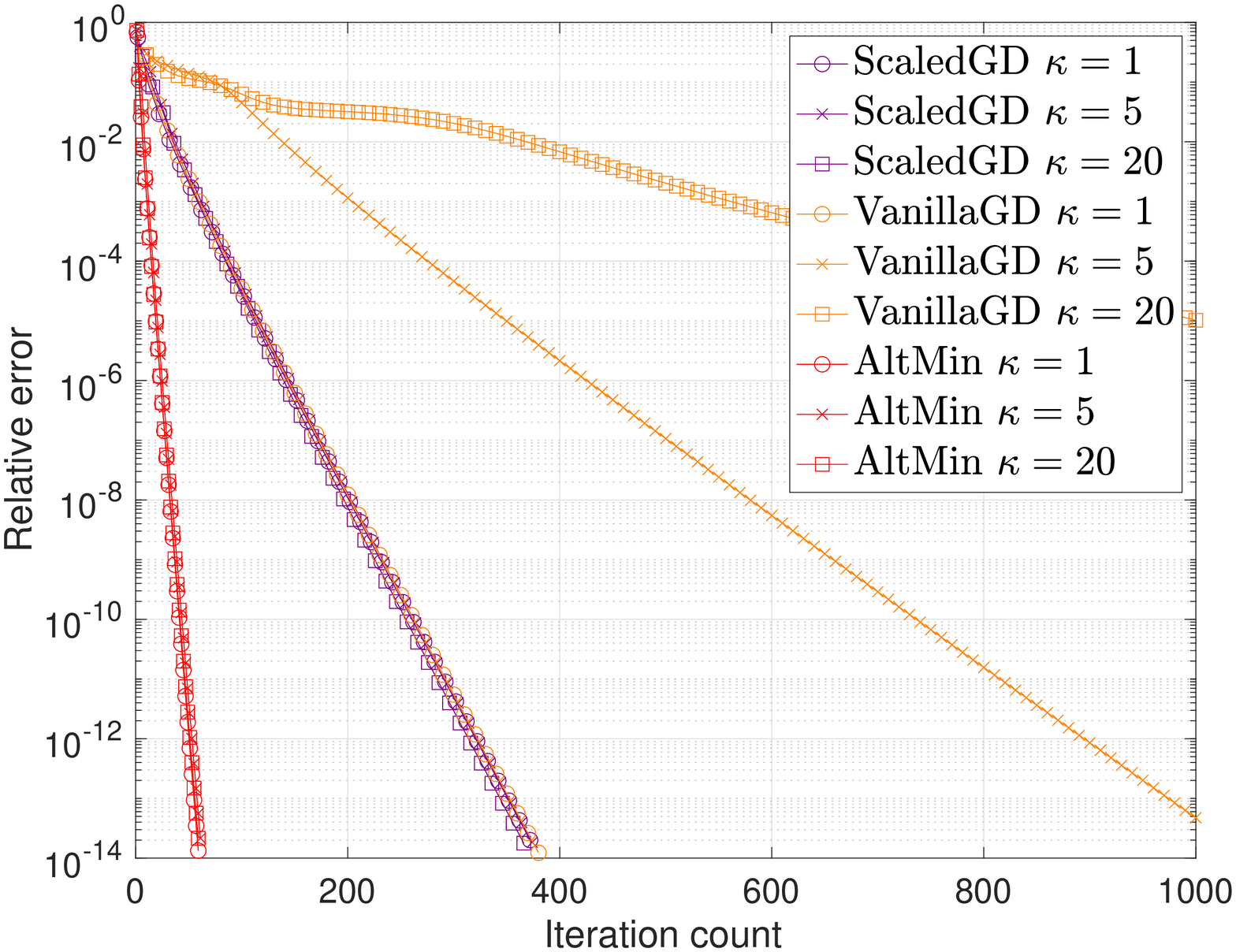} & 
 \includegraphics[width=0.45\textwidth]{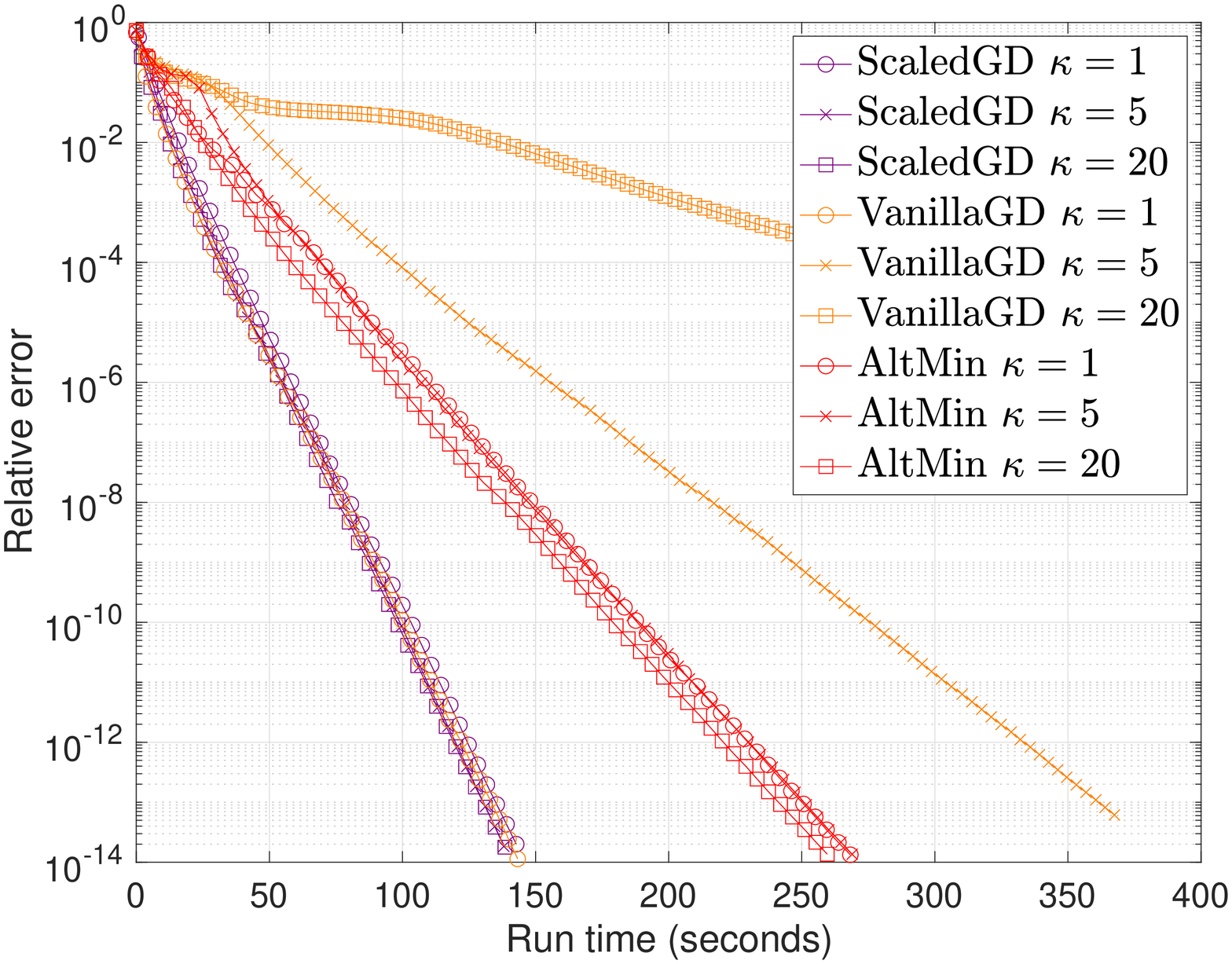} \\
	(a)  iteration count with $r=10$ & (b) run time with $r = 10$ \\
 \includegraphics[width=0.45\textwidth]{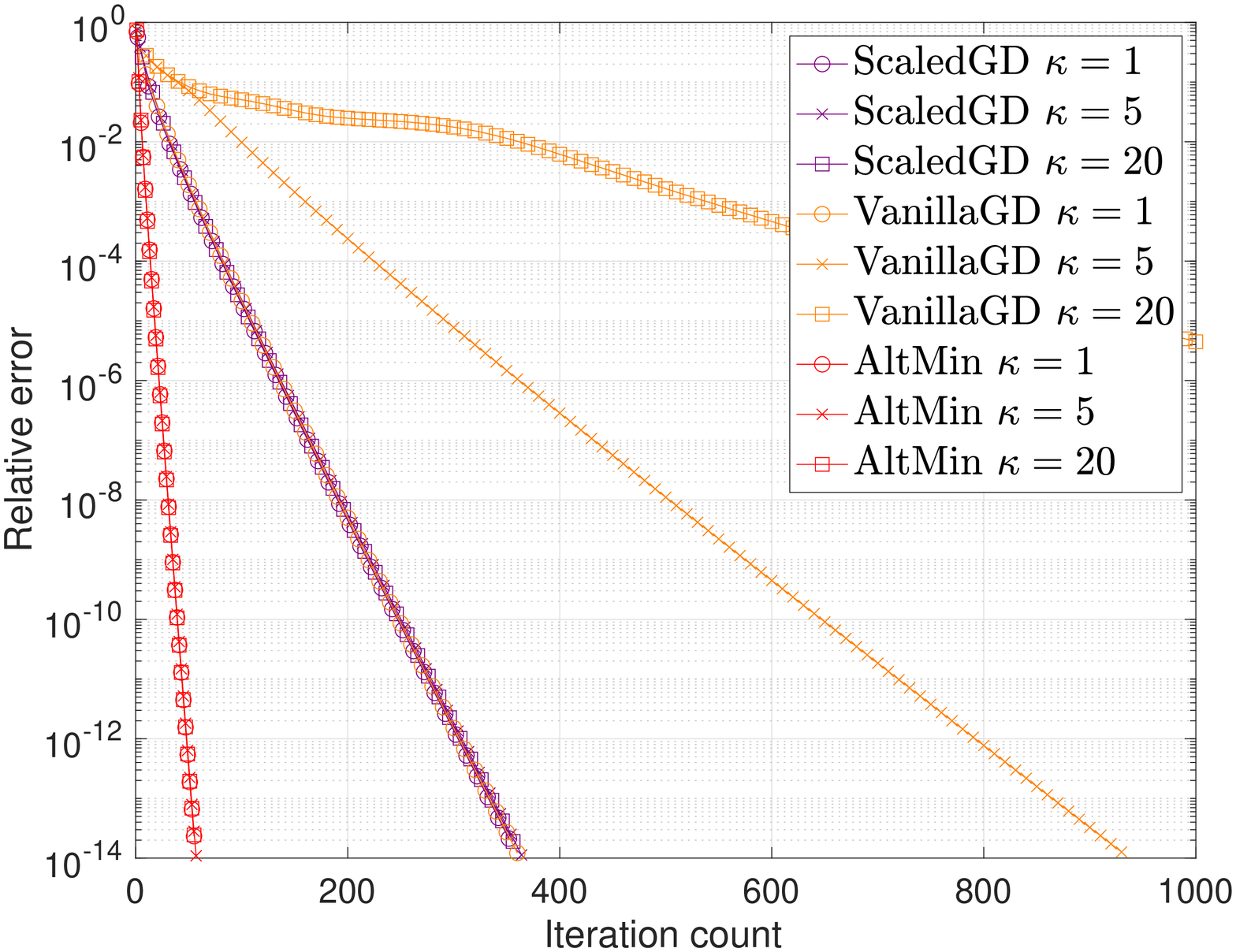} & 
 \includegraphics[width=0.45\textwidth]{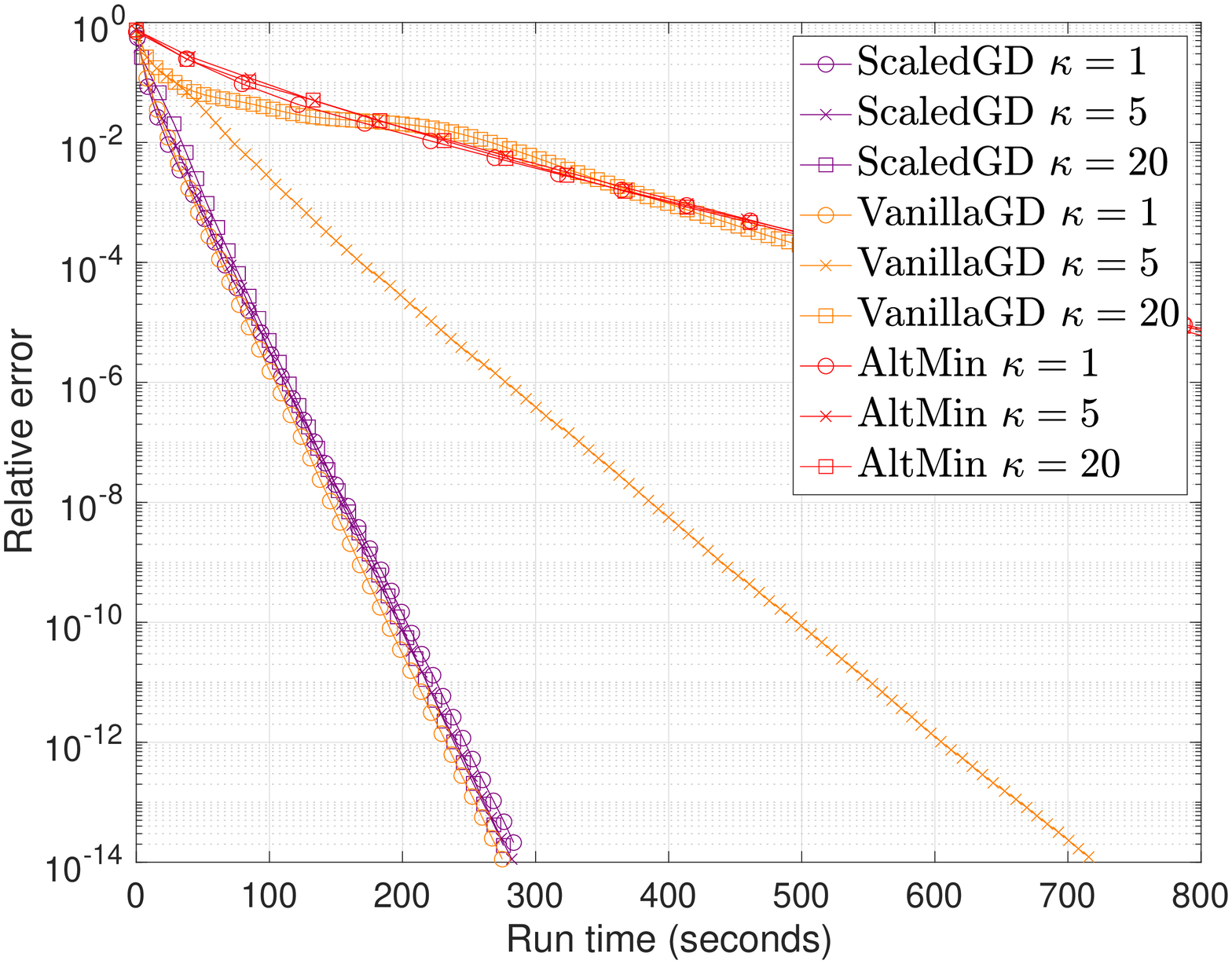} \\
	(c)  iteration count with $r=20$ & (d) run time with $r = 20$ \\
\end{tabular}
\caption{The relative errors of \texttt{ScaledGD}, vanilla GD and \texttt{AltMin} with respect to the iteration count and run time (in seconds) under different condition numbers $\kappa=1,5,20$ for matrix sensing with $n=200$, and $m=5nr$. (a, b): $r=10$; (c, d): $r=20$.}\label{fig:scaledGD_time_MS}
\end{figure}

\begin{figure}[!ht]
\centering
\begin{tabular}{cc}
 \includegraphics[width=0.45\textwidth]{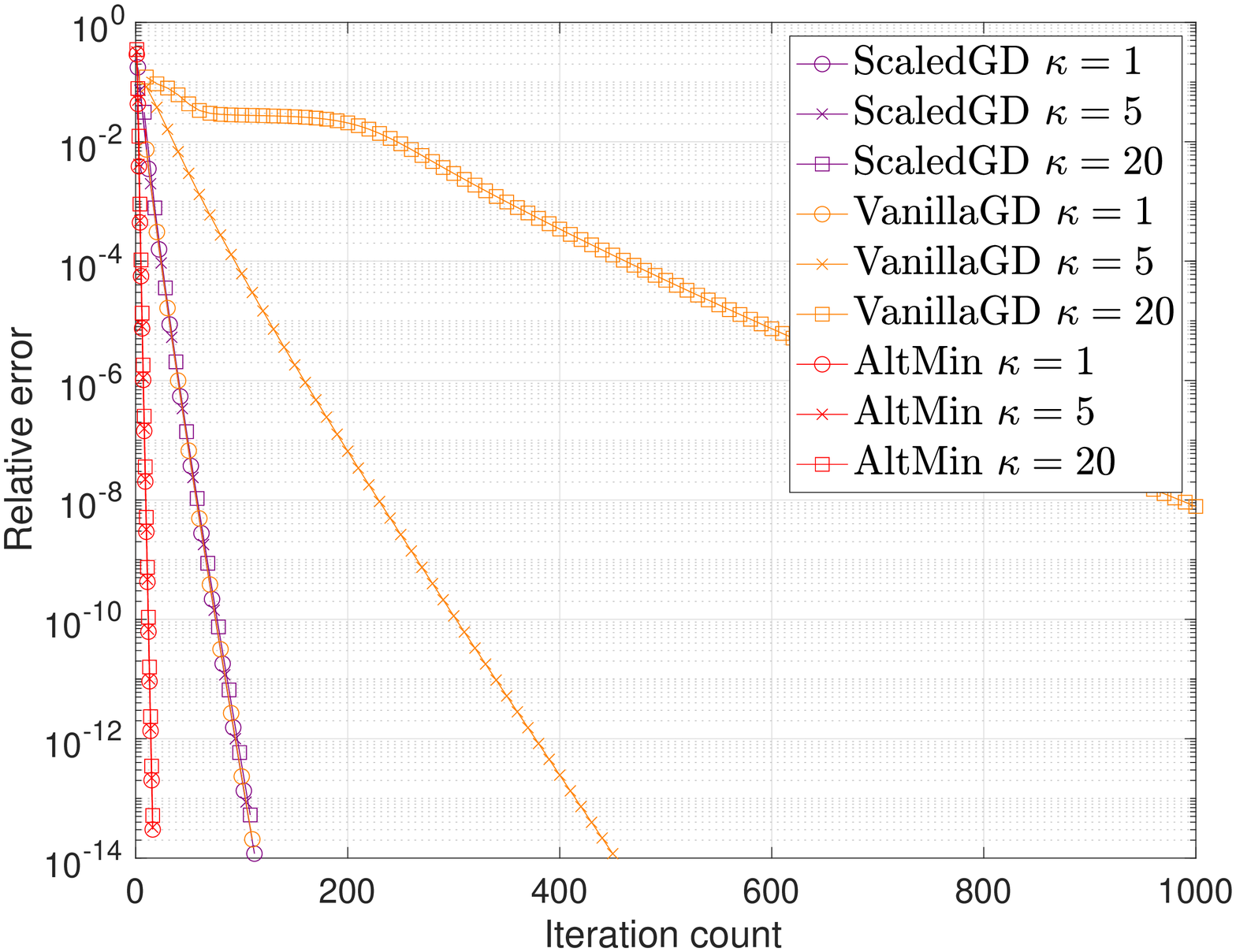} & 
 \includegraphics[width=0.45\textwidth]{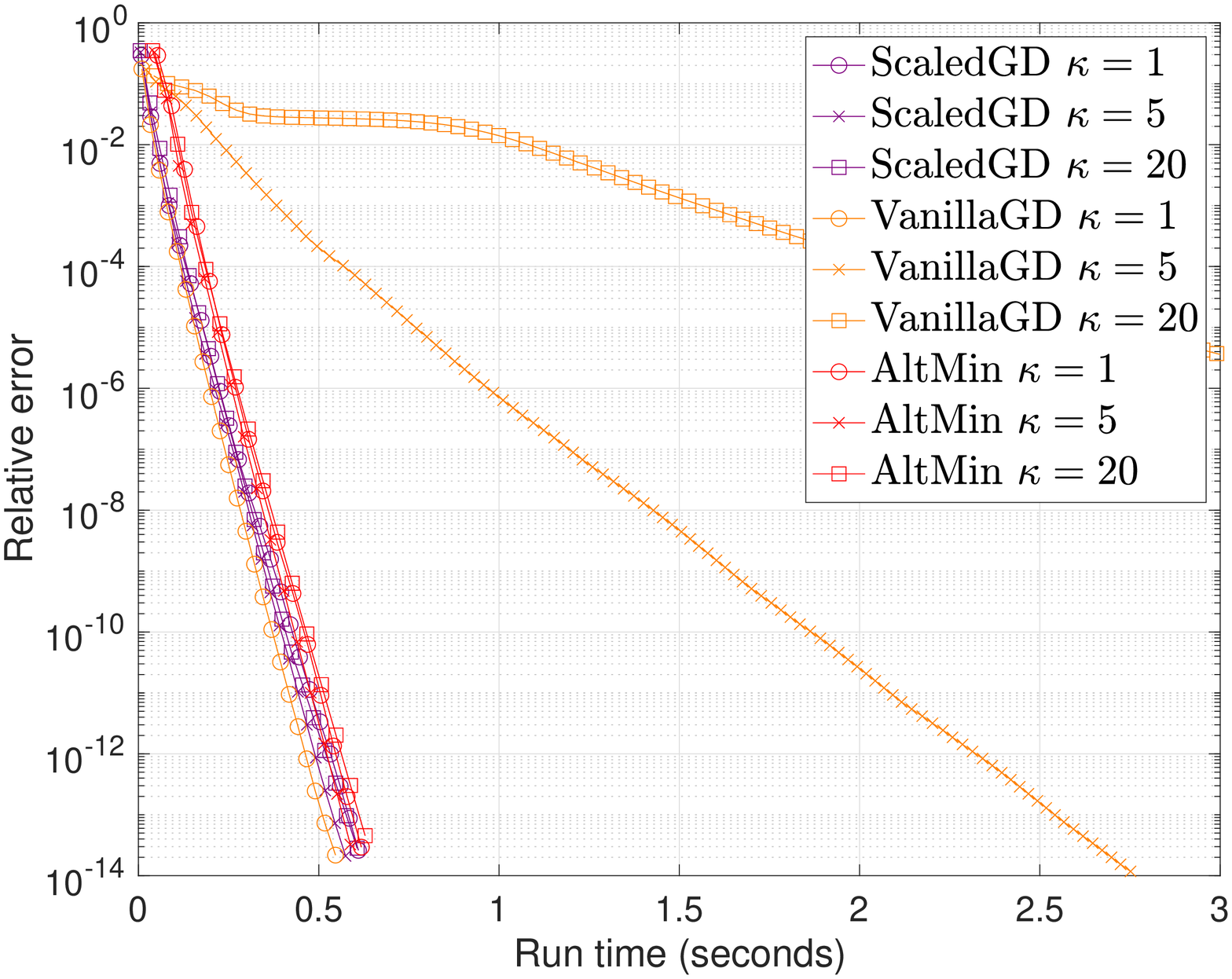} \\ 
	(a)  iteration count with $r=10$ & (b) run time with $r = 10$ \\
 \includegraphics[width=0.45\textwidth]{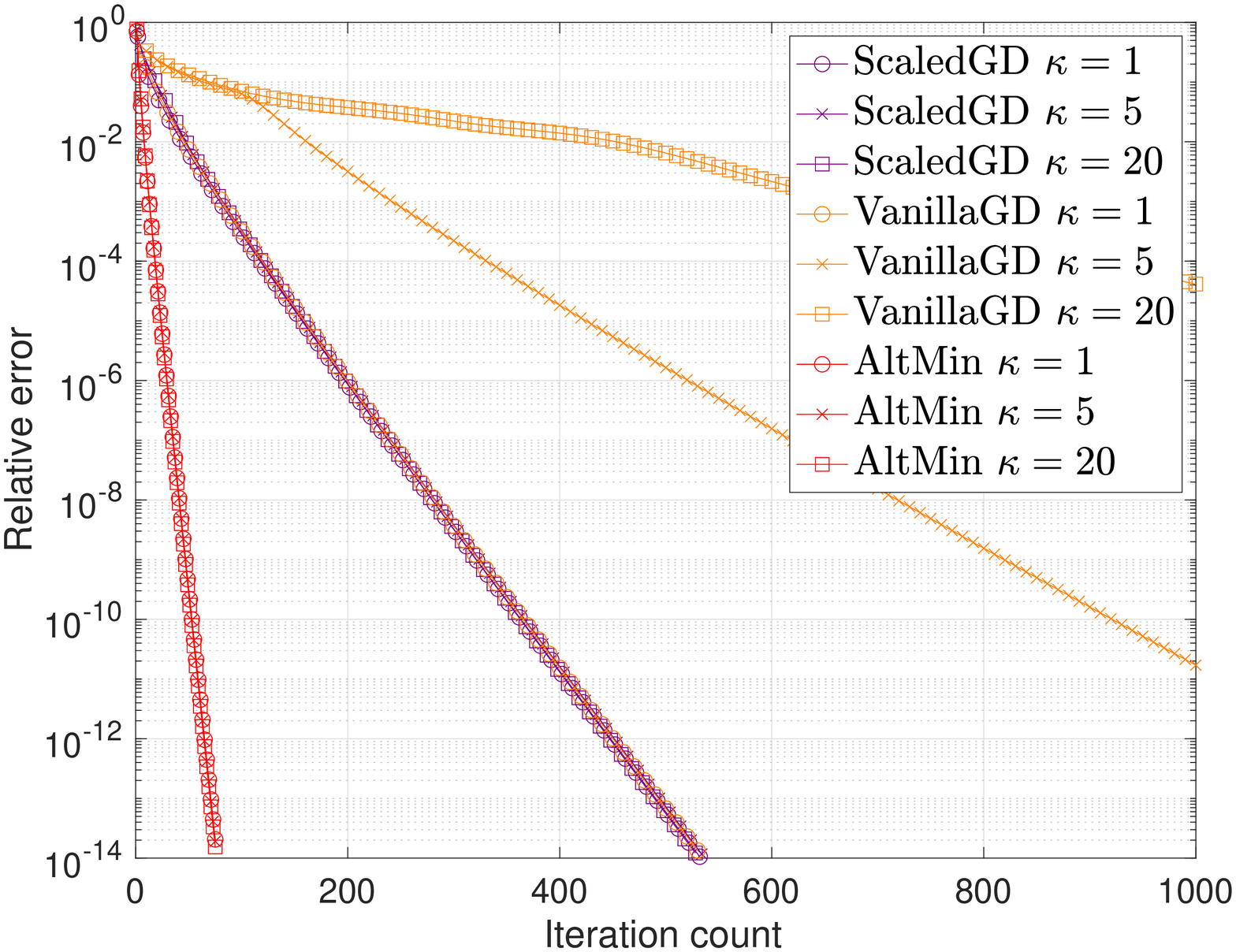} &
 \includegraphics[width=0.45\textwidth]{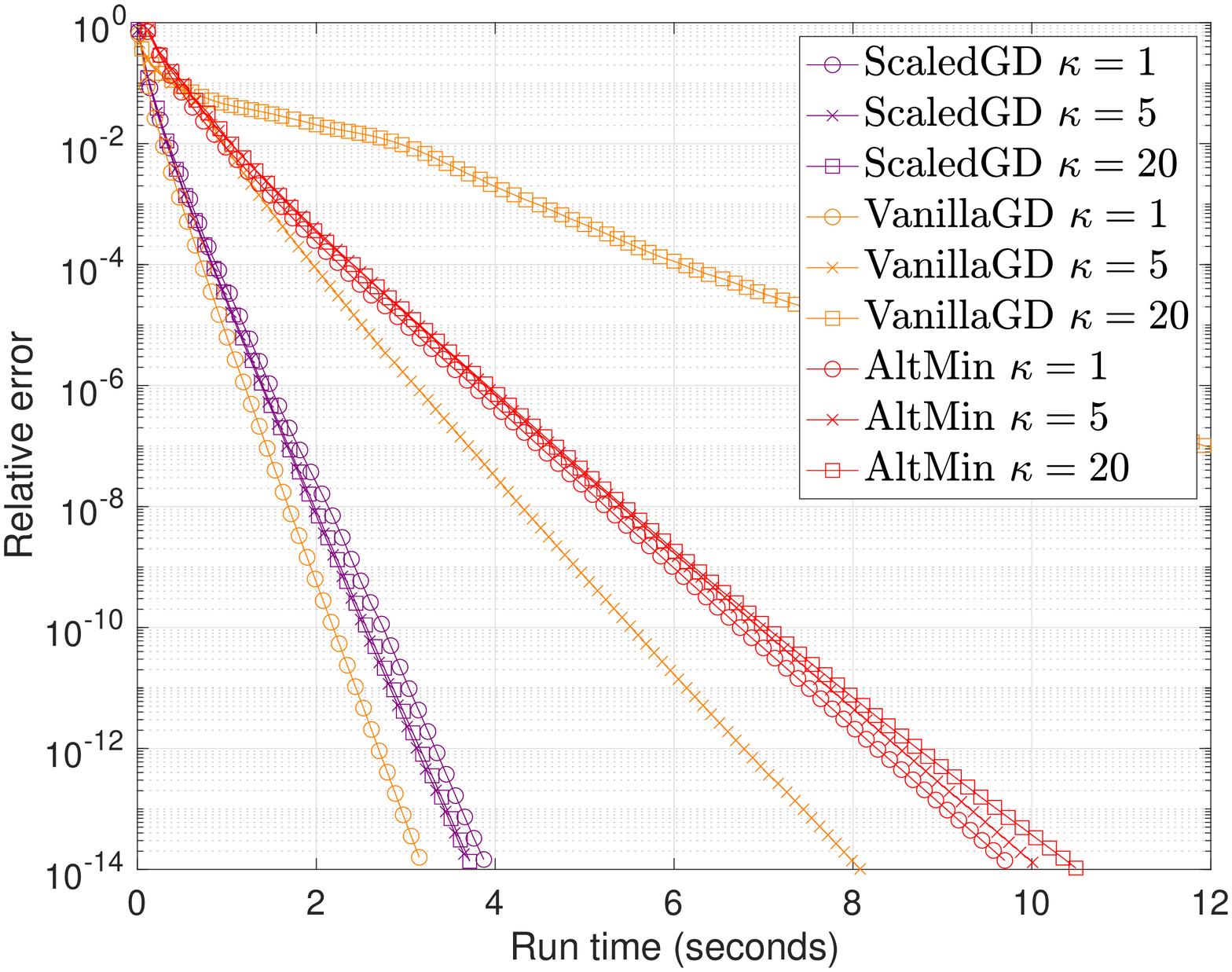} \\
	(c)  iteration count with $r=50$ & (d) run time with $r = 50$ \\
\end{tabular}
\caption{The relative errors of \texttt{ScaledGD}, vanilla GD and \texttt{AltMin} with respect to the iteration count and run time (in seconds) under different condition numbers $\kappa=1,5,20$ for matrix completion with $n=1000$, and $p=0.2$. (a, b): $r=10$; (c, d): $r=50$.}\label{fig:scaledGD_time_MC}
\end{figure}

We now compare the run time of \texttt{ScaledGD} with vanilla GD and alternating minimization (\texttt{AltMin}) \cite{jain2013low}. Specifically, for matrix sensing, alternating minimization (\texttt{AltMinSense}) updates the factors alternatively as
\begin{align*}
\bL_{t+1} &= \argmin_{\bL}\;\left\Vert \cA(\bL\bR_{t}^{\top})-\by\right\Vert _{2}^2, \\
\bR_{t+1} &= \argmin_{\bR}\;\left\Vert \cA(\bL_{t+1}\bR^{\top})-\by\right\Vert _{2}^2,
\end{align*}
which corresponds to solving two least-squares problems. For matrix completion, the update rule of alternating minimization proceeds as
\begin{align*}
\bL_{t+1} &= \argmin_{\bL}\;\left\Vert \cP_{\Omega}(\bL\bR_{t}^{\top}-\bY)\right\Vert _{2}^2, \\
\bR_{t+1} &= \argmin_{\bR}\;\left\Vert \cP_{\Omega}(\bL_{t+1}\bR^{\top}-\bY)\right\Vert _{2}^2,
\end{align*}
which can be implemented more efficiently since each row of $\bL$ (resp.~$\bR$) can be updated independently via solving a much smaller least-squares problem due to the decomposable structure of the objective function. It is worth noting that, to the best of our knowledge, this most natural variant of alternating minimization for matrix completion still eludes from a provable performance guarantee, nonetheless, we choose it to compare against due to its popularity and excellent empirical performance.

Figure~\ref{fig:scaledGD_time_MS} plots the relative errors of \texttt{ScaledGD}, vanilla GD and alternating minimization (\texttt{AltMin}) with respect to the iteration count and run time (in seconds) under different condition numbers $\kappa=1,5,20$; and similarly, Figure~\ref{fig:scaledGD_time_MC} plots the corresponding results for matrix completion. It can be seen that, both \texttt{ScaledGD} and \texttt{AltMin} admit a convergence rate that is independent of the condition number, where the per-iteration complexity of \texttt{AltMin} is much higher than that of \texttt{ScaledGD}. As expected, the run time of \texttt{ScaledGD} only adds a minimal overhead to vanilla GD while being much more robust to ill-conditioning. Noteworthily, \texttt{AltMin} takes much more time and becomes significantly slower than \texttt{ScaledGD} when the rank $r$ is larger. Nonetheless, we emphasize that since the run time is impacted by many factors in terms of problem parameters as well as implementation details, our purpose is to demonstrate the competitive performance of \texttt{ScaledGD} over alternatives, rather than claiming it as the state-of-the-art.

%% file: discussion.tex
\section{Conclusions}\label{sec:discussion}

This paper proposes scaled gradient descent (\texttt{ScaledGD}) for factored low-rank matrix estimation, which maintains the low per-iteration computational complexity of vanilla gradient descent, but offers significant speed-up in terms of the convergence rate with respect to the condition number $\kappa$ of the low-rank matrix. In particular, we rigorously establish that for low-rank matrix sensing, robust PCA, and matrix completion, to reach $\epsilon$-accuracy, \texttt{ScaledGD} only takes $O(\log(1/\epsilon))$ iterations without the dependency on the condition number when initialized via the spectral method, under standard assumptions. The key to our analysis is the introduction of a new distance metric that takes into account the preconditioning and unbalancedness of the low-rank factors, and we have developed new tools to analyze the trajectory of \texttt{ScaledGD} under this new metric. This work opens up many venues for future research, as we discuss below. 
\begin{itemize}
\item {\em Improved analysis.} In this paper, we have focused on establishing the fast local convergence rate. It is interesting to study if the theory developed herein can be further strengthened in terms of sample complexity and the size of basin of attraction. For matrix completion, it will be interesting to see if a similar guarantee continues to hold in the absence of the projection, which will generalize recent works \cite{ma2017implicit,chen2019nonconvex} that successfully removed these projections for vanilla gradient descent. 
\item {\em Other low-rank recovery problems.} Besides the problems studied herein, there are many other applications involving the recovery of an ill-conditioned low-rank matrix, such as robust PCA with missing data, quadratic sampling, and so on. It is of interest to establish fast convergence rates of \texttt{ScaledGD} that are independent of the condition number for these problems as well. In addition, it is worthwhile to explore if a similar preconditioning trick can be useful to problems beyond low-rank matrix estimation. One recent attempt is to generalize \texttt{ScaledGD} for low-rank tensor estimation \cite{tong2021scaling}. 

\item {\em Acceleration schemes?} As it is evident from our analysis of the general loss case, \texttt{ScaledGD} may still converge slowly when the loss function is ill-conditioned over low-rank matrices, i.e.~$\kappa_{f}$ is large. In this case, it might be of interest to combine techniques such as momentum \cite{kyrillidis2012matrix} from the optimization literature to further accelerate the convergence. In our companion paper \cite{tong2021low}, we have extended \texttt{ScaledGD} to nonsmooth formulations, which possess better curvatures than their smooth counterparts for certain problems. 
\end{itemize}

%% file: appendix.tex
\section{Technical Lemmas}

This section gathers several useful lemmas that will be used in the appendix. Throughout all lemmas, we use $\bX_{\star}$ to denote the ground truth low-rank matrix, with its compact SVD as $\bX_{\star}=\bU_{\star}\bSigma_{\star}\bV_{\star}^{\top}$, and the stacked factor matrix is defined as $\bF_{\star}=\begin{bmatrix}\bL_{\star}\\ \bR_{\star}\end{bmatrix}=\begin{bmatrix}\bU_{\star}\bSigma_{\star}^{1/2}\\ \bV_{\star}\bSigma_{\star}^{1/2}\end{bmatrix}$.

\subsection{New distance metric}\label{subsec:distance_metric}

We begin with the investigation of the new distance metric \eqref{eq:dist}, where the matrix $\bQ$ that attains the infimum, if exists, is called the optimal alignment matrix between $\bF$ and $\bF_{\star}$; see \eqref{eq:Q_def}. Notice that \eqref{eq:dist} involves a minimization problem over an open set (the set of invertible matrices). Hence the minimizer, i.e.~the optimal alignment matrix between $\bF$ and $\bF_{\star}$ is not guaranteed to be attained. Fortunately, a simple sufficient condition guarantees the existence of the minimizer; see the lemma below.

\begin{lemma}\label{lemma:Q_existence} Fix any factor matrix $\bF=\begin{bmatrix}\bL \\ \bR\end{bmatrix}\in\RR^{(n_{1}+n_2)\times r}$. Suppose that
\begin{align}
\dist(\bF,\bF_{\star})=\sqrt{\inf_{\bQ\in\GL(r)}\left\Vert \left(\bL\bQ-\bL_{\star}\right)\bSigma_{\star}^{1/2}\right\Vert _{\fro}^{2}+\left\Vert \left(\bR\bQ^{-\top}-\bR_{\star}\right)\bSigma_{\star}^{1/2}\right\Vert _{\fro}^{2}}
<\sigma_{r}(\bX_{\star}),\label{eq:Q_existence_condition}
\end{align}
then the minimizer of the above minimization problem is attained at some $\bQ\in\GL(r)$, i.e.~the optimal alignment matrix $\bQ$ between $\bF$ and $\bF_{\star}$ exists. 
\end{lemma}

\begin{proof} In view of the condition \eqref{eq:Q_existence_condition} and the definition of infimum, one knows that there must exist a matrix $\bar{\bQ}\in\GL(r)$ such that 
\begin{align*}
\sqrt{\left\Vert \left(\bL\bar{\bQ}-\bL_{\star}\right)\bSigma_{\star}^{1/2}\right\Vert _{\fro}^{2}+\left\Vert \left(\bR\bar{\bQ}^{-\top}-\bR_{\star}\right)\bSigma_{\star}^{1/2}\right\Vert _{\fro}^{2}} \le \epsilon\sigma_{r}(\bX_{\star}),
\end{align*}
for some $\epsilon$ obeying $0<\epsilon<1$. It further implies that 
\begin{align*}
\left\Vert \left(\bL\bar{\bQ}-\bL_{\star}\right)\bSigma_{\star}^{-1/2}\right\Vert _{\op} \vee \left\Vert \left(\bR\bar{\bQ}^{-\top}-\bR_{\star}\right)\bSigma_{\star}^{-1/2}\right\Vert _{\op} \le \epsilon.
\end{align*}
Invoke Weyl's inequality $|\sigma_{r}(\bA)-\sigma_{r}(\bB)| \le \|\bA-\bB\|_{\op}$, and use that $\sigma_{r}(\bL_{\star}\bSigma_{\star}^{-1/2})=\sigma_{r}(\bU_{\star})=1$ to obtain
\begin{align}
\sigma_{r}(\bL\bar{\bQ}\bSigma_{\star}^{-1/2}) \ge \sigma_{r}(\bL_{\star}\bSigma_{\star}^{-1/2}) - \left\Vert \left(\bL\bar{\bQ}-\bL_{\star}\right)\bSigma_{\star}^{-1/2}\right\Vert _{\op} \ge 1-\epsilon.\label{eq:first_sigma}
\end{align}
In addition, it is straightforward to verify that 
\begin{align}
 & \inf_{\bQ\in\GL(r)}\quad\left\Vert \left(\bL\bQ-\bL_{\star}\right)\bSigma_{\star}^{1/2}\right\Vert _{\fro}^{2}+\left\Vert \left(\bR\bQ^{-\top}-\bR_{\star}\right)\bSigma_{\star}^{1/2}\right\Vert _{\fro}^{2}\label{eq:first_inf}\\
 & \quad=\inf_{\bH\in\GL(r)}\quad\left\Vert \left(\bL\bar{\bQ}\bH-\bL_{\star}\right)\bSigma_{\star}^{1/2}\right\Vert _{\fro}^{2}+\left\Vert \left(\bR\bar{\bQ}^{-\top}\bH^{-\top}-\bR_{\star}\right)\bSigma_{\star}^{1/2}\right\Vert _{\fro}^{2}.\label{eq:second_inf}
\end{align}
Indeed, if the minimizer of the second optimization problem (cf.~\eqref{eq:second_inf}) is attained at some $\bH$, then $\bar{\bQ}\bH$ must be the minimizer of the first problem \eqref{eq:first_inf}. 
Therefore, from now on, we focus on proving that the minimizer of the second problem \eqref{eq:second_inf} is attained at some $\bH$. In view of \eqref{eq:first_inf} and \eqref{eq:second_inf}, one has 
\begin{align*}
 & \inf_{\bH\in\GL(r)}\;\left\Vert \left(\bL\bar{\bQ}\bH-\bL_{\star}\right)\bSigma_{\star}^{1/2}\right\Vert _{\fro}^{2}+\left\Vert \left(\bR\bar{\bQ}^{-\top}\bH^{-\top}-\bR_{\star}\right)\bSigma_{\star}^{1/2}\right\Vert _{\fro}^{2}\nonumber \\
 & \quad\le\left\Vert \left(\bL\bar{\bQ}-\bL_{\star}\right)\bSigma_{\star}^{1/2}\right\Vert _{\fro}^{2}+\left\Vert \left(\bR\bar{\bQ}^{-\top}-\bR_{\star}\right)\bSigma_{\star}^{1/2}\right\Vert _{\fro}^{2},
\end{align*}
Clearly, for any $\bar{\bQ}\bH$ to yield a smaller distance than $\bar{\bQ}$, $\bH$ must obey
\begin{align*}
\sqrt{\left\Vert \left(\bL\bar{\bQ}\bH-\bL_{\star}\right)\bSigma_{\star}^{1/2}\right\Vert _{\fro}^{2}+\left\Vert \left(\bR\bar{\bQ}^{-\top}\bH^{-\top}-\bR_{\star}\right)\bSigma_{\star}^{1/2}\right\Vert _{\fro}^{2}} \le\epsilon\sigma_{r}(\bX_{\star}).
\end{align*}
It further implies that
\begin{align*}
\left\Vert \left(\bL\bar{\bQ}\bH-\bL_{\star}\right)\bSigma_{\star}^{-1/2}\right\Vert _{\op} \vee \left\Vert \left(\bR\bar{\bQ}^{-\top}\bH^{-\top}-\bR_{\star}\right)\bSigma_{\star}^{-1/2}\right\Vert _{\op} \le \epsilon.
\end{align*}
Invoke Weyl's inequality $|\sigma_{1}(\bA)-\sigma_{1}(\bB)| \le \|\bA-\bB\|_{\op}$, and use that $\sigma_{1}(\bL_{\star}\bSigma_{\star}^{-1/2})=\sigma_{1}(\bU_{\star})=1$ to obtain
\begin{align}
\sigma_{1}(\bL\bar{\bQ}\bH\bSigma_{\star}^{-1/2}) \le \sigma_{1}(\bL_{\star}\bSigma_{\star}^{-1/2}) + \left\Vert\left(\bL\bar{\bQ}\bH-\bL_{\star}\right)\bSigma_{\star}^{-1/2}\right\Vert _{\op} \le 1+\epsilon.\label{eq:second_sigma}
\end{align}
Combine \eqref{eq:first_sigma} and \eqref{eq:second_sigma}, and use the relation $\sigma_{r}(\bA)\sigma_{1}(\bB) \le \sigma_{1}(\bA\bB)$ to obtain
\begin{align*}
\sigma_{r}(\bL\bar{\bQ}\bSigma_{\star}^{-1/2})\sigma_{1}(\bSigma_{\star}^{1/2}\bH\bSigma_{\star}^{-1/2}) \le \sigma_{1}(\bL\bar{\bQ}\bH\bSigma_{\star}^{-1/2}) \le \frac{1+\epsilon}{1-\epsilon}\sigma_{r}(\bL\bar{\bQ}\bSigma_{\star}^{-1/2}).
\end{align*}
As a result, one has $\sigma_{1}(\bSigma_{\star}^{1/2}\bH\bSigma_{\star}^{-1/2}) \le \frac{1+\epsilon}{1-\epsilon}$.

Similarly, one can show that $\sigma_{1}(\bSigma_{\star}^{1/2}\bH^{-\top}\bSigma_{\star}^{-1/2}) \le \frac{1+\epsilon}{1-\epsilon}$, equivalently, $\sigma_{r}(\bSigma_{\star}^{1/2}\bH\bSigma_{\star}^{-1/2}) \ge \frac{1-\epsilon}{1+\epsilon}$. 
Combining the above two arguments reveals that the minimization problem \eqref{eq:second_inf} is equivalent to the constrained problem: 
\begin{align*}
\minimize_{\bH\in\GL(r)}\quad & \left\Vert \left(\bL\bar{\bQ}\bH-\bL_{\star}\right)\bSigma_{\star}^{1/2}\right\Vert _{\fro}^{2}+\left\Vert \left(\bR\bar{\bQ}^{-\top}\bH^{-\top}-\bR_{\star}\right)\bSigma_{\star}^{1/2}\right\Vert _{\fro}^{2}\\
\mbox{s.t.}\quad & \frac{1-\epsilon}{1+\epsilon} \le \sigma_{r}(\bSigma_{\star}^{1/2}\bH\bSigma_{\star}^{-1/2}) \le \sigma_{1}(\bSigma_{\star}^{1/2}\bH\bSigma_{\star}^{-1/2})\le \frac{1+\epsilon}{1-\epsilon}.
\end{align*}
Notice that this is a continuous optimization problem over a compact set. Apply the Weierstrass extreme value theorem to finish the proof.
\end{proof}

With the existence of the optimal alignment matrix in place, the following lemma provides the first-order necessary condition for the minimizer.

\begin{lemma}\label{lemma:Q_criterion} For any factor matrix $\bF=\begin{bmatrix}\bL \\ \bR\end{bmatrix}\in\RR^{(n_{1}+n_{2})\times r}$, suppose that the optimal alignment matrix 
\begin{align*}
\bQ=\argmin_{\bQ\in\GL(r)}\;\left\Vert (\bL\bQ-\bL_{\star})\bSigma_{\star}^{1/2}\right\Vert _{\fro}^{2}+\left\Vert (\bR\bQ^{-\top}-\bR_{\star})\bSigma_{\star}^{1/2}\right\Vert _{\fro}^{2}
\end{align*}
between $\bF$ and $\bF_{\star}$ exists, then $\bQ$ obeys
\begin{align}
(\bL\bQ)^{\top}(\bL\bQ-\bL_{\star})\bSigma_{\star}=\bSigma_{\star}(\bR\bQ^{-\top}-\bR_{\star})^{\top}\bR\bQ^{-\top}.\label{eq:Q_criterion}
\end{align}
\end{lemma}

\begin{proof} Expand the squares in the definition of $\bQ$ to obtain
\begin{align*}
\bQ=\argmin_{\bQ\in\GL(r)}\;\tr\left((\bL\bQ-\bL_{\star})^{\top}(\bL\bQ-\bL_{\star})\bSigma_{\star}\right)+\tr\left((\bR\bQ^{-\top}-\bR_{\star})^{\top}(\bR\bQ^{-\top}-\bR_{\star})\bSigma_{\star}\right).
\end{align*}
Clearly, the first order necessary condition (i.e.~the gradient is zero) yields 
\begin{align*}
 & 2\bL^{\top}(\bL\bQ-\bL_{\star})\bSigma_{\star}-2\bQ^{-\top}\bSigma_{\star}(\bR\bQ^{-\top}-\bR_{\star})^{\top}\bR\bQ^{-\top}=\zero,
\end{align*}
which implies the optimal alignment criterion \eqref{eq:Q_criterion}.
\end{proof}

Last but not least, we connect the newly proposed distance to the usual Frobenius norm in Lemma~\ref{lemma:Procrustes}, the proof of which is a slight modification to \cite[Lemma~5.4]{tu2015low} and \cite[Lemma~41]{ge2017no}.

\begin{lemma}\label{lemma:Procrustes} For any factor matrix $\bF=\begin{bmatrix}\bL \\ \bR \end{bmatrix}\in\RR^{(n_{1}+n_{2})\times r}$, the distance between $\bF$ and $\bF_{\star}$ satisfies
\begin{align*}
\dist(\bF,\bF_{\star})\le \left(\sqrt{2}+1 \right)^{1/2}\|\bL\bR^{\top}-\bX_{\star}\|_{\fro}.
\end{align*}
\end{lemma} 

\begin{proof} Suppose that $\bX\coloneqq\bL\bR^{\top}$ has compact SVD as $\bX=\bU\bSigma\bV^{\top}$. Without loss of generality, we can assume that $\bF=\begin{bmatrix}\bU\bSigma^{1/2} \\ \bV\bSigma^{1/2}\end{bmatrix}$, since any factorization of $\bL\bR^{\top}$ yields the same distance. Introduce two auxiliary matrices $\bar{\bF}\coloneqq\begin{bmatrix}\bU\bSigma^{1/2}\\
-\bV\bSigma^{1/2}
\end{bmatrix}$ and $\bar{\bF}_{\star}\coloneqq\begin{bmatrix}\bU_{\star}\bSigma_{\star}^{1/2}\\
-\bV_{\star}\bSigma_{\star}^{1/2}
\end{bmatrix}$. Apply the dilation trick to obtain 
\begin{align*}
2\begin{bmatrix}\zero & \bX\\
\bX^{\top} & \zero
\end{bmatrix}=\bF\bF^{\top}-\bar{\bF}\bar{\bF}^{\top},\quad2\begin{bmatrix}\zero & \bX_{\star}\\
\bX_{\star}^{\top} & \zero
\end{bmatrix}=\bF_{\star}\bF_{\star}^{\top}-\bar{\bF}_{\star}\bar{\bF}_{\star}^{\top}.
\end{align*}
As a result, the squared Frobenius norm of $\bX-\bX_{\star}$ is given by
\begin{align*}
8\|\bX-\bX_{\star}\|_{\fro}^{2} & =\left\Vert \bF\bF^{\top}-\bar{\bF}\bar{\bF}^{\top}-\bF_{\star}\bF_{\star}^{\top}+\bar{\bF}_{\star}\bar{\bF}_{\star}^{\top}\right\Vert _{\fro}^{2}\\
 & =\left\Vert \bF\bF^{\top}-\bF_{\star}\bF_{\star}^{\top}\right\Vert _{\fro}^{2}+\left\Vert \bar{\bF}\bar{\bF}^{\top}-\bar{\bF}_{\star}\bar{\bF}_{\star}^{\top}\right\Vert _{\fro}^{2}-2\tr\left((\bF\bF^{\top}-\bF_{\star}\bF_{\star}^{\top})(\bar{\bF}\bar{\bF}^{\top}-\bar{\bF}_{\star}\bar{\bF}_{\star}^{\top})\right)\\
 & =2\left\Vert \bF\bF^{\top}-\bF_{\star}\bF_{\star}^{\top}\right\Vert _{\fro}^{2}+2\|\bF^{\top}\bar{\bF}_{\star}\|_{\fro}^{2}+2\|\bF_{\star}^{\top}\bar{\bF}\|_{\fro}^{2}\\
 & \ge2\left\Vert \bF\bF^{\top}-\bF_{\star}\bF_{\star}^{\top}\right\Vert _{\fro}^{2},
\end{align*}
where we use the facts that $\left\Vert \bF\bF^{\top}-\bF_{\star}\bF_{\star}^{\top}\right\Vert _{\fro}^{2}=\left\Vert \bar{\bF}\bar{\bF}^{\top}-\bar{\bF}_{\star}\bar{\bF}_{\star}^{\top}\right\Vert _{\fro}^{2}$ and $\bF^{\top}\bar{\bF}=\bF_{\star}^{\top}\bar{\bF}_{\star}=\zero$. 

Let $\bO\coloneqq\sgn(\bF^{\top}\bF_{\star})$\footnote{Let $\bA\bS\bB^{\top}$ be the SVD of $\bF^{\top}\bF_{\star}$, then the matrix sign is $\sgn(\bF^{\top}\bF_{\star})\coloneqq\bA\bB^{\top}$.} be the optimal orthonormal alignment matrix between $\bF$ and $\bF_{\star}$. Denote $\bDelta\coloneqq\bF\bO-\bF_{\star}$.
Follow the same argument as \cite[Lemma~5.14]{tu2015low} and \cite[Lemma~41]{ge2017no} to obtain 
\begin{align*}
4\|\bX-\bX_{\star}\|_{\fro}^{2} & \ge\left\Vert \bF_{\star}\bDelta^{\top}+\bDelta\bF_{\star}^{\top}+\bDelta\bDelta^{\top}\right\Vert _{\fro}^{2}\\
 & =\tr\left(2\bF_{\star}^{\top}\bF_{\star}\bDelta^{\top}\bDelta+(\bDelta^{\top}\bDelta)^{2}+2(\bF_{\star}^{\top}\bDelta)^{2}+4\bF_{\star}^{\top}\bDelta\bDelta^{\top}\bDelta\right)\\
 & =\tr\left(2\bF_{\star}^{\top}\bF_{\star}\bDelta^{\top}\bDelta+(\bDelta^{\top}\bDelta+\sqrt{2}\bF_{\star}^{\top}\bDelta)^{2}+(4-2\sqrt{2})\bF_{\star}^{\top}\bDelta\bDelta^{\top}\bDelta\right)\\
 & =\tr\left(2(\sqrt{2}-1)\bF_{\star}^{\top}\bF_{\star}\bDelta^{\top}\bDelta+(\bDelta^{\top}\bDelta+\sqrt{2}\bF_{\star}^{\top}\bDelta)^{2}+(4-2\sqrt{2})\bF_{\star}^{\top}\bF\bO\bDelta^{\top}\bDelta\right)\\
 & \ge\tr\left(4(\sqrt{2}-1)\bSigma_{\star}\bDelta^{\top}\bDelta\right)=4(\sqrt{2}-1)\left\Vert (\bF\bO-\bF_{\star})\bSigma_{\star}^{1/2}\right\Vert _{\fro}^{2},
\end{align*}
where the last inequality follows from the facts that $\bF_{\star}^{\top}\bF_{\star}=2\bSigma_{\star}$ and that $\bF_{\star}^{\top}\bF\bO$ is positive semi-definite. Therefore we obtain 
\begin{align*}
\left\Vert (\bF\bO-\bF_{\star})\bSigma_{\star}^{1/2}\right\Vert _{\fro} \le \left( \sqrt{2}+1 \right)^{1/2} \|\bX-\bX_{\star}\|_{\fro}.
\end{align*}
This in conjunction with $\dist(\bF,\bF_{\star})\le\|(\bF\bO-\bF_{\star})\bSigma_{\star}^{1/2}\|_{\fro}$ yields the claimed result. 
\end{proof}

\subsection{Matrix perturbation bounds}

\begin{lemma}\label{lemma:Weyl} For any $\bL\in\RR^{n_{1}\times r},\bR\in\RR^{n_{2}\times r}$, denote $\bDelta_{L}\coloneqq\bL-\bL_{\star}$ and $\bDelta_{R}\coloneqq\bR-\bR_{\star}$. Suppose that $\|\bDelta_{L}\bSigma_{\star}^{-1/2}\|_{\op}\vee\|\bDelta_{R}\bSigma_{\star}^{-1/2}\|_{\op}<1$, then one has 
\begin{subequations}
\begin{align}
\left\Vert \bL(\bL^{\top}\bL)^{-1}\bSigma_{\star}^{1/2}\right\Vert _{\op} & \le \frac{1}{1-\|\bDelta_{L}\bSigma_{\star}^{-1/2}\|_{\op}}; \label{eq:Weyl-1L} \\
\left\Vert \bR(\bR^{\top}\bR)^{-1}\bSigma_{\star}^{1/2}\right\Vert _{\op} & \le \frac{1}{1-\|\bDelta_{R}\bSigma_{\star}^{-1/2}\|_{\op}}; \label{eq:Weyl-1R} \\
\left\Vert \bL(\bL^{\top}\bL)^{-1}\bSigma_{\star}^{1/2}-\bU_{\star}\right\Vert _{\op}& \le\frac{\sqrt{2}\|\bDelta_{L}\bSigma_{\star}^{-1/2}\|_{\op}}{1-\|\bDelta_{L}\bSigma_{\star}^{-1/2}\|_{\op}}; \label{eq:Weyl-2L} \\
\left\Vert \bR(\bR^{\top}\bR)^{-1}\bSigma_{\star}^{1/2}-\bV_{\star}\right\Vert _{\op}& \le\frac{\sqrt{2}\|\bDelta_{R}\bSigma_{\star}^{-1/2}\|_{\op}}{1-\|\bDelta_{R}\bSigma_{\star}^{-1/2}\|_{\op}}. \label{eq:Weyl-2R}
\end{align}
\end{subequations}
\end{lemma}

\begin{proof} We only prove claims \eqref{eq:Weyl-1L} and \eqref{eq:Weyl-2L} on the factor $\bL$, while the claims on the factor $\bR$ follow from a similar argument. We start to prove \eqref{eq:Weyl-1L}. Notice that 
\begin{align*}
\left\Vert \bL(\bL^{\top}\bL)^{-1}\bSigma_{\star}^{1/2}\right\Vert _{\op}=\frac{1}{\sigma_{r}(\bL\bSigma_{\star}^{-1/2})}.
\end{align*}
In addition, invoke Weyl's inequality to obtain 
\begin{align*}
\sigma_{r}(\bL\bSigma_{\star}^{-1/2})\ge\sigma_{r}(\bL_{\star}\bSigma_{\star}^{-1/2})-\|\bDelta_{L}\bSigma_{\star}^{-1/2}\|_{\op}=1-\|\bDelta_{L}\bSigma_{\star}^{-1/2}\|_{\op},
\end{align*}
where we have used the fact that $\bU_{\star}=\bL_{\star}\bSigma_{\star}^{-1/2}$ satisfies $\sigma_{r}(\bU_{\star})=1$. Combine the preceding two relations to prove \eqref{eq:Weyl-1L}.

We proceed to prove \eqref{eq:Weyl-2L}. Combine $\bL_{\star}^{\top}\bU_{\star}=\bSigma_{\star}^{1/2}$ and $(\bI_{n_1}-\bL(\bL^{\top}\bL)^{-1}\bL^{\top})\bL=\zero$ to obtain the decomposition
\begin{align*}
\bL(\bL^{\top}\bL)^{-1}\bSigma_{\star}^{1/2}-\bU_{\star}=-\bL(\bL^{\top}\bL)^{-1}\bDelta_{L}^{\top}\bU_{\star}+(\bI_{n_{1}}-\bL(\bL^{\top}\bL)^{-1}\bL^{\top})\bDelta_{L}\bSigma_{\star}^{-1/2}.
\end{align*}
The fact that $\bL(\bL^{\top}\bL)^{-1}\bDelta_{L}^{\top}\bU_{\star}$ and $(\bI_{n_{1}}-\bL(\bL^{\top}\bL)^{-1}\bL^{\top})\bDelta_{L}\bSigma_{\star}^{-1/2}$ are orthogonal implies
\begin{align*}
\left\Vert\bL(\bL^{\top}\bL)^{-1}\bSigma_{\star}^{1/2}-\bU_{\star}\right\Vert_{\op}^2 & \le\left\Vert\bL(\bL^{\top}\bL)^{-1}\bDelta_{L}^{\top}\bU_{\star}\right\Vert_{\op}^2 + \left\Vert(\bI_{n_{1}}-\bL(\bL^{\top}\bL)^{-1}\bL^{\top})\bDelta_{L}\bSigma_{\star}^{-1/2}\right\Vert_{\op}^2 \\
 & \le\left\Vert\bL(\bL^{\top}\bL)^{-1}\bSigma_{\star}^{1/2}\right\Vert_{\op}^2\|\bDelta_{L}\bSigma_{\star}^{-1/2}\|_{\op}^2 + \left\Vert\bI_{n_{1}}-\bL(\bL^{\top}\bL)^{-1}\bL^{\top}\right\Vert_{\op}^2\|\bDelta_{L}\bSigma_{\star}^{-1/2}\|_{\op}^2 \\
 & \le\frac{\|\bDelta_{L}\bSigma_{\star}^{-1/2}\|_{\op}^2}{(1-\|\bDelta_{L}\bSigma_{\star}^{-1/2}\|_{\op})^2} + \|\bDelta_{L}\bSigma_{\star}^{-1/2}\|_{\op}^2 \\
 & \le\frac{2\|\bDelta_{L}\bSigma_{\star}^{-1/2}\|_{\op}^2}{(1-\|\bDelta_{L}\bSigma_{\star}^{-1/2}\|_{\op})^2},
\end{align*}
where we have used \eqref{eq:Weyl-1L} and the fact that $\|\bI_{n_{1}}-\bL(\bL^{\top}\bL)^{-1}\bL^{\top}\|_{\op}\le 1$ in the third line.
\end{proof}
 
\begin{lemma}\label{lemma:matrix2factor} For any $\bL\in\RR^{n_{1}\times r},\bR\in\RR^{n_{2}\times r}$, denote $\bDelta_{L}\coloneqq\bL-\bL_{\star}$ and $\bDelta_{R}\coloneqq\bR-\bR_{\star}$, then one has
\begin{align*}
 \|\bL\bR^{\top}-\bX_{\star}\|_{\fro} & \le\|\bDelta_{L}\bR_{\star}^{\top}\|_{\fro}+\|\bL_{\star}\bDelta_{R}^{\top}\|_{\fro}+\|\bDelta_{L}\bDelta_{R}^{\top}\|_{\fro}\\
 & \le\left(1+\frac{1}{2}(\|\bDelta_{L}\bSigma_{\star}^{-1/2}\|_{\op}\vee\|\bDelta_{R}\bSigma_{\star}^{-1/2}\|_{\op})\right)\left(\|\bDelta_{L}\bSigma_{\star}^{1/2}\|_{\fro}+\|\bDelta_{R}\bSigma_{\star}^{1/2}\|_{\fro}\right).
\end{align*}
\end{lemma}

\begin{proof} In light of the decomposition $\bL\bR^{\top}-\bX_{\star}=\bDelta_{L}\bR_{\star}^{\top}+\bL_{\star}\bDelta_{R}^{\top}+\bDelta_{L}\bDelta_{R}^{\top}$ and the triangle inequality, one has 
\begin{align*}
\|\bL\bR^{\top}-\bX_{\star}\|_{\fro} & \le\|\bDelta_{L}\bR_{\star}^{\top}\|_{\fro}+\|\bL_{\star}\bDelta_{R}^{\top}\|_{\fro}+\|\bDelta_{L}\bDelta_{R}^{\top}\|_{\fro}\\
 & =\|\bDelta_{L}\bSigma_{\star}^{1/2}\|_{\fro}+\|\bDelta_{R}\bSigma_{\star}^{1/2}\|_{\fro}+\|\bDelta_{L}\bDelta_{R}^{\top}\|_{\fro},
\end{align*}
where we have used the facts that 
\begin{align*}
\|\bDelta_{L}\bR_{\star}^{\top}\|_{\fro}=\|\bDelta_{L}\bSigma_{\star}^{1/2}\bV_{\star}^{\top}\|_{\fro}=\|\bDelta_{L}\bSigma_{\star}^{1/2}\|_{\fro},\quad\mbox{and}\quad\|\bL_{\star}\bDelta_{R}^{\top}\|_{\fro}=\|\bU_{\star}\bSigma_{\star}^{1/2}\bDelta_{R}^{\top}\|_{\fro}=\|\bDelta_{R}\bSigma_{\star}^{1/2}\|_{\fro}.
\end{align*}
This together with the simple upper bound 
\begin{align*}
\|\bDelta_{L}\bDelta_{R}^{\top}\|_{\fro} & =\frac{1}{2}\|\bDelta_{L}\bSigma_{\star}^{1/2}(\bDelta_{R}\bSigma_{\star}^{-1/2})^{\top}\|_{\fro}+\frac{1}{2}\|\bDelta_{L}\bSigma_{\star}^{-1/2}(\bDelta_{R}\bSigma_{\star}^{1/2})^{\top}\|_{\fro}\\
 & \le\frac{1}{2}\|\bDelta_{L}\bSigma_{\star}^{1/2}\|_{\fro}\|\bDelta_{R}\bSigma_{\star}^{-1/2}\|_{\op}+\frac{1}{2}\|\bDelta_{L}\bSigma_{\star}^{-1/2}\|_{\op}\|\bDelta_{R}\bSigma_{\star}^{1/2}\|_{\fro}\\
 &\le\frac{1}{2}(\|\bDelta_{L}\bSigma_{\star}^{-1/2}\|_{\op}\vee\|\bDelta_{R}\bSigma_{\star}^{-1/2}\|_{\op})\left(\|\bDelta_{L}\bSigma_{\star}^{1/2}\|_{\fro}+\|\bDelta_{R}\bSigma_{\star}^{1/2}\|_{\fro}\right)
\end{align*}
finishes the proof.
\end{proof}

\begin{lemma}\label{lemma:Q_perturbation} For any $\bL\in\RR^{n_{1}\times r},\bR\in\RR^{n_{2}\times r}$ and any invertible matrices $\bQ,\bar{\bQ}\in\GL(r)$, suppose that $\|(\bL\bQ-\bL_{\star})\bSigma_{\star}^{-1/2}\|_{\op}\vee\|(\bR\bQ^{-\top}-\bR_{\star})\bSigma_{\star}^{-1/2}\|_{\op}<1$, then one has
\begin{align*}
\left\Vert \bSigma_{\star}^{1/2}\bar{\bQ}^{-1}\bQ\bSigma_{\star}^{1/2}-\bSigma_{\star}\right\Vert _{\op} & \le\frac{\|\bR(\bar{\bQ}^{-\top}-\bQ^{-\top})\bSigma_{\star}^{1/2}\|_{\op}}{1-\|(\bR\bQ^{-\top}-\bR_{\star})\bSigma_{\star}^{-1/2}\|_{\op}}; \\
\left\Vert \bSigma_{\star}^{1/2}\bar{\bQ}^{\top}\bQ^{-\top}\bSigma_{\star}^{1/2}-\bSigma_{\star}\right\Vert _{\op} & \le\frac{\|\bL(\bar{\bQ}-\bQ)\bSigma_{\star}^{1/2}\|_{\op}}{1-\|(\bL\bQ-\bL_{\star})\bSigma_{\star}^{-1/2}\|_{\op}}.
\end{align*}
\end{lemma}

\begin{proof} Insert $\bR^{\top}\bR(\bR^{\top}\bR)^{-1}$, and use the relation $\|\bA\bB\|_{\op} \le \|\bA\|_{\op}\|\bB\|_{\op}$ to obtain 
\begin{align*}
 \left\Vert \bSigma_{\star}^{1/2}\bar{\bQ}^{-1}\bQ\bSigma_{\star}^{1/2}-\bSigma_{\star}\right\Vert _{\op} & =\left\Vert \bSigma_{\star}^{1/2}(\bar{\bQ}^{-1}-\bQ^{-1})\bR^{\top}\bR(\bR^{\top}\bR)^{-1}\bQ\bSigma_{\star}^{1/2}\right\Vert _{\op}\\
 &  \le\left\Vert \bR(\bar{\bQ}^{-\top}-\bQ^{-\top})\bSigma_{\star}^{1/2}\right\Vert _{\op}\left\Vert \bR(\bR^{\top}\bR)^{-1}\bQ\bSigma_{\star}^{1/2}\right\Vert _{\op}\\
 &  =\left\Vert \bR(\bar{\bQ}^{-\top}-\bQ^{-\top})\bSigma_{\star}^{1/2}\right\Vert _{\op}\left\Vert \bR\bQ^{-\top}((\bR\bQ^{-\top})^{\top}\bR\bQ^{-\top})^{-1}\bSigma_{\star}^{1/2}\right\Vert _{\op}\\
 &  \le\frac{\|\bR(\bar{\bQ}^{-\top}-\bQ^{-\top})\bSigma_{\star}^{1/2}\|_{\op}}{1-\|(\bR\bQ^{-\top}-\bR_{\star})\bSigma_{\star}^{-1/2}\|_{\op}},
\end{align*}
where the last line uses Lemma~\ref{lemma:Weyl}.

Similarly, insert $\bL^{\top}\bL(\bL^{\top}\bL)^{-1}$, and use the relation $\|\bA\bB\|_{\op} \le \|\bA\|_{\op}\|\bB\|_{\op}$ to obtain
\begin{align*}
 \left\Vert \bSigma_{\star}^{1/2}\bar{\bQ}^{\top}\bQ^{-\top}\bSigma_{\star}^{1/2}-\bSigma_{\star}\right\Vert _{\op} &=\left\Vert \bSigma_{\star}^{1/2}(\bar{\bQ}^{\top}-\bQ^{\top})\bL^{\top}\bL(\bL^{\top}\bL)^{-1}\bQ^{-\top}\bSigma_{\star}^{1/2}\right\Vert _{\op}\\
 & \le\left\Vert \bL(\bar{\bQ}-\bQ)\bSigma_{\star}^{1/2}\right\Vert _{\op}\left\Vert \bL(\bL^{\top}\bL)^{-1}\bQ^{-\top}\bSigma_{\star}^{1/2}\right\Vert _{\op}\\
 & =\left\Vert \bL(\bar{\bQ}-\bQ)\bSigma_{\star}^{1/2}\right\Vert _{\op}\left\Vert \bL\bQ((\bL\bQ)^{\top}\bL\bQ)^{-1}\bSigma_{\star}^{1/2}\right\Vert _{\op}\\
 & \le\frac{\|\bL(\bar{\bQ}-\bQ)\bSigma_{\star}^{1/2}\|_{\op}}{1-\|(\bL\bQ-\bL_{\star})\bSigma_{\star}^{-1/2}\|_{\op}},
\end{align*}
where the last line uses Lemma~\ref{lemma:Weyl}. 
\end{proof}

\subsection{Partial Frobenius norm}

We introduce the partial Frobenius norm 
\begin{align}
\|\bX\|_{\fro,r}\coloneqq\sqrt{\sum_{i=1}^r \sigma_{i}^2(\bX)}=\|\cP_{r}(\bX)\|_{\fro}\label{eq:norm_Fr_def}
\end{align}
as the $\ell_{2}$ norm of the vector composed of the top-$r$ singular values of the matrix $\bX$, or equivalently as the Frobenius norm of the rank-$r$ approximation $\cP_{r}(\bX)$ defined in \eqref{eq:rank_r_proj}. 
It is straightforward to verify that $\|\cdot\|_{\fro,r}$ is a norm; see also \cite{mazeika2016singular}. The following lemma provides several equivalent and useful characterizations of this partial Frobenius norm. 
\begin{lemma}\label{lemma:norm_Fr_variation} For any $\bX\in\RR^{n_{1}\times n_{2}}$, one has
\begin{subequations}
\begin{align}
\|\bX\|_{\fro,r} & =\max_{\tilde{\bV}\in\RR^{n_{2}\times r}:\tilde{\bV}^{\top}\tilde{\bV}=\bI_{r}}\;\|\bX\tilde{\bV}\|_{\fro}\label{eq:norm_Fr_variation-1}\\
 & =\max_{\tilde{\bX}\in\RR^{n_{1}\times n_{2}}:\|\tilde{\bX}\|_{\fro}\le 1,\rank(\tilde{\bX})\le r}\;|\langle\bX,\tilde{\bX}\rangle|\label{eq:norm_Fr_variation-2}\\
 & =\max_{\tilde{\bR}\in\RR^{n_{2}\times r}:\|\tilde{\bR}\|_{\op}\le1}\;\|\bX\tilde{\bR}\|_{\fro}\label{eq:norm_Fr_variation-3}.
\end{align}
\end{subequations}
\end{lemma}

\begin{proof} The first representation \eqref{eq:norm_Fr_variation-1} follows immediately from the extremal partial trace identity; see \cite[Proposition~4.4]{mazeika2016singular}, by noticing the following relation
\begin{align*}
\sum_{i=1}^{r}\sigma_{i}^2(\bX) = \max_{\VV\subseteq\RR^{n_{2}}:\dim(\VV)=r}\;\tr\left(\bX^{\top}\bX\mid\VV\right) = \max_{\tilde{\bV}\in\RR^{n_2\times r}:\tilde{\bV}^{\top}\tilde{\bV}=\bI_{r}}\;\|\bX\tilde{\bV}\|_{\fro}^2.
\end{align*}
Here the partial trace over a vector space $\VV$ is defined as
\begin{align*}
\tr(\bX^{\top}\bX\mid\VV)\coloneqq\sum_{i=1}^{r}\tilde{\bv}_{i}^{\top}\bX^{\top}\bX\tilde{\bv}_{i},
\end{align*}
where $\{\tilde{\bv}_{i}\}_{1\le i \le r}$ is any orthonormal basis of $\VV$. The partial trace is invariant to the choice of orthonormal basis and therefore well-defined.

To prove the second representation \eqref{eq:norm_Fr_variation-2}, for any $\tilde{\bX}\in\RR^{n_{1}\times n_{2}}$ obeying $\rank(\tilde{\bX})\le r$ and $\|\tilde{\bX}\|_{\fro}\le 1$, denoting $\tilde{\bX}=\tilde{\bU}\tilde{\bSigma}\tilde{\bV}^{\top}$ as its compact SVD, one has
\begin{align*}
|\langle\bX,\tilde{\bX}\rangle| = |\langle\bX,\tilde{\bU}\tilde{\bSigma}\tilde{\bV}^{\top}\rangle| = |\langle\bX\tilde{\bV},\tilde{\bU}\tilde{\bSigma}\rangle| \le \|\bX\tilde{\bV}\|_{\fro}\|\tilde{\bU}\tilde{\bSigma}\|_{\fro} \le \|\bX\|_{\fro,r},
\end{align*}
where the last inequality follows from \eqref{eq:norm_Fr_variation-1}.
In addition, the maximum in \eqref{eq:norm_Fr_variation-2} is attained at $\tilde{\bX}=\cP_{r}(\bX)/\|\cP_{r}(\bX)\|_{\fro}$.

To prove the third representation \eqref{eq:norm_Fr_variation-3}, for any $\tilde{\bR}\in\RR^{n_{2}\times r}$ obeying $\|\tilde{\bR}\|_{\op}\le 1$, combine the variational representation of the Frobenius norm and \eqref{eq:norm_Fr_variation-2} to obtain
\begin{align*}
\|\bX\tilde{\bR}\|_{\fro} & =\max_{\tilde{\bL}\in\RR^{n_{1}\times n_{2}}:\|\tilde{\bL}\|_{\fro}\le1}\;|\langle\bX\tilde{\bR},\tilde{\bL}\rangle| \\
 & =\max_{\tilde{\bL}\in\RR^{n_{1}\times n_{2}}:\|\tilde{\bL}\|_{\fro}\le1}\;|\langle\bX,\tilde{\bL}\tilde{\bR}^{\top}\rangle|\le\|\bX\|_{\fro,r},
\end{align*}
where the last inequality follows from \eqref{eq:norm_Fr_variation-2}. In addition, the maximum in \eqref{eq:norm_Fr_variation-3} is attained at $\tilde{\bR}=\bV$, where $\bV$ denotes the top-$r$ right singular vectors of $\bX$.
\end{proof}

\begin{remark} For self-completeness, we also provide a detailed proof of the first representation \eqref{eq:norm_Fr_variation-1}. This proof is inductive on $r$. When $r=1$, we have
\begin{align*}
\sigma_{1}(\bX)=\|\bX\bv_{1}\|_{2}=\max_{\tilde{\bv}\in\RR^{n_2}:\|\tilde{\bv}\|_{2}=1}\;\|\bX\tilde{\bv}\|_{2},
\end{align*}
where $\bv_{1}$ denotes the top right singular vector of $\bX$. 
Assume that the statement holds for $\|\cdot\|_{\fro,r-1}$. Now consider $\|\cdot\|_{\fro,r}$. For any $\tilde{\bV}\in\RR^{n_{2}\times r}$ such that $\tilde{\bV}^{\top}\tilde{\bV}=\bI_{r}$, we can first pick $\tilde{\bv}_{2},\dots,\tilde{\bv}_{r}$ as a set of orthonormal vectors in the column space of $\tilde{\bV}$ that are orthogonal to $\bv_{1}$, and then pick $\tilde{\bv}_{1}$ via the Gram-Schmidt process, so that $\{\tilde{\bv}_i\}_{i=1}^r$ provides an orthonormal basis of the column space of $\tilde{\bV}$. Further, by the orthogonality of $\tilde{\bV}$, there exists an orthonormal matrix $\bO$ such that
\begin{align*}
\tilde{\bV}=[\tilde{\bv}_{1},\dots,\tilde{\bv}_{r}]\bO.
\end{align*}
Combining this formula with the induction hypothesis yields
\begin{align*}
\|\bX\tilde{\bV}\|_{\fro}^2 & =\|\bX[\tilde{\bv}_{1},\dots,\tilde{\bv}_{r}]\|_{\fro}^2 \\
& =\|\bX\tilde{\bv}_{1}\|_{2}^2+\|\bX[\tilde{\bv}_{2},\dots,\tilde{\bv}_{r}]\|_{\fro}^2 \\
& =\|\bX\tilde{\bv}_{1}\|_{2}^2+\|(\bX-\cP_{1}(\bX))[\tilde{\bv}_{2},\dots,\tilde{\bv}_{r}]\|_{\fro}^2 \\
& \le\sigma_{1}^2(\bX) + \|\bX-\cP_{1}(\bX)\|_{\fro,r-1}^2 \\
& =\sum_{i=1}^{r}\sigma_{i}^2(\bX)=\|\bX\|_{\fro,r}^2,
\end{align*}
where the first line holds since $\bO$ is orthonormal, the third line holds since $\cP_{1}(\bX)[\tilde{\bv}_{2},\dots,\tilde{\bv}_{r}]=\zero$, the fourth line follows from the induction hypothesis, and the last line follows from the definition \eqref{eq:norm_Fr_def}. 
In addition, the maximum in \eqref{eq:norm_Fr_variation-1} is attained at $\tilde{\bV}=\bV$, where $\bV$ denotes the top-$r$ right singular vectors of $\bX$. This finishes the proof.
\end{remark}

Recall that $\cP_{r}(\bX)$ denotes the best rank-$r$ approximation of $\bX$ under the Frobenius norm. It turns out that $\cP_{r}(\bX)$ is also the best rank-$r$ approximation of $\bX$ under the partial Frobenius norm $\|\cdot\|_{\fro,r}$. This claim is formally stated below; see also \cite[Theorem 4.21]{mazeika2016singular}.
\begin{lemma}\label{lemma:norm_Fr_Eckart-Yang-Mirsky} Fix any $\bX\in\RR^{n_1\times n_2}$ and recall the definition of $\cP_{r}(\bX)$ in \eqref{eq:rank_r_proj}. One has
\begin{align*}
\cP_{r}(\bX)=\argmin_{\tilde{\bX}\in\RR^{n_1\times n_2}:\rank(\tilde{\bX})\le r}\;\|\bX-\tilde{\bX}\|_{\fro,r}.
\end{align*}
\end{lemma}

\begin{proof} For any $\tilde{\bX}$ of rank at most $r$, invoke Weyl's inequality to obtain 
$\sigma_{r+i}(\bX)\le\sigma_{i}(\bX-\tilde{\bX})+\sigma_{r+1}(\tilde{\bX})=\sigma_{i}(\bX-\tilde{\bX})$, for $i=1,\dots,r$. Thus one has
\begin{align*}
\|\bX-\cP_{r}(\bX)\|_{\fro,r}^2=\sum_{i=1}^{r}\sigma_{r+i}^2(\bX)\le\sum_{i=1}^{r}\sigma_{i}^2(\bX-\tilde{\bX})=\|\bX-\tilde{\bX}\|_{\fro,r}^2.
\end{align*}
The proof is finished by observing that the rank of $\cP_{r}(\bX)$ is at most $r$.
\end{proof}

\section{Proof for Low-Rank Matrix Factorization}

\subsection{Proof of Proposition~\ref{prop:vectorize}}\label{subsec:quasi_hessian}

The gradients of $\cL(\bF)$ in \eqref{eq:loss_MF} with respect to $\bL$ and $\bR$ are given as 
\begin{align*}
\nabla_{\bL}\cL(\bF)=(\bL\bR^{\top}-\bX_{\star})\bR,\quad \nabla_{\bR}\cL(\bF)=(\bL\bR^{\top}-\bX_{\star})^{\top}\bL,
\end{align*}
which can be used to compute the Hessian with respect to $\bL$ and $\bR$. Writing for the vectorized variables, the Hessians are given as 
\begin{align*}
\nabla_{\bL,\bL}^{2}\cL(\bF)=(\bR^{\top}\bR)\otimes\bI_{n_{1}},\quad \nabla_{\bR,\bR}^{2}\cL(\bF)=(\bL^{\top}\bL)\otimes\bI_{n_{2}}.
\end{align*}
Viewed in the vectorized form, the \texttt{ScaledGD} update in \eqref{eq:scaledGD} can be rewritten as 
\begin{align*}
\vc(\bL_{t+1}) & =\vc(\bL_{t})-\eta((\bR_{t}^{\top}\bR_{t})^{-1}\otimes\bI_{n_{1}})\vc((\bL_{t}\bR_{t}^{\top}-\bX_{\star})\bR_{t})\\
 & =\vc(\bL_{t})-\eta(\nabla_{\bL,\bL}^{2}\cL(\bF_{t}))^{-1}\vc(\nabla_{\bL}\cL(\bF_{t})),\\
\vc(\bR_{t+1}) & =\vc(\bR_{t})-\eta((\bL_{t}^{\top}\bL_{t})^{-1}\otimes\bI_{n_{2}})\vc((\bL_{t}\bR_{t}^{\top}-\bX_{\star})^{\top}\bL_{t})\\
 & =\vc(\bR_{t})-\eta(\nabla_{\bR,\bR}^{2}\cL(\bF_{t}))^{-1}\vc(\nabla_{\bR}\cL(\bF_{t})).
\end{align*}

\subsection{Proof of Theorem~\ref{thm:MF}}\label{subsec:proof_MF}

The proof is inductive in nature. More specifically, we intend to show that for all $t\ge0$, 
\begin{enumerate}
\item $\dist(\bF_{t},\bF_{\star})\le(1-0.7\eta)^{t}\dist(\bF_{0},\bF_{\star})\le0.1(1-0.7\eta)^{t}\sigma_{r}(\bX_{\star})$, and
\item the optimal alignment matrix $\bQ_{t}$ between $\bF_{t}$ and $\bF_{\star}$ exists.
\end{enumerate}
For the base case, i.e.~$t=0$, the first induction hypothesis trivially holds, while the second also holds true in view of Lemma~\ref{lemma:Q_existence} and the assumption that $\dist(\bF_{0},\bF_{\star})\le0.1\sigma_{r}(\bX_{\star})$.
We therefore concentrate on the induction step. Suppose that the $t$-th iterate $\bF_{t}$ obeys the aforementioned induction hypotheses. Our goal is to show that $\bF_{t+1}$ continues to satisfy those.

For notational convenience, denote $\bL\coloneqq\bL_{t}\bQ_{t}$, $\bR\coloneqq\bR_{t}\bQ_{t}^{-\top}$, $\bDelta_{L}\coloneqq\bL-\bL_{\star}$, $\bDelta_{R}\coloneqq\bR-\bR_{\star}$, and $\epsilon\coloneqq0.1$. By the definition of $\dist(\bF_{t+1},\bF_{\star})$, one has
\begin{align}
\dist^{2}(\bF_{t+1},\bF_{\star}) & \le\left\Vert (\bL_{t+1}\bQ_{t}-\bL_{\star})\bSigma_{\star}^{1/2}\right\Vert _{\fro}^{2}+\left\Vert (\bR_{t+1}\bQ_{t}^{-\top}-\bR_{\star})\bSigma_{\star}^{1/2}\right\Vert _{\fro}^{2},\label{eq:MF_expand}
\end{align}
where we recall that $\bQ_{t}$ is the optimal alignment matrix between $\bF_{t}$ and $\bF_{\star}$. Utilize the \texttt{ScaledGD} update rule \eqref{eq:scaledGD_MF} and the decomposition $\bL\bR^{\top}-\bX_{\star}=\bDelta_{L}\bR^{\top}+\bL_{\star}\bDelta_{R}^{\top}$ to obtain 
\begin{align*}
(\bL_{t+1}\bQ_{t}-\bL_{\star})\bSigma_{\star}^{1/2} & =\left(\bL-\eta(\bL\bR^{\top}-\bX_{\star})\bR(\bR^{\top}\bR)^{-1}-\bL_{\star}\right)\bSigma_{\star}^{1/2}\\
 & =\left(\bDelta_{L}-\eta(\bDelta_{L}\bR^{\top}+\bL_{\star}\bDelta_{R}^{\top})\bR(\bR^{\top}\bR)^{-1}\right)\bSigma_{\star}^{1/2}\\
 & =(1-\eta)\bDelta_{L}\bSigma_{\star}^{1/2}-\eta\bL_{\star}\bDelta_{R}^{\top}\bR(\bR^{\top}\bR)^{-1}\bSigma_{\star}^{1/2}.
\end{align*}
As a result, one can expand the first square in \eqref{eq:MF_expand} as
\begin{align}
\left\Vert (\bL_{t+1}\bQ_{t}-\bL_{\star})\bSigma_{\star}^{1/2}\right\Vert _{\fro}^{2} & =(1-\eta)^{2}\tr\left(\bDelta_{L}\bSigma_{\star}\bDelta_{L}^{\top}\right)-2\eta(1-\eta)\underbrace{\tr\left(\bL_{\star}\bDelta_{R}^{\top}\bR(\bR^{\top}\bR)^{-1}\bSigma_{\star}\bDelta_{L}^{\top}\right)}_{\mfk{M}_{1}}\nonumber \\
 & \quad+\eta^{2}\underbrace{\left\Vert \bL_{\star}\bDelta_{R}^{\top}\bR(\bR^{\top}\bR)^{-1}\bSigma_{\star}^{1/2}\right\Vert _{\fro}^{2}}_{\mfk{M}_{2}}.\label{eq:MF_Lt}
\end{align}
The first term $\tr(\bDelta_{L}\bSigma_{\star}\bDelta_{L}^{\top})$ is closely related to $\dist(\bF_{t},\bF_{\star})$, and hence our focus will be on relating $\mfk{M}_{1}$ and $\mfk{M}_{2}$ to $\dist(\bF_{t},\bF_{\star})$.
We start with the term $\mfk{M}_{1}$. Since $\bL$ and $\bR$ are aligned with $\bL_{\star}$ and $\bR_{\star}$, Lemma~\ref{lemma:Q_criterion} tells that $\bSigma_{\star}\bDelta_{L}^{\top}\bL=\bR^{\top}\bDelta_{R}\bSigma_{\star}$.
This together with $\bL_{\star}=\bL-\bDelta_{L}$ allows us to rewrite $\mfk{M}_{1}$ as 
\begin{align*}
\mfk{M}_{1} & =\tr\left(\bR(\bR^{\top}\bR)^{-1}\bSigma_{\star}\bDelta_{L}^{\top}\bL_{\star}\bDelta_{R}^{\top}\right)\\
 & =\tr\left(\bR(\bR^{\top}\bR)^{-1}\bSigma_{\star}\bDelta_{L}^{\top}\bL\bDelta_{R}^{\top}\right)-\tr\left(\bR(\bR^{\top}\bR)^{-1}\bSigma_{\star}\bDelta_{L}^{\top}\bDelta_{L}\bDelta_{R}^{\top}\right)\\
 & =\tr\left(\bR(\bR^{\top}\bR)^{-1}\bR^{\top}\bDelta_{R}\bSigma_{\star}\bDelta_{R}^{\top}\right)-\tr\left(\bR(\bR^{\top}\bR)^{-1}\bSigma_{\star}\bDelta_{L}^{\top}\bDelta_{L}\bDelta_{R}^{\top}\right).
\end{align*}
Moving on to $\mfk{M}_{2}$, we can utilize the fact $\bL_{\star}^{\top}\bL_{\star}=\bSigma_{\star}$ and the decomposition $\bSigma_{\star}=\bR^{\top}\bR-(\bR^{\top}\bR-\bSigma_{\star})$ to obtain 
\begin{align*}
\mfk{M}_{2} & =\tr\left(\bR(\bR^{\top}\bR)^{-1}\bSigma_{\star}(\bR^{\top}\bR)^{-1}\bR^{\top}\bDelta_{R}\bSigma_{\star}\bDelta_{R}^{\top}\right)\\
 & =\tr\left(\bR(\bR^{\top}\bR)^{-1}\bR^{\top}\bDelta_{R}\bSigma_{\star}\bDelta_{R}^{\top}\right)-\tr\left(\bR(\bR^{\top}\bR)^{-1}(\bR^{\top}\bR-\bSigma_{\star})(\bR^{\top}\bR)^{-1}\bR^{\top}\bDelta_{R}\bSigma_{\star}\bDelta_{R}^{\top}\right).
\end{align*}
Putting $\mfk{M}_{1}$ and $\mfk{M}_{2}$ back to \eqref{eq:MF_Lt} yields 
\begin{align*}
\left\Vert (\bL_{t+1}\bQ_{t}-\bL_{\star})\bSigma_{\star}^{1/2}\right\Vert _{\fro}^{2} & =(1-\eta)^{2}\tr\left(\bDelta_{L}\bSigma_{\star}\bDelta_{L}^{\top}\right)-\eta(2-3\eta)\underbrace{\tr\left(\bR(\bR^{\top}\bR)^{-1}\bR^{\top}\bDelta_{R}\bSigma_{\star}\bDelta_{R}^{\top}\right)}_{\mfk{F}_{1}}\\
 & \quad+2\eta(1-\eta)\underbrace{\tr\left(\bR(\bR^{\top}\bR)^{-1}\bSigma_{\star}\bDelta_{L}^{\top}\bDelta_{L}\bDelta_{R}^{\top}\right)}_{\mfk{F}_{2}}\\
 & \quad-\eta^{2}\underbrace{\tr\left(\bR(\bR^{\top}\bR)^{-1}(\bR^{\top}\bR-\bSigma_{\star})(\bR^{\top}\bR)^{-1}\bR^{\top}\bDelta_{R}\bSigma_{\star}\bDelta_{R}^{\top}\right)}_{\mfk{F}_{3}}.
\end{align*}
In what follows, we will control the three terms $\mfk{F}_{1},\mfk{F}_{2}$ and $\mfk{F}_{3}$ separately. 
\begin{enumerate}
\item Notice that $\mfk{F}_{1}$ is the inner product of two positive semi-definite matrices $\bR(\bR^{\top}\bR)^{-1}\bR^{\top}$ and $\bDelta_{R}\bSigma_{\star}\bDelta_{R}^{\top}$. Consequently we have $\mfk{F}_{1}\ge0$.
\item To control $\mfk{F}_{2}$, we need certain control on $\|\bDelta_{L}\bSigma_{\star}^{-1/2}\|_{\op}$ and $\|\bDelta_{R}\bSigma_{\star}^{-1/2}\|_{\op}$. The first induction hypothesis 
\begin{align*}
\dist(\bF_{t},\bF_{\star})=\sqrt{\|\bDelta_{L}\bSigma_{\star}^{-1/2}\bSigma_{\star}\|_{\fro}^{2}+\|\bDelta_{R}\bSigma_{\star}^{-1/2}\bSigma_{\star}\|_{\fro}^{2}} & \le\epsilon\sigma_{r}(\bX_{\star})
\end{align*}
together with the relation $\|\bA\bB\|_{\fro}\ge\|\bA\|_{\fro}\sigma_{r}(\bB)$ tells that 
\begin{align*}
\sqrt{\|\bDelta_{L}\bSigma_{\star}^{-1/2}\|_{\fro}^{2}+\|\bDelta_{R}\bSigma_{\star}^{-1/2}\|_{\fro}^{2}}\;\sigma_{r}(\bX_{\star})\le\epsilon\sigma_{r}(\bX_{\star}).
\end{align*}
In light of the relation $\|\bA\|_{\op}\le\|\bA\|_{\fro}$, this further implies 
\begin{align}
\|\bDelta_{L}\bSigma_{\star}^{-1/2}\|_{\op}\vee\|\bDelta_{R}\bSigma_{\star}^{-1/2}\|_{\op}\le\epsilon.\label{eq:cond_MF}
\end{align}
Invoke Lemma~\ref{lemma:Weyl} to see 
\begin{align*}
\left\Vert \bR(\bR^{\top}\bR)^{-1}\bSigma_{\star}^{1/2}\right\Vert _{\op} & \le \frac{1}{1-\epsilon}.
\end{align*}
With these consequences, one can bound $|\mfk{F}_{2}|$ by
\begin{align*}
|\mfk{F}_{2}| & =\Big|\tr\left(\bSigma_{\star}^{-1/2}\bDelta_{R}^{\top}\bR(\bR^{\top}\bR)^{-1}\bSigma_{\star}\bDelta_{L}^{\top}\bDelta_{L}\bSigma_{\star}^{1/2}\right)\Big|\\
 & \le\left\Vert \bSigma_{\star}^{-1/2}\bDelta_{R}^{\top}\bR(\bR^{\top}\bR)^{-1}\bSigma_{\star}^{1/2}\right\Vert _{\op}\tr\left(\bSigma_{\star}^{1/2}\bDelta_{L}^{\top}\bDelta_{L}\bSigma_{\star}^{1/2}\right)\\
 & \le\|\bDelta_{R}\bSigma_{\star}^{-1/2}\|_{\op}\left\Vert \bR(\bR^{\top}\bR)^{-1}\bSigma_{\star}^{1/2}\right\Vert _{\op}\tr\left(\bDelta_{L}\bSigma_{\star}\bDelta_{L}^{\top}\right)\\
 & \le\frac{\epsilon}{1-\epsilon}\tr\left(\bDelta_{L}\bSigma_{\star}\bDelta_{L}^{\top}\right).
\end{align*}

\item Similarly, one can bound $|\mfk{F}_{3}|$ by 
\begin{align*}
|\mfk{F}_{3}| & \le\left\Vert \bR(\bR^{\top}\bR)^{-1}(\bR^{\top}\bR-\bSigma_{\star})(\bR^{\top}\bR)^{-1}\bR^{\top}\right\Vert _{\op}\tr\left(\bDelta_{R}\bSigma_{\star}\bDelta_{R}^{\top}\right)\\
 & \le\left\Vert \bR(\bR^{\top}\bR)^{-1}\bSigma_{\star}^{1/2}\right\Vert _{\op}^{2}\left\Vert \bSigma_{\star}^{-1/2}(\bR^{\top}\bR-\bSigma_{\star})\bSigma_{\star}^{-1/2}\right\Vert _{\op}\tr\left(\bDelta_{R}\bSigma_{\star}\bDelta_{R}^{\top}\right)\\
 & \le\frac{1}{(1-\epsilon)^{2}}\left\Vert \bSigma_{\star}^{-1/2}(\bR^{\top}\bR-\bSigma_{\star})\bSigma_{\star}^{-1/2}\right\Vert _{\op}\tr\left(\bDelta_{R}\bSigma_{\star}\bDelta_{R}^{\top}\right).
\end{align*}
Further notice that 
\begin{align*}
\left\Vert \bSigma_{\star}^{-1/2}(\bR^{\top}\bR-\bSigma_{\star})\bSigma_{\star}^{-1/2}\right\Vert _{\op} & =\left\Vert \bSigma_{\star}^{-1/2}(\bR_{\star}^{\top}\bDelta_{R}+\bDelta_{R}^{\top}\bR_{\star}+\bDelta_{R}^{\top}\bDelta_{R})\bSigma_{\star}^{-1/2}\right\Vert _{\op}\\
 & \le2\|\bDelta_{R}\bSigma_{\star}^{-1/2}\|_{\op}+\|\bDelta_{R}\bSigma_{\star}^{-1/2}\|_{\op}^{2}\\
 & \le2\epsilon+\epsilon^{2}.
\end{align*}
Take the preceding two bounds together to arrive at 
\begin{align*}
|\mfk{F}_{3}|\le\frac{2\epsilon+\epsilon^{2}}{(1-\epsilon)^{2}}\tr\left(\bDelta_{R}\bSigma_{\star}\bDelta_{R}^{\top}\right).
\end{align*}
\end{enumerate}

Combining the bounds for $\mfk{F}_{1},\mfk{F}_{2},\mfk{F}_{3}$, one has 
\begin{align}
 & \left\Vert (\bL_{t+1}\bQ_{t}-\bL_{\star})\bSigma_{\star}^{1/2}\right\Vert _{\fro}^{2}=\left\Vert (1-\eta)\bDelta_{L}\bSigma_{\star}^{1/2}-\eta\bL_{\star}\bDelta_{R}^{\top}\bR(\bR^{\top}\bR)^{-1}\bSigma_{\star}^{1/2}\right\Vert _{\fro}^{2}\nonumber \\
 & \quad\le\left((1-\eta)^{2}+\frac{2\epsilon}{1-\epsilon}\eta(1-\eta)\right)\tr\left(\bDelta_{L}\bSigma_{\star}\bDelta_{L}^{\top}\right)+\frac{2\epsilon+\epsilon^{2}}{(1-\epsilon)^{2}}\eta^{2}\tr\left(\bDelta_{R}\bSigma_{\star}\bDelta_{R}^{\top}\right).\label{eq:MF_Lt_bound}
\end{align}
A similarly bound holds for the second square $\|(\bR_{t+1}\bQ_{t}-\bR_{\star})\bSigma_{\star}^{1/2}\|_{\fro}^{2}$ in \eqref{eq:MF_expand}. Therefore we obtain 
\begin{align*}
 & \left\Vert (\bL_{t+1}\bQ_{t}-\bL_{\star})\bSigma_{\star}^{1/2}\right\Vert _{\fro}^{2}+\left\Vert (\bR_{t+1}\bQ_{t}^{-\top}-\bR_{\star})\bSigma_{\star}^{1/2}\right\Vert _{\fro}^{2}\le\rho^{2}(\eta;\epsilon)\dist^{2}(\bF_{t},\bF_{\star}),
\end{align*}
where we identify 
\begin{align}
\dist^{2}(\bF_{t},\bF_{\star})=\tr(\bDelta_{L}\bSigma_{\star}\bDelta_{L}^{\top})+\tr(\bDelta_{R}\bSigma_{\star}\bDelta_{R}^{\top})\label{eq:dist_Ft}
\end{align}
and the contraction rate $\rho^{2}(\eta;\epsilon)$ is given by 
\begin{align*}
 & \rho^{2}(\eta;\epsilon)\coloneqq(1-\eta)^{2}+\frac{2\epsilon}{1-\epsilon}\eta(1-\eta)+\frac{2\epsilon+\epsilon^{2}}{(1-\epsilon)^{2}}\eta^2.
\end{align*}
With $\epsilon=0.1$ and $0<\eta\le2/3$, one has $\rho(\eta;\epsilon)\le1-0.7\eta$. Thus we conclude that 
\begin{align*}
\dist(\bF_{t+1},\bF_{\star}) & \le\sqrt{\left\Vert (\bL_{t+1}\bQ_{t}-\bL_{\star})\bSigma_{\star}^{1/2}\right\Vert _{\fro}^{2}+\left\Vert (\bR_{t+1}\bQ_{t}^{-\top}-\bR_{\star})\bSigma_{\star}^{1/2}\right\Vert _{\fro}^{2}}\\
 & \le(1-0.7\eta)\dist(\bF_{t},\bF_{\star})\\
 & \le(1-0.7\eta)^{t+1}\dist(\bF_{0},\bF_{\star})\le(1-0.7\eta)^{t+1}0.1\sigma_{r}(\bX_{\star}).
\end{align*}
This proves the first induction hypothesis. The existence of the optimal alignment matrix $\bQ_{t+1}$ between $\bF_{t+1}$ and $\bF_{\star}$ is assured by Lemma~\ref{lemma:Q_existence}, which finishes the proof for the second hypothesis.

So far, we have demonstrated the first conclusion in the theorem. The second conclusion is an easy consequence of Lemma~\ref{lemma:matrix2factor} as
\begin{align}
\begin{split}
\left\Vert \bL_{t}\bR_{t}^{\top}-\bX_{\star}\right\Vert _{\fro} & \le\left(1+\frac{\epsilon}{2}\right)\left(\|\bDelta_{L}\bSigma_{\star}^{1/2}\|_{\fro}+\|\bDelta_{R}\bSigma_{\star}^{1/2}\|_{\fro}\right)\\
 & \le\left(1+\frac{\epsilon}{2}\right)\sqrt{2}\dist(\bF_{t},\bF_{\star})\\
 & \le1.5\dist(\bF_{t},\bF_{\star}).
 \end{split}\label{eq:dist_matrix}
\end{align}
Here, the second line follows from the elementary inequality $a+b\le\sqrt{2(a^{2}+b^{2})}$ and the expression of $\dist(\bF_{t},\bF_{\star})$ in \eqref{eq:dist_Ft}.
The proof is now completed.

%% file: appendix-sensing.tex
\section{Proof for Low-Rank Matrix Sensing}\label{sec:proof_MS}

We start by recording a useful lemma. 
\begin{lemma}[\cite{candes2011tight}]\label{lemma:RIP_distance} Suppose that $\cA(\cdot)$ obeys the $2r$-RIP with a constant $\delta_{2r}$. Then for any $\bX_{1},\bX_{2}\in\RR^{n_{1}\times n_{2}}$ of rank at most $r$, one has
\begin{align*}
\left|\langle\cA(\bX_{1}),\cA(\bX_{2})\rangle-\langle\bX_{1},\bX_{2}\rangle\right|\le\delta_{2r}\|\bX_{1}\|_{\fro}\|\bX_{2}\|_{\fro},
\end{align*}
which can be stated equivalently as 
\begin{align}
\left|\tr\left((\cA^{*}\cA-\cI)(\bX_{1})\bX_{2}^{\top}\right)\right|\le\delta_{2r}\|\bX_{1}\|_{\fro}\|\bX_{2}\|_{\fro}.\label{eq:2r-RIP}
\end{align}
\end{lemma} 
As a simple corollary, one has that for any matrix $\bR\in\RR^{n_{2}\times r}$:
\begin{align}
\left\Vert (\cA^{*}\cA-\cI)(\bX_{1})\bR\right\Vert _{\fro}\le\delta_{2r}\|\bX_{1}\|_{\fro}\|\bR\|_{\op}.\label{eq:coro-RIP}
\end{align}
This is due to the fact that 
\begin{align*}
\left\Vert (\cA^{*}\cA-\cI)(\bX_{1})\bR\right\Vert _{\fro} & =\max_{\tilde{\bL}:\|\tilde{\bL}\|_{\fro}\le1}\;\tr\left((\cA^{*}\cA-\cI)(\bX_{1})\bR\tilde{\bL}^{\top}\right)\\
 & \le\max_{\tilde{\bL}:\|\tilde{\bL}\|_{\fro}\le1}\;\delta_{2r}\|\bX_{1}\|_{\fro}\|\tilde{\bL}\bR^{\top}\|_{\fro}\\
 & \le\delta_{2r}\|\bX_{1}\|_{\fro}\|\bR\|_{\op}.
\end{align*}
Here, the first line follows from the variational representation of the Frobenius norm, the second line follows from \eqref{eq:2r-RIP}, and the last line follows from the relation $\|\bA\bB\|_{\fro}\le\|\bA\|_{\fro}\|\bB\|_{\op}$.

\subsection{Proof of Lemma~\ref{lemma:contraction_MS}}

The proof mostly mirrors that in Section~\ref{subsec:proof_MF}. 
First, in view of the condition $\dist(\bF_{t},\bF_{\star})\le0.1\sigma_{r}(\bX_{\star})$ and Lemma~\ref{lemma:Q_existence}, one knows that $\bQ_{t}$, the optimal alignment matrix between $\bF_{t}$ and $\bF_{\star}$ exists.
Therefore, for notational convenience, denote $\bL\coloneqq\bL_{t}\bQ_{t}$, $\bR\coloneqq\bR_{t}\bQ_{t}^{-\top}$, $\bDelta_{L}\coloneqq\bL-\bL_{\star}$, $\bDelta_{R}\coloneqq\bR-\bR_{\star}$, and $\epsilon\coloneqq0.1$. Similar to the derivation in \eqref{eq:cond_MF}, we have 
\begin{align}
\|\bDelta_{L}\bSigma_{\star}^{-1/2}\|_{\op}\vee\|\bDelta_{R}\bSigma_{\star}^{-1/2}\|_{\op}\le\epsilon.\label{eq:cond_MS}
\end{align}
The conclusion $\|\bL_{t}\bR_{t}^{\top}-\bX_{\star}\|_{\fro} \le 1.5\dist(\bF_{t},\bF_{\star})$ is a simple consequence of Lemma~\ref{lemma:matrix2factor}; see \eqref{eq:dist_matrix} for a detailed argument. From now on, we focus on proving the distance contraction.

With these notations in place, we have by the definition of $\dist(\bF_{t+1},\bF_{\star})$ that 
\begin{align}
\dist^{2}(\bF_{t+1},\bF_{\star}) & \le\left\Vert (\bL_{t+1}\bQ_{t}-\bL_{\star})\bSigma_{\star}^{1/2}\right\Vert _{\fro}^{2}+\left\Vert (\bR_{t+1}\bQ_{t}^{-\top}-\bR_{\star})\bSigma_{\star}^{1/2}\right\Vert _{\fro}^{2}.\label{eq:MS_expand}
\end{align}
Apply the update rule \eqref{eq:iterates_MS} and the decomposition $\bL\bR^{\top}-\bX_{\star}=\bDelta_{L}\bR^{\top}+\bL_{\star}\bDelta_{R}^{\top}$ to obtain 
\begin{align*}
 & (\bL_{t+1}\bQ_{t}-\bL_{\star})\bSigma_{\star}^{1/2}=\left(\bL-\eta\cA^{*}\cA(\bL\bR^{\top}-\bX_{\star})\bR(\bR^{\top}\bR)^{-1}-\bL_{\star}\right)\bSigma_{\star}^{1/2}\\
 & \quad=\left(\bDelta_{L}-\eta(\bL\bR^{\top}-\bX_{\star})\bR(\bR^{\top}\bR)^{-1}-\eta(\cA^{*}\cA-\cI)(\bL\bR^{\top}-\bX_{\star})\bR(\bR^{\top}\bR)^{-1}\right)\bSigma_{\star}^{1/2}\\
 & \quad=(1-\eta)\bDelta_{L}\bSigma_{\star}^{1/2}-\eta\bL_{\star}\bDelta_{R}^{\top}\bR(\bR^{\top}\bR)^{-1}\bSigma_{\star}^{1/2}-\eta(\cA^{*}\cA-\cI)(\bL\bR^{\top}-\bX_{\star})\bR(\bR^{\top}\bR)^{-1}\bSigma_{\star}^{1/2}.
\end{align*}
This allows us to expand the first square in \eqref{eq:MS_expand} as 
\begin{align*}
\left\Vert (\bL_{t+1}\bQ_{t}-\bL_{\star})\bSigma_{\star}^{1/2}\right\Vert _{\fro}^{2} & =\underbrace{\left\Vert (1-\eta)\bDelta_{L}\bSigma_{\star}^{1/2}-\eta\bL_{\star}\bDelta_{R}^{\top}\bR(\bR^{\top}\bR)^{-1}\bSigma_{\star}^{1/2}\right\Vert _{\fro}^{2}}_{\mfk{S}_{1}}\\
 & \quad-2\eta(1-\eta)\underbrace{\tr\left((\cA^{*}\cA-\cI)(\bL\bR^{\top}-\bX_{\star})\bR(\bR^{\top}\bR)^{-1}\bSigma_{\star}\bDelta_{L}^{\top}\right)}_{\mfk{S}_{2}}\\
 & \quad+2\eta^{2}\underbrace{\tr\left((\cA^{*}\cA-\cI)(\bL\bR^{\top}-\bX_{\star})\bR(\bR^{\top}\bR)^{-1}\bSigma_{\star}(\bR^{\top}\bR)^{-1}\bR^{\top}\bDelta_{R}\bL_{\star}^{\top}\right)}_{\mfk{S}_{3}}\\
 & \quad+\eta^{2}\underbrace{\left\Vert (\cA^{*}\cA-\cI)(\bL\bR^{\top}-\bX_{\star})\bR(\bR^{\top}\bR)^{-1}\bSigma_{\star}^{1/2}\right\Vert _{\fro}^{2}}_{\mfk{S}_{4}}.
\end{align*}
In what follows, we shall control the four terms separately, of which $\mfk{S}_{1}$ is the main term, and $\mfk{S}_{2},\mfk{S}_{3}$ and $\mfk{S}_{4}$ are perturbation terms. 
\begin{enumerate}
\item Notice that the main term $\mfk{S}_{1}$ has already been controlled in \eqref{eq:MF_Lt_bound} under the condition \eqref{eq:cond_MS}. It obeys
\begin{align*}
\mfk{S}_{1}\le\left((1-\eta)^{2}+\frac{2\epsilon}{1-\epsilon}\eta(1-\eta)\right)\|\bDelta_{L}\bSigma_{\star}^{1/2}\|_{\fro}^{2}+\frac{2\epsilon+\epsilon^{2}}{(1-\epsilon)^{2}}\eta^{2}\|\bDelta_{R}\bSigma_{\star}^{1/2}\|_{\fro}^{2}.
\end{align*}
\item For the second term $\mfk{S}_{2}$, decompose $\bL\bR^{\top}-\bX_{\star}=\bDelta_{L}\bR_{\star}^{\top}+\bL_{\star}\bDelta_{R}^{\top}+\bDelta_{L}\bDelta_{R}^{\top}$ and apply the triangle inequality to obtain
\begin{align*}
|\mfk{S}_{2}| & =\Big|\tr\left((\cA^{*}\cA-\cI)(\bDelta_{L}\bR_{\star}^{\top}+\bL_{\star}\bDelta_{R}^{\top}+\bDelta_{L}\bDelta_{R}^{\top})\bR(\bR^{\top}\bR)^{-1}\bSigma_{\star}\bDelta_{L}^{\top}\right)\Big|\\
 & \le\Big|\tr\left((\cA^{*}\cA-\cI)(\bDelta_{L}\bR_{\star}^{\top})\bR(\bR^{\top}\bR)^{-1}\bSigma_{\star}\bDelta_{L}^{\top}\right)\Big|\\
 & \quad+\Big|\tr\left((\cA^{*}\cA-\cI)(\bL_{\star}\bDelta_{R}^{\top})\bR(\bR^{\top}\bR)^{-1}\bSigma_{\star}\bDelta_{L}^{\top}\right)\Big|\\
 & \quad+\Big|\tr\left((\cA^{*}\cA-\cI)(\bDelta_{L}\bDelta_{R}^{\top})\bR(\bR^{\top}\bR)^{-1}\bSigma_{\star}\bDelta_{L}^{\top}\right)\Big|.
\end{align*}
Invoke Lemma~\ref{lemma:RIP_distance} to further obtain 
\begin{align*}
|\mfk{S}_{2}| & \le\delta_{2r}\left(\|\bDelta_{L}\bR_{\star}^{\top}\|_{\fro}+\|\bL_{\star}\bDelta_{R}^{\top}\|_{\fro}+\|\bDelta_{L}\bDelta_{R}^{\top}\|_{\fro}\right)\left\Vert \bR(\bR^{\top}\bR)^{-1}\bSigma_{\star}\bDelta_{L}^{\top}\right\Vert _{\fro}\\
 & \le\delta_{2r}\left(\|\bDelta_{L}\bR_{\star}^{\top}\|_{\fro}+\|\bL_{\star}\bDelta_{R}^{\top}\|_{\fro}+\|\bDelta_{L}\bDelta_{R}^{\top}\|_{\fro}\right)\left\Vert \bR(\bR^{\top}\bR)^{-1}\bSigma_{\star}^{1/2}\right\Vert _{\op}\|\bDelta_{L}\bSigma_{\star}^{1/2}\|_{\fro},
\end{align*}
where the second line follows from the relation $\|\bA\bB\|_{\fro}\le\|\bA\|_{\op}\|\bB\|_{\fro}$. Take the condition \eqref{eq:cond_MS} and Lemmas~\ref{lemma:Weyl} and \ref{lemma:matrix2factor} together to obtain 
\begin{align*}
\left\Vert \bR(\bR^{\top}\bR)^{-1}\bSigma_{\star}^{1/2}\right\Vert _{\op} & \le \frac{1}{1-\epsilon}; \\
\|\bDelta_{L}\bR_{\star}^{\top}\|_{\fro}+\|\bL_{\star}\bDelta_{R}^{\top}\|_{\fro}+\|\bDelta_{L}\bDelta_{R}^{\top}\|_{\fro} & \le(1+\frac{\epsilon}{2})\left(\|\bDelta_{L}\bSigma_{\star}^{1/2}\|_{\fro}+\|\bDelta_{R}\bSigma_{\star}^{1/2}\|_{\fro}\right).
\end{align*}
These consequences further imply that
\begin{align*}
|\mfk{S}_{2}| & \le\frac{\delta_{2r}(2+\epsilon)}{2(1-\epsilon)}\left(\|\bDelta_{L}\bSigma_{\star}^{1/2}\|_{\fro}+\|\bDelta_{R}\bSigma_{\star}^{1/2}\|_{\fro}\right)\|\bDelta_{L}\bSigma_{\star}^{1/2}\|_{\fro}\\
 & =\frac{\delta_{2r}(2+\epsilon)}{2(1-\epsilon)}\left(\|\bDelta_{L}\bSigma_{\star}^{1/2}\|_{\fro}^{2}+\|\bDelta_{L}\bSigma_{\star}^{1/2}\|_{\fro}\|\bDelta_{R}\bSigma_{\star}^{1/2}\|_{\fro}\right).
\end{align*}
For the term $\|\bDelta_{L}\bSigma_{\star}^{1/2}\|_{\fro}\|\bDelta_{R}\bSigma_{\star}^{1/2}\|_{\fro}$, we can apply the elementary inequality $2ab\le a^{2}+b^{2}$ to see
\begin{align*}
\|\bDelta_{L}\bSigma_{\star}^{1/2}\|_{\fro}\|\bDelta_{R}\bSigma_{\star}^{1/2}\|_{\fro}\le\frac{1}{2}\|\bDelta_{L}\bSigma_{\star}^{1/2}\|_{\fro}^{2}+\frac{1}{2}\|\bDelta_{R}\bSigma_{\star}^{1/2}\|_{\fro}^{2}.
\end{align*}
The preceding two bounds taken collectively yield 
\begin{align*}
|\mfk{S}_{2}| & \le\frac{\delta_{2r}(2+\epsilon)}{2\left(1-\epsilon\right)}\left(\frac{3}{2}\|\bDelta_{L}\bSigma_{\star}^{1/2}\|_{\fro}^{2}+\frac{1}{2}\|\bDelta_{L}\bSigma_{\star}^{1/2}\|_{\fro}^{2}\right).
\end{align*}

\item The third term $\mfk{S}_{3}$ can be similarly bounded as 
\begin{align*}
|\mfk{S}_{3}| & \le \delta_{2r}\left(\|\bDelta_{L}\bR_{\star}^{\top}\|_{\fro}+\|\bL_{\star}\bDelta_{R}^{\top}\|_{\fro}+\|\bDelta_{L}\bDelta_{R}^{\top}\|_{\fro}\right)\left\Vert \bR(\bR^{\top}\bR)^{-1}\bSigma_{\star}(\bR^{\top}\bR)^{-1}\bR^{\top}\bDelta_{R}\bL_{\star}^{\top}\right\Vert _{\fro}\\
 & \le\delta_{2r}\left(\|\bDelta_{L}\bR_{\star}^{\top}\|_{\fro}+\|\bL_{\star}\bDelta_{R}^{\top}\|_{\fro}+\|\bDelta_{L}\bDelta_{R}^{\top}\|_{\fro}\right)\left\Vert \bR(\bR^{\top}\bR)^{-1}\bSigma_{\star}^{1/2}\right\Vert _{\op}^2\|\bDelta_{R}\bL_{\star}^{\top}\|_{\fro}\\
 & \le\frac{\delta_{2r}(2+\epsilon)}{2(1-\epsilon)^{2}}\left(\|\bDelta_{L}\bSigma_{\star}^{1/2}\|_{\fro}+\|\bDelta_{R}\bSigma_{\star}^{1/2}\|_{\fro}\right)\|\bDelta_{R}\bSigma_{\star}^{1/2}\|_{\fro}\\
 & \le\frac{\delta_{2r}(2+\epsilon)}{2(1-\epsilon)^{2}}\left(\frac{1}{2}\|\bDelta_{L}\bSigma_{\star}^{1/2}\|_{\fro}^{2}+\frac{3}{2}\|\bDelta_{R}\bSigma_{\star}^{1/2}\|_{\fro}^{2}\right).
\end{align*}
\item We are then left with the last term $\mfk{S}_{4}$, for which we have 
\begin{align*}
\sqrt{\mfk{S}_{4}} & =\left\Vert (\cA^{*}\cA-\cI)(\bL\bR^{\top}-\bX_{\star})\bR(\bR^{\top}\bR)^{-1}\bSigma_{\star}^{1/2}\right\Vert _{\fro}\\
 & \le\left\Vert (\cA^{*}\cA-\cI)(\bDelta_{L}\bR_{\star}^{\top})\bR(\bR^{\top}\bR)^{-1}\bSigma_{\star}^{1/2}\right\Vert _{\fro}\\
 & \quad+\left\Vert (\cA^{*}\cA-\cI)(\bL_{\star}\bDelta_{R}^{\top})\bR(\bR^{\top}\bR)^{-1}\bSigma_{\star}^{1/2}\right\Vert _{\fro}\\
 & \quad+\left\Vert (\cA^{*}\cA-\cI)(\bDelta_{L}\bDelta_{R}^{\top})\bR(\bR^{\top}\bR)^{-1}\bSigma_{\star}^{1/2}\right\Vert _{\fro},
\end{align*}
where once again we use the decomposition $\bL\bR^{\top}-\bX_{\star}=\bDelta_{L}\bR_{\star}^{\top}+\bL_{\star}\bDelta_{R}^{\top}+\bDelta_{L}\bDelta_{R}^{\top}$. Use \eqref{eq:coro-RIP} to see that 
\begin{align*}
\sqrt{\mfk{S}_{4}} & \le\delta_{2r}\left(\|\bDelta_{L}\bR_{\star}^{\top}\|_{\fro}+\|\bL_{\star}\bDelta_{R}^{\top}\|_{\fro}+\|\bDelta_{L}\bDelta_{R}^{\top}\|_{\fro}\right)\left\Vert \bR(\bR^{\top}\bR)^{-1}\bSigma_{\star}^{1/2}\right\Vert _{\op}.
\end{align*}
Repeating the same argument in bounding $\mfk{S}_{2}$ yields 
\begin{align*}
\sqrt{\mfk{S}_{4}}\le\frac{\delta_{2r}\left(2+\epsilon\right)}{2\left(1-\epsilon\right)}\left(\|\bDelta_{L}\bSigma_{\star}^{1/2}\|_{\fro}+\|\bDelta_{R}\bSigma_{\star}^{1/2}\|_{\fro}\right).
\end{align*}
We can then take the squares of both sides and use $(a+b)^{2}\le2a^{2}+2b^{2}$ to reach 
\begin{align*}
 \mfk{S}_{4}\le\frac{\delta_{2r}^{2}(2+\epsilon)^{2}}{2(1-\epsilon)^{2}}\left(\|\bDelta_{L}\bSigma_{\star}^{1/2}\|_{\fro}^{2}+\|\bDelta_{R}\bSigma_{\star}^{1/2}\|_{\fro}^{2}\right).
\end{align*}
\end{enumerate}
Taking the bounds for $\mfk{S}_{1},\mfk{S}_{2},\mfk{S}_{3},\mfk{S}_{4}$ collectively yields
\begin{align*}
\left\Vert (\bL_{t+1}\bQ_{t}-\bL_{\star})\bSigma_{\star}^{1/2}\right\Vert _{\fro}^{2} & \le \left((1-\eta)^{2}+\frac{2\epsilon}{1-\epsilon}\eta(1-\eta)\right)\|\bDelta_{L}\bSigma_{\star}^{1/2}\|_{\fro}^{2}+\frac{2\epsilon+\epsilon^{2}}{(1-\epsilon)^{2}}\eta^{2}\|\bDelta_{R}\bSigma_{\star}^{1/2}\|_{\fro}^{2}\\
 & \quad + \frac{\delta_{2r}(2+\epsilon)}{1-\epsilon}\eta(1-\eta)\left(\frac{3}{2}\|\bDelta_{L}\bSigma_{\star}^{1/2}\|_{\fro}^{2}+\frac{1}{2}\|\bDelta_{R}\bSigma_{\star}^{1/2}\|_{\fro}^2\right)\\
 & \quad + \frac{\delta_{2r}(2+\epsilon)}{(1-\epsilon)^{2}}\eta^2\left(\frac{1}{2}\|\bDelta_{L}\bSigma_{\star}^{1/2}\|_{\fro}^{2}+\frac{3}{2}\|\bDelta_{R}\bSigma_{\star}^{1/2}\|_{\fro}^{2}\right) \\
 & \quad + \frac{\delta_{2r}^{2}(2+\epsilon)^{2}}{2(1-\epsilon)^{2}}\eta^2\left(\|\bDelta_{L}\bSigma_{\star}^{1/2}\|_{\fro}^{2}+\|\bDelta_{R}\bSigma_{\star}^{1/2}\|_{\fro}^{2}\right).
\end{align*}
Similarly, we can expand the second square in \eqref{eq:MS_expand} and obtain a similar bound. Combine both to obtain 
\begin{align*}
\left\Vert (\bL_{t+1}\bQ_{t}-\bL_{\star})\bSigma_{\star}^{1/2}\right\Vert _{\fro}^{2}+\left\Vert (\bR_{t+1}\bQ_{t}^{-\top}-\bR_{\star})\bSigma_{\star}^{1/2}\right\Vert _{\fro}^{2} & \le\rho^{2}(\eta;\epsilon,\delta_{2r})\dist^{2}(\bF_{t},\bF_{\star}),
\end{align*}
where the contraction rate is given by
\begin{align*}
\rho^{2}(\eta;\epsilon,\delta_{2r})\coloneqq(1-\eta)^{2}+\frac{2\epsilon+\delta_{2r}(4+2\epsilon)}{1-\epsilon}\eta(1-\eta)+\frac{2\epsilon+\epsilon^{2}+\delta_{2r}(4+2\epsilon)+\delta_{2r}^{2}(2+\epsilon)^{2}}{(1-\epsilon)^{2}}\eta^{2}.
\end{align*}
With $\epsilon=0.1$, $\delta_{2r}\le0.02$, and $0<\eta\le2/3$, one has $\rho(\eta;\epsilon,\delta_{2r})\le1-0.6\eta$. Thus we conclude that 
\begin{align*}
\dist(\bF_{t+1},\bF_{\star}) & \le\sqrt{\left\Vert (\bL_{t+1}\bQ_{t}-\bL_{\star})\bSigma_{\star}^{1/2}\right\Vert _{\fro}^{2}+\left\Vert (\bR_{t+1}\bQ_{t}^{-\top}-\bR_{\star})\bSigma_{\star}^{1/2}\right\Vert _{\fro}^{2}}\\
 & \le(1-0.6\eta)\dist(\bF_{t},\bF_{\star}).
\end{align*}

\subsection{Proof of Lemma~\ref{lemma:init_MS}}

With the knowledge of partial Frobenius norm $\|\cdot\|_{\fro,r}$, we are ready to establish the claimed result. Invoke Lemma~\ref{lemma:Procrustes} to relate $\dist(\bF_{0},\bF_{\star})$ to $\|\bL_{0}\bR_{0}^{\top}-\bX_{\star}\|_{\fro}$, and use that $\bL_{0}\bR_{0}^{\top}-\bX_{\star}$ has rank at most $2r$ to obtain
\begin{align*}
\dist(\bF_{0},\bF_{\star})\le\sqrt{\sqrt{2}+1}\left\Vert \bL_{0}\bR_{0}^{\top}-\bX_{\star}\right\Vert _{\fro}\le\sqrt{2(\sqrt{2}+1)}\left\Vert \bL_{0}\bR_{0}^{\top}-\bX_{\star}\right\Vert _{\fro,r}.
\end{align*}
Note that $\bL_{0}\bR_{0}^{\top}$ is the best rank-$r$ approximation of $\cA^{*}\cA(\bX_{\star})$, and apply the triangle inequality combined with Lemma~\ref{lemma:norm_Fr_Eckart-Yang-Mirsky} to obtain
\begin{align*}
\left\Vert \bL_{0}\bR_{0}^{\top}-\bX_{\star}\right\Vert _{\fro,r} & \le\left\Vert \cA^{*}\cA(\bX_{\star})-\bL_{0}\bR_{0}^{\top}\right\Vert _{\fro,r}+\left\Vert \cA^{*}\cA(\bX_{\star})-\bX_{\star}\right\Vert _{\fro,r}\\
 & \le2\left\Vert (\cA^{*}\cA-\cI)(\bX_{\star})\right\Vert _{\fro,r}\le2\delta_{2r}\|\bX_{\star}\|_{\fro}.
\end{align*}
Here, the last inequality follows from combining Lemma~\ref{lemma:norm_Fr_variation} and \eqref{eq:coro-RIP} as 
\begin{align*}
 \left\Vert (\cA^{*}\cA-\cI)(\bX_{\star})\right\Vert _{\fro,r} & =\max_{\tilde{\bR}\in\RR^{n_2\times r}:\|\tilde{\bR}\|_{\op}\le1}\;\left\Vert (\cA^{*}\cA-\cI)(\bX_{\star})\tilde{\bR}\right\Vert _{\fro}\le\delta_{2r}\|\bX_{\star}\|_{\fro}.
\end{align*}
As a result, one has 
\begin{align*}
\dist(\bF_{0},\bF_{\star}) & \le2\sqrt{2(\sqrt{2}+1)}\delta_{2r}\|\bX_{\star}\|_{\fro}\le5\delta_{2r}\sqrt{r}\kappa\sigma_{r}(\bX_{\star}).
\end{align*}

%% file: appendix-rpca.tex
\section{Proof for Robust PCA}\label{sec:proof_RPCA}

We first establish a useful property regarding the truncation operator $\cT_{2\alpha}[\cdot]$. 
\begin{lemma}\label{lemma:S} Given $\bS_{\star}\in\cS_{\alpha}$ and $\bS=\cT_{2\alpha}[\bX_{\star}+\bS_{\star}-\bL\bR^{\top}]$, one has 
\begin{align}
\|\bS-\bS_{\star}\|_{\infty}\le2\|\bL\bR^{\top}-\bX_{\star}\|_{\infty}.\label{eq:S_delta_inf}
\end{align}
In addition, for any low-rank matrix $\bM=\bL_{M}\bR_{M}^{\top}\in\RR^{n_{1}\times n_{2}}$ with $\bL_{M}\in\RR^{n_{1}\times r},\bR_{M}\in\RR^{n_{2}\times r}$, one has 
\begin{align}
\begin{split} |\langle\bS-\bS_{\star},\bM\rangle| &\le \sqrt{3\alpha}\nu\left(\|(\bL-\bL_{\star})\bSigma_{\star}^{1/2}\|_{\fro}+\|(\bR-\bR_{\star})\bSigma_{\star}^{1/2}\|_{\fro}\right)\|\bM\|_{\fro}\\
 & \quad+2\sqrt{\alpha}\left(\sqrt{n_{1}}\|\bL_{M}\|_{2,\infty}\|\bR_{M}\|_{\fro}\wedge\sqrt{n_{2}}\|\bL_{M}\|_{\fro}\|\bR_{M}\|_{2,\infty}\right)\|\bL\bR^{\top}-\bX_{\star}\|_{\fro},
\end{split}\label{eq:S_delta_inner}
\end{align}
where $\nu$ obeys
\begin{align*}
\nu\ge\frac{\sqrt{n_{1}}}{2}\left(\|\bL\bSigma_{\star}^{-1/2}\|_{2,\infty}+\|\bL_{\star}\bSigma_{\star}^{-1/2}\|_{2,\infty}\right)\vee\frac{\sqrt{n_{2}}}{2}\left(\|\bR\bSigma_{\star}^{-1/2}\|_{2,\infty}+\|\bR_{\star}\bSigma_{\star}^{-1/2}\|_{2,\infty}\right).
\end{align*}
\end{lemma}

\begin{proof} Denote $\bDelta_{L}\coloneqq\bL-\bL_{\star}$, $\bDelta_{R}\coloneqq\bR-\bR_{\star}$, and $\bDelta_{X}\coloneqq\bL\bR^{\top}-\bX_{\star}$. Let $\Omega,\Omega_{\star}$ be the support of $\bS$ and $\bS_{\star}$, respectively. As a result, $\bS-\bS_{\star}$ is supported on $\Omega\cup\Omega_{\star}$.

We start with proving the first claim, i.e.~\eqref{eq:S_delta_inf}.
For $(i,j)\in\Omega$, by the definition of $\cT_{2\alpha}[\cdot]$, we have $(\bS-\bS_{\star})_{i,j}=(-\bDelta_{X})_{i,j}$. For $(i,j)\in\Omega_{\star}\setminus\Omega$, one necessarily has $\bS_{i,j}=0$ and therefore $(\bS-\bS_{\star})_{i,j}=(-\bS_{\star})_{i,j}$.
Again by the definition of the operator $\cT_{2\alpha}[\cdot]$, we know $|\bS_{\star}-\bDelta_{X}|_{i,j}$ is either smaller than $|\bS_{\star}-\bDelta_{X}|_{i,(2\alpha n_{2})}$ or $|\bS_{\star}-\bDelta_{X}|_{(2\alpha n_{1}),j}$.
Furthermore, we know that $\bS_{\star}$ contains at most $\alpha$-fraction nonzero entries per row and column. Consequently, one has $|\bS_{\star}-\bDelta_{X}|_{i,j}\le|\bDelta_{X}|_{i,(\alpha n_{2})}\vee|\bDelta_{X}|_{(\alpha n_{1}),j}$. Combining the two cases above, we conclude that 
\begin{align}
|\bS-\bS_{\star}|_{i,j}\le\begin{cases}
|\bDelta_{X}|_{i,j}, & (i,j)\in\Omega\\
|\bDelta_{X}|_{i,j}+\left(|\bDelta_{X}|_{i,(\alpha n_{2})}\vee|\bDelta_{X}|_{(\alpha n_{1}),j}\right), & (i,j)\in\Omega_{\star}\setminus\Omega
\end{cases}.\label{eq:S_upper_bound}
\end{align}
This immediately implies the $\ell_{\infty}$ norm bound \eqref{eq:S_delta_inf}.

Next, we prove the second claim \eqref{eq:S_delta_inner}. Recall that $\bS-\bS_{\star}$ is supported on $\Omega\cup\Omega_{\star}$. We then have 
\begin{align*}
 |\langle\bS-\bS_{\star},\bM\rangle| &\le \sum_{(i,j)\in\Omega}|\bS-\bS_{\star}|_{i,j}|\bM|_{i,j}+\sum_{(i,j)\in\Omega_{\star}\setminus\Omega}|\bS-\bS_{\star}|_{i,j}|\bM|_{i,j}\\
 & \le\sum_{(i,j)\in\Omega\cup\Omega_{\star}}|\bDelta_{X}|_{i,j}|\bM|_{i,j}+\sum_{(i,j)\in\Omega_{\star}\setminus\Omega}\left(|\bDelta_{X}|_{i,(\alpha n_{2})}+|\bDelta_{X}|_{(\alpha n_{1}),j}\right)|\bM|_{i,j},
\end{align*}
where the second line follows from \eqref{eq:S_upper_bound}. Let $\beta>0$ be some positive number, whose value will be determined later. Use $2ab\le\beta^{-1}a^{2}+\beta b^{2}$ to further obtain 
\begin{align*}
|\langle\bS-\bS_{\star},\bM\rangle| & \le\underbrace{\sum_{(i,j)\in\Omega\cup\Omega_{\star}}|\bDelta_{X}|_{i,j}|\bM|_{i,j}}_{\mfk{A}_{1}}+\frac{1}{2\beta}\underbrace{\sum_{(i,j)\in\Omega_{\star}\setminus\Omega}\left(|\bDelta_{X}|_{i,(\alpha n_{2})}^{2}+|\bDelta_{X}|_{(\alpha n_{1}),j}^{2}\right)}_{\mfk{A}_{2}}+\beta\underbrace{\sum_{(i,j)\in\Omega_{\star}\setminus\Omega}|\bM|_{i,j}^{2}}_{\mfk{A}_{3}}.
\end{align*}
In regard to the three terms $\mfk{A}_{1},\mfk{A}_{2}$ and $\mfk{A}_{3}$, we have the following claims, whose proofs are deferred to the end.

\begin{claim}\label{claim:A_1} The first term $\mfk{A}_{1}$ satisfies
\begin{align*}
\mfk{A}_{1}\le\sqrt{3\alpha}\nu\left(\|\bDelta_{L}\bSigma_{\star}^{1/2}\|_{\fro}+\|\bDelta_{R}\bSigma_{\star}^{1/2}\|_{\fro}\right)\|\bM\|_{\fro}.
\end{align*}
\end{claim}
\begin{claim}\label{claim:A_2} The second term $\mfk{A}_{2}$ satisfies
\begin{align*}
\mfk{A}_{2}\le2\|\bDelta_{X}\|_{\fro}^{2}.
\end{align*}
\end{claim}
\begin{claim}\label{claim:A_3} The third term $\mfk{A}_{3}$ satisfies
\begin{align*}
\mfk{A}_{3}\le\alpha\left(n_{1}\|\bL_{M}\|_{2,\infty}^{2}\|\bR_{M}\|_{\fro}^{2}\wedge n_{2}\|\bL_{M}\|_{\fro}^{2}\|\bR_{M}\|_{2,\infty}^{2}\right).
\end{align*}
\end{claim}

Combine the pieces to reach 
\begin{align*}
 |\langle\bS-\bS_{\star},\bM\rangle| &\le \sqrt{3\alpha}\nu\left(\|\bDelta_{L}\bSigma_{\star}^{1/2}\|_{\fro}+\|\bDelta_{R}\bSigma_{\star}^{1/2}\|_{\fro}\right)\|\bM\|_{\fro}\\
 & \quad+\frac{\|\bDelta_{X}\|_{\fro}^{2}}{\beta}+\beta\alpha\left(n_{1}\|\bL_{M}\|_{2,\infty}^{2}\|\bR_{M}\|_{\fro}^{2}\wedge n_{2}\|\bL_{M}\|_{\fro}^{2}\|\bR_{M}\|_{2,\infty}^{2}\right).
\end{align*}
One can then choose $\beta$ optimally to yield 
\begin{align*}
 |\langle\bS-\bS_{\star},\bM\rangle| &\le \sqrt{3\alpha}\nu\left(\|\bDelta_{L}\bSigma_{\star}^{1/2}\|_{\fro}+\|\bDelta_{R}\bSigma_{\star}^{1/2}\|_{\fro}\right)\|\bM\|_{\fro}\\
 & \quad+2\sqrt{\alpha}\left(\sqrt{n_{1}}\|\bL_{M}\|_{2,\infty}\|\bR_{M}\|_{\fro}\wedge\sqrt{n_{2}}\|\bL_{M}\|_{\fro}\|\bR_{M}\|_{2,\infty}\right)\|\bDelta_{X}\|_{\fro}.
\end{align*}
This finishes the proof. 
\end{proof}

\begin{proof}[Proof of Claim~\ref{claim:A_1}] Use the decomposition $\bDelta_{X}=\bDelta_{L}\bR^{\top}+\bL_{\star}\bDelta_{R}^{\top}=\bDelta_{L}\bR_{\star}^{\top}+\bL\bDelta_{R}^{\top}$ to obtain 
\begin{align*}
|\bDelta_{X}|_{i,j} & \le\|(\bDelta_{L}\bSigma_{\star}^{1/2})_{i,\cdot}\|_{2}\|\bR\bSigma_{\star}^{-1/2}\|_{2,\infty}+\|\bL_{\star}\bSigma_{\star}^{-1/2}\|_{2,\infty}\|(\bDelta_{R}\bSigma_{\star}^{1/2})_{j,\cdot}\|_{2},\quad\mbox{and}\\
|\bDelta_{X}|_{i,j} & \le\|(\bDelta_{L}\bSigma_{\star}^{1/2})_{i,\cdot}\|_{2}\|\bR_{\star}\bSigma_{\star}^{-1/2}\|_{2,\infty}+\|\bL\bSigma_{\star}^{-1/2}\|_{2,\infty}\|(\bDelta_{R}\bSigma_{\star}^{1/2})_{j,\cdot}\|_{2}.
\end{align*}
Take the average to yield 
\begin{align*}
|\bDelta_{X}|_{i,j}\le\frac{\nu}{\sqrt{n_{2}}}\|(\bDelta_{L}\bSigma_{\star}^{1/2})_{i,\cdot}\|_{2}+\frac{\nu}{\sqrt{n_{1}}}\|(\bDelta_{R}\bSigma_{\star}^{1/2})_{j,\cdot}\|_{2},
\end{align*}
where we have used the assumption on $\nu$. With this upper bound on $|\bDelta_{X}|_{i,j}$ in place, we can further control $\mfk{A}_{1}$ as 
\begin{align*}
 \mfk{A}_{1} &\le \sum_{(i,j)\in\Omega\cup\Omega_{\star}}\frac{\nu}{\sqrt{n_{2}}}\|(\bDelta_{L}\bSigma_{\star}^{1/2})_{i,\cdot}\|_{2}|\bM|_{i,j}+\sum_{(i,j)\in\Omega\cup\Omega_{\star}}\frac{\nu}{\sqrt{n_{1}}}\|(\bDelta_{R}\bSigma_{\star}^{1/2})_{j,\cdot}\|_{2}|\bM|_{i,j}\\
 & \le\left(\sqrt{\sum_{(i,j)\in\Omega\cup\Omega_{\star}}\|(\bDelta_{L}\bSigma_{\star}^{1/2})_{i,\cdot}\|_{2}^{2}/n_{2}}+\sqrt{\sum_{(i,j)\in\Omega\cup\Omega_{\star}}\|(\bDelta_{R}\bSigma_{\star}^{1/2})_{j,\cdot}\|_{2}^{2}/n_{1}}\right)\nu\|\bM\|_{\fro}.
\end{align*}
Regarding the first term, one has 
\begin{align*}
\sum_{(i,j)\in\Omega\cup\Omega_{\star}}\|(\bDelta_{L}\bSigma_{\star}^{1/2})_{i,\cdot}\|_{2}^{2} & =\sum_{i=1}^{n_{1}}\sum_{j:(i,j)\in\Omega\cup\Omega_{\star}}\|(\bDelta_{L}\bSigma_{\star}^{1/2})_{i,\cdot}\|_{2}^{2}\\
 & \le3\alpha n_{2}\sum_{i=1}^{n_{1}}\|(\bDelta_{L}\bSigma_{\star}^{1/2})_{i,\cdot}\|_{2}^{2}\\
 & =3\alpha n_{2}\|\bDelta_{L}\bSigma_{\star}^{1/2}\|_{\fro}^{2},
\end{align*}
where the second line follows from the fact that $\Omega\cup\Omega_{\star}$ contains at most $3\alpha n_{2}$ non-zero entries in each row. Similarly, we can show that 
\begin{align*}
\sum_{(i,j)\in\Omega\cup\Omega_{\star}}\|(\bDelta_{R}\bSigma_{\star}^{1/2})_{j,\cdot}\|_{2}^{2} & \le3\alpha n_{1}\|\bDelta_{R}\bSigma_{\star}^{1/2}\|_{\fro}^{2}.
\end{align*}
In all, we arrive at 
\begin{align*}
\mfk{A}_{1}\le\sqrt{3\alpha}\nu\left(\|\bDelta_{L}\bSigma_{\star}^{1/2}\|_{\fro}+\|\bDelta_{R}\bSigma_{\star}^{1/2}\|_{\fro}\right)\|\bM\|_{\fro},
\end{align*}
which is the desired claim. 
\end{proof}

\begin{proof}[Proof of Claim~\ref{claim:A_2}] Recall that $(\bDelta_{X})_{i,(\alpha n_{2})}$ denotes the $(\alpha n_{2})$-th largest entry in the $i$-th row of $\bDelta_{X}$. One necessarily has 
\begin{align*}
\alpha n_{2}|\bDelta_{X}|_{i,(\alpha n_{2})}^{2}\le\|(\bDelta_{X})_{i,\cdot}\|_{2}^{2}.
\end{align*}
As a result, we obtain 
\begin{align*}
\sum_{(i,j)\in\Omega_{\star}\setminus\Omega}|\bDelta_{X}|_{i,(\alpha n_{2})}^{2} & \le\sum_{(i,j)\in\Omega_{\star}}|\bDelta_{X}|_{i,(\alpha n_{2})}^{2}\\
 & \le\sum_{i=1}^{n_{1}}\sum_{j:(i,j)\in\Omega_{\star}}\frac{\|(\bDelta_{X})_{i,\cdot}\|_{2}^{2}}{\alpha n_{2}}\\
 & \le\sum_{i=1}^{n_{1}}\|(\bDelta_{X})_{i,\cdot}\|_{2}^{2}=\|\bDelta_{X}\|_{\fro}^{2},
\end{align*}
where the last line follows from the fact that $\Omega_{\star}$ contains at most $\alpha n_{2}$ nonzero entries in each row. Similarly one can show that 
\begin{align*}
\sum_{(i,j)\in\Omega_{\star}\setminus\Omega}|\bDelta_{X}|_{(\alpha n_{1}),j}^{2}\le\|\bDelta_{X}\|_{\fro}^{2}.
\end{align*}
Combining the above two bounds with the definition of $\mfk{A}_{2}$ completes the proof. 
\end{proof}

\begin{proof}[Proof of Claim~\ref{claim:A_3}] By definition, $\bM=\bL_{M}\bR_{M}^{\top}$, and hence one has 
\begin{align*}
\mfk{A}_{3}=\sum_{(i,j)\in\Omega_{\star}\setminus\Omega}|(\bL_{M})_{i,\cdot}(\bR_{M})_{j,\cdot}^{\top}|^{2}\le\sum_{(i,j)\in\Omega_{\star}}|(\bL_{M})_{i,\cdot}(\bR_{M})_{j,\cdot}^{\top}|^{2}.
\end{align*}
We can further upper bound $\mfk{A}_{3}$ as 
\begin{align*}
\mfk{A}_{3} & \le\sum_{(i,j)\in\Omega_{\star}}\|(\bL_{M})_{i,\cdot}\|_{2}^{2}\|(\bR_{M})_{j,\cdot}\|_{2}^{2}\\
 & \le\sum_{i=1}^{n_{1}}\sum_{j:(i,j)\in\Omega_{\star}}\|(\bL_{M})_{i,\cdot}\|_{2}^{2}\|\bR_{M}\|_{2,\infty}^{2}\\
 & \le\sum_{i=1}^{n_{1}}\alpha n_{2}\|(\bL_{M})_{i,\cdot}\|_{2}^{2}\|\bR_{M}\|_{2,\infty}^{2}=\alpha n_{2}\|\bL_{M}\|_{\fro}^{2}\|\bR_{M}\|_{2,\infty}^{2},
\end{align*}
where the last line follows from the fact that $\Omega_{\star}$ contains at most $\alpha n_{2}$ non-zero entries in each row. Similarly, one can obtain
\begin{align*}
\mfk{A}_{3}\le\alpha n_{1}\|\bL_{M}\|_{2,\infty}^{2}\|\bR_{M}\|_{\fro}^{2},
\end{align*}
which completes the proof.
\end{proof}

\subsection{Proof of Lemma~\ref{lemma:contraction_RPCA}}

We begin with introducing several useful notations and facts. In view of the condition $\dist(\bF_{t},\bF_{\star})\le0.02\sigma_{r}(\bX_{\star})$ and Lemma~\ref{lemma:Q_existence}, one knows that $\bQ_{t}$, the optimal alignment matrix between $\bF_{t}$ and $\bF_{\star}$ exists. Therefore, for notational convenience, denote $\bL\coloneqq\bL_{t}\bQ_{t}$, $\bR\coloneqq\bR_{t}\bQ_{t}^{-\top}$, $\bDelta_{L}\coloneqq\bL-\bL_{\star}$, $\bDelta_{R}\coloneqq\bR-\bR_{\star}$, $\bS\coloneqq\bS_{t}=\cT_{2\alpha}[\bX_{\star}+\bS_{\star}-\bL\bR^{\top}]$, and $\epsilon\coloneqq 0.02$. Similar to the derivation in \eqref{eq:cond_MF}, we have 
\begin{align}
\|\bDelta_{L}\bSigma_{\star}^{-1/2}\|_{\op}\vee\|\bDelta_{R}\bSigma_{\star}^{-1/2}\|_{\op}\le\epsilon.\label{eq:cond_RPCA_op}
\end{align}
Moreover, the incoherence condition 
\begin{align}
\sqrt{n_{1}}\|\bDelta_{L}\bSigma_{\star}^{1/2}\|_{2,\infty}\vee\sqrt{n_{2}}\|\bDelta_{R}\bSigma_{\star}^{1/2}\|_{2,\infty}\le\sqrt{\mu r}\sigma_{r}(\bX_{\star})\label{eq:cond_RPCA_2inf-1}
\end{align}
implies 
\begin{align}
\sqrt{n_{1}}\|\bDelta_{L}\bSigma_{\star}^{-1/2}\|_{2,\infty}\vee\sqrt{n_{2}}\|\bDelta_{R}\bSigma_{\star}^{-1/2}\|_{2,\infty}\le\sqrt{\mu r},\label{eq:cond_RPCA_2inf-2}
\end{align}
which combined with the triangle inequality further implies 
\begin{align}
\sqrt{n_{1}}\|\bL\bSigma_{\star}^{-1/2}\|_{2,\infty}\vee\sqrt{n_{2}}\|\bR\bSigma_{\star}^{-1/2}\|_{2,\infty}\le2\sqrt{\mu r}.\label{eq:cond_RPCA_2inf-3}
\end{align}
The conclusion $\|\bL_{t}\bR_{t}^{\top}-\bX_{\star}\|_{\fro} \le 1.5\dist(\bF_{t},\bF_{\star})$ is a simple consequence of Lemma~\ref{lemma:matrix2factor}; see \eqref{eq:dist_matrix} for a detailed argument. 
In what follows, we shall prove the distance contraction and the incoherence condition separately.

\subsubsection{Distance contraction}

By the definition of $\dist^{2}(\bF_{t+1},\bF_{\star})$, one has
\begin{align}
\dist^{2}(\bF_{t+1},\bF_{\star}) & \le\left\Vert (\bL_{t+1}\bQ_{t}-\bL_{\star})\bSigma_{\star}^{1/2}\right\Vert _{\fro}^{2}+\left\Vert (\bR_{t+1}\bQ_{t}^{-\top}-\bR_{\star})\bSigma_{\star}^{1/2}\right\Vert _{\fro}^{2}.\label{eq:RPCA_expand}
\end{align}
From now on, we focus on controlling the first square $\|(\bL_{t+1}\bQ_{t}-\bL_{\star})\bSigma_{\star}^{1/2}\|_{\fro}^{2}$. In view of the update rule \eqref{eq:iterates_RPCA}, one has 
\begin{align}
(\bL_{t+1}\bQ_{t}-\bL_{\star})\bSigma_{\star}^{1/2} &= \left(\bL-\eta(\bL\bR^{\top}+\bS-\bX_{\star}-\bS_{\star})\bR(\bR^{\top}\bR)^{-1}-\bL_{\star}\right)\bSigma_{\star}^{1/2}\nonumber \\
 & =\left(\bDelta_{L}-\eta(\bL\bR^{\top}-\bX_{\star})\bR(\bR^{\top}\bR)^{-1}-\eta(\bS-\bS_{\star})\bR(\bR^{\top}\bR)^{-1}\right)\bSigma_{\star}^{1/2}\nonumber \\
 & =(1-\eta)\bDelta_{L}\bSigma_{\star}^{1/2}-\eta\bL_{\star}\bDelta_{R}^{\top}\bR(\bR^{\top}\bR)^{-1}\bSigma_{\star}^{1/2}-\eta(\bS-\bS_{\star})\bR(\bR^{\top}\bR)^{-1}\bSigma_{\star}^{1/2}.\label{eq:RPCA_Lt}
\end{align}
Here, we use the notation introduced above and the decomposition $\bL\bR^{\top}-\bX_{\star}=\bDelta_{L}\bR^{\top}+\bL_{\star}\bDelta_{R}^{\top}$. Take the squared Frobenius norm of both sides of \eqref{eq:RPCA_Lt} to obtain
\begin{align*}
\left\Vert (\bL_{t+1}\bQ_{t}-\bL_{\star})\bSigma_{\star}^{1/2}\right\Vert _{\fro}^{2} &=\underbrace{\left\Vert (1-\eta)\bDelta_{L}\bSigma_{\star}^{1/2}-\eta\bL_{\star}\bDelta_{R}^{\top}\bR(\bR^{\top}\bR)^{-1}\bSigma_{\star}^{1/2}\right\Vert _{\fro}^{2}}_{\mfk{R}_{1}}\\
 & \quad-2\eta(1-\eta)\underbrace{\tr\left((\bS-\bS_{\star})\bR(\bR^{\top}\bR)^{-1}\bSigma_{\star}\bDelta_{L}^{\top}\right)}_{\mfk{R}_{2}}\\
 & \quad+2\eta^{2}\underbrace{\tr\left((\bS-\bS_{\star})\bR(\bR^{\top}\bR)^{-1}\bSigma_{\star}(\bR^{\top}\bR)^{-1}\bR^{\top}\bDelta_{R}\bL_{\star}^{\top}\right)}_{\mfk{R}_{3}}\\
 & \quad+\eta^{2}\underbrace{\left\Vert (\bS-\bS_{\star})\bR(\bR^{\top}\bR)^{-1}\bSigma_{\star}^{1/2}\right\Vert _{\fro}^{2}}_{\mfk{R}_{4}}.
\end{align*}
In the sequel, we shall bound the four terms separately, of which $\mfk{R}_{1}$ is the main term, and $\mfk{R}_{2},\mfk{R}_{3}$ and $\mfk{R}_{4}$ are perturbation terms. 
\begin{enumerate}
\item Notice that the main term $\mfk{R}_{1}$ has already been controlled in \eqref{eq:MF_Lt_bound} under the condition \eqref{eq:cond_RPCA_op}. It obeys
\begin{align*}
\mfk{R}_{1}\le\left((1-\eta)^{2}+\frac{2\epsilon}{1-\epsilon}\eta(1-\eta)\right)\|\bDelta_{L}\bSigma_{\star}^{1/2}\|_{\fro}^{2}+\frac{2\epsilon+\epsilon^{2}}{(1-\epsilon)^{2}}\eta^{2}\|\bDelta_{R}\bSigma_{\star}^{1/2}\|_{\fro}^{2}.
\end{align*}
\item For the second term $\mfk{R}_{2}$, set $\bM\coloneqq\bDelta_{L}\bSigma_{\star}(\bR^{\top}\bR)^{-1}\bR^{\top}$ with $\bL_{M}\coloneqq\bDelta_{L}\bSigma_{\star}(\bR^{\top}\bR)^{-1}\bSigma_{\star}^{1/2}$, $\bR_{M}\coloneqq\bR\bSigma_{\star}^{-1/2}$, and then invoke Lemma~\ref{lemma:S} with $\nu\coloneqq3\sqrt{\mu r}/2$ to see
\begin{align*}
|\mfk{R}_{2}| & \le\frac{3}{2}\sqrt{3\alpha\mu r}\left(\|\bDelta_{L}\bSigma_{\star}^{1/2}\|_{\fro}+\|\bDelta_{R}\bSigma_{\star}^{1/2}\|_{\fro}\right)\left\Vert \bDelta_{L}\bSigma_{\star}(\bR^{\top}\bR)^{-1}\bR^{\top}\right\Vert _{\fro}\\
 & \quad+2\sqrt{\alpha n_{2}}\left\Vert \bDelta_{L}\bSigma_{\star}(\bR^{\top}\bR)^{-1}\bSigma_{\star}^{1/2}\right\Vert _{\fro}\|\bR\bSigma_{\star}^{-1/2}\|_{2,\infty}\|\bL\bR^{\top}-\bX_{\star}\|_{\fro}\\
 & \le\frac{3}{2}\sqrt{3\alpha\mu r}\left(\|\bDelta_{L}\bSigma_{\star}^{1/2}\|_{\fro}+\|\bDelta_{R}\bSigma_{\star}^{1/2}\|_{\fro}\right)\|\bDelta_{L}\bSigma_{\star}^{1/2}\|_{\fro}\left\Vert \bR(\bR^{\top}\bR)^{-1}\bSigma_{\star}^{1/2}\right\Vert _{\op}\\
 & \quad+2\sqrt{\alpha n_{2}}\|\bDelta_{L}\bSigma_{\star}^{1/2}\|_{\fro}\left\Vert \bSigma_{\star}^{1/2}(\bR^{\top}\bR)^{-1}\bSigma_{\star}^{1/2}\right\Vert _{\op}\|\bR\bSigma_{\star}^{-1/2}\|_{2,\infty}\|\bL\bR^{\top}-\bX_{\star}\|_{\fro}.
\end{align*}
Take the condition \eqref{eq:cond_RPCA_op} and Lemmas~\ref{lemma:Weyl} and \ref{lemma:matrix2factor} together to obtain 
\begin{align}
\begin{split}\left\Vert \bR(\bR^{\top}\bR)^{-1}\bSigma_{\star}^{1/2}\right\Vert _{\op} & \le\frac{1}{1-\epsilon};\\
\left\Vert \bSigma_{\star}^{1/2}(\bR^{\top}\bR)^{-1}\bSigma_{\star}^{1/2}\right\Vert _{\op} & =\left\Vert \bR(\bR^{\top}\bR)^{-1}\bSigma_{\star}^{1/2}\right\Vert _{\op}^{2}\le \frac{1}{(1-\epsilon)^{2}};\\
\|\bL\bR^{\top}-\bX_{\star}\|_{\fro} & \le(1+\frac{\epsilon}{2})\left(\|\bDelta_{L}\bSigma_{\star}^{1/2}\|_{\fro}+\|\bDelta_{R}\bSigma_{\star}^{1/2}\|_{\fro}\right).
\end{split}\label{eq:consequences_RPCA}
\end{align}
These consequences combined with the condition \eqref{eq:cond_RPCA_2inf-3} yield
\begin{align*}
|\mfk{R}_{2}| &\le\frac{3\sqrt{3\alpha\mu r}}{2(1-\epsilon)}\left(\|\bDelta_{L}\bSigma_{\star}^{1/2}\|_{\fro}+\|\bDelta_{R}\bSigma_{\star}^{1/2}\|_{\fro}\right)\|\bDelta_{L}\bSigma_{\star}^{1/2}\|_{\fro}\\
 & \quad+\frac{4\sqrt{\alpha\mu r}}{(1-\epsilon)^{2}}\|\bDelta_{L}\bSigma_{\star}^{1/2}\|_{\fro}(1+\frac{\epsilon}{2})\left(\|\bDelta_{L}\bSigma_{\star}^{1/2}\|_{\fro}+\|\bDelta_{R}\bSigma_{\star}^{1/2}\|_{\fro}\right)\\
 & \le\sqrt{\alpha\mu r}\frac{3\sqrt{3}+\frac{4(2+\epsilon)}{1-\epsilon}}{2(1-\epsilon)}\left(\|\bDelta_{L}\bSigma_{\star}^{1/2}\|_{\fro}^{2}+\|\bDelta_{L}\bSigma_{\star}^{1/2}\|_{\fro}\|\bDelta_{R}\bSigma_{\star}^{1/2}\|_{\fro}\right)\\
 & \le\sqrt{\alpha\mu r}\frac{3\sqrt{3}+\frac{4(2+\epsilon)}{1-\epsilon}}{2(1-\epsilon)}\left(\frac{3}{2}\|\bDelta_{L}\bSigma_{\star}^{1/2}\|_{\fro}^{2}+\frac{1}{2}\|\bDelta_{R}\bSigma_{\star}^{1/2}\|_{\fro}^{2}\right),
\end{align*}
where the last inequality holds since $2ab\le a^{2}+b^{2}$. 

\item The third term $\mfk{R}_{3}$ can be controlled similarly. Set $\bM\coloneqq\bL_{\star}\bDelta_{R}^{\top}\bR(\bR^{\top}\bR)^{-1}\bSigma_{\star}(\bR^{\top}\bR)^{-1}\bR^{\top}$ with $\bL_{M}\coloneqq\bL_{\star}\bSigma_{\star}^{-1/2}$ and $\bR_{M}\coloneqq\bR(\bR^{\top}\bR)^{-1}\bSigma_{\star}(\bR^{\top}\bR)^{-1}\bR^{\top}\bDelta_{R}\bSigma_{\star}^{1/2}$, and invoke Lemma~\ref{lemma:S} with $\nu\coloneqq3\sqrt{\mu r}/2$ to arrive at 
\begin{align*}
|\mfk{R}_{3}| & \le\frac{3}{2}\sqrt{3\alpha\mu r}\left(\|\bDelta_{L}\bSigma_{\star}^{1/2}\|_{\fro}+\|\bDelta_{R}\bSigma_{\star}^{1/2}\|_{\fro}\right)\left\Vert \bL_{\star}\bDelta_{R}^{\top}\bR(\bR^{\top}\bR)^{-1}\bSigma_{\star}(\bR^{\top}\bR)^{-1}\bR^{\top}\right\Vert _{\fro}\\
 & \quad+2\sqrt{\alpha n_{1}}\|\bL_{\star}\bSigma_{\star}^{-1/2}\|_{2,\infty}\left\Vert \bR(\bR^{\top}\bR)^{-1}\bSigma_{\star}(\bR^{\top}\bR)^{-1}\bR^{\top}\bDelta_{R}\bSigma_{\star}^{1/2}\right\Vert _{\fro}\|\bL\bR^{\top}-\bX_{\star}\|_{\fro}\\
 & \le\frac{3}{2}\sqrt{3\alpha\mu r}\left(\|\bDelta_{L}\bSigma_{\star}^{1/2}\|_{\fro}+\|\bDelta_{R}\bSigma_{\star}^{1/2}\|_{\fro}\right)\|\bDelta_{R}\bSigma_{\star}^{1/2}\|_{\fro}\left\Vert \bR(\bR^{\top}\bR)^{-1}\bSigma_{\star}^{1/2}\right\Vert _{\op}^{2}\\
 & \quad+2\sqrt{\alpha n_{1}}\|\bL_{\star}\bSigma_{\star}^{-1/2}\|_{2,\infty}\left\Vert \bR(\bR^{\top}\bR)^{-1}\bSigma_{\star}^{1/2}\right\Vert _{\op}^{2}\|\bDelta_{R}\bSigma_{\star}^{1/2}\|_{\fro}\|\bL\bR^{\top}-\bX_{\star}\|_{\fro}.
\end{align*}
Use the consequences \eqref{eq:consequences_RPCA} again to obtain 
\begin{align*}
|\mfk{R}_{3}| & \le\frac{3\sqrt{3\alpha\mu r}}{2(1-\epsilon)^{2}}\left(\|\bDelta_{L}\bSigma_{\star}^{1/2}\|_{\fro}+\|\bDelta_{R}\bSigma_{\star}^{1/2}\|_{\fro}\right)\|\bDelta_{R}\bSigma_{\star}^{1/2}\|_{\fro}\\
 & \quad+\frac{2\sqrt{\alpha\mu r}}{(1-\epsilon)^{2}}\|\bDelta_{R}\bSigma_{\star}^{1/2}\|_{\fro}(1+\frac{\epsilon}{2})\left(\|\bDelta_{L}\bSigma_{\star}^{1/2}\|_{\fro}+\|\bDelta_{R}\bSigma_{\star}^{1/2}\|_{\fro}\right)\\
 & \le\sqrt{\alpha\mu r}\frac{3\sqrt{3}+2(2+\epsilon)}{2(1-\epsilon)^{2}}\left(\|\bDelta_{L}\bSigma_{\star}^{1/2}\|_{\fro}\|\bDelta_{R}\bSigma_{\star}^{1/2}\|_{\fro}+\|\bDelta_{R}\bSigma_{\star}^{1/2}\|_{\fro}^{2}\right)\\
 & \le\sqrt{\alpha\mu r}\frac{3\sqrt{3}+2(2+\epsilon)}{2(1-\epsilon)^{2}}\left(\frac{1}{2}\|\bDelta_{L}\bSigma_{\star}^{1/2}\|_{\fro}^{2}+\frac{3}{2}\|\bDelta_{R}\bSigma_{\star}^{1/2}\|_{\fro}^{2}\right).
\end{align*}

\item For the last term $\mfk{R}_{4}$, utilize the variational representation of the Frobenius norm to see 
\begin{align*}
\sqrt{\mfk{R}_{4}}=\tr\left((\bS-\bS_{\star})\bR(\bR^{\top}\bR)^{-1}\bSigma_{\star}^{1/2}\tilde{\bL}^{\top}\right)
\end{align*}
for some $\tilde{\bL}\in\RR^{n_{1}\times r}$ obeying $\|\tilde{\bL}\|_{\fro}=1$.
Setting $\bM\coloneqq\tilde{\bL}\bSigma_{\star}^{1/2}(\bR^{\top}\bR)^{-1}\bR^{\top} =\bL_{M}\bR_{M}^{\top} $ with $\bL_{M}\coloneqq\tilde{\bL}\bSigma_{\star}^{1/2}(\bR^{\top}\bR)^{-1}\bSigma_{\star}^{1/2}$ and $\bR_{M}\coloneqq\bR\bSigma_{\star}^{-1/2}$, we are ready to apply Lemma~\ref{lemma:S} again with $\nu\coloneqq3\sqrt{\mu r}/2$ to see 
\begin{align*}
\sqrt{\mfk{R}_{4}} & \le\frac{3}{2}\sqrt{3\alpha\mu r}\left(\|\bDelta_{L}\bSigma_{\star}^{1/2}\|_{\fro}+\|\bDelta_{R}\bSigma_{\star}^{1/2}\|_{\fro}\right)\left\Vert \tilde{\bL}\bSigma_{\star}^{1/2}(\bR^{\top}\bR)^{-1}\bR^{\top}\right\Vert _{\fro}\\
 & \quad+2\sqrt{\alpha n_{2}}\left\Vert \tilde{\bL}\bSigma_{\star}^{1/2}(\bR^{\top}\bR)^{-1}\bSigma_{\star}^{1/2}\right\Vert _{\fro}\|\bR\bSigma_{\star}^{-1/2}\|_{2,\infty}\|\bL\bR^{\top}-\bX_{\star}\|_{\fro}\\
 & \le\frac{3}{2}\sqrt{3\alpha\mu r}\left(\|\bDelta_{L}\bSigma_{\star}^{1/2}\|_{\fro}+\|\bDelta_{R}\bSigma_{\star}^{1/2}\|_{\fro}\right)\left\Vert \bR(\bR^{\top}\bR)^{-1}\bSigma_{\star}^{1/2}\right\Vert _{\op}\\
 & \quad+2\sqrt{\alpha n_{2}}\left\Vert \bSigma_{\star}^{1/2}(\bR^{\top}\bR)^{-1}\bSigma_{\star}^{1/2}\right\Vert _{\op}\|\bR\bSigma_{\star}^{-1/2}\|_{2,\infty}\|\bL\bR^{\top}-\bX_{\star}\|_{\fro}.
\end{align*}
This combined with the consequences \eqref{eq:consequences_RPCA} and condition \eqref{eq:cond_RPCA_2inf-3} yields
\begin{align*}
\sqrt{\mfk{R}_{4}}\le\sqrt{\alpha\mu r}\frac{3\sqrt{3}+\frac{4(2+\epsilon)}{1-\epsilon}}{2(1-\epsilon)}\left(\|\bDelta_{L}\bSigma_{\star}^{1/2}\|_{\fro}+\|\bDelta_{R}\bSigma_{\star}^{1/2}\|_{\fro}\right).
\end{align*}
Take the square, and use the elementary inequality $(a+b)^{2}\le2a^{2}+2b^{2}$ to reach
\begin{align*}
\mfk{R}_{4}\le\alpha\mu r\frac{(3\sqrt{3}+\frac{4(2+\epsilon)}{1-\epsilon})^{2}}{2(1-\epsilon)^{2}}\left(\|\bDelta_{L}\bSigma_{\star}^{1/2}\|_{\fro}^{2}+\|\bDelta_{R}\bSigma_{\star}^{1/2}\|_{\fro}^{2}\right).
\end{align*}
\end{enumerate}
Taking collectively the bounds for $\mfk{R}_{1},\mfk{R}_{2},\mfk{R}_{3}$ and $\mfk{R}_{4}$ yields the control of $\|(\bL_{t+1}\bQ_{t}-\bL_{\star})\bSigma_{\star}^{1/2}\|_{\fro}^{2}$ as
\begin{align*}
\left\Vert (\bL_{t+1}\bQ_{t}-\bL_{\star})\bSigma_{\star}^{1/2}\right\Vert _{\fro}^{2} & \le \left((1-\eta)^{2}+\frac{2\epsilon}{1-\epsilon}\eta(1-\eta)\right)\|\bDelta_{L}\bSigma_{\star}^{1/2}\|_{\fro}^{2}+\frac{2\epsilon+\epsilon^{2}}{(1-\epsilon)^{2}}\eta^{2}\|\bDelta_{R}\bSigma_{\star}^{1/2}\|_{\fro}^{2} \\
 & \quad +\sqrt{\alpha\mu r}\frac{3\sqrt{3}+\frac{4(2+\epsilon)}{1-\epsilon}}{1-\epsilon}\eta(1-\eta)\left(\frac{3}{2}\|\bDelta_{L}\bSigma_{\star}^{1/2}\|_{\fro}^{2}+\frac{1}{2}\|\bDelta_{R}\bSigma_{\star}^{1/2}\|_{\fro}^{2}\right) \\
 & \quad +\sqrt{\alpha\mu r}\frac{3\sqrt{3}+2(2+\epsilon)}{(1-\epsilon)^{2}}\eta^2\left(\frac{1}{2}\|\bDelta_{L}\bSigma_{\star}^{1/2}\|_{\fro}^{2}+\frac{3}{2}\|\bDelta_{R}\bSigma_{\star}^{1/2}\|_{\fro}^{2}\right)\\
 & \quad +\alpha\mu r\frac{(3\sqrt{3}+\frac{4(2+\epsilon)}{1-\epsilon})^{2}}{2(1-\epsilon)^{2}}\eta^2\left(\|\bDelta_{L}\bSigma_{\star}^{1/2}\|_{\fro}^{2}+\|\bDelta_{R}\bSigma_{\star}^{1/2}\|_{\fro}^{2}\right).
\end{align*}
Similarly, we can obtain the control of $\|(\bR_{t+1}\bQ_{t}^{-\top}-\bR_{\star})\bSigma_{\star}^{1/2}\|_{\fro}^{2}$.
Combine them together and identify $\dist^{2}(\bF_{t},\bF_{\star}) = \|\bDelta_{L}\bSigma_{\star}^{1/2}\|_{\fro}^{2}+\|\bDelta_{R}\bSigma_{\star}^{1/2}\|_{\fro}^{2}$ to reach
\begin{align*}
\left\Vert (\bL_{t+1}\bQ_{t}-\bL_{\star})\bSigma_{\star}^{1/2}\right\Vert _{\fro}^{2}+\left\Vert (\bR_{t+1}\bQ_{t}^{-\top}-\bR_{\star})\bSigma_{\star}^{1/2}\right\Vert _{\fro}^{2}\le\rho^{2}(\eta;\epsilon,\alpha\mu r)\dist^{2}(\bF_{t},\bF_{\star}),
\end{align*}
where the contraction rate $\rho^{2}(\eta;\epsilon,\alpha\mu r)$ is given by 
\begin{align*}
\rho^{2}(\eta;\epsilon,\alpha\mu r) &\coloneqq (1-\eta)^{2}+\frac{2\epsilon+\sqrt{\alpha\mu r}(6\sqrt{3}+\frac{8(2+\epsilon)}{1-\epsilon})}{1-\epsilon}\eta(1-\eta)\\
 & \quad+\frac{2\epsilon+\epsilon^{2}+\sqrt{\alpha\mu r}(6\sqrt{3}+4(2+\epsilon))+\alpha\mu r(3\sqrt{3}+\frac{4(2+\epsilon)}{1-\epsilon})^{2}}{(1-\epsilon)^{2}}\eta^{2}.
\end{align*}
With $\epsilon=0.02$, $\alpha\mu r\le10^{-4}$, and $0<\eta\le2/3$, one has $\rho(\eta;\epsilon,\alpha\mu r)\le1-0.6\eta$. Thus we conclude that
\begin{align}
\dist(\bF_{t+1},\bF_{\star}) &\le \sqrt{\left\Vert (\bL_{t+1}\bQ_{t}-\bL_{\star})\bSigma_{\star}^{1/2}\right\Vert _{\fro}^{2}+\left\Vert (\bR_{t+1}\bQ_{t}^{-\top}-\bR_{\star})\bSigma_{\star}^{1/2}\right\Vert _{\fro}^{2}}\nonumber \\
 &\le(1-0.6\eta)\dist(\bF_{t},\bF_{\star}).\label{eq:RPCA_bound}
\end{align}

\subsubsection{Incoherence condition}

We start by controlling the term $\|(\bL_{t+1}\bQ_{t}-\bL_{\star})\bSigma_{\star}^{1/2}\|_{2,\infty}$. We know from \eqref{eq:RPCA_Lt} that 
\begin{align*}
(\bL_{t+1}\bQ_{t}-\bL_{\star})\bSigma_{\star}^{1/2}=(1-\eta)\bDelta_{L}\bSigma_{\star}^{1/2}-\eta\bL_{\star}\bDelta_{R}^{\top}\bR(\bR^{\top}\bR)^{-1}\bSigma_{\star}^{1/2}-\eta(\bS-\bS_{\star})\bR(\bR^{\top}\bR)^{-1}\bSigma_{\star}^{1/2}.
\end{align*}
Apply the triangle inequality to obtain 
\begin{align*}
\left\Vert (\bL_{t+1}\bQ_{t}-\bL_{\star})\bSigma_{\star}^{1/2}\right\Vert _{2,\infty} & \le(1-\eta)\|\bDelta_{L}\bSigma_{\star}^{1/2}\|_{2,\infty}+\eta\underbrace{\left\Vert \bL_{\star}\bDelta_{R}^{\top}\bR(\bR^{\top}\bR)^{-1}\bSigma_{\star}^{1/2}\right\Vert _{2,\infty}}_{\mfk{T}_{1}}\\
 & \quad+\eta\underbrace{\left\Vert (\bS-\bS_{\star})\bR(\bR^{\top}\bR)^{-1}\bSigma_{\star}^{1/2}\right\Vert _{2,\infty}}_{\mfk{T}_{2}}.
\end{align*}
The first term $\|\bDelta_{L}\bSigma_{\star}^{1/2}\|_{2,\infty}$ follows from the incoherence condition \eqref{eq:cond_RPCA_2inf-1} as
\begin{align*}
\|\bDelta_{L}\bSigma_{\star}^{1/2}\|_{2,\infty}\le\sqrt{\frac{\mu r}{n_{1}}}\sigma_{r}(\bX_{\star}).
\end{align*}
In the sequel, we shall bound the terms $\mfk{T}_{1}$ and $\mfk{T}_{2}$. 
\begin{enumerate}
\item For the term $\mfk{T}_{1}$, use the relation $\|\bA\bB\|_{2,\infty}\le\|\bA\|_{2,\infty}\|\bB\|_{\op}$, and combine the condition \eqref{eq:cond_RPCA_op}  with the consequences \eqref{eq:consequences_RPCA} to obtain
\begin{align*}
\mfk{T}_{1} & \le\|\bL_{\star}\bSigma_{\star}^{-1/2}\|_{2,\infty}\left\Vert \bSigma_{\star}^{1/2}\bDelta_{R}^{\top}\bR(\bR^{\top}\bR)^{-1}\bSigma_{\star}^{1/2}\right\Vert _{\op}\\
 & \le\|\bL_{\star}\bSigma_{\star}^{-1/2}\|_{2,\infty}\|\bDelta_{R}\bSigma_{\star}^{1/2}\|_{\op}\left\Vert \bR(\bR^{\top}\bR)^{-1}\bSigma_{\star}^{1/2}\right\Vert _{\op}\\
 & \le\frac{\epsilon}{1-\epsilon}\sqrt{\frac{\mu r}{n_{1}}}\sigma_{r}(\bX_{\star}),
\end{align*}
\item For the term $\mfk{T}_{2}$, use the relation $\|\bA\bB\|_{2,\infty}\le\|\bA\|_{2,\infty}\|\bB\|_{\op}$ to obtain 
\begin{align*}
\mfk{T}_{2} & \le\|\bS-\bS_{\star}\|_{2,\infty}\left\Vert \bR(\bR^{\top}\bR)^{-1}\bSigma_{\star}^{1/2}\right\Vert _{\op}.
\end{align*}
We know from Lemma~\ref{lemma:S} that $\bS-\bS_{\star}$ has at most $3\alpha n_{2}$ non-zero entries in each row, and $\|\bS-\bS_{\star}\|_{\infty}\le2\|\bL\bR^{\top}-\bX_{\star}\|_{\infty}$. Upper bound the $\ell_{2,\infty}$ norm by the $\ell_{\infty}$ norm as 
\begin{align*}
\|\bS-\bS_{\star}\|_{2,\infty} & \le\sqrt{3\alpha n_{2}}\|\bS-\bS_{\star}\|_{\infty}\le2\sqrt{3\alpha n_{2}}\|\bL\bR^{\top}-\bX_{\star}\|_{\infty}.
\end{align*}
Split $\bL\bR^{\top}-\bX_{\star}=\bDelta_{L}\bR^{\top}+\bL_{\star}\bDelta_{R}^{\top}$, and take the conditions \eqref{eq:cond_RPCA_2inf-1} and \eqref{eq:cond_RPCA_2inf-3} to obtain 
\begin{align*}
\|\bL\bR^{\top}-\bX_{\star}\|_{\infty} & \le\|\bDelta_{L}\bR^{\top}\|_{\infty}+\|\bL_{\star}\bDelta_{R}^{\top}\|_{\infty}\\
 & \le\|\bDelta_{L}\bSigma_{\star}^{1/2}\|_{2,\infty}\|\bR\bSigma_{\star}^{-1/2}\|_{2,\infty}+\|\bL_{\star}\bSigma_{\star}^{-1/2}\|_{2,\infty}\|\bDelta_{R}\bSigma_{\star}^{1/2}\|_{2,\infty}\\
 & \le\sqrt{\frac{\mu r}{n_{1}}}\sigma_{r}(\bX_{\star})2\sqrt{\frac{\mu r}{n_{2}}}+\sqrt{\frac{\mu r}{n_{1}}}\sqrt{\frac{\mu r}{n_{2}}}\sigma_{r}(\bX_{\star}) \\
 & =\frac{3\mu r}{\sqrt{n_{1}n_{2}}}\sigma_{r}(\bX_{\star}).
\end{align*}
This combined with the consequences \eqref{eq:consequences_RPCA} yields
\begin{align*}
\mfk{T}_{2}\le\frac{6\sqrt{3\alpha\mu r}}{1-\epsilon}\sqrt{\frac{\mu r}{n_{1}}}\sigma_{r}(\bX_{\star}).
\end{align*}
\end{enumerate}
Taking collectively the bounds for $\mfk{T}_{1},\mfk{T}_{2}$ yields the control 
\begin{align}
\left\Vert (\bL_{t+1}\bQ_{t}-\bL_{\star})\bSigma_{\star}^{1/2}\right\Vert _{2,\infty}\le\left(1-\eta+\frac{\epsilon+6\sqrt{3\alpha\mu r}}{1-\epsilon}\eta\right)\sqrt{\frac{\mu r}{n_{1}}}\sigma_{r}(\bX_{\star}).\label{eq:RPCA_Lt_bound_2inf}
\end{align}

The last step is to switch the alignment matrix from $\bQ_{t}$ to $\bQ_{t+1}$. \eqref{eq:RPCA_bound} together with Lemma~\ref{lemma:Q_existence} demonstrates the existence of $\bQ_{t+1}$. Apply the triangle inequality to obtain 
\begin{align*}
 & \left\Vert (\bL_{t+1}\bQ_{t+1}-\bL_{\star})\bSigma_{\star}^{1/2}\right\Vert _{2,\infty}\le\left\Vert (\bL_{t+1}\bQ_{t}-\bL_{\star})\bSigma_{\star}^{1/2}\right\Vert _{2,\infty}+\left\Vert \bL_{t+1}(\bQ_{t+1}-\bQ_{t})\bSigma_{\star}^{1/2}\right\Vert _{2,\infty}\\
 & \qquad\le\left\Vert (\bL_{t+1}\bQ_{t}-\bL_{\star})\bSigma_{\star}^{1/2}\right\Vert _{2,\infty}+\|\bL_{t+1}\bQ_{t}\bSigma_{\star}^{-1/2}\|_{2,\infty}\left\Vert \bSigma_{\star}^{1/2}\bQ_{t}^{-1}\bQ_{t+1}\bSigma_{\star}^{1/2}-\bSigma_{\star}\right\Vert _{\op}.
\end{align*}
We deduct from \eqref{eq:RPCA_Lt_bound_2inf} that 
\begin{align*}
\|\bL_{t+1}\bQ_{t}\bSigma_{\star}^{-1/2}\|_{2,\infty} \le \|\bL_{\star}\bSigma_{\star}^{-1/2}\|_{2,\infty}+\left\Vert (\bL_{t+1}\bQ_{t}-\bL_{\star})\bSigma_{\star}^{-1/2}\right\Vert _{2,\infty}\le\left(2-\eta+\frac{\epsilon+6\sqrt{3\alpha\mu r}}{1-\epsilon}\eta\right)\sqrt{\frac{\mu r}{n_{1}}}.
\end{align*}
Regarding the alignment matrix term, invoke Lemma~\ref{lemma:Q_perturbation} to obtain 
\begin{align*}
\left\Vert \bSigma_{\star}^{1/2}\bQ_{t}^{-1}\bQ_{t+1}\bSigma_{\star}^{1/2}-\bSigma_{\star}\right\Vert _{\op} &\le \frac{\|(\bR_{t+1}(\bQ_{t}^{-\top}-\bQ_{t+1}^{-\top})\bSigma_{\star}^{1/2}\|_{\op}}{1-\|(\bR_{t+1}\bQ_{t+1}^{-\top}-\bR_{\star})\bSigma_{\star}^{-1/2}\|_{\op}}\\
 &\le\frac{\|(\bR_{t+1}\bQ_{t}^{-\top}-\bR_{\star})\bSigma_{\star}^{1/2}\|_{\op}+\|(\bR_{t+1}\bQ_{t+1}^{-\top}-\bR_{\star})\bSigma_{\star}^{1/2}\|_{\op}}{1-\|(\bR_{t+1}\bQ_{t+1}^{-\top}-\bR_{\star})\bSigma_{\star}^{-1/2}\|_{\op}} \\
 &\le\frac{2\epsilon}{1-\epsilon}\sigma_{r}(\bX_{\star}),
\end{align*}
where we deduct from \eqref{eq:RPCA_bound} that the distances using either $\bQ_{t}$ or $\bQ_{t+1}$ are bounded by 
\begin{align*}
\|(\bR_{t+1}\bQ_{t}^{-\top}-\bR_{\star})\bSigma_{\star}^{1/2}\|_{\op} & \le\epsilon\sigma_{r}(\bX_{\star});\\
\|(\bR_{t+1}\bQ_{t+1}^{-\top}-\bR_{\star})\bSigma_{\star}^{1/2}\|_{\op} & \le\epsilon\sigma_{r}(\bX_{\star});\\
\|(\bR_{t+1}\bQ_{t+1}^{-\top}-\bR_{\star})\bSigma_{\star}^{-1/2}\|_{\op} & \le\epsilon.
\end{align*}
Combine all pieces to reach 
\begin{align*}
\left\Vert (\bL_{t+1}\bQ_{t+1}-\bL_{\star})\bSigma_{\star}^{1/2}\right\Vert _{2,\infty}\le\left(\frac{1+\epsilon}{1-\epsilon}\left(1-\eta+\frac{\epsilon+6\sqrt{3\alpha\mu r}}{1-\epsilon}\eta\right)+\frac{2\epsilon}{1-\epsilon}\right)\sqrt{\frac{\mu r}{n_{1}}}\sigma_{r}(\bX_{\star}).
\end{align*}
With $\epsilon=0.02$, $\alpha\mu r\le10^{-4}$, and $0.1\le\eta\le2/3$, we get the desired incoherence condition 
\begin{align*}
\left\Vert (\bL_{t+1}\bQ_{t+1}-\bL_{\star})\bSigma_{\star}^{1/2}\right\Vert _{2,\infty}\le\sqrt{\frac{\mu r}{n_{1}}}\sigma_{r}(\bX_{\star}).
\end{align*}
Similarly, we can prove the other part 
\begin{align*}
\left\Vert (\bR_{t+1}\bQ_{t+1}^{-\top}-\bR_{\star})\bSigma_{\star}^{1/2}\right\Vert _{2,\infty}\le\sqrt{\frac{\mu r}{n_{2}}}\sigma_{r}(\bX_{\star}).
\end{align*}

\subsection{Proof of Lemma~\ref{lemma:init_RPCA}}

We first record two lemmas from \cite{yi2016fast}, which are useful for studying the properties of the initialization. 

\begin{lemma}[{\cite[Section~6.1]{yi2016fast}}]\label{lemma:S_inf}
Given $\bS_{\star}\in\cS_{\alpha}$, one has $\|\bS_{\star}-\cT_{\alpha}[\bX_{\star}+\bS_{\star}]\|_{\infty}\le2\|\bX_{\star}\|_{\infty}$.
\end{lemma}

\begin{lemma}[{\cite[Lemma~1]{yi2016fast}}]\label{lemma:S_op}
For any matrix $\bM\in\cS_{\alpha}$, one has $\|\bM\|_{\op}\le\alpha\sqrt{n_{1}n_{2}}\|\bM\|_{\infty}$.
\end{lemma}

With these two lemmas in place, we are ready to establish the claimed result. Invoke Lemma~\ref{lemma:Procrustes} to obtain
\begin{align*}
\dist(\bF_{0},\bF_{\star}) & \le \sqrt{\sqrt{2}+1}\left\Vert \bL_{0}\bR_{0}^{\top}-\bX_{\star}\right\Vert _{\fro} \le \sqrt{(\sqrt{2}+1)2r}\left\Vert \bL_{0}\bR_{0}^{\top}-\bX_{\star}\right\Vert _{\op},
\end{align*}
where the last relation uses the fact that $\bL_{0}\bR_{0}^{\top}-\bX_{\star}$ has rank at most $2r$. We can further apply the triangle inequality to see 
\begin{align*}
\left\Vert \bL_{0}\bR_{0}^{\top}-\bX_{\star}\right\Vert _{\op} & \le\left\Vert \bY-\cT_{\alpha}[\bY]-\bL_{0}\bR_{0}^{\top}\right\Vert _{\op}+\left\Vert \bY-\cT_{\alpha}[\bY]-\bX_{\star}\right\Vert _{\op}\\
 & \le2\left\Vert \bY-\cT_{\alpha}[\bY]-\bX_{\star}\right\Vert _{\op}=2\left\Vert \bS_{\star}-\cT_{\alpha}[\bX_{\star}+\bS_{\star}]\right\Vert _{\op}.
\end{align*}
Here the second inequality hinges on the fact that $\bL_{0}\bR_{0}^{\top}$ is the best rank-$r$ approximation of $\bY-\cT_{\alpha}[\bY]$, and the last identity arises from $\bY=\bX_{\star}+\bS_{\star}$. 
Follow the same argument as \cite[Section~6.1]{yi2016fast}, combining Lemmas~\ref{lemma:S_inf} and \ref{lemma:S_op} to reach
\begin{align*}
\left\Vert \bS_{\star}-\cT_{\alpha}[\bX_{\star}+\bS_{\star}]\right\Vert _{\op} & \le2\alpha\sqrt{n_{1}n_{2}}\left\Vert \bS_{\star}-\cT_{\alpha}[\bX_{\star}+\bS_{\star}]\right\Vert _{\infty}\\
 & \le4\alpha\sqrt{n_{1}n_{2}}\|\bX_{\star}\|_{\infty}\le4\alpha\mu r\kappa\sigma_{r}(\bX_{\star}),
\end{align*}
where the last inequality follows from the incoherence assumption
\begin{align}
\|\bX_{\star}\|_{\infty}\le\|\bU_{\star}\|_{2,\infty}\|\bSigma_{\star}\|_{\op}\|\bV_{\star}\|_{2,\infty}\le\frac{\mu r}{\sqrt{n_{1}n_{2}}}\kappa\sigma_{r}(\bX_{\star}).\label{eq:incoherence_inf}
\end{align}
Take the above inequalities together to arrive at
\begin{align*}
\dist(\bF_{0},\bF_{\star})\le8\sqrt{2(\sqrt{2}+1)}\alpha\mu r^{3/2}\kappa\sigma_{r}(\bX_{\star})\le20\alpha\mu r^{3/2}\kappa\sigma_{r}(\bX_{\star}).
\end{align*}

\subsection{Proof of Lemma~\ref{lemma:init_2inf_RPCA}}

In view of the condition $\dist(\bF_{0},\bF_{\star})\le0.02\sigma_{r}(\bX_{\star})$ and Lemma~\ref{lemma:Q_existence}, one knows that $\bQ_{0}$, the optimal alignment matrix between $\bF_{0}$ and $\bF_{\star}$ exists. Therefore, for notational convenience, denote $\bL\coloneqq\bL_{0}\bQ_{0}$, $\bR\coloneqq\bR_{0}\bQ_{0}^{-\top}$, $\bDelta_{L}\coloneqq\bL-\bL_{\star}$, $\bDelta_{R}\coloneqq\bR-\bR_{\star}$, and $\epsilon\coloneqq0.02$. Our objective is then translated to demonstrate
\begin{align*}
\sqrt{n_{1}}\|\bDelta_{L}\bSigma_{\star}^{1/2}\|_{2,\infty}\vee\sqrt{n_{2}}\|\bDelta_{R}\bSigma_{\star}^{1/2}\|_{2,\infty}\le\sqrt{\mu r}\sigma_{r}(\bX_{\star}).
\end{align*}
From now on, we focus on bounding $\|\bDelta_{L}\bSigma_{\star}^{1/2}\|_{2,\infty}$. 
Since $\bU_{0}\bSigma_{0}\bV_{0}^{\top}$ is the top-$r$ SVD of $\bY-\cT_{\alpha}[\bY]$, and recall that $\bY=\bX_{\star}+\bS_{\star}$, we have the relation
\begin{align*}
(\bX_{\star}+\bS_{\star}-\cT_{\alpha}[\bX_{\star}+\bS_{\star}])\bV_{0}=\bU_{0}\bSigma_{0},
\end{align*}
which further implies the following decomposition of $\bDelta_{L}\bSigma_{\star}^{1/2}$. 
\begin{claim}\label{claim:decomposition} One has 
\begin{align*}
\bDelta_{L}\bSigma_{\star}^{1/2}=(\bS_{\star}-\cT_{\alpha}[\bX_{\star}+\bS_{\star}])\bR(\bR^{\top}\bR)^{-1}\bSigma_{\star}^{1/2}-\bL_{\star}\bDelta_{R}^{\top}\bR(\bR^{\top}\bR)^{-1}\bSigma_{\star}^{1/2}.
\end{align*}
\end{claim}

Combining Claim~\ref{claim:decomposition} with the triangle inequality yields
\begin{align*}
\|\bDelta_{L}\bSigma_{\star}^{1/2}\|_{2,\infty}\le\underbrace{\left\Vert \bL_{\star}\bDelta_{R}^{\top}\bR(\bR^{\top}\bR)^{-1}\bSigma_{\star}^{1/2}\right\Vert _{2,\infty}}_{\mfk{I}_{1}}+\underbrace{\left\Vert (\bS_{\star}-\cT_{\alpha}[\bX_{\star}+\bS_{\star}])\bR(\bR^{\top}\bR)^{-1}\bSigma_{\star}^{1/2}\right\Vert _{2,\infty}}_{\mfk{I}_{2}}.
\end{align*}
In what follows, we shall control $\mfk{I}_{1}$ and $\mfk{I}_{2}$ in turn. 
\begin{enumerate}
\item For the term $\mfk{I}_{1}$, use the relation $\|\bA\bB\|_{2,\infty}\le\|\bA\|_{2,\infty}\|\bB\|_{\op}$ to obtain
\begin{align*}
\mfk{I}_{1}\le\|\bL_{\star}\bSigma_{\star}^{-1/2}\|_{2,\infty}\|\bDelta_{R}\bSigma_{\star}^{1/2}\|_{\op}\left\Vert \bR(\bR^{\top}\bR)^{-1}\bSigma_{\star}^{1/2}\right\Vert _{\op}.
\end{align*}
The incoherence assumption tells $\|\bL_{\star}\bSigma_{\star}^{-1/2}\|_{2,\infty}=\|\bU_{\star}\|_{2,\infty}\le\sqrt{\mu r/n_{1}}$.
In addition, the assumption $\dist(\bF_{0},\bF_{\star})\le\epsilon\sigma_{r}(\bX_{\star})$ entails the bound $\|\bDelta_{R}\bSigma_{\star}^{1/2}\|_{\op}\le\epsilon\sigma_{r}(\bX_{\star})$.
Finally, repeating the argument for obtaining \eqref{eq:cond_RPCA_op} yields $\|\bDelta_{R}\bSigma_{\star}^{-1/2}\|_{\op}\le\epsilon$, which together with Lemma~\ref{lemma:Weyl} reveals 
\begin{align*}
\left\Vert \bR(\bR^{\top}\bR)^{-1}\bSigma_{\star}^{1/2}\right\Vert _{\op}\le\frac{1}{1-\epsilon}.
\end{align*}
In all, we arrive at
\begin{align*}
\mfk{I}_{1}\le\frac{\epsilon}{1-\epsilon}\sqrt{\frac{\mu r}{n_{1}}}\sigma_{r}(\bX_{\star}).
\end{align*}
\item Proceeding to the term $\mfk{I}_{2}$, use the relations $\|\bA\bB\|_{2,\infty}\le\|\bA\|_{1,\infty}\|\bB\|_{2,\infty}$ and $\|\bA\bB\|_{2,\infty}\le\|\bA\|_{2,\infty}\|\bB\|_{\op}$ to obtain 
\begin{align*}
\mfk{I}_{2} & \le\left\Vert \bS_{\star}-\cT_{\alpha}[\bX_{\star}+\bS_{\star}]\right\Vert _{1,\infty}\left\Vert \bR(\bR^{\top}\bR)^{-1}\bSigma_{\star}^{1/2}\right\Vert _{2,\infty}\\
 & \le\left\Vert \bS_{\star}-\cT_{\alpha}[\bX_{\star}+\bS_{\star}]\right\Vert _{1,\infty}\|\bR\bSigma_{\star}^{-1/2}\|_{2,\infty}\left\Vert \bSigma_{\star}^{1/2}(\bR^{\top}\bR)^{-1}\bSigma_{\star}^{1/2}\right\Vert _{\op}.
\end{align*}
Regarding $\bS_{\star}-\cT_{\alpha}[\bX_{\star}+\bS_{\star}]$, Lemma~\ref{lemma:S_inf} tells that $\bS_{\star}-\cT_{\alpha}[\bX_{\star}+\bS_{\star}]$ has at most $2\alpha n_{2}$ non-zero entries in each row, and $\|\bS_{\star}-\cT_{\alpha}[\bX_{\star}+\bS_{\star}]\|_{\infty}\le2\|\bX_{\star}\|_{\infty}$.
Consequently, we can upper bound the $\ell_{1,\infty}$ norm by the $\ell_{\infty}$ norm as
\begin{align*}
\left\Vert \bS_{\star}-\cT_{\alpha}[\bX_{\star}+\bS_{\star}]\right\Vert _{1,\infty} & \le2\alpha n_{2}\left\Vert \bS_{\star}-\cT_{\alpha}[\bX_{\star}+\bS_{\star}]\right\Vert _{\infty}\\
 & \le4\alpha n_{2}\|\bX_{\star}\|_{\infty}\\
 & \le4\alpha n_{2}\frac{\mu r}{\sqrt{n_{1}n_{2}}}\kappa\sigma_{r}(\bX_{\star}).
\end{align*}
Here the last inequality follows from the incoherence assumption \eqref{eq:incoherence_inf}. For the term $\|\bR\bSigma_{\star}^{-1/2}\|_{2,\infty}$, one can apply the triangle inequality to see 
\begin{align*}
\|\bR\bSigma_{\star}^{-1/2}\|_{2,\infty}\le\|\bR_{\star}\bSigma_{\star}^{-1/2}\|_{2,\infty}+\|\bDelta_{R}\bSigma_{\star}^{-1/2}\|_{2,\infty}\le\sqrt{\frac{\mu r}{n_{2}}}+\frac{\|\bDelta_{R}\bSigma_{\star}^{1/2}\|_{2,\infty}}{\sigma_{r}(\bX_{\star})}.
\end{align*}
Last but not least, repeat the argument for \eqref{eq:consequences_RPCA} to obtain 
\begin{align*}
\left\Vert \bSigma_{\star}^{1/2}(\bR^{\top}\bR)^{-1}\bSigma_{\star}^{1/2}\right\Vert _{\op}=\left\Vert \bR(\bR^{\top}\bR)^{-1}\bSigma_{\star}^{1/2}\right\Vert _{\op}^{2}\le \frac{1}{(1-\epsilon)^{2}}.
\end{align*}
Taking together the above bounds yields 
\begin{align*}
\mfk{I}_{2} & \le\frac{4\alpha\mu r\kappa}{(1-\epsilon)^{2}}\sqrt{\frac{\mu r}{n_{1}}}\sigma_{r}(\bX_{\star})+\frac{4\alpha\mu r\kappa}{(1-\epsilon)^{2}}\sqrt{\frac{n_{2}}{n_{1}}}\|\bDelta_{R}\bSigma_{\star}^{1/2}\|_{2,\infty}.
\end{align*}
\end{enumerate}
Combine the bounds on $\mfk{I}_{1}$ and $\mfk{I}_{2}$ to reach 
\begin{align*}
\sqrt{n_{1}}\|\bDelta_{L}\bSigma_{\star}^{1/2}\|_{2,\infty}\le\left(\frac{\epsilon}{1-\epsilon}+\frac{4\alpha\mu r\kappa}{(1-\epsilon)^{2}}\right)\sqrt{\mu r}\sigma_{r}(\bX_{\star})+\frac{4\alpha\mu r\kappa}{(1-\epsilon)^{2}}\sqrt{n_{2}}\|\bDelta_{R}\bSigma_{\star}^{1/2}\|_{2,\infty}.
\end{align*}
Similarly, we have 
\begin{align*}
\sqrt{n_{2}}\|\bDelta_{R}\bSigma_{\star}^{1/2}\|_{2,\infty}\le\left(\frac{\epsilon}{1-\epsilon}+\frac{4\alpha\mu r\kappa}{(1-\epsilon)^{2}}\right)\sqrt{\mu r}\sigma_{r}(\bX_{\star})+\frac{4\alpha\mu r\kappa}{(1-\epsilon)^{2}}\sqrt{n_{1}}\|\bDelta_{L}\bSigma_{\star}^{1/2}\|_{2,\infty}.
\end{align*}
Taking the maximum and solving for $\sqrt{n_{1}}\|\bDelta_{L}\bSigma_{\star}^{1/2}\|_{2,\infty}\vee\sqrt{n_{2}}\|\bDelta_{L}\bSigma_{\star}^{1/2}\|_{2,\infty}$ yield the relation
\begin{align*}
\sqrt{n_{1}}\|\bDelta_{L}\bSigma_{\star}^{1/2}\|_{2,\infty}\vee\sqrt{n_{2}}\|\bDelta_{L}\bSigma_{\star}^{1/2}\|_{2,\infty}\le\frac{\epsilon(1-\epsilon)+4\alpha\mu r\kappa}{(1-\epsilon)^{2}-4\alpha\mu r\kappa}\sqrt{\mu r}\sigma_{r}(\bX_{\star}).
\end{align*}
With $\epsilon=0.02$ and $\alpha\mu r\kappa\le0.1$, we get the desired conclusion
\begin{align*}
\sqrt{n_{1}}\|\bDelta_{L}\bSigma_{\star}^{1/2}\|_{2,\infty}\vee\sqrt{n_{2}}\|\bDelta_{L}\bSigma_{\star}^{1/2}\|_{2,\infty}\le\sqrt{\mu r}\sigma_{r}(\bX_{\star}).	 
\end{align*}

\begin{proof}[Proof of Claim~\ref{claim:decomposition}] Identify $\bU_{0}$ (resp.~$\bV_{0}$) with $\bL_{0}\bSigma_{0}^{-1/2}$ (resp.~$\bR_{0}\bSigma_{0}^{-1/2}$) to yield
\begin{align*}
(\bX_{\star}+\bS_{\star}-\cT_{\alpha}[\bX_{\star}+\bS_{\star}])\bR_{0}\bSigma_{0}^{-1}=\bL_{0},
\end{align*}
which is equivalent to $(\bX_{\star}+\bS_{\star}-\cT_{\alpha}[\bX_{\star}+\bS_{\star}])\bR_{0}(\bR_{0}^{\top}\bR_{0})^{-1}=\bL_{0}$ since $\bSigma_{0}=\bR_{0}^{\top}\bR_{0}$. Multiply both sides by $\bQ_{0}\bSigma_{\star}^{1/2}$ to obtain 
\begin{align*}
(\bX_{\star}+\bS_{\star}-\cT_{\alpha}[\bX_{\star}+\bS_{\star}])\bR(\bR^{\top}\bR)^{-1}\bSigma_{\star}^{1/2}=\bL\bSigma_{\star}^{1/2},
\end{align*}
where we recall that $\bL=\bL_{0}\bQ_{0}$ and $\bR=\bR_{0}\bQ_{0}^{-\top}$.
In the end, subtract $\bX_{\star}\bR(\bR^{\top}\bR)^{-1}\bSigma_{\star}^{1/2}$ from both sides to reach 
\begin{align*}
(\bS_{\star}-\cT_{\alpha}[\bX_{\star}+\bS_{\star}])\bR(\bR^{\top}\bR)^{-1}\bSigma_{\star}^{1/2}
 & =\bL\bSigma_{\star}^{1/2}-\bL_{\star}\bR_{\star}^{\top}\bR(\bR^{\top}\bR)^{-1}\bSigma_{\star}^{1/2}\\
 & =(\bL-\bL_{\star}) \bSigma_{\star}^{1/2}+\bL_{\star}(\bR-\bR_{\star})^{\top}\bR(\bR^{\top}\bR)^{-1}\bSigma_{\star}^{1/2}\\
 & =\bDelta_{L}\bSigma_{\star}^{1/2}+\bL_{\star}\bDelta_{R}^{\top}\bR(\bR^{\top}\bR)^{-1}\bSigma_{\star}^{1/2}.
\end{align*}
This finishes the proof. 
\end{proof}

%% file: appendix-mc.tex
\section{Proof for Matrix Completion}\label{sec:proof_MC}

\subsection{New projection operator}

\subsubsection{Proof of Proposition~\ref{prop:scaled_proj_sol}}\label{proof:scaled_proj_sol}

First, notice that the optimization of $\bL$ and $\bR$ in \eqref{eq:scaled_proj_opt} can be decomposed and done in parallel, hence we focus on the optimization of $\bL$ below:
\begin{align*}
\bL=\argmin_{\bL\in\RR^{n_{1}\times r}}\;\left\Vert(\bL-\tilde{\bL})(\tilde{\bR}^{\top}\tilde{\bR})^{1/2}\right\Vert_{\fro}^{2} \qquad\mbox{s.t.}\quad\sqrt{n_{1}}\left\Vert\bL(\tilde{\bR}^{\top}\tilde{\bR})^{1/2}\right\Vert_{2,\infty} \le B.
\end{align*} 
By a change of variables as $\bG\coloneqq\bL(\tilde{\bR}^{\top}\tilde{\bR})^{1/2}$ and $\tilde{\bG}\coloneqq\tilde{\bL}(\tilde{\bR}^{\top}\tilde{\bR})^{1/2}$, we rewrite the above problem equivalently as
\begin{align*}
\bG=\argmin_{\bG\in\RR^{n_{1}\times r}}\;\|\bG-\tilde{\bG}\|_{\fro}^{2}\qquad\mbox{s.t.}\quad\sqrt{n_{1}}\left\Vert \bG \right\Vert_{2,\infty} \le B,
\end{align*} 
whose solution is given as \cite{chen2015fast}
\begin{align*}
\bG_{i,\cdot} & =\left(1\wedge\frac{B}{\sqrt{n_{1}}\|\tilde{\bG}_{i,\cdot}\|_{2}}\right)\tilde{\bG}_{i,\cdot}, \quad 1\le i \le n_{1}.  
\end{align*}
By applying again the change of variable $\bL=\bG(\tilde{\bR}^{\top}\tilde{\bR})^{-1/2}$ and $\tilde{\bL}=\tilde{\bG} (\tilde{\bR}^{\top}\tilde{\bR})^{-1/2}$, we obtain the claimed solution.

\subsubsection{Proof of Lemma~\ref{lemma:scaled_proj}}
 
We begin with proving the non-expansiveness property. Denote the optimal alignment matrix between $\tilde{\bF}$ and $\bF_{\star}$ as $\tilde{\bQ}$, whose existence is guaranteed by Lemma~\ref{lemma:Q_existence}. 
Denoting $\cP_{B}(\tilde{\bF})=[\bL^{\top},\bR^{\top}]^{\top}$, by the definition of $\dist(\cP_{B}(\tilde{\bF}),\bF_{\star})$, we know that 
\begin{align}
\dist^{2}(\cP_{B}(\tilde{\bF}),\bF_{\star}) & \le\sum_{i=1}^{n_{1}}\left\Vert \bL_{i,\cdot}\tilde{\bQ}\bSigma_{\star}^{1/2}-(\bL_{\star}\bSigma_{\star}^{1/2})_{i,\cdot}\right\Vert _{2}^{2}+\sum_{j=1}^{n_{2}}\left\Vert \bR_{j,\cdot}\tilde{\bQ}^{-\top}\bSigma_{\star}^{1/2}-(\bR_{\star}\bSigma_{\star}^{1/2})_{j,\cdot}\right\Vert _{2}^{2}.\label{eq:dist_trailer}
\end{align}
Recall that the condition $\dist(\tilde{\bF},\bF_{\star})\le\epsilon\sigma_{r}(\bX_{\star})$
implies 
\begin{align*}
\left\Vert (\tilde{\bL}\tilde{\bQ}-\bL_{\star})\bSigma_{\star}^{-1/2}\right\Vert _{\op}\vee\left\Vert (\tilde{\bR}\tilde{\bQ}^{-\top}-\bR_{\star})\bSigma_{\star}^{-1/2}\right\Vert _{\op}\le\epsilon,
\end{align*}
which, together with $\bR_{\star}\bSigma_{\star}^{-1/2}=\bV_{\star}$, further implies that 
\begin{align*}
\left\|\tilde{\bL}_{i,\cdot}\tilde{\bR}^{\top}\right\|_{2}& \le\left\Vert \tilde{\bL}_{i,\cdot}\tilde{\bQ}\bSigma_{\star}^{1/2}\right\Vert _{2}\left\Vert \tilde{\bR}\tilde{\bQ}^{-\top}\bSigma_{\star}^{-1/2}\right\Vert _{\op} \\
&\le\left\Vert \tilde{\bL}_{i,\cdot}\tilde{\bQ}\bSigma_{\star}^{1/2}\right\Vert _{2} \left(\|\bV_{\star}\|_{\op}+\left\Vert (\tilde{\bR}\tilde{\bQ}^{-\top}-\bR_{\star})\bSigma_{\star}^{-1/2}\right\Vert _{\op}\right) \le (1+\epsilon)\left\Vert \tilde{\bL}_{i,\cdot}\tilde{\bQ}\bSigma_{\star}^{1/2}\right\Vert _{2}.
\end{align*}
In addition, the $\mu$-incoherence of $\bX_{\star}$ yields
\begin{align*}
\sqrt{n_{1}}\left\Vert (\bL_{\star}\bSigma_{\star}^{1/2})_{i,\cdot}\right\Vert _{2} & \le\sqrt{n_{1}}\|\bU_{\star}\|_{2,\infty}\|\bSigma_{\star}\|_{\op}\le\sqrt{\mu r}\sigma_{1}(\bX_{\star})\le\frac{B}{1+\epsilon},
\end{align*}
where the last inequality follows from the choice of $B$. Take the above two relations collectively to reach 
\begin{align*}
\frac{B}{\sqrt{n_{1}}\|\tilde{\bL}_{i,\cdot}\tilde{\bR}^{\top}\|_{2}}\ge\frac{\left\Vert (\bL_{\star}\bSigma_{\star}^{1/2})_{i,\cdot}\right\Vert _{2}}{\left\Vert \tilde{\bL}_{i,\cdot}\tilde{\bQ}\bSigma_{\star}^{1/2}\right\Vert _{2}}.
\end{align*}
We claim that performing the following projection yields a contraction on each row; see also \cite[Lemma~11]{zheng2016convergence}.
\begin{claim}\label{claim:nonexpansive} For vectors $\bu,\bu_{\star}\in\RR^{n}$ and $\lambda\ge\|\bu_{\star}\|_{2}/\|\bu\|_{2}$, it holds that 
\begin{align*}
\|(1\wedge\lambda)\bu-\bu_{\star}\|_{2}\le\|\bu-\bu_{\star}\|_{2}.
\end{align*}
\end{claim}

Apply Claim~\ref{claim:nonexpansive} with $\bu\coloneqq\tilde{\bL}_{i,\cdot}\tilde{\bQ}\bSigma_{\star}^{1/2}$, $\bu_{\star}\coloneqq(\bL_{\star}\bSigma_{\star}^{1/2})_{i,\cdot}$,
and $\lambda\coloneqq B/(\sqrt{n_{1}}\|\tilde{\bL}_{i,\cdot}\tilde{\bR}^{\top}\|_{2})$ to obtain 
\begin{align*}
\left\Vert \bL_{i,\cdot}\tilde{\bQ}\bSigma_{\star}^{1/2}-(\bL_{\star}\bSigma_{\star}^{1/2})_{i,\cdot}\right\Vert _{2}^{2} & =\left\Vert \left(1\wedge\frac{B}{\sqrt{n_{1}}\|\tilde{\bL}_{i,\cdot}\tilde{\bR}^{\top}\|_{2}}\right)\tilde{\bL}_{i,\cdot}\tilde{\bQ}\bSigma_{\star}^{1/2}-(\bL_{\star}\bSigma_{\star}^{1/2})_{i,\cdot}\right\Vert _{2}^{2}\\
 & \le\left\Vert \tilde{\bL}_{i,\cdot}\tilde{\bQ}\bSigma_{\star}^{1/2}-(\bL_{\star}\bSigma_{\star}^{1/2})_{i,\cdot}\right\Vert _{2}^{2}.
\end{align*}
Following a similar argument for $\bR$, and plugging them back to \eqref{eq:dist_trailer}, we conclude that 
\begin{align*}
\dist^{2}(\cP_{B}(\tilde{\bF}),\bF_{\star}) & \le\sum_{i=1}^{n_{1}}\left\Vert \tilde{\bL}_{i,\cdot}\tilde{\bQ}\bSigma_{\star}^{1/2}-(\bL_{\star}\bSigma_{\star}^{1/2})_{i,\cdot}\right\Vert _{2}^{2}+\sum_{j=1}^{n_{2}}\left\Vert \tilde{\bR}_{j,\cdot}\tilde{\bQ}^{-\top}\bSigma_{\star}^{1/2}-(\bR_{\star}\bSigma_{\star}^{1/2})_{j,\cdot}\right\Vert _{2}^{2}=\dist^{2}(\tilde{\bF},\bF_{\star}).
\end{align*}
We move on to the incoherence condition. For any $1\le i\le n_{1}$, one has
\begin{align*}
\|\bL_{i,\cdot}\bR^{\top}\|_{2}^{2} & =\sum_{j=1}^{n_{2}}\langle\bL_{i,\cdot},\bR_{j,\cdot}\rangle^{2}=\sum_{j=1}^{n_{2}}\left(1\wedge\frac{B}{\sqrt{n_{1}}\|\tilde{\bL}_{i,\cdot}\tilde{\bR}^{\top}\|_{2}}\right)^{2}\langle\tilde{\bL}_{i,\cdot},\tilde{\bR}_{j,\cdot}\rangle^{2}\left(1\wedge\frac{B}{\sqrt{n_{2}}\|\tilde{\bR}_{j,\cdot}\tilde{\bL}^{\top}\|_{2}}\right)^{2}\\
 & \overset{\mathrm{(i)}}{\le} \left(1\wedge\frac{B}{\sqrt{n_{1}}\|\tilde{\bL}_{i,\cdot}\tilde{\bR}^{\top}\|_{2}}\right)^{2}\sum_{j=1}^{n_{2}}\langle\tilde{\bL}_{i,\cdot},\tilde{\bR}_{j,\cdot}\rangle^{2}=\left(1\wedge\frac{B}{\sqrt{n_{1}}\|\tilde{\bL}_{i,\cdot}\tilde{\bR}^{\top}\|_{2}}\right)^{2}\|\tilde{\bL}_{i,\cdot}\tilde{\bR}^{\top}\|_{2}^{2}\\
 & \overset{\mathrm{(ii)}}{\le} \frac{B^{2}}{n_{1}} .
\end{align*}
where $\mathrm{(i)}$ follows from $1\wedge\frac{B}{\sqrt{n_{2}}\|\tilde{\bR}_{j,\cdot}\tilde{\bL}^{\top}\|_{2}} \le 1$, and $\mathrm{(ii)}$ follows from $1\wedge\frac{B}{\sqrt{n_{1}}\|\tilde{\bL}_{i,\cdot}\tilde{\bR}^{\top}\|_{2}}\le \frac{B}{\sqrt{n_{1}}\|\tilde{\bL}_{i,\cdot}\tilde{\bR}^{\top}\|_{2}}$. Similarly, one has $\|\bR_{j,\cdot}\bL^{\top}\|_{2}^{2}\le B^{2}/n_{2}$.
Combining these two bounds completes the proof. 

\begin{proof}[Proof of Claim~\ref{claim:nonexpansive}] 
When $\lambda>1$, the claim holds as an identity. Otherwise $\lambda\le1$. Denote $h(\bar{\lambda})\coloneqq\|\bar{\lambda}\bu-\bu_{\star}\|_{2}^{2}$. Calculate its derivative to conclude that $h(\bar{\lambda})$ is monotonically increasing when $\bar{\lambda}\ge\lambda_{\star}\coloneqq\langle\bu,\bu_{\star}\rangle/\|\bu\|_{2}^{2}$. Note that $\lambda\ge\|\bu_{\star}\|_{2}/\|\bu\|_{2}\ge\lambda_{\star}$, thus $h(\lambda)\le h(1)$, i.e.~the claim holds. 
\end{proof}

\subsection{Proof of Lemma~\ref{lemma:contraction_MC}}

We first record two useful lemmas regarding the projector $\cP_{\Omega}(\cdot)$.

\begin{lemma}[{\cite[Lemma~10]{zheng2016convergence}}]
\label{lemma:P_Omega_tangent} Suppose that $\bX_{\star}$ is $\mu$-incoherent, and $p\gtrsim\mu r\log(n_{1}\vee n_{2})/(n_{1}\wedge n_{2})$. With overwhelming probability, one has 
\begin{align*}
 & \left|\left\langle (p^{-1}\cP_{\Omega}-\cI)(\bL_{\star}\bR_{A}^{\top}+\bL_{A}\bR_{\star}^{\top}),\bL_{\star}\bR_{B}^{\top}+\bL_{B}\bR_{\star}^{\top}\right\rangle \right| \\
 & \qquad\qquad \le C_{1}\sqrt{\frac{\mu r\log(n_{1}\vee n_{2})}{p(n_{1}\wedge n_{2})}}\|\bL_{\star}\bR_{A}^{\top}+\bL_{A}\bR_{\star}^{\top}\|_{\fro}\|\bL_{\star}\bR_{B}^{\top}+\bL_{B}\bR_{\star}^{\top}\|_{\fro},
\end{align*}
simultaneously for all $\bL_{A},\bL_{B}\in\RR^{n_{1}\times r}$ and $\bR_{A},\bR_{B}\in\RR^{n_{2}\times r}$, where $C_{1}>0$ is some universal constant.
\end{lemma}

\begin{lemma}[{\cite[Lemma~8]{chen2019model},\cite[Lemma~12]{chen2019nonconvex}}]
\label{lemma:P_Omega_Chen} Suppose that $p\gtrsim\log(n_{1}\vee n_{2})/(n_{1}\wedge n_{2})$. With overwhelming probability, one has
\begin{align*}
 & \left|\left\langle (p^{-1}\cP_{\Omega}-\cI)(\bL_{A}\bR_{A}^{\top}),\bL_{B}\bR_{B}^{\top}\right\rangle \right| \\
 & \qquad\qquad \le C_{2}\sqrt{\frac{n_{1}\vee n_{2}}{p}} \left(\|\bL_{A}\|_{\fro}\|\bL_{B}\|_{2,\infty}\wedge\|\bL_{A}\|_{2,\infty}\|\bL_{B}\|_{\fro}\right)\left(\|\bR_{A}\|_{\fro}\|\bR_{B}\|_{2,\infty}\wedge\|\bR_{A}\|_{2,\infty}\|\bR_{B}\|_{\fro}\right),
\end{align*}
simultaneously for all $\bL_{A},\bL_{B}\in\RR^{n_{1}\times r}$ and $\bR_{A},\bR_{B}\in\RR^{n_{2}\times r}$, where $C_{2}>0$ is some universal constant.
\end{lemma}

In view of the above two lemmas, define the event $\cE$ as the intersection of the events that the bounds in Lemmas~\ref{lemma:P_Omega_tangent} and \ref{lemma:P_Omega_Chen} hold, which happens with overwhelming probability. The rest of the proof is then performed under the event that $\cE$ holds. 

By the condition $\dist(\bF_{t},\bF_{\star})\le0.02\sigma_{r}(\bX_{\star})$ and Lemma~\ref{lemma:Q_existence}, one knows that $\bQ_{t}$, the optimal alignment matrix between $\bF_{t}$ and $\bF_{\star}$ exists.
Therefore, for notational convenience, we denote $\bL\coloneqq\bL_{t}\bQ_{t}$, $\bR\coloneqq\bR_{t}\bQ_{t}^{-\top}$, $\bDelta_{L}\coloneqq\bL-\bL_{\star}$, $\bDelta_{R}\coloneqq\bR-\bR_{\star}$, and $\epsilon\coloneqq0.02$. In addition, denote $\tilde{\bF}_{t+1}$ as the update before projection as
\begin{align*}
\tilde{\bF}_{t+1}\coloneqq \begin{bmatrix}\tilde{\bL}_{t+1} \\ \tilde{\bR}_{t+1}\end{bmatrix}=\begin{bmatrix}\bL_{t}-\eta p^{-1}\cP_{\Omega}(\bL_{t}\bR_{t}^{\top}-\bX_{\star})\bR_{t}(\bR_{t}^{\top}\bR_{t})^{-1}\\
\bR_{t}-\eta p^{-1}\cP_{\Omega}(\bL_{t}\bR_{t}^{\top}-\bX_{\star})^{\top}\bL_{t}(\bL_{t}^{\top}\bL_{t})^{-1}
\end{bmatrix},
\end{align*}
and therefore $\bF_{t+1}=\cP_{B}(\tilde{\bF}_{t+1})$. 
Note that in view of Lemma~\ref{lemma:scaled_proj}, it suffices to prove the following relation 
\begin{align}
\dist(\tilde{\bF}_{t+1},\bF_{\star})\le(1-0.6\eta)\dist(\bF_{t},\bF_{\star}).\label{eq:MC_goal}
\end{align}
The conclusion $\|\bL_{t}\bR_{t}^{\top}-\bX_{\star}\|_{\fro}\le1.5\dist(\bF_{t},\bF_{\star})$
is a simple consequence of Lemma~\ref{lemma:matrix2factor}; see \eqref{eq:dist_matrix} for a detailed argument. In what follows, we concentrate on proving \eqref{eq:MC_goal}. 

To begin with, we list a few easy consequences under the assumed conditions.
\begin{claim}\label{claim:cond_MC} Under conditions $\dist(\bF_{t},\bF_{\star})\le\epsilon\sigma_{r}(\bX_{\star})$ and $\sqrt{n_{1}}\|\bL\bR^{\top}\|_{2,\infty}\vee\sqrt{n_{2}}\|\bR\bL^{\top}\|_{2,\infty}\le C_{B}\sqrt{\mu r}\sigma_{1}(\bX_{\star})$, one has
\begin{subequations}
\begin{align}
\|\bDelta_{L}\bSigma_{\star}^{-1/2}\|_{\op}\vee\|\bDelta_{R}\bSigma_{\star}^{-1/2}\|_{\op} & \le\epsilon;\label{eq:cond_MC_op} \\
\left\Vert \bR(\bR^{\top}\bR)^{-1}\bSigma_{\star}^{1/2}\right\Vert _{\op} & \le\frac{1}{1-\epsilon}; \label{eq:MC_consequences-1}\\
\left\Vert \bSigma_{\star}^{1/2}(\bR^{\top}\bR)^{-1}\bSigma_{\star}^{1/2}\right\Vert _{\op} & \le\frac{1}{(1-\epsilon)^{2}}; \label{eq:MC_consequences-2} \\
\sqrt{n_{1}}\|\bL\bSigma_{\star}^{1/2}\|_{2,\infty}\vee\sqrt{n_{2}}\|\bR\bSigma_{\star}^{1/2}\|_{2,\infty} & \le\frac{C_{B}}{1-\epsilon}\sqrt{\mu r}\sigma_{1}(\bX_{\star}); \label{eq:cond_MC_2inf-1p}\\
\sqrt{n_{1}}\|\bL\bSigma_{\star}^{-1/2}\|_{2,\infty}\vee\sqrt{n_{2}}\|\bR\bSigma_{\star}^{-1/2}\|_{2,\infty} & \le\frac{C_{B}\kappa}{1-\epsilon}\sqrt{\mu r}; \label{eq:cond_MC_2inf-1n} \\
\sqrt{n_{1}}\|\bDelta_{L}\bSigma_{\star}^{1/2}\|_{2,\infty}\vee\sqrt{n_{2}}\|\bDelta_{R}\bSigma_{\star}^{1/2}\|_{2,\infty} & \le \left(1+\frac{C_{B}}{1-\epsilon}\right)\sqrt{\mu r}\sigma_{1}(\bX_{\star}). \label{eq:cond_MC_2inf-2p}
\end{align}
\end{subequations}
\end{claim}

Now we are ready to embark on the proof of \eqref{eq:MC_goal}. By the definition of $\dist(\tilde{\bF}_{t+1},\bF_{\star})$, one has 
\begin{align}
\dist^{2}(\tilde{\bF}_{t+1},\bF_{\star}) & \le\left\Vert (\tilde{\bL}_{t+1}\bQ_{t}-\bL_{\star})\bSigma_{\star}^{1/2}\right\Vert _{\fro}^{2}+\left\Vert (\tilde{\bR}_{t+1}\bQ_{t}^{-\top}-\bR_{\star})\bSigma_{\star}^{1/2}\right\Vert _{\fro}^{2},\label{eq:MC_expand}
\end{align}
where we recall that $\bQ_{t}$ is the optimal alignment matrix between $\bF_{t}$ and $\bF_{\star}$. 
Plug in the update rule \eqref{eq:iterates_MC} and the decomposition $\bL\bR^{\top}-\bX_{\star}=\bDelta_{L}\bR^{\top}+\bL_{\star}\bDelta_{R}^{\top}$ to obtain 
\begin{align*}
& (\tilde{\bL}_{t+1}\bQ_{t}-\bL_{\star})\bSigma_{\star}^{1/2}  =\left(\bL-\eta p^{-1}\cP_{\Omega}(\bL\bR^{\top}-\bX_{\star})\bR(\bR^{\top}\bR)^{-1}-\bL_{\star}\right)\bSigma_{\star}^{1/2}\\
 &\qquad =\bDelta_{L}\bSigma_{\star}^{1/2}-\eta(\bL\bR^{\top}-\bX_{\star})\bR(\bR^{\top}\bR)^{-1}\bSigma_{\star}^{1/2}-\eta(p^{-1}\cP_{\Omega}-\cI)(\bL\bR^{\top}-\bX_{\star})\bR(\bR^{\top}\bR)^{-1}\bSigma_{\star}^{1/2} \\
 &\qquad =(1-\eta)\bDelta_{L}\bSigma_{\star}^{1/2}-\eta\bL_{\star}\bDelta_{R}^{\top}\bR(\bR^{\top}\bR)^{-1}\bSigma_{\star}^{1/2}-\eta(p^{-1}\cP_{\Omega}-\cI)(\bL\bR^{\top}-\bX_{\star})\bR(\bR^{\top}\bR)^{-1}\bSigma_{\star}^{1/2}.
\end{align*}
This allows us to expand the first square in \eqref{eq:MC_expand} as 
\begin{align*}
\left\Vert (\tilde{\bL}_{t+1}\bQ_{t}-\bL_{\star})\bSigma_{\star}^{1/2}\right\Vert _{\fro}^{2} & =\underbrace{\left\Vert (1-\eta)\bDelta_{L}\bSigma_{\star}^{1/2}-\eta\bL_{\star}\bDelta_{R}^{\top}\bR(\bR^{\top}\bR)^{-1}\bSigma_{\star}^{1/2}\right\Vert _{\fro}^{2}}_{\mfk{P}_{1}}\\
 & \quad-2\eta(1-\eta)\underbrace{\tr\left((p^{-1}\cP_{\Omega}-\cI)(\bL\bR^{\top}-\bX_{\star})\bR(\bR^{\top}\bR)^{-1}\bSigma_{\star}\bDelta_{L}^{\top}\right)}_{\mfk{P}_{2}}\\
 & \quad+2\eta^{2}\underbrace{\tr\left((p^{-1}\cP_{\Omega}-\cI)(\bL\bR^{\top}-\bX_{\star})\bR(\bR^{\top}\bR)^{-1}\bSigma_{\star}(\bR^{\top}\bR)^{-1}\bR^{\top}\bDelta_{R}\bL_{\star}^{\top}\right)}_{\mfk{P}_{3}}\\
 & \quad+\eta^{2}\underbrace{\left\Vert (p^{-1}\cP_{\Omega}-\cI)(\bL\bR^{\top}-\bX_{\star})\bR(\bR^{\top}\bR)^{-1}\bSigma_{\star}^{1/2}\right\Vert _{\fro}^{2}}_{\mfk{P}_{4}}.
\end{align*}
In the sequel, we shall control the four terms separately, of which $\mfk{P}_{1}$ is the main term, and $\mfk{P}_{2},\mfk{P}_{3}$ and $\mfk{P}_{4}$ are perturbation terms. 
\begin{enumerate}
\item Notice that the main term $\mfk{P}_{1}$ has already been controlled in \eqref{eq:MF_Lt_bound} under the condition \eqref{eq:cond_MC_op}. It obeys 
\begin{align*}
\mfk{P}_{1}\le\left((1-\eta)^{2}+\frac{2\epsilon}{1-\epsilon}\eta(1-\eta)\right)\|\bDelta_{L}\bSigma_{\star}^{1/2}\|_{\fro}^{2}+\frac{2\epsilon+\epsilon^{2}}{(1-\epsilon)^{2}}\eta^{2}\|\bDelta_{R}\bSigma_{\star}^{1/2}\|_{\fro}^{2}.
\end{align*}
\item For the second term $\mfk{P}_{2}$, decompose $\bL\bR^{\top}-\bX_{\star}=\bDelta_{L}\bR_{\star}^{\top}+\bL\bDelta_{R}^{\top}$ and apply the triangle inequality to obtain 
\begin{align*}
|\mfk{P}_{2}| & =\Big|\tr\left((p^{-1}\cP_{\Omega}-\cI)(\bDelta_{L}\bR_{\star}^{\top}+\bL\bDelta_{R}^{\top})\bR(\bR^{\top}\bR)^{-1}\bSigma_{\star}\bDelta_{L}^{\top}\right)\Big|\\
 & \le\underbrace{\Big|\tr\left((p^{-1}\cP_{\Omega}-\cI)(\bDelta_{L}\bR_{\star}^{\top})\bR_{\star}(\bR^{\top}\bR)^{-1}\bSigma_{\star}\bDelta_{L}^{\top}\right)\Big|}_{\mfk{P}_{2,1}}\\
 & \quad+\underbrace{\Big|\tr\left((p^{-1}\cP_{\Omega}-\cI)(\bDelta_{L}\bR_{\star}^{\top})\bDelta_{R}(\bR^{\top}\bR)^{-1}\bSigma_{\star}\bDelta_{L}^{\top}\right)\Big|}_{\mfk{P}_{2,2}}\\
 & \quad+\underbrace{\Big|\tr\left((p^{-1}\cP_{\Omega}-\cI)(\bL\bDelta_{R}^{\top})\bR(\bR^{\top}\bR)^{-1}\bSigma_{\star}\bDelta_{L}^{\top}\right)\Big|}_{\mfk{P}_{2,3}}.
\end{align*}
For the first term $\mfk{P}_{2,1}$, under the event $\cE$, we can invoke Lemma~\ref{lemma:P_Omega_tangent} to obtain 
\begin{align*}
\mfk{P}_{2,1} & \le C_{1}\sqrt{\frac{\mu r\log(n_{1}\vee n_{2})}{p(n_{1}\wedge n_{2})}}\|\bDelta_{L}\bR_{\star}^{\top}\|_{\fro}\left\Vert \bDelta_{L}\bSigma_{\star}(\bR^{\top}\bR)^{-1}\bR_{\star}^{\top}\right\Vert _{\fro}\\
 & \le C_{1}\sqrt{\frac{\mu r\log(n_{1}\vee n_{2})}{p(n_{1}\wedge n_{2})}}\|\bDelta_{L}\bSigma_{\star}^{1/2}\|_{\fro}^{2}\left\Vert \bSigma_{\star}^{1/2}(\bR^{\top}\bR)^{-1}\bSigma_{\star}^{1/2}\right\Vert _{\op},
\end{align*}
where the second line follows from the relation $\|\bA\bB\|_{\fro}\le\|\bA\|_{\op}\|\bB\|_{\fro}$. 
Use the condition~\eqref{eq:MC_consequences-2} to obtain 
\begin{align*}
\mfk{P}_{2,1} & \le\frac{C_{1}}{(1-\epsilon)^{2}}\sqrt{\frac{\mu r\log(n_{1}\vee n_{2})}{p(n_{1}\wedge n_{2})}}\|\bDelta_{L}\bSigma_{\star}^{1/2}\|_{\fro}^2.
\end{align*}
Regarding the remaining terms $\mfk{P}_{2,2}$ and $\mfk{P}_{2,3}$, our main hammer is Lemma~\ref{lemma:P_Omega_Chen}. Invoking Lemma~\ref{lemma:P_Omega_Chen} under the event $\cE$ with $\bL_{A}\coloneqq\bDelta_{L}\bSigma_{\star}^{1/2}$, $\bR_{A}\coloneqq\bR_{\star}\bSigma_{\star}^{-1/2}$, $\bL_{B}\coloneqq\bDelta_{L}\bSigma_{\star}^{1/2}$, and $\bR_{B}\coloneqq\bDelta_{R}(\bR^{\top}\bR)^{-1}\bSigma_{\star}^{1/2}$, we arrive at 
\begin{align*}
\mfk{P}_{2,2} & \le C_{2}\sqrt{\frac{n_{1}\vee n_{2}}{p}}\|\bDelta_{L}\bSigma_{\star}^{1/2}\|_{2,\infty}\|\bDelta_{L}\bSigma_{\star}^{1/2}\|_{\fro}\|\bR_{\star}\bSigma_{\star}^{-1/2}\|_{2,\infty}\left\Vert \bDelta_{R}(\bR^{\top}\bR)^{-1}\bSigma_{\star}^{1/2}\right\Vert _{\fro}\\
 & \le C_{2}\sqrt{\frac{n_{1}\vee n_{2}}{p}}\|\bDelta_{L}\bSigma_{\star}^{1/2}\|_{2,\infty}\|\bDelta_{L}\bSigma_{\star}^{1/2}\|_{\fro}\|\bR_{\star}\bSigma_{\star}^{-1/2}\|_{2,\infty}\|\bDelta_{R}\bSigma_{\star}^{-1/2}\|_{\fro}\left\Vert \bSigma_{\star}^{1/2}(\bR^{\top}\bR)^{-1}\bSigma_{\star}^{1/2}\right\Vert _{\op}.
\end{align*}
Similarly, with the help of Lemma~\ref{lemma:P_Omega_Chen}, one has 
\begin{align*}
\mfk{P}_{2,3} & \le C_{2}\sqrt{\frac{n_{1}\vee n_{2}}{p}}\|\bL\bSigma_{\star}^{-1/2}\|_{2,\infty}\|\bDelta_{L}\bSigma_{\star}^{1/2}\|_{\fro}\|\bDelta_{R}\bSigma_{\star}^{1/2}\|_{\fro}\|\bR\bSigma_{\star}^{-1/2}\|_{2,\infty}\left\Vert \bSigma_{\star}^{1/2}(\bR^{\top}\bR)^{-1}\bSigma_{\star}^{1/2}\right\Vert _{\op}.
\end{align*}
Utilizing the consequences in Claim~\ref{claim:cond_MC}, we arrive at
\begin{align*}
\mfk{P}_{2,2} & \le\frac{C_{2}\kappa}{(1-\epsilon)^{2}}\left(1+\frac{C_{B}}{1-\epsilon}\right)\frac{\mu r}{\sqrt{p(n_{1}\wedge n_{2})}}\|\bDelta_{L}\bSigma_{\star}^{1/2}\|_{\fro}\|\bDelta_{R}\bSigma_{\star}^{1/2}\|_{\fro}; \\
\mfk{P}_{2,3} & \le\frac{C_{2}C_{B}^{2}\kappa^{2}}{(1-\epsilon)^{4}}\frac{\mu r}{\sqrt{p(n_{1}\wedge n_{2})}}\|\bDelta_{L}\bSigma_{\star}^{1/2}\|_{\fro}\|\bDelta_{R}\bSigma_{\star}^{1/2}\|_{\fro}.
\end{align*}
We then combine the bounds for $\mfk{P}_{2,1},\mfk{P}_{2,2}$ and $\mfk{P}_{2,3}$  to see 
\begin{align*}
\mfk{P}_{2} & \le\frac{C_{1}}{(1-\epsilon)^{2}}\sqrt{\frac{\mu r\log(n_{1}\vee n_{2})}{p(n_{1}\wedge n_{2})}}\|\bDelta_{L}\bSigma_{\star}^{1/2}\|_{\fro}^{2} \\
 & \quad+\frac{C_{2}\kappa}{(1-\epsilon)^{2}}\left(1+\frac{C_{B}}{1-\epsilon}+\frac{C_B^2\kappa}{(1-\epsilon)^2}\right) \frac{\mu r}{\sqrt{p(n_{1}\wedge n_{2})}}\|\bDelta_{L}\bSigma_{\star}^{1/2}\|_{\fro}\|\bDelta_{R}\bSigma_{\star}^{1/2}\|_{\fro} \\
 & =\delta_{1}\|\bDelta_{L}\bSigma_{\star}^{1/2}\|_{\fro}^{2}+\delta_{2}\|\bDelta_{L}\bSigma_{\star}^{1/2}\|_{\fro}\|\bDelta_{R}\bSigma_{\star}^{1/2}\|_{\fro}\\
 & \le (\delta_{1}+\frac{\delta_{2}}{2})\|\bDelta_{L}\bSigma_{\star}^{1/2}\|_{\fro}^{2}+\frac{\delta_{2}}{2}\|\bDelta_{R}\bSigma_{\star}^{1/2}\|_{\fro}^{2},
\end{align*}
where we denote 
\begin{align*}
\delta_{1}\coloneqq\frac{C_{1}}{(1-\epsilon)^{2}}\sqrt{\frac{\mu r\log(n_{1}\vee n_{2})}{p(n_{1}\wedge n_{2})}},\quad\mbox{and}\quad\delta_{2}\coloneqq\frac{C_{2}\kappa}{(1-\epsilon)^{2}}\left(1+\frac{C_{B}}{1-\epsilon}+\frac{C_B^2\kappa}{(1-\epsilon)^2}\right) \frac{\mu r}{\sqrt{p(n_{1}\wedge n_{2})}}.
\end{align*}

\item Following a similar argument for controlling $\mfk{P}_{2}$ (i.e.~repeatedly using Lemmas~\ref{lemma:P_Omega_tangent} and \ref{lemma:P_Omega_Chen}), we can obtain the following bounds for $\mfk{P}_{3}$ and $\mfk{P}_{4}$, whose proof are deferred to the end of this section. 
\begin{claim}\label{claim:P3-4} Under the event $\cE$, one has 
\begin{align*}
\mfk{P}_{3} & \le\frac{\delta_{2}}{2}\|\bDelta_{L}\bSigma_{\star}^{1/2}\|_{\fro}^{2}+(\delta_{1}+\frac{\delta_{2}}{2})\|\bDelta_{R}\bSigma_{\star}^{1/2}\|_{\fro}^{2};\\
\mfk{P}_{4} & \le\delta_{1}(\delta_{1}+\delta_{2})\|\bDelta_{L}\bSigma_{\star}^{1/2}\|_{\fro}^{2}+\delta_{2}(\delta_{1}+\delta_{2})\|\bDelta_{R}\bSigma_{\star}^{1/2}\|_{\fro}^{2}.
\end{align*}
\end{claim}
\end{enumerate}
Taking the bounds for $\mfk{P}_{1},\mfk{P}_{2},\mfk{P}_{3}$ and $\mfk{P}_{4}$ collectively yields 
\begin{align*}
\left\Vert (\tilde{\bL}_{t+1}\bQ_{t}-\bL_{\star})\bSigma_{\star}^{1/2}\right\Vert _{\fro}^{2} & \le\left((1-\eta)^{2}+\frac{2\epsilon}{1-\epsilon}\eta(1-\eta)\right)\|\bDelta_{L}\bSigma_{\star}^{1/2}\|_{\fro}^{2}+\frac{2\epsilon+\epsilon^{2}}{(1-\epsilon)^{2}}\eta^{2}\|\bDelta_{R}\bSigma_{\star}^{1/2}\|_{\fro}^{2}\\
 & \quad+\eta(1-\eta)\left((2\delta_{1}+\delta_{2})\|\bDelta_{L}\bSigma_{\star}^{1/2}\|_{\fro}^{2}+\delta_{2}\|\bDelta_{R}\bSigma_{\star}^{1/2}\|_{\fro}^{2}\right)\\
 & \quad+\eta^{2}\left(\delta_{2}\|\bDelta_{L}\bSigma_{\star}^{1/2}\|_{\fro}^{2}+(2\delta_{1}+\delta_{2})\|\bDelta_{R}\bSigma_{\star}^{1/2}\|_{\fro}^{2}\right)\\
 & \quad+\eta^{2}\left(\delta_{1}(\delta_{1}+\delta_{2})\|\bDelta_{L}\bSigma_{\star}^{1/2}\|_{\fro}^{2}+\delta_{2}(\delta_{1}+\delta_{2})\|\bDelta_{R}\bSigma_{\star}^{1/2}\|_{\fro}^{2}\right).
\end{align*}
A similar upper bound holds for the second square in \eqref{eq:MC_expand}.
As a result, we reach the conclusion that 
\begin{align*}
\left\Vert (\tilde{\bL}_{t+1}\bQ_{t}-\bL_{\star})\bSigma_{\star}^{1/2}\right\Vert _{\fro}^{2}+\left\Vert (\tilde{\bR}_{t+1}\bQ_{t}^{-\top}-\bR_{\star})\bSigma_{\star}^{1/2}\right\Vert _{\fro}^{2} & \le\rho^{2}(\eta;\epsilon,\delta_{1},\delta_{2})\dist^{2}(\bF_{t},\bF_{\star}),
\end{align*}
where the contraction rate $\rho^{2}(\eta;\epsilon,\delta_{1},\delta_{2})$
is given by 
\begin{align*}
\rho^{2}(\eta;\epsilon,\delta_{1},\delta_{2})\coloneqq(1-\eta)^{2}+\left(\frac{2\epsilon}{1-\epsilon}+2(\delta_{1}+\delta_{2})\right)\eta(1-\eta)+\left(\frac{2\epsilon+\epsilon^{2}}{(1-\epsilon)^{2}}+2(\delta_{1}+\delta_{2})+(\delta_{1}+\delta_{2})^{2}\right)\eta^{2}.
\end{align*}
As long as $p\ge C(\mu r\kappa^{4}\vee\log(n_{1}\vee n_{2}))\mu r/(n_{1}\wedge n_{2})$ for some sufficiently large constant $C$, one has $\delta_{1}+\delta_{2}\le0.1$ under the setting $\epsilon=0.02$. When $0<\eta\le2/3$, one further has $\rho(\eta;\epsilon,\delta_{1},\delta_{2})\le1-0.6\eta$.
Thus we conclude that
\begin{align*}
\dist(\tilde{\bF}_{t+1},\bF_{\star}) & \le\sqrt{\left\Vert (\tilde{\bL}_{t+1}\bQ_{t}-\bL_{\star})\bSigma_{\star}^{1/2}\right\Vert _{\fro}^{2}+\left\Vert (\tilde{\bR}_{t+1}\bQ_{t}^{-\top}-\bR_{\star})\bSigma_{\star}^{1/2}\right\Vert _{\fro}^{2}}\\
 & \le(1-0.6\eta)\dist(\bF_{t},\bF_{\star}),
\end{align*}
 which is exactly the upper bound we are after; see~\eqref{eq:MC_goal}. This finishes the proof. 

\begin{proof}[Proof of Claim~\ref{claim:cond_MC}] 
First, repeating the derivation for \eqref{eq:cond_MF} obtains \eqref{eq:cond_MC_op}. 
Second, take the condition \eqref{eq:cond_MC_op} and Lemma~\ref{lemma:Weyl} together to obtain \eqref{eq:MC_consequences-1} and \eqref{eq:MC_consequences-2}.
Third, take the incoherence condition $\sqrt{n_{1}}\|\bL\bR^{\top}\|_{2,\infty}\vee\sqrt{n_{2}}\|\bR\bL^{\top}\|_{2,\infty}\le C_{B}\sqrt{\mu r}\sigma_{1}(\bX_{\star})$ together with the relations
\begin{align*}
\|\bL\bR^{\top}\|_{2,\infty} & \ge\sigma_{r}(\bR\bSigma_{\star}^{-1/2})\|\bL\bSigma_{\star}^{1/2}\|_{2,\infty} \\ 
 & \ge\left(\sigma_{r}(\bR_{\star}\bSigma_{\star}^{-1/2})-\|\bDelta_{R}\bSigma_{\star}^{-1/2}\|_{\op}\right)\|\bL\bSigma_{\star}^{1/2}\|_{2,\infty} \\ 
 & \ge(1-\epsilon)\|\bL\bSigma_{\star}^{1/2}\|_{2,\infty}; \\
\|\bR\bL^{\top}\|_{2,\infty} & \ge\sigma_{r}(\bL\bSigma_{\star}^{-1/2})\|\bR\bSigma_{\star}^{1/2}\|_{2,\infty} \\ 
 & \ge\left(\sigma_{r}(\bL_{\star}\bSigma_{\star}^{-1/2})-\|\bDelta_{L}\bSigma_{\star}^{-1/2}\|_{\op}\right)\|\bR\bSigma_{\star}^{1/2}\|_{2,\infty} \\ 
 & \ge(1-\epsilon)\|\bR\bSigma_{\star}^{1/2}\|_{2,\infty}
\end{align*}
to obtain \eqref{eq:cond_MC_2inf-1p} and \eqref{eq:cond_MC_2inf-1n}.
Finally, apply the triangle inequality together with incoherence assumption to obtain \eqref{eq:cond_MC_2inf-2p}.
\end{proof}

\begin{proof}[Proof of Claim~\ref{claim:P3-4}] We start with the term $\mfk{P}_{3}$, for which we have 
\begin{align*}
|\mfk{P}_{3}| & \le\underbrace{\Big|\tr\left((p^{-1}\cP_{\Omega}-\cI)(\bL_{\star}\bDelta_{R}^{\top})\bR(\bR^{\top}\bR)^{-1}\bSigma_{\star}(\bR^{\top}\bR)^{-1}\bR^{\top}\bDelta_{R}\bL_{\star}^{\top}\right)\Big|}_{\mfk{P}_{3,1}}\\
 & \quad+\underbrace{\Big|\tr\left((p^{-1}\cP_{\Omega}-\cI)(\bDelta_{L}\bR^{\top})\bR(\bR^{\top}\bR)^{-1}\bSigma_{\star}(\bR^{\top}\bR)^{-1}\bR^{\top}\bDelta_{R}\bL_{\star}^{\top}\right)\Big|}_{\mfk{P}_{3,2}}.
\end{align*}
Invoke Lemma~\ref{lemma:P_Omega_tangent} to bound $\mfk{P}_{3,1}$ as 
\begin{align*}
\mfk{P}_{3,1} & \le C_{1}\sqrt{\frac{\mu r\log(n_{1}\vee n_{2})}{p(n_{1}\wedge n_{2})}}\|\bL_{\star}\bDelta_{R}^{\top}\|_{\fro}\left\Vert \bL_{\star}\bDelta_{R}^{\top}\bR(\bR^{\top}\bR)^{-1}\bSigma_{\star}(\bR^{\top}\bR)^{-1}\bR^{\top}\right\Vert _{\fro}\\
 & \le C_{1}\sqrt{\frac{\mu r\log(n_{1}\vee n_{2})}{p(n_{1}\wedge n_{2})}}\|\bDelta_{R}\bSigma_{\star}^{1/2}\|_{\fro}^{2}\left\Vert \bR(\bR^{\top}\bR)^{-1}\bSigma_{\star}^{1/2}\right\Vert _{\op}^{2}.
\end{align*}
The condition \eqref{eq:MC_consequences-1} allows us to obtain a simplified bound 
\begin{align*}
\mfk{P}_{3,1} & \le\frac{C_{1}}{(1-\epsilon)^{2}}\sqrt{\frac{\mu r\log(n_{1}\vee n_{2})}{p(n_{1}\wedge n_{2})}}\|\bDelta_{R}\bSigma_{\star}^{1/2}\|_{\fro}^{2}.
\end{align*}
In regard to $\mfk{P}_{3,2}$, we apply Lemma~\ref{lemma:P_Omega_Chen} with $\bL_{A}\coloneqq\bDelta_{L}\bSigma_{\star}^{1/2}$, $\bR_{A}\coloneqq\bR\bSigma_{\star}^{-1/2}$, $\bL_{B}\coloneqq\bL_{\star}\bSigma_{\star}^{-1/2}$, and \\
$\bR_{B}\coloneqq\bR(\bR^{\top}\bR)^{-1}\bSigma_{\star}(\bR^{\top}\bR)^{-1}\bR^{\top}\bDelta_{R}\bSigma_{\star}^{1/2}$ to see
\begin{align*}
\mfk{P}_{3,2} & \le C_{2}\sqrt{\frac{n_{1}\vee n_{2}}{p}}\|\bDelta_{L}\bSigma_{\star}^{1/2}\|_{\fro}\|\bL_{\star}\bSigma_{\star}^{-1/2}\|_{2,\infty}\|\bR\bSigma_{\star}^{-1/2}\|_{2,\infty}\left\Vert \bR(\bR^{\top}\bR)^{-1}\bSigma_{\star}(\bR^{\top}\bR)^{-1}\bR^{\top}\bDelta_{R}\bSigma_{\star}^{1/2}\right\Vert _{\fro}\\
 & \le C_{2}\sqrt{\frac{n_{1}\vee n_{2}}{p}}\|\bDelta_{L}\bSigma_{\star}^{1/2}\|_{\fro}\|\bL_{\star}\bSigma_{\star}^{-1/2}\|_{2,\infty}\|\bR\bSigma_{\star}^{-1/2}\|_{2,\infty}\left\Vert \bR(\bR^{\top}\bR)^{-1}\bSigma_{\star}^{1/2}\right\Vert _{\op}^{2}\|\bDelta_{R}\bSigma_{\star}^{1/2}\|_{\fro}.
\end{align*}
Again, use the consequences in Claim~\ref{claim:cond_MC} to reach
\begin{align*}
\mfk{P}_{3,2} & \le C_{2}\sqrt{\frac{n_{1}\vee n_{2}}{p}}\|\bDelta_{L}\bSigma_{\star}^{1/2}\|_{\fro}\sqrt{\frac{\mu r}{n_{1}}} \frac{C_{B}\kappa}{1-\epsilon}\sqrt{\frac{\mu r}{n_{2}}}\frac{1}{(1-\epsilon)^{2}}\|\bDelta_{R}\bSigma_{\star}^{1/2}\|_{\fro}\\
 & =\frac{C_{2}C_B\kappa}{(1-\epsilon)^{3}} \frac{\mu r}{\sqrt{p(n_{1}\wedge n_{2})}}\|\bDelta_{L}\bSigma_{\star}^{1/2}\|_{\fro}\|\bDelta_{R}\bSigma_{\star}^{1/2}\|_{\fro}.
\end{align*}
Combine the bounds of $\mfk{P}_{3,1}$ and $\mfk{P}_{3,2}$ to reach
\begin{align*}
\mfk{P}_{3} & \le\frac{C_{1}}{(1-\epsilon)^{2}}\sqrt{\frac{\mu r\log(n_{1}\vee n_{2})}{p(n_{1}\wedge n_{2})}}\|\bDelta_{R}\bSigma_{\star}^{1/2}\|_{\fro}^{2}\\
 & \quad+\frac{C_{2}C_{B}\kappa}{(1-\epsilon)^{3}}\frac{\mu r}{\sqrt{p(n_{1}\wedge n_{2})}}\|\bDelta_{L}\bSigma_{\star}^{1/2}\|_{\fro}\|\bDelta_{R}\bSigma_{\star}^{1/2}\|_{\fro}\\
 & \le \delta_{1}\|\bDelta_{R}\bSigma_{\star}^{1/2}\|_{\fro}^{2} + \delta_{2}\|\bDelta_{L}\bSigma_{\star}^{1/2}\|_{\fro}\|\bDelta_{R}\bSigma_{\star}^{1/2}\|_{\fro} \\
 & \le \frac{\delta_{2}}{2}\|\bDelta_{L}\bSigma_{\star}^{1/2}\|_{\fro}^{2} + (\delta_{1}+\frac{\delta_{2}}{2})\|\bDelta_{R}\bSigma_{\star}^{1/2}\|_{\fro}^{2}.
\end{align*}
Moving on to the term $\mfk{P}_{4}$, we have 
\begin{align*}
\sqrt{\mfk{P}_{4}} & =\left\Vert (p^{-1}\cP_{\Omega}-\cI)(\bL\bR^{\top}-\bX_{\star})\bR(\bR^{\top}\bR)^{-1}\bSigma_{\star}^{1/2}\right\Vert _{\fro}\\
 & \le\underbrace{\Big|\tr\left((p^{-1}\cP_{\Omega}-\cI)(\bDelta_{L}\bR_{\star}^{\top})\bR_{\star}(\bR^{\top}\bR)^{-1}\bSigma_{\star}^{1/2}\tilde{\bL}^{\top}\right)\Big|}_{\mfk{P}_{4,1}}\\
 & \quad+\underbrace{\Big|\tr\left((p^{-1}\cP_{\Omega}-\cI)(\bDelta_{L}\bR_{\star}^{\top})\bDelta_{R}(\bR^{\top}\bR)^{-1}\bSigma_{\star}^{1/2}\tilde{\bL}^{\top}\right)\Big|}_{\mfk{P}_{4,2}}\\
 & \quad+\underbrace{\Big|\tr\left((p^{-1}\cP_{\Omega}-\cI)(\bL\bDelta_{R}^{\top})\bR(\bR^{\top}\bR)^{-1}\bSigma_{\star}^{1/2}\tilde{\bL}^{\top}\right)\Big|}_{\mfk{P}_{4,3}},
\end{align*}
where we have used the variational representation of the Frobenius norm for some $\tilde{\bL}\in\RR^{n_{1}\times r}$ obeying $\|\tilde{\bL}\|_{\fro}=1$.
Note that the decomposition of $\sqrt{\mfk{P}_{4}}$ is extremely similar to that of $\mfk{P}_{2}$. Therefore we can follow a similar argument (i.e.~applying Lemmas~\ref{lemma:P_Omega_tangent} and \ref{lemma:P_Omega_Chen}) to control these terms as
\begin{align*}
\mfk{P}_{4,1} & \le \frac{C_{1}}{(1-\epsilon)^{2}}\sqrt{\frac{\mu r\log(n_{1}\vee n_{2})}{p(n_{1}\wedge n_{2})}}\|\bDelta_{L}\bSigma_{\star}^{1/2}\|_{\fro}; \\
\mfk{P}_{4,2} & \le \frac{C_{2}\kappa}{(1-\epsilon)^{2}}\left(1+\frac{C_{B}}{1-\epsilon}\right)\frac{\mu r}{\sqrt{p(n_{1}\wedge n_{2})}}\|\bDelta_{R}\bSigma_{\star}^{1/2}\|_{\fro}; \\
\mfk{P}_{4,3} & \le \frac{C_{2}C_{B}^{2}\kappa^{2}}{(1-\epsilon)^{4}}\frac{\mu r}{\sqrt{p(n_{1}\wedge n_{2})}}\|\bDelta_{R}\bSigma_{\star}^{1/2}\|_{\fro}.
\end{align*}
For conciseness, we omit the details for bounding each term. Combine them to reach  
\begin{align*}
\sqrt{\mfk{P}_{4}}\le\delta_{1}\|\bDelta_{L}\bSigma_{\star}^{1/2}\|_{\fro}+\delta_{2}\|\bDelta_{R}\bSigma_{\star}^{1/2}\|_{\fro}.
\end{align*}
Finally take the square on both sides and use $2ab\le a^{2}+b^{2}$ to obtain the upper bound
\begin{align*}
\mfk{P}_{4}\le\delta_{1}(\delta_{1}+\delta_{2})\|\bDelta_{L}\bSigma_{\star}^{1/2}\|_{\fro}^{2}+\delta_{2}(\delta_{1}+\delta_{2})\|\bDelta_{R}\bSigma_{\star}^{1/2}\|_{\fro}^{2}.
\end{align*}
\end{proof}

\subsection{Proof of Lemma~\ref{lemma:init_MC}}\label{subsec:proof_init_MC}

We start by recording a useful lemma below.
\begin{lemma}[{\cite[Lemma~2]{chen2015incoherence}, \cite[Lemma~4]{chen2019nonconvex}}]\label{lemma:P_Omega_fixed}
For any fixed $\bX\in\RR^{n_{1}\times n_{2}}$, with overwhelming probability, one has 
\begin{align*}
\left\Vert (p^{-1}\cP_{\Omega}-\cI)(\bX)\right\Vert _{\op}\le C_{0}\frac{\log(n_{1}\vee n_{2})}{p}\|\bX\|_{\infty}+C_{0}\sqrt{\frac{\log(n_{1}\vee n_{2})}{p}}(\|\bX\|_{2,\infty}\vee\|\bX^{\top}\|_{2,\infty}),
\end{align*}
where $C_{0}>0$ is some universal constant that does not depend on $\bX$. 
\end{lemma}

In view of Lemma~\ref{lemma:Procrustes}, one has 
\begin{align}
\dist(\tilde{\bF}_{0},\bF_{\star})\le\sqrt{\sqrt{2}+1}\left\Vert \bU_{0}\bSigma_{0}\bV_{0}^{\top}-\bX_{\star}\right\Vert _{\fro}\le\sqrt{(\sqrt{2}+1)2r}\left\Vert \bU_{0}\bSigma_{0}\bV_{0}^{\top}-\bX_{\star}\right\Vert _{\op},\label{eq:init_MC_helper-1}
\end{align}
where the last relation uses the fact that $\bU_{0}\bSigma_{0}\bV_{0}^{\top}-\bX_{\star}$ has rank at most $2r$. Applying the triangle inequality, we obtain 
\begin{align}
\left\Vert \bU_{0}\bSigma_{0}\bV_{0}^{\top}-\bX_{\star}\right\Vert _{\op} & \le\left\Vert p^{-1}\cP_{\Omega}(\bX_{\star})-\bU_{0}\bSigma_{0}\bV_{0}^{\top}\right\Vert _{\op}+\left\Vert p^{-1}\cP_{\Omega}(\bX_{\star})-\bX_{\star}\right\Vert _{\op} \nonumber\\
 & \le2\left\Vert (p^{-1}\cP_{\Omega}-\cI)(\bX_{\star})\right\Vert _{\op}.\label{eq:init_MC_helper-2}
\end{align}
Here the second inequality hinges on the fact that $\bU_{0}\bSigma_{0}\bV_{0}^{\top}$ is the best rank-$r$ approximation to $p^{-1}\cP_{\Omega}(\bX_{\star})$, i.e. 
\begin{align*} 
\left\Vert p^{-1}\cP_{\Omega}(\bX_{\star})-\bU_{0}\bSigma_{0}\bV_{0}^{\top}\right\Vert _{\op}\le\left\Vert p^{-1}\cP_{\Omega}(\bX_{\star})-\bX_{\star}\right\Vert _{\op}.
\end{align*}
Combining \eqref{eq:init_MC_helper-1} and \eqref{eq:init_MC_helper-2} yields
\begin{align*}
\dist(\tilde{\bF}_{0},\bF_{\star})\le2\sqrt{(\sqrt{2}+1)2r}\left\Vert (p^{-1}\cP_{\Omega}-\cI)(\bX_{\star})\right\Vert _{\op}\le5\sqrt{r}\left\Vert (p^{-1}\cP_{\Omega}-\cI)(\bX_{\star})\right\Vert _{\op}.
\end{align*}
It then boils down to controlling $\left\Vert p^{-1}\cP_{\Omega}(\bX_{\star})-\bX_{\star}\right\Vert _{\op}$,
which is readily supplied by Lemma~\ref{lemma:P_Omega_fixed} as
\begin{align*}
\left\Vert (p^{-1}\cP_{\Omega}-\cI)(\bX_{\star})\right\Vert _{\op} & \le C_{0}\frac{\log(n_{1}\vee n_{2})}{p}\|\bX_{\star}\|_{\infty}+C_{0}\sqrt{\frac{\log(n_{1}\vee n_{2})}{p}} (\|\bX_{\star}\|_{2,\infty}\vee\|\bX_{\star}^{\top}\|_{2,\infty}),
\end{align*}
which holds with overwhelming probability. The proof is finished by plugging the following bounds from incoherence assumption of $\bX_{\star}$:
\begin{align*}
\|\bX_{\star}\|_{\infty} & \le\|\bU_{\star}\|_{2,\infty}\|\bSigma_{\star}\|_{\op}\|\bV_{\star}\|_{2,\infty}\le\frac{\mu r}{\sqrt{n_{1}n_{2}}}\kappa\sigma_{r}(\bX_{\star});\\
\|\bX_{\star}\|_{2,\infty} & \le\|\bU_{\star}\|_{2,\infty}\|\bSigma_{\star}\|_{\op}\|\bV_{\star}\|_{\op}\le\sqrt{\frac{\mu r}{n_{1}}}\kappa\sigma_{r}(\bX_{\star});\\
\|\bX_{\star}^{\top}\|_{2,\infty} & \le\|\bU_{\star}\|_{\op}\|\bSigma_{\star}\|_{\op}\|\bV_{\star}\|_{2,\infty}\le\sqrt{\frac{\mu r}{n_{2}}}\kappa\sigma_{r}(\bX_{\star}).
\end{align*}

%% file: appendix-general-loss.tex
\section{Proof for General Loss Functions}\label{sec:proof_GL}

We first present a useful property of restricted smooth and convex functions.
\begin{lemma}\label{lemma:resticted_smooth_convex} Suppose that $f:\RR^{n_{1}\times n_{2}}\mapsto\RR$ is rank-$2r$ restricted $L$-smooth and rank-$2r$ restricted convex. Then for any $\bX_{1},\bX_{2}\in\RR^{n_{1}\times n_{2}}$ of rank at most $r$, one has
\begin{align*}
\langle\nabla f(\bX_{1})-\nabla f(\bX_{2}), \bX_{1}-\bX_{2}\rangle\ge\frac{1}{L}\|\nabla f(\bX_{1})-\nabla f(\bX_{2})\|_{\fro,r}^2.
\end{align*}
\end{lemma}

\begin{proof}
Since $f(\cdot)$ is rank-$2r$ restricted $L$-smooth and convex, it holds for any $\bar{\bX}\in\RR^{n_{1}\times n_{2}}$ with rank at most $2r$ that
\begin{align*}
f(\bX_{1})+\langle\nabla f(\bX_{1}),\bar{\bX}-\bX_{1}\rangle\le f(\bar{\bX})\le f(\bX_{2})+\langle\nabla f(\bX_{2}),\bar{\bX}-\bX_{2}\rangle+\frac{L}{2}\|\bar{\bX}-\bX_{2}\|_{\fro}^2.
\end{align*}
Reorganize the terms to yield
\begin{align*}
f(\bX_{1})+\langle\nabla f(\bX_{1}),\bX_{2}-\bX_{1}\rangle\le f(\bX_{2})+\langle\nabla f(\bX_{2})-\nabla f(\bX_{1}),\bar{\bX}-\bX_{2}\rangle+\frac{L}{2}\|\bar{\bX}-\bX_{2}\|_{\fro}^2.
\end{align*}
Take $\bar{\bX}=\bX_{2}-\frac{1}{L}\cP_{r}(\nabla f(\bX_{2})-\nabla f(\bX_{1}))$, whose rank is at most $2r$, to see
\begin{align*}
f(\bX_{1})+\langle\nabla f(\bX_{1}),\bX_{2}-\bX_{1}\rangle+\frac{1}{2L}\left\Vert\nabla f(\bX_{2})-\nabla f(\bX_{1})\right\Vert_{\fro,r}^2 \le f(\bX_{2}).
\end{align*}
We can further switch the roles of $\bX_{1}$ and $\bX_{2}$ to obtain 
\begin{align*}
f(\bX_{2})+\langle\nabla f(\bX_{2}),\bX_{1}-\bX_{2}\rangle+\frac{1}{2L}\left\Vert\nabla f(\bX_{2})-\nabla f(\bX_{1})\right\Vert_{\fro,r}^2 \le f(\bX_{1}).
\end{align*}
Adding the above two inequalities yields the desired bound.
\end{proof}

% \begin{lemma} Suppose that $f(\cdot):\RR^{n_1 \times n_2}\mapsto\RR$ is rank-$r$ restricted $L$-smooth. Then it holds for any $\bX_{1},\bX_{2}\in\RR^{n_{1}\times n_{2}}$ of rank at most $r$ that
% \begin{align*}
% \langle\nabla f(\bX_{1})-\nabla f(\bX_{2}), \bX_{1}-\bX_{2}\rangle\le L\|\bX_{1}-\bX_{2}\|_{\fro}^2.
% \end{align*}
% \end{lemma}
% \begin{proof} This follows immediately from adding the two inequalities:
% \begin{align*}
% f(\bX_{1}) & \le f(\bX_{2})+\langle\nabla f(\bX_{2}),\bX_{1}-\bX_{2}\rangle+\frac{L}{2}\|\bX_{1}-\bX_{2}\|_{\fro}^2, \quad\mbox{and}\\
% f(\bX_{2}) & \le f(\bX_{1})+\langle\nabla f(\bX_{1}),\bX_{2}-\bX_{1}\rangle+\frac{L}{2}\|\bX_{1}-\bX_{2}\|_{\fro}^2.
% \end{align*}
% \end{proof}

\subsection{Proof of Theorem~\ref{thm:GL}}

Suppose that the $t$-th iterate $\bF_{t}$ obeys the condition $\dist(\bF_{t},\bF_{\star})\le 0.1\sigma_{r}(\bX_{\star})/\sqrt{\kappa_{f}}$. In view of Lemma~\ref{lemma:Q_existence}, one knows that $\bQ_{t}$, the optimal alignment matrix between $\bF_{t}$ and $\bF_{\star}$ exists. Therefore, for notational convenience, denote $\bL\coloneqq\bL_{t}\bQ_{t}$, $\bR\coloneqq\bR_{t}\bQ_{t}^{-\top}$, $\bDelta_{L}\coloneqq\bL-\bL_{\star}$, $\bDelta_{R}\coloneqq\bR-\bR_{\star}$, and $\epsilon\coloneqq0.1/\sqrt{\kappa_{f}}$. Similar to the derivation in \eqref{eq:cond_MF}, we have 
\begin{align}
\|\bDelta_{L}\bSigma_{\star}^{-1/2}\|_{\op}\vee\|\bDelta_{R}\bSigma_{\star}^{-1/2}\|_{\op}\le\epsilon.\label{eq:cond_GL}
\end{align}
The conclusion $\|\bL_{t}\bR_{t}^{\top}-\bX_{\star}\|_{\fro} \le 1.5\dist(\bF_{t},\bF_{\star})$ is a simple consequence of Lemma~\ref{lemma:matrix2factor}; see \eqref{eq:dist_matrix} for a detailed argument. From now on, we focus on proving the distance contraction. 

By the definition of $\dist(\bF_{t+1},\bF_{\star})$, one has
\begin{align}
\dist^{2}(\bF_{t+1},\bF_{\star}) & \le\left\Vert (\bL_{t+1}\bQ_{t}-\bL_{\star})\bSigma_{\star}^{1/2}\right\Vert _{\fro}^{2}+\left\Vert (\bR_{t+1}\bQ_{t}^{-\top}-\bR_{\star})\bSigma_{\star}^{1/2}\right\Vert _{\fro}^{2}.\label{eq:GL_expand}
\end{align}
Introduce an auxiliary function
\begin{align*}
f_{\mu}(\bX)=f(\bX)-\frac{\mu}{2}\|\bX-\bX_{\star}\|_{\fro}^{2}, 
\end{align*}
which is rank-$2r$ restricted $(L-\mu)$-smooth and rank-$2r$ restricted convex.
Using the \texttt{ScaledGD} update rule \eqref{eq:scaledGD_GL} and the decomposition $\bL\bR^{\top}-\bX_{\star}=\bDelta_{L}\bR^{\top}+\bL_{\star}\bDelta_{R}^{\top}$, we obtain 
\begin{align*}
(\bL_{t+1}\bQ_{t}-\bL_{\star})\bSigma_{\star}^{1/2} & =\left(\bL-\eta\nabla f(\bL\bR^{\top})\bR(\bR^{\top}\bR)^{-1}-\bL_{\star}\right)\bSigma_{\star}^{1/2}\\
 & =\left(\bL-\eta\mu(\bL\bR^{\top}-\bX_{\star})\bR(\bR^{\top}\bR)^{-1}-\eta\nabla f_{\mu}(\bL\bR^{\top})\bR(\bR^{\top}\bR)^{-1}-\bL_{\star}\right)\bSigma_{\star}^{1/2}\\
 & =(1-\eta\mu)\bDelta_{L}\bSigma_{\star}^{1/2}-\eta\mu\bL_{\star}\bDelta_{R}^{\top}\bR(\bR^{\top}\bR)^{-1}\bSigma_{\star}^{1/2} - \eta\nabla f_{\mu}(\bL\bR^{\top})\bR(\bR^{\top}\bR)^{-1}\bSigma_{\star}^{1/2}.
\end{align*}
As a result, one can expand the first square in \eqref{eq:GL_expand} as
\begin{align*}
\left\Vert(\bL_{t+1}\bQ_{t}-\bL_{\star})\bSigma_{\star}^{1/2}\right\Vert_{\fro}^2 & =\underbrace{\left\Vert(1-\eta\mu)\bDelta_{L}\bSigma_{\star}^{1/2}-\eta\mu\bL_{\star}\bDelta_{R}^{\top}\bR(\bR^{\top}\bR)^{-1}\bSigma_{\star}^{1/2}\right\Vert_{\fro}^2}_{\mfk{G}_{1}} \\
& \quad -2\eta(1-\eta\mu)\underbrace{\left\langle\nabla f_{\mu}(\bL\bR^{\top}), \bDelta_{L}\bSigma_{\star}(\bR^{\top}\bR)^{-1}\bR^{\top}-\bDelta_{L}\bR_{\star}^{\top}-\frac{1}{2}\bDelta_{L}\bDelta_{R}^{\top}\right\rangle}_{\mfk{G}_{2}} \\
& \quad -2\eta(1-\eta\mu)\left\langle\nabla f_{\mu}(\bL\bR^{\top}), \bDelta_{L}\bR_{\star}^{\top}+\frac{1}{2}\bDelta_{L}\bDelta_{R}^{\top}\right\rangle \\
& \quad +2\eta^2\mu\underbrace{\left\langle\nabla f_{\mu}(\bL\bR^{\top}), \bL_{\star}\bDelta_{R}^{\top}\bR(\bR^{\top}\bR)^{-1}\bSigma_{\star}(\bR^{\top}\bR)^{-1}\bR^{\top}\right\rangle}_{\mfk{G}_{3}} \\
& \quad +\eta^2\underbrace{\left\Vert\nabla f_{\mu}(\bL\bR^{\top})\bR(\bR^{\top}\bR)^{-1}\bSigma_{\star}^{1/2}\right\Vert_{\fro}^2}_{\mfk{G}_{4}}.
\end{align*}
In the sequel, we shall bound the four terms separately. 
\begin{enumerate}
\item Notice that the main term $\mfk{G}_{1}$ has already been controlled in \eqref{eq:MF_Lt_bound} under the condition \eqref{eq:cond_GL}. It obeys
\begin{align*}
\mfk{G}_{1}\le\left((1-\eta\mu)^{2}+\frac{2\epsilon}{1-\epsilon}\eta\mu(1-\eta\mu)\right)\|\bDelta_{L}\bSigma_{\star}^{1/2}\|_{\fro}^{2}+\frac{2\epsilon+\epsilon^{2}}{(1-\epsilon)^{2}}\eta^{2}\mu^{2}\|\bDelta_{R}\bSigma_{\star}^{1/2}\|_{\fro}^{2},
\end{align*}
as long as $\eta\mu\le2/3$.
\item For the second term $\mfk{G}_{2}$, note that $\bDelta_{L}\bSigma_{\star}(\bR^{\top}\bR)^{-1}\bR^{\top}-\bDelta_{L}\bR_{\star}^{\top}-\frac{1}{2}\bDelta_{L}\bDelta_{R}^{\top}$ has rank at most $r$. Hence we can invoke Lemma~\ref{lemma:norm_Fr_variation} to obtain
\begin{align*}
|\mfk{G}_{2}| & \le\|\nabla f_{\mu}(\bL\bR^{\top})\|_{\fro,r}\left\Vert\bDelta_{L}\bSigma_{\star}(\bR^{\top}\bR)^{-1}\bR^{\top}-\bDelta_{L}\bR_{\star}^{\top}-\frac{1}{2}\bDelta_{L}\bDelta_{R}^{\top}\right\Vert_{\fro} \\
 & \le\|\nabla f_{\mu}(\bL\bR^{\top})\|_{\fro,r}\|\bDelta_{L}\bSigma_{\star}^{1/2}\|_{\fro}\left(\left\Vert\bR(\bR^{\top}\bR)^{-1}\bSigma_{\star}^{1/2}-\bV_{\star}\right\Vert_{\op}+\frac{1}{2}\|\bDelta_{R}\bSigma_{\star}^{-1/2}\|_{\op}\right),
\end{align*}
where the second line uses $\bR_{\star} = \bV_{\star} \bSigma_{\star}^{1/2}$. Take the condition \eqref{eq:cond_GL} and Lemma~\ref{lemma:Weyl} together to obtain
\begin{align*}
\left\Vert\bR(\bR^{\top}\bR)^{-1}\bSigma_{\star}^{1/2}\right\Vert_{\op} & \le\frac{1}{1-\epsilon};\\
\left\Vert\bR(\bR^{\top}\bR)^{-1}\bSigma_{\star}^{1/2}-\bV_{\star}\right\Vert_{\op} & \le\frac{\sqrt{2}\epsilon}{1-\epsilon}.
\end{align*}
These consequences further imply that
\begin{align*}
|\mfk{G}_{2}| & \le(\frac{\sqrt{2}\epsilon}{1-\epsilon}+\frac{\epsilon}{2})\|\nabla f_{\mu}(\bL\bR^{\top})\|_{\fro,r}\|\bDelta_{L}\bSigma_{\star}^{1/2}\|_{\fro}.
\end{align*}

\item As above, the third term $\mfk{G}_{3}$ can be similarly bounded as
\begin{align*}
|\mfk{G}_{3}| & \le\|\nabla f_{\mu}(\bL\bR^{\top})\|_{\fro,r}\left\Vert\bL_{\star}\bDelta_{R}^{\top}\bR(\bR^{\top}\bR)^{-1}\bSigma_{\star}(\bR^{\top}\bR)^{-1}\bR^{\top}\right\Vert_{\fro} \\
 & \le\|\nabla f_{\mu}(\bL\bR^{\top})\|_{\fro,r}\|\bDelta_{R}\bSigma_{\star}^{1/2}\|_{\fro}\left\Vert\bR(\bR^{\top}\bR)^{-1}\bSigma_{\star}^{1/2}\right\Vert_{\op}^{2} \\
 & \le\frac{1}{(1-\epsilon)^2}\|\nabla f_{\mu}(\bL\bR^{\top})\|_{\fro,r}\|\bDelta_{R}\bSigma_{\star}^{1/2}\|_{\fro}.
\end{align*}

\item For the last term $\mfk{G}_{4}$, invoke Lemma~\ref{lemma:norm_Fr_variation} to obtain
\begin{align*}
\mfk{G}_{4}\le\|\nabla f_{\mu}(\bL\bR^{\top})\|_{\fro,r}^{2}\left\Vert\bR(\bR^{\top}\bR)^{-1}\bSigma_{\star}^{1/2}\right\Vert_{\op}^{2}\le\frac{1}{(1-\epsilon)^{2}}\|\nabla f_{\mu}(\bL\bR^{\top})\|_{\fro,r}^{2}.
\end{align*}
\end{enumerate}

\noindent Taking collectively the bounds for $\mfk{G}_{1},\mfk{G}_{2},\mfk{G}_{3}$ and $\mfk{G}_{4}$ yields 
\begin{align*}
\left\Vert (\bL_{t+1}\bQ_{t}-\bL_{\star})\bSigma_{\star}^{1/2}\right\Vert _{\fro}^{2} & \le \left((1-\eta\mu)^{2}+\frac{2\epsilon}{1-\epsilon}\eta\mu(1-\eta\mu)\right)\|\bDelta_{L}\bSigma_{\star}^{1/2}\|_{\fro}^{2}+\frac{2\epsilon+\epsilon^{2}}{(1-\epsilon)^{2}}\eta^{2}\mu^2\|\bDelta_{R}\bSigma_{\star}^{1/2}\|_{\fro}^{2} \\
 & \quad +2\eta(\frac{\sqrt{2}\epsilon}{1-\epsilon}+\frac{\epsilon}{2})(1-\eta\mu)\|\nabla f_{\mu}(\bL\bR^{\top})\|_{\fro,r}\|\bDelta_{L}\bSigma_{\star}^{1/2}\|_{\fro} \\
 & \quad -2\eta(1-\eta\mu)\left\langle\nabla f_{\mu}(\bL\bR^{\top}), \bDelta_{L}\bR_{\star}^{\top}+\frac{1}{2}\bDelta_{L}\bDelta_{R}^{\top}\right\rangle \\
 & \quad +\frac{2\eta^{2}\mu}{(1-\epsilon)^{2}}\|\nabla f_{\mu}(\bL\bR^{\top})\|_{\fro,r}\|\bDelta_{R}\bSigma_{\star}^{1/2}\|_{\fro}+\frac{\eta^2}{(1-\epsilon)^{2}}\|\nabla f_{\mu}(\bL\bR^{\top})\|_{\fro,r}^{2}.
\end{align*}
Similarly, we can obtain the control of $\|(\bR_{t+1}\bQ_{t}^{-\top}-\bR_{\star})\bSigma_{\star}^{1/2}\|_{\fro}^{2}$. Combine them together to reach
\begin{align*}
 & \left\Vert (\bL_{t+1}\bQ_{t}-\bL_{\star})\bSigma_{\star}^{1/2}\right\Vert _{\fro}^{2}+\left\Vert (\bR_{t+1}\bQ_{t}^{-\top}-\bR_{\star})\bSigma_{\star}^{1/2}\right\Vert _{\fro}^{2}\\
 & \quad\le\left((1-\eta\mu)^{2}+\frac{2\epsilon}{1-\epsilon}\eta\mu(1-\eta\mu)+\frac{2\epsilon+\epsilon^{2}}{(1-\epsilon)^{2}}\eta^{2}\mu^{2}\right)\left(\|\bDelta_{L}\bSigma_{\star}^{1/2}\|_{\fro}^{2}+\|\bDelta_{R}\bSigma_{\star}^{1/2}\|_{\fro}^{2}\right) \\
 & \qquad+2\eta\left((\frac{\sqrt{2}\epsilon}{1-\epsilon}+\frac{\epsilon}{2})(1-\eta\mu)+\frac{\eta\mu}{(1-\epsilon)^{2}}\right)\|\nabla f_{\mu}(\bL\bR^{\top})\|_{\fro,r}\left(\|\bDelta_{L}\bSigma_{\star}^{1/2}\|_{\fro}+\|\bDelta_{R}\bSigma_{\star}^{1/2}\|_{\fro}\right) \\
 & \qquad-2\eta(1-\eta\mu)\left\langle\nabla f_{\mu}(\bL\bR^{\top}), \bDelta_{L}\bR_{\star}^{\top}+\bL_{\star}\bDelta_{R}^{\top}+\bDelta_{L}\bDelta_{R}^{\top}\right\rangle+\frac{2\eta^2}{(1-\epsilon)^{2}}\|\nabla f_{\mu}(\bL\bR^{\top})\|_{\fro,r}^{2} \\
 & \quad\le\left((1-\eta\mu)^{2}+\frac{2\epsilon}{1-\epsilon}\eta\mu(1-\eta\mu)+\frac{2\epsilon+\epsilon^{2}}{(1-\epsilon)^{2}}\eta^{2}\mu^{2}\right)\left(\|\bDelta_{L}\bSigma_{\star}^{1/2}\|_{\fro}^{2}+\|\bDelta_{R}\bSigma_{\star}^{1/2}\|_{\fro}^{2}\right) \\
 & \qquad+2\eta\underbrace{\left((\frac{\sqrt{2}\epsilon}{1-\epsilon}+\frac{\epsilon}{2})(1-\eta\mu)+\frac{\eta\mu}{(1-\epsilon)^{2}}\right)}_{\mfk{C}_{1}}\|\nabla f_{\mu}(\bL\bR^{\top})\|_{\fro,r}\left(\|\bDelta_{L}\bSigma_{\star}^{1/2}\|_{\fro}+\|\bDelta_{R}\bSigma_{\star}^{1/2}\|_{\fro}\right) \\
 & \qquad-2\eta\underbrace{\left(\frac{1-\eta\mu}{L-\mu}-\frac{\eta}{(1-\epsilon)^{2}}\right)}_{\mfk{C}_{2}}\|\nabla f_{\mu}(\bL\bR^{\top})\|_{\fro,r}^{2},
\end{align*}
where the last line follows from Lemma~\ref{lemma:resticted_smooth_convex} (notice that $\nabla f_{\mu}(\bX_{\star})=\zero$) as
\begin{align*}
\langle\nabla f_{\mu}(\bL\bR^{\top}),\bDelta_{L}\bR_{\star}^{\top}+\bL_{\star}\bDelta_{R}^{\top}+\bDelta_{L}\bDelta_{R}^{\top}\rangle =
\langle\nabla f_{\mu}(\bL\bR^{\top}),\bL\bR^{\top}-\bX_{\star}\rangle\ge\frac{1}{L-\mu}\|\nabla f_{\mu}(\bL\bR^{\top})\|_{\fro,r}^{2}.
\end{align*}
Notice that $\mfk{C}_{2}>0$ as long as $\eta\le(1-\epsilon)^2/L$. Maximizing the quadratic function of $\|\nabla f_{\mu}(\bL\bR^{\top})\|_{\fro,r}$ yields
\begin{align*}
\mfk{C}_{1}\|\nabla f_{\mu}(\bL\bR^{\top})\|_{\fro,r}\left(\|\bDelta_{L}\bSigma_{\star}^{1/2}\|_{\fro}+\|\bDelta_{R}\bSigma_{\star}^{1/2}\|_{\fro}\right)-\mfk{C}_{2}\|\nabla f_{\mu}(\bL\bR^{\top})\|_{\fro,r}^{2} & \le\frac{\mfk{C}_{1}^{2}}{4\mfk{C}_{2}}\left(\|\bDelta_{L}\bSigma_{\star}^{1/2}\|_{\fro}+\|\bDelta_{R}\bSigma_{\star}^{1/2}\|_{\fro}\right)^{2} \\
 & \le\frac{\mfk{C}_{1}^{2}}{2\mfk{C}_{2}}\left(\|\bDelta_{L}\bSigma_{\star}^{1/2}\|_{\fro}^{2}+\|\bDelta_{R}\bSigma_{\star}^{1/2}\|_{\fro}^{2}\right),
\end{align*}
where the last inequality holds since $(a+b)^{2}\le2(a^{2}+b^{2})$. Identify $\dist^{2}(\bF_{t},\bF_{\star}) = \|\bDelta_{L}\bSigma_{\star}^{1/2}\|_{\fro}^{2}+\|\bDelta_{R}\bSigma_{\star}^{1/2}\|_{\fro}^{2}$ to obtain
\begin{align*} 
\left\Vert (\bL_{t+1}\bQ_{t}-\bL_{\star})\bSigma_{\star}^{1/2}\right\Vert _{\fro}^{2}+\left\Vert (\bR_{t+1}\bQ_{t}^{-\top}-\bR_{\star})\bSigma_{\star}^{1/2}\right\Vert _{\fro}^{2}\le\rho^{2}(\eta;\epsilon,\mu,L)\dist^{2}(\bF_{t},\bF_{\star}),
\end{align*}
where the contraction rate is given by
\begin{align*}
\rho^{2}(\eta;\epsilon,\mu,L) & \coloneqq(1-\eta\mu)^{2}+\frac{2\epsilon}{1-\epsilon}\eta\mu(1-\eta\mu)+\frac{2\epsilon+\epsilon^{2}}{(1-\epsilon)^{2}}\eta^{2}\mu^2+\frac{\left((\frac{\sqrt{2}\epsilon}{1-\epsilon}+\frac{\epsilon}{2})(1-\eta\mu)+\frac{\eta\mu}{(1-\epsilon)^{2}}\right)^2}{1-\eta\mu-\frac{\eta(L-\mu)}{(1-\epsilon)^{2}}}\eta(L-\mu).
\end{align*}
With $\epsilon=0.1/\sqrt{\kappa_{f}}$ and $0<\eta\le0.4/L$, one has $\rho(\eta;\epsilon,\mu,L)\le1-0.7\eta\mu$. Thus we conclude that 
\begin{align*}
\dist(\bF_{t+1},\bF_{\star}) &\le \sqrt{\left\Vert (\bL_{t+1}\bQ_{t}-\bL_{\star})\bSigma_{\star}^{1/2}\right\Vert _{\fro}^{2}+\left\Vert (\bR_{t+1}\bQ_{t}^{-\top}-\bR_{\star})\bSigma_{\star}^{1/2}\right\Vert _{\fro}^{2}} \\
 &\le(1-0.7\eta\mu)\dist(\bF_{t},\bF_{\star}),
\end{align*}
which is the desired claim. 

\begin{remark} We provide numerical details for the contraction rate. For simplicity, we shall prove $\rho(\eta;\epsilon,\mu,L)\le1-0.7\eta\mu$ under a stricter condition $\epsilon=0.02/\sqrt{\kappa_{f}}$. The stronger result under the condition $\epsilon=0.1/\sqrt{\kappa_{f}}$ can be verified through a subtler analysis.

With $\epsilon=0.02/\sqrt{\kappa_{f}}$ and $0<\eta\le0.4/L$, one can bound the terms in $\rho^{2}(\eta;\epsilon,\mu,L)$ as
\begin{align}
(1-\eta\mu)^{2}+\frac{2\epsilon}{1-\epsilon}\eta\mu(1-\eta\mu)+\frac{2\epsilon+\epsilon^{2}}{(1-\epsilon)^{2}}\eta^{2}\mu^2 & \le 1-1.959\eta\mu+1.002\eta^{2}\mu^2; \label{eq:GL_rho-1} \\
\frac{\left((\frac{\sqrt{2}\epsilon}{1-\epsilon}+\frac{\epsilon}{2})(1-\eta\mu)+\frac{\eta\mu}{(1-\epsilon)^{2}}\right)^2}{1-\eta\mu-\frac{\eta(L-\mu)}{(1-\epsilon)^{2}}}\eta(L-\mu) & \le \frac{\frac{0.0016}{\kappa_{f}}+0.078\eta\mu+1.005\eta^{2}\mu^{2}}{1-1.042\eta L}\eta L \nonumber\\
 & \le\frac{0.0016\eta\frac{L}{\kappa_{f}}+0.4\times(0.078\eta\mu+1.005\eta^{2}\mu^{2})}{1-0.4\times1.042} \nonumber \\ 
 & \le0.057\eta\mu+0.69\eta^{2}\mu^{2}, \label{eq:GL_rho-2}
\end{align}
where the last line uses the definition \eqref{eq:kappa_f} of $\kappa_{f}$. Putting \eqref{eq:GL_rho-1} and \eqref{eq:GL_rho-2} together further implies
\begin{align*}
\rho^{2}(\eta;\epsilon,\mu,L) & \le1-1.9\eta\mu+1.7\eta^{2}\mu^2\le(1-0.7\eta\mu)^{2},
\end{align*}
as long as $0<\eta\mu\le0.4$.
\end{remark}

%% file: ScaledGD.bbl
\newcommand{\etalchar}[1]{$^{#1}$}
\begin{thebibliography}{TMPB{\etalchar{+}}21}

\bibitem[BH89]{baldi1989neural}
P.~Baldi and K.~Hornik.
\newblock Neural networks and principal component analysis: Learning from
  examples without local minima.
\newblock {\em Neural networks}, 2(1):53--58, 1989.

\bibitem[BKS16]{bhojanapalli2016dropping}
S.~Bhojanapalli, A.~Kyrillidis, and S.~Sanghavi.
\newblock Dropping convexity for faster semi-definite optimization.
\newblock In {\em Conference on Learning Theory}, pages 530--582. PMLR, 2016.

\bibitem[BNS16]{bhojanapalli2016global}
S.~Bhojanapalli, B.~Neyshabur, and N.~Srebro.
\newblock Global optimality of local search for low rank matrix recovery.
\newblock In {\em Advances in Neural Information Processing Systems}, pages
  3873--3881, 2016.

\bibitem[CC14]{chen2014robust}
Y.~Chen and Y.~Chi.
\newblock Robust spectral compressed sensing via structured matrix completion.
\newblock {\em IEEE Transactions on Information Theory}, 60(10):6576--6601,
  2014.

\bibitem[CC18]{chen2018harnessing}
Y.~Chen and Y.~Chi.
\newblock Harnessing structures in big data via guaranteed low-rank matrix
  estimation: Recent theory and fast algorithms via convex and nonconvex
  optimization.
\newblock {\em IEEE Signal Processing Magazine}, 35(4):14 -- 31, 2018.

\bibitem[CCD{\etalchar{+}}21]{charisopoulos2019low}
V.~Charisopoulos, Y.~Chen, D.~Davis, M.~D{\'\i}az, L.~Ding, and
  D.~Drusvyatskiy.
\newblock Low-rank matrix recovery with composite optimization: good
  conditioning and rapid convergence.
\newblock {\em Foundations of Computational Mathematics}, pages 1--89, 2021.

\bibitem[CCF{\etalchar{+}}20]{chen2019noisy}
Y.~Chen, Y.~Chi, J.~Fan, C.~Ma, and Y.~Yan.
\newblock Noisy matrix completion: Understanding statistical guarantees for
  convex relaxation via nonconvex optimization.
\newblock {\em SIAM Journal on Optimization}, 30(4):3098--3121, 2020.

\bibitem[CFMY20]{chen2020bridging}
Y.~Chen, J.~Fan, C.~Ma, and Y.~Yan.
\newblock Bridging convex and nonconvex optimization in robust {PCA}: Noise,
  outliers, and missing data.
\newblock {\em arXiv preprint arXiv:2001.05484}, 2020.

\bibitem[Che15]{chen2015incoherence}
Y.~Chen.
\newblock Incoherence-optimal matrix completion.
\newblock {\em IEEE Transactions on Information Theory}, 61(5):2909--2923,
  2015.

\bibitem[CL19]{chen2019model}
J.~Chen and X.~Li.
\newblock Model-free nonconvex matrix completion: Local minima analysis and
  applications in memory-efficient kernel {PCA}.
\newblock {\em Journal of Machine Learning Research}, 20(142):1--39, 2019.

\bibitem[CLC19]{chi2019nonconvex}
Y.~Chi, Y.~M. Lu, and Y.~Chen.
\newblock Nonconvex optimization meets low-rank matrix factorization: An
  overview.
\newblock {\em IEEE Transactions on Signal Processing}, 67(20):5239--5269,
  2019.

\bibitem[CLL20]{chen2019nonconvex}
J.~Chen, D.~Liu, and X.~Li.
\newblock Nonconvex rectangular matrix completion via gradient descent without
  $\ell_{2,\infty}$ regularization.
\newblock {\em IEEE Transactions on Information Theory}, 66(9):5806--5841,
  2020.

\bibitem[CLMW11]{candes2009robustPCA}
E.~J. Cand{\`e}s, X.~Li, Y.~Ma, and J.~Wright.
\newblock Robust principal component analysis?
\newblock {\em Journal of the ACM}, 58(3):11:1--11:37, 2011.

\bibitem[CLS15]{candes2015phase}
E.~Cand\`es, X.~Li, and M.~Soltanolkotabi.
\newblock Phase retrieval via {W}irtinger flow: Theory and algorithms.
\newblock {\em Information Theory, IEEE Transactions on}, 61(4):1985--2007,
  2015.

\bibitem[CP11]{candes2011tight}
E.~J. Cand\`es and Y.~Plan.
\newblock Tight oracle inequalities for low-rank matrix recovery from a minimal
  number of noisy random measurements.
\newblock {\em IEEE Transactions on Information Theory}, 57(4):2342--2359,
  2011.

\bibitem[CR09]{candes_mc}
E.~J. Cand{\`e}s and B.~Recht.
\newblock Exact matrix completion via convex optimization.
\newblock {\em Foundations of Computational Mathematics}, 9(6):717--772, 2009.

\bibitem[CSPW11]{chandrasekaran2011siam}
V.~Chandrasekaran, S.~Sanghavi, P.~Parrilo, and A.~Willsky.
\newblock Rank-sparsity incoherence for matrix decomposition.
\newblock {\em SIAM {Journal} on {Optimization}}, 21(2):572--596, 2011.

\bibitem[CW15]{chen2015fast}
Y.~Chen and M.~J. Wainwright.
\newblock Fast low-rank estimation by projected gradient descent: General
  statistical and algorithmic guarantees.
\newblock {\em arXiv preprint arXiv:1509.03025}, 2015.

\bibitem[CWW18]{cai2018spectral}
J.-F. Cai, T.~Wang, and K.~Wei.
\newblock Spectral compressed sensing via projected gradient descent.
\newblock {\em SIAM Journal on Optimization}, 28(3):2625--2653, 2018.

\bibitem[DDP17]{davis2017nonsmooth}
D.~Davis, D.~Drusvyatskiy, and C.~Paquette.
\newblock The nonsmooth landscape of phase retrieval.
\newblock {\em arXiv preprint arXiv:1711.03247}, 2017.

\bibitem[DHL18]{du2018algorithmic}
S.~S. Du, W.~Hu, and J.~D. Lee.
\newblock Algorithmic regularization in learning deep homogeneous models:
  Layers are automatically balanced.
\newblock In {\em Advances in Neural Information Processing Systems}, pages
  384--395, 2018.

\bibitem[GHJY15]{ge2015escaping}
R.~Ge, F.~Huang, C.~Jin, and Y.~Yuan.
\newblock Escaping from saddle points-online stochastic gradient for tensor
  decomposition.
\newblock In {\em Conference on Learning Theory (COLT)}, pages 797--842, 2015.

\bibitem[GJZ17]{ge2017no}
R.~Ge, C.~Jin, and Y.~Zheng.
\newblock No spurious local minima in nonconvex low rank problems: A unified
  geometric analysis.
\newblock In {\em International Conference on Machine Learning}, pages
  1233--1242, 2017.

\bibitem[GLM16]{ge2016matrix}
R.~Ge, J.~D. Lee, and T.~Ma.
\newblock Matrix completion has no spurious local minimum.
\newblock In {\em Advances in Neural Information Processing Systems}, pages
  2973--2981, 2016.

\bibitem[GRG14]{gunasekar2014exponential}
S.~Gunasekar, P.~Ravikumar, and J.~Ghosh.
\newblock Exponential family matrix completion under structural constraints.
\newblock In {\em International Conference on Machine Learning}, pages
  1917--1925, 2014.

\bibitem[HW14]{hardt2014fast}
M.~Hardt and M.~Wootters.
\newblock Fast matrix completion without the condition number.
\newblock In {\em Proceedings of The 27th Conference on Learning Theory}, pages
  638--678, 2014.

\bibitem[JGN{\etalchar{+}}17]{jin2017escape}
C.~Jin, R.~Ge, P.~Netrapalli, S.~M. Kakade, and M.~I. Jordan.
\newblock How to escape saddle points efficiently.
\newblock In {\em International Conference on Machine Learning}, pages
  1724--1732, 2017.

\bibitem[JK17]{jain2017non}
P.~Jain and P.~Kar.
\newblock Non-convex optimization for machine learning.
\newblock {\em Foundations and Trends{\textregistered} in Machine Learning},
  10(3-4):142--336, 2017.

\bibitem[JMD10]{jain2010guaranteed}
P.~Jain, R.~Meka, and I.~S. Dhillon.
\newblock Guaranteed rank minimization via singular value projection.
\newblock In {\em Advances in Neural Information Processing Systems}, pages
  937--945, 2010.

\bibitem[JNS13]{jain2013low}
P.~Jain, P.~Netrapalli, and S.~Sanghavi.
\newblock Low-rank matrix completion using alternating minimization.
\newblock In {\em Proceedings of the forty-fifth annual ACM symposium on Theory
  of computing}, pages 665--674. ACM, 2013.

\bibitem[Kaw16]{kawaguchi2016deep}
K.~Kawaguchi.
\newblock Deep learning without poor local minima.
\newblock In {\em Advances in neural information processing systems}, pages
  586--594, 2016.

\bibitem[KC12]{kyrillidis2012matrix}
A.~Kyrillidis and V.~Cevher.
\newblock Matrix {ALPS}: Accelerated low rank and sparse matrix reconstruction.
\newblock In {\em 2012 IEEE Statistical Signal Processing Workshop (SSP)},
  pages 185--188. IEEE, 2012.

\bibitem[Laf15]{lafond2015low}
J.~Lafond.
\newblock Low rank matrix completion with exponential family noise.
\newblock In {\em Conference on Learning Theory}, pages 1224--1243, 2015.

\bibitem[LHLZ20]{luo2020recursive}
Y.~Luo, W.~Huang, X.~Li, and A.~R. Zhang.
\newblock Recursive importance sketching for rank constrained least squares:
  Algorithms and high-order convergence.
\newblock {\em arXiv preprint arXiv:2011.08360}, 2020.

\bibitem[LLSW19]{li2019rapid}
X.~Li, S.~Ling, T.~Strohmer, and K.~Wei.
\newblock Rapid, robust, and reliable blind deconvolution via nonconvex
  optimization.
\newblock {\em Applied and computational harmonic analysis}, 47(3):893--934,
  2019.

\bibitem[LMCC21]{li2018nonconvex}
Y.~{Li}, C.~{Ma}, Y.~{Chen}, and Y.~{Chi}.
\newblock Nonconvex matrix factorization from rank-one measurements.
\newblock {\em IEEE Transactions on Information Theory}, 67(3):1928--1950,
  2021.

\bibitem[MAS12]{mishra2012riemannian}
B.~Mishra, K.~A. Apuroop, and R.~Sepulchre.
\newblock A {R}iemannian geometry for low-rank matrix completion.
\newblock {\em arXiv preprint arXiv:1211.1550}, 2012.

\bibitem[Maz16]{mazeika2016singular}
M.~Mazeika.
\newblock The singular value decomposition and low rank approximation.
\newblock Technical report, University of Chicago, 2016.

\bibitem[MBM18]{mei2016landscape}
S.~Mei, Y.~Bai, and A.~Montanari.
\newblock The landscape of empirical risk for nonconvex losses.
\newblock {\em The Annals of Statistics}, 46(6A):2747--2774, 2018.

\bibitem[MLC21]{ma2021beyond}
C.~Ma, Y.~Li, and Y.~Chi.
\newblock Beyond {P}rocrustes: Balancing-free gradient descent for asymmetric
  low-rank matrix sensing.
\newblock {\em IEEE Transactions on Signal Processing}, 69:867--877, 2021.

\bibitem[MS16]{mishra2016riemannian}
B.~Mishra and R.~Sepulchre.
\newblock Riemannian preconditioning.
\newblock {\em SIAM Journal on Optimization}, 26(1):635--660, 2016.

\bibitem[MWCC19]{ma2017implicit}
C.~Ma, K.~Wang, Y.~Chi, and Y.~Chen.
\newblock Implicit regularization in nonconvex statistical estimation: Gradient
  descent converges linearly for phase retrieval, matrix completion, and blind
  deconvolution.
\newblock {\em Foundations of Computational Mathematics}, pages 1--182, 2019.

\bibitem[NNS{\etalchar{+}}14]{netrapalli2014non}
P.~Netrapalli, U.~Niranjan, S.~Sanghavi, A.~Anandkumar, and P.~Jain.
\newblock Non-convex robust {PCA}.
\newblock In {\em Advances in Neural Information Processing Systems}, pages
  1107--1115, 2014.

\bibitem[NP06]{nesterov2006cubic}
Y.~Nesterov and B.~T. Polyak.
\newblock Cubic regularization of {N}ewton method and its global performance.
\newblock {\em Mathematical Programming}, 108(1):177--205, 2006.

\bibitem[PKCS17]{park2017non}
D.~Park, A.~Kyrillidis, C.~Carmanis, and S.~Sanghavi.
\newblock Non-square matrix sensing without spurious local minima via the
  {Burer-Monteiro} approach.
\newblock In {\em Artificial Intelligence and Statistics}, pages 65--74, 2017.

\bibitem[PKCS18]{park2018finding}
D.~Park, A.~Kyrillidis, C.~Caramanis, and S.~Sanghavi.
\newblock Finding low-rank solutions via nonconvex matrix factorization,
  efficiently and provably.
\newblock {\em SIAM Journal on Imaging Sciences}, 11(4):2165--2204, 2018.

\bibitem[RFP10]{recht2010guaranteed}
B.~Recht, M.~Fazel, and P.~A. Parrilo.
\newblock Guaranteed minimum-rank solutions of linear matrix equations via
  nuclear norm minimization.
\newblock {\em SIAM review}, 52(3):471--501, 2010.

\bibitem[SL16]{sun2016guaranteed}
R.~Sun and Z.-Q. Luo.
\newblock Guaranteed matrix completion via non-convex factorization.
\newblock {\em IEEE Transactions on Information Theory}, 62(11):6535--6579,
  2016.

\bibitem[SQW15]{sun2015complete}
J.~Sun, Q.~Qu, and J.~Wright.
\newblock Complete dictionary recovery using nonconvex optimization.
\newblock In {\em Proceedings of the 32nd International Conference on Machine
  Learning}, pages 2351--2360, 2015.

\bibitem[SQW18]{sun2018geometric}
J.~Sun, Q.~Qu, and J.~Wright.
\newblock A geometric analysis of phase retrieval.
\newblock {\em Foundations of Computational Mathematics}, 18(5):1131--1198,
  2018.

\bibitem[SWW17]{sanghavi2017local}
S.~Sanghavi, R.~Ward, and C.~D. White.
\newblock The local convexity of solving systems of quadratic equations.
\newblock {\em Results in Mathematics}, 71(3-4):569--608, 2017.

\bibitem[TBS{\etalchar{+}}16]{tu2015low}
S.~Tu, R.~Boczar, M.~Simchowitz, M.~Soltanolkotabi, and B.~Recht.
\newblock Low-rank solutions of linear matrix equations via {P}rocrustes flow.
\newblock In {\em International Conference Machine Learning}, pages 964--973,
  2016.

\bibitem[TMC21]{tong2021low}
T.~Tong, C.~Ma, and Y.~Chi.
\newblock Low-rank matrix recovery with scaled subgradient methods: Fast and
  robust convergence without the condition number.
\newblock {\em IEEE Transactions on Signal Processing}, 2021.

\bibitem[TMPB{\etalchar{+}}21]{tong2021scaling}
T.~Tong, C.~Ma, A.~Prater-Bennette, E.~Tripp, and Y.~Chi.
\newblock Scaling and scalability: Provable nonconvex low-rank tensor
  estimation from incomplete measurements.
\newblock {\em arXiv preprint arXiv:2104.14526}, 2021.

\bibitem[TW16]{tanner2016low}
J.~Tanner and K.~Wei.
\newblock Low rank matrix completion by alternating steepest descent methods.
\newblock {\em Applied and Computational Harmonic Analysis}, 40(2):417--429,
  2016.

\bibitem[WCCL16]{wei2016guarantees}
K.~Wei, J.-F. Cai, T.~F. Chan, and S.~Leung.
\newblock Guarantees of {R}iemannian optimization for low rank matrix recovery.
\newblock {\em SIAM Journal on Matrix Analysis and Applications},
  37(3):1198--1222, 2016.

\bibitem[YPCC16]{yi2016fast}
X.~Yi, D.~Park, Y.~Chen, and C.~Caramanis.
\newblock Fast algorithms for robust {PCA} via gradient descent.
\newblock In {\em Advances in neural information processing systems}, pages
  4152--4160, 2016.

\bibitem[ZL15]{zheng2015convergent}
Q.~Zheng and J.~Lafferty.
\newblock A convergent gradient descent algorithm for rank minimization and
  semidefinite programming from random linear measurements.
\newblock In {\em Advances in Neural Information Processing Systems}, pages
  109--117, 2015.

\bibitem[ZL16]{zheng2016convergence}
Q.~Zheng and J.~Lafferty.
\newblock Convergence analysis for rectangular matrix completion using
  {B}urer-{M}onteiro factorization and gradient descent.
\newblock {\em arXiv preprint arXiv:1605.07051}, 2016.

\bibitem[ZLTW18]{zhu2017global}
Z.~Zhu, Q.~Li, G.~Tang, and M.~B. Wakin.
\newblock Global optimality in low-rank matrix optimization.
\newblock {\em IEEE Transactions on Signal Processing}, 66(13):3614--3628,
  2018.

\end{thebibliography}
